\documentclass{article}[12pt]

\usepackage[margin=1.1in]{geometry}

\usepackage[utf8]{inputenc} %
\usepackage[T1]{fontenc}    %
\usepackage[hidelinks,colorlinks,linkcolor=black,citecolor=orange]{hyperref}  
 \hypersetup{
     colorlinks=true,
     linkcolor=blue,
     filecolor=blue,    
     urlcolor=cyan,
     }
\usepackage{placeins}
\usepackage{url}            %

\usepackage{booktabs}       %
\usepackage{amsfonts}       %
\usepackage{nicefrac}       %
\usepackage{microtype}      %
\usepackage{xcolor}         %
\usepackage[sortcites=true,citestyle=numeric,natbib=true,bibstyle=numeric,maxbibnames=99]{biblatex}

\usepackage{tikz}
\usetikzlibrary{decorations.pathreplacing}
\usetikzlibrary{positioning}
\usepackage{makecell}
\usepackage{graphicx}
\usepackage{caption}
\usepackage{subcaption}
\usepackage{amsmath,amssymb,amsthm}
\usepackage{bbm}
\usepackage{enumitem}   
\usepackage{capt-of}
\usepackage{wrapfig}
\usepackage{algorithm}
\usepackage{algpseudocode}
\newcommand{\upperR}[1]{\uppercase\expandafter{\romannumeral#1}}
\usepackage{xcolor}                %
\usepackage[most]{tcolorbox}             %
\usepackage{titletoc}  
\usepackage{relsize}
\usepackage{empheq}

\usepackage{varwidth}
\newtcbox{\mybox}{colback=blue!5,
	colframe=blue!30!black, center, enhanced, varwidth upper}
\newtcbox{\mymath}[1][]{%
	nobeforeafter, math upper, tcbox raise base,
	enhanced, colframe=blue!30!black,
	colback=blue!5, boxrule=0.5pt, top=1mm,bottom=1mm,
	#1}

\usepackage{amsmath,amsfonts,bm,amssymb}
\usepackage{amsthm,bbm}

\def\ceil#1{\lceil #1 \rceil}

\def\1{\bm{1}}

\DeclareMathAlphabet{\mathsfit}{\encodingdefault}{\sfdefault}{m}{sl}
\SetMathAlphabet{\mathsfit}{bold}{\encodingdefault}{\sfdefault}{bx}{n}

\def\gB{{\mathcal{B}}}

\def\gD{{\mathcal{D}}}

\def\gH{{\mathcal{H}}}
\def\gI{{\mathcal{I}}}
\def\gJ{{\mathcal{J}}}

\def\gM{{\mathcal{M}}}
\def\gN{{\mathcal{N}}}

\def\gU{{\mathcal{U}}}

\newcommand{\E}{\mathbb{E}}

\renewcommand{\H}{\mathbb{H}}
\newcommand{\R}{\mathbb{R}}

\renewcommand\S{{\mathcal{S}}}

\newcommand\F{{\mathcal{F}}}

\renewcommand\R{{\mathbb{R}}}
\renewcommand\P{{\mathbb{P}}}

\newcommand\Diag{{\mathsf{Diag}}}
\newcommand\offd{{\mathsf{off}}}

\newcommand\diag{{\mathsf{diag}}}

\usepackage{physics}

\newcommand\J{{\mathcal{J}}}

\newcommand\Q{\mathcal{Q}}

\newcommand\T{{\mathcal{T}}}

\newcommand\N{{\mathbb{N}}}
\newcommand{\inner}[1]{\left<#1\right>}
\renewcommand{\norm}[1]{\left\|#1\right\|}
\newcommand{\snorm}[1]{\left\|#1\right\|_{\mathrm{op}}}

\newcommand{\round}[1]{\left(#1\right)}
\newcommand{\brac}[1]{\left[#1\right]}
\renewcommand{\abs}[1]{\left|#1\right|}

\newtheorem{theorem}{Theorem}
\newtheorem{lemma}[theorem]{Lemma}
\newtheorem{corollary}[theorem]{Corollary}
\newtheorem{claim}{Claim}
\newtheorem{proposition}[theorem]{Proposition}

\newtheorem{definition}{Definition}

\newtheorem{assumption}{Assumption}

\newcommand{\tran}{^\top}
\newcommand{\inv}{^{-1}}

\newcommand{\vepsilon}{{\bm{\varepsilon}}}

\newcommand{\He}{{\rm{He}}}
\renewcommand{\op}{{\rm{op}}}
\usepackage[normalem]{ulem}

\makeatletter
\@removefromreset{theorem}{section}
\makeatother
\title{Average Gradient Outer Product in kernel regression provably recovers the central subspace for multi-index models
}
\author{
Libin Zhu\thanks{Department of Mathematics, University of Washington, Seattle, WA 98195; \texttt{https://libinzhu.github.io/}. Research of Zhu was supported by NSF TRIPODS II DMS-2023166.}
\and
Damek Davis\thanks{Wharton Department of Statistics and Data Science, University of Pennsylvania,
		Philadelphia, PA 19104, USA;
		\texttt{www.damekdavis.com}. Research of Davis supported by an Alfred P. Sloan research fellowship and NSF DMS award 2047637}
\and
Dmitriy Drusvyatskiy\thanks{Department of Mathematics, U. Washington, Seattle, WA 98195; \texttt{https://sites.google.com/view/dmitriy-drusvyatskiy}.
		Research of Drusvyatskiy was supported by NSF DMS-2306322, NSF DMS-2023166, and AFOSR FA9550-24-1-0092 awards.}
\and 
Maryam Fazel\thanks{Department of Electrical \& Computer Engineering, University of Washington, Seattle, WA 98195, and Amazon, Inc. \texttt{https://people.ece.uw.edu/fazel\_maryam/}. Research supported by NSF TRIPODS II DMS-2023166, CCF 2007036, CCF 2212261, CCF 2312775 and the Moorthy Family professorship at UW.}
}
\begin{document}
\maketitle

\begin{abstract}

We study a prototypical situation when a learned
predictor can discover useful low-dimensional structure in data, while using fewer samples than are needed for accurate prediction. Specifically,  we consider the problem of recovering a multi-index polynomial $f^*(x)=h(Ux)$, with $U\in\R^{r\times d}$ and $r\ll d$, from finitely many data/label pairs. Importantly, the target function depends on input $x$ only through the projection onto an unknown $r$-dimensional central subspace. The algorithm we analyze is appealingly simple: fit kernel ridge regression (KRR) to the data and compute the Average Gradient Outer Product (AGOP) from the fitted predictor.  Our main results show that under reasonable assumptions the top $r$-dimensional eigenspace of AGOP provably recovers the central subspace, even in regimes when the prediction error remains large. 
Specifically, if the target function $f^*$ has degree $p^*$, it is known that $n\asymp d^{p^*}$ samples are necessary for KRR to achieve accurate prediction. In contrast, we show that if a low degree $p$ component of $f^*$ already carries all relevant directions for prediction, subspace recovery occurs in the much lower sample regime $n\asymp d^{p+\delta}$ for any $\delta\in(0,1)$.
Our results thus demonstrate a separation between prediction and representation, and provide an explanation for why iterative kernel methods such as Recursive Feature Machines (RFM) can be sample-efficient in practice.
\end{abstract}

\section{Introduction}
Modern machine learning systems often succeed not only by fitting input-output relationships, but also by learning useful intermediate representations. This theme is central to contemporary representation learning: models are trained on high-dimensional data, yet the task-relevant structure is often believed to depend on a much smaller number of latent factors. A basic theoretical question is whether a trained predictor can reveal such low-dimensional structure with far less data than it needs for accurate prediction.

In this paper, we study this question in a simple and tractable setting. We consider multi-index regression, in which the response $y$ depends on input data \(x\in\R^d\) only through a small number of linear measurements,
\begin{equation}\label{eq:intro_mim}
y \;=\; f^*(x) + \mathrm{noise},
\qquad
f^*(x) \;=\; h(Ux),
\end{equation}
where \(U\in\R^{r\times d}\) has orthonormal rows with \(r\ll d\) and $h\colon\R^r\rightarrow \R$ is an unknown link function.
The row space of \(U\), denoted \(\mathrm{row}(U)\), is called the central subspace. Recovering this subspace is valuable in its own right: it provides an interpretable representation of the data, enables dimension reduction, and reduces downstream learning to an \(r\)-dimensional problem.

Gradient information provides a direct population-level characterization of the central subspace. If \(f^*(x)=h(Ux)\) is differentiable, then the population Average Gradient Outer Product (AGOP) takes the form
\begin{equation}\label{eq:ambient_agop}
\E[\nabla f^*(x)\nabla f^*(x)^\top]
=
U^\top
\E[\nabla h(Ux)\nabla h(Ux)^\top]
U,
\end{equation}
and its range coincides with \(\mathrm{row}(U)\) in nondegenerate settings. This observation underlies average-derivative and outer-product-of-gradients methods in efficient dimension reduction.

At the same time, this is only a population statement: the regression function is unknown, and naively estimating its gradient field in high dimension is typically no easier than solving the original nonparametric regression problem. Recent work on Recursive Feature Machines (RFM) uses AGOP-based updates for adaptive kernel machines, and has shown strong empirical performance on multi-index models and related learning tasks~\cite{rfm_science,radhakrishnan2024linear,beaglehole2026xrfm}. One aim of the present paper is to explain the feature-learning mechanism underlying this empirical behavior. To this end, we study the AGOP formed from a fitted kernel ridge regression (KRR) predictor, and investigate whether the gradients of the fitted predictor can reveal the central subspace even when the predictor has not yet learned the full target function.

We show that the empirical AGOP constructed from a fitted KRR predictor provably recovers the central subspace in a high-dimensional regime. We focus on Boolean data drawn uniformly from \(\{-1,1\}^d\), fit KRR with an inner-product kernel, and form the empirical AGOP matrix
\begin{equation}\label{eq:intro_emp_agop}
\widehat M \;:=\; \frac{1}{n} \sum_{i=1}^n \!\big[\nabla \hat f(x^{(i)})\,(\nabla \hat f(x^{(i)}))^\top\big] \in \R^{d\times d},
\end{equation}
where \(\hat f\) is the KRR predictor and the gradients are ambient gradients evaluated at the training inputs \(x^{(i)}\). Our main results show that the empirical AGOP \(\widehat M\) concentrates around a low-degree population AGOP target (Theorem~\ref{thm:main_agop}), and under additional nondegeneracy and a weak subspace coherence condition~\cite{candes2012exact}, the top-\(r\) eigenspace of \(\widehat M\) consistently recovers the central subspace \(\mathrm{row}(U)\) (Corollary~\ref{cor:main_result}). 

\paragraph{Subspace recovery before accurate prediction.}
The central message of this work is that recovering the central subspace from a fitted predictor can be statistically easier than achieving low test error. Specifically, in the sample regime \(n\asymp d^{p+\delta}\) with \(\delta\in(0,1)\), KRR primarily captures the degree-\(\le p\) component of the target function. The prediction error can therefore remain large when a substantial portion of the signal lies in higher degree components. Our contribution is to show that this does not preclude accurate recovery of the central subspace: if the low-degree component already captures all relevant directions in the central subspace, then the gradients of the fitted predictor remain informative enough for the empirical AGOP to recover the subspace. Thus, the sample complexity of representation learning is governed by the first degree at which the low-degree gradient covariance contains all relevant directions, rather than by the degree needed to learn the full target function.

This viewpoint suggests a natural two-stage approach. First, fit KRR and recover the central subspace \(\mathrm{row}(U)\) via the empirical AGOP. Second, project the input data onto the estimated subspace and learn the full link function 
in dimension \(r\). Since \(r\) is small, the second stage is expected to learn higher-degree structure without incurring a \(d\)-dimensional sample-complexity cost. This viewpoint is also related to iterative kernel methods such as RFM~\cite{radhakrishnan2024linear} and IRKM~\cite{zhu2025iteratively}, which aim to refine feature discovery over successive iterations. See the discussion in Section~\ref{sec:fl}. Our experiments show that RFM identifies the central subspace well before
its prediction error becomes small, consistent with our subspace-recovery
results; subsequent iterations then improve prediction empirically. See Section~\ref{sec:exp} for numerical experiments.

\paragraph{Contributions.}
Our contributions are threefold. First, we establish a subspace recovery guarantee for the empirical AGOP of the fitted KRR predictor in a polynomial high-dimensional regime. Second, we show a prediction--representation separation: the central subspace can be recovered accurately even in low-sample regimes where the prediction test error remains large. Third, we develop analytical tools for random matrix analysis with hypercube data, which may be of independent interest.

\subsection{Notation}\label{sec:notation}

\paragraph{Indices and linear-algebra notation.}We denote by \(\mathbb N\) the set of natural numbers and write
\(\mathbb N^+ := \mathbb N\setminus\{0\}\). For \(n\in\mathbb N^+\), let
\([n] := \{1,\dots,n\}\). A vector \(\lambda\in\mathbb N^d\) is called a
multi-index, and we write \(|\lambda| := \lambda_1+\cdots+\lambda_d\). For
\(x\in\R^d\) and \(\lambda\in\mathbb N^d\), we use the standard monomial notation
\(x^\lambda := x_1^{\lambda_1}\cdots x_d^{\lambda_d}\), and we write
\(\He_\lambda(x) := \prod_{i=1}^d \He_{\lambda_i}(x_i)\), where \(\He_k\) is the
univariate probabilist's Hermite polynomial of degree \(k\). 

For \(u,v\in\R^d\),  write \(\inner{u,v} := \sum_{i=1}^d u_i v_i\) for the
Euclidean inner product and \(\norm{v}_p := (\sum_{i=1}^d |v_i|^p)^{1/p}\) for the
\(\ell_p\)-norm of \(v\). For a matrix \(M\), we denote its \((i,j)\)-entry by
\(M_{i,j}\), its \(i\)th row by \(M_{i,:}\), and its \(j\)th column by
\(M_{:,j}\). We write \(\diag(M)\) for the vector of diagonal entries of \(M\) and
\(\Diag(v)\) for the diagonal matrix with diagonal \(v\). For a matrix \(U\in\R^{r\times d}\),
the symbol \(\mathrm{row}(U)\subseteq\R^d\) denotes its row space, and
\(U_\perp\in\R^{(d-\mathrm{rank}(U))\times d}\) denotes the orthonormal complement of \(\mathrm{row}(U)\). The symbols \(\snorm{M}\) and \(\norm{M}_F\)
denote the spectral and Frobenius norms of \(M\), respectively.

We write \(\tau_d\) for the uniform probability measure on \(\H^d := \{-1,1\}^d\), and for a function
\(f\colon\H^d\to\R\) we define the usual square $L_2$-norm
\(\|f\|_{L_2(\tau_d)}^2 := \int_{\H^d}f^2\,d\tau_d\).

\paragraph{Differential notation.}
For a differentiable function \(f\colon\R^d\to\R\), we write
\(\nabla f(x) := (\partial_1 f(x),\dots,\partial_d f(x))\) for its Euclidean
gradient. For any point on the hypercube \(x\in\H^d\), we regard \(x\) as a point of \(\R^d\) and use the
same notation \(\nabla f(x)\) for the ambient gradient evaluated at \(x\). 
\paragraph{Asymptotic notation.}For
nonnegative functions \(f\) and \(g\), we write \(f\lesssim g\) if there exists a numerical constant
\(C>0\) such that \(f(x)\le Cg(x)\) for all \(x\); similarly, we write \(f\gtrsim g\) to  mean
\(g\lesssim f\), and \(f\asymp g\) means both \(f\lesssim g\) and \(f\gtrsim g\).
For functions \(f,g\colon\R^d\times\mathbb N\to\R\), the notation \(f=O_d(g)\) means
that there exist constants \(C>0\) and \(d_0\in\mathbb N\) such that
\(|f(x,d)|\le C|g(x,d)|\) for all \(x\) and all \(d\ge d_0\), while \(f=o_d(g)\)
means that for every \(\epsilon>0\) there exists \(d_\epsilon\in\mathbb N\) such
that \(|f(x,d)|\le \epsilon |g(x,d)|\) for all \(x\) and all \(d\ge d_\epsilon\).
In both cases, the comparison is uniform in \(x\). 
Finally, for sequences of random
variables \(f_d\) and \(g_d\), we write \(f_d=O_{d,\mathbb P}(g_d)\) if for every
\(\epsilon>0\) there exists \(C_\epsilon>0\) such that
\(\limsup_{d\to\infty}\mathbb P(|f_d|>C_\epsilon g_d)\le\epsilon\), and we write 
\(f_d=o_{d,\mathbb P}(g_d)\) if for every \(\epsilon>0\), we have
\(\lim_{d\to\infty}\mathbb P(|f_d|>\epsilon g_d)=0\).

\section{Related work}\label{sec:related_work}

\paragraph{Sufficient dimension reduction and multi-index models.}
Sufficient dimension reduction (SDR) studies regression settings in which the conditional distribution, or at least the conditional mean, of the response depends on a covariate \(x\in\mathbb R^d\) only through a low-dimensional linear projection. Classical inverse-regression methods include sliced inverse regression (SIR) and principal Hessian directions (pHd) \cite{Li1991SIR,li1992principal}; see also \cite{cook2007fisher,ma2013review}.A complementary gradient-based line estimates dimension-reduction directions using average-derivative estimators
\cite{hardle1989investigating,powell1989semiparametric},
outer products of gradients and MAVE-type local-linear procedures
\cite{xia2002adaptive,xia2007constructive},
and structure-adaptive estimators
\cite{hristache2001structure}. Closely related are kernel SDR methods: kernel dimension reduction (KDR) characterizes sufficient directions through RKHS conditional covariance operators \cite{fukumizu2009kernel}, whereas gradient-based KDR estimates directions from gradients of a kernel-estimated regression function \cite{fukumizu2014gradient}. Single- and multi-index models have also been studied extensively in semiparametric statistics and, more recently, under high-dimensional structural assumptions such as sparsity or smoothness
\cite{ichimura1993semiparametric,lin2018consistency,yuan2025efficient}. Our work differs from this literature in both estimator and regime: we analyze the empirical AGOP of a \emph{trained} KRR predictor and prove post-hoc subspace recovery in a polynomial high-dimensional regime.

\paragraph{High-dimensional KRR and spectral structure.}
High-dimensional kernel methods have been analyzed through random-matrix and spectral descriptions of empirical kernel matrices and KRR risk
\cite{elkaroui2010spectrum,pandit2025universality}. For inner-product kernels in polynomial high-dimensional regimes, analyses under spherical and hypercube covariate models show that KRR prediction is governed by low-degree polynomial components, leading to a degree-truncation phenomenon
\cite{ghorbani2021linearized,mei2022generalization}. We use this low-degree structure but study a different object: rather than characterizing the test error of the KRR predictor, we analyze the concentration and eigenspace structure of the empirical AGOP formed from its fitted gradients.

\paragraph{Feature learning in neural networks.}
Single- and multi-index models are standard testbeds for understanding feature learning in neural networks. For single-index targets, learnability under online SGD, gradient flow, and related shallow-network dynamics is often governed by Hermite coefficients, information exponents, or generative exponents
\cite{arous2021online,bietti2022learning,damian2023smoothing,mousavineural,lee2024neural,tsiolis2025from}. Related high-dimensional analyses show how first-layer adaptation can improve over fixed random features or kernel methods
\cite{ba2022high,moniri2024theory}. 
For multi-index and sparse low-dimensional targets, recent work studies representation recovery \cite{damian2022neural,mousavi2023neural,bietti2023learning,mousavi2025learning,simsek2025learning}; staircase and leap-complexity phenomena \cite{abbe2022merged,abbe2023sgd}; finite-step effects \cite{dandi2024how}; batch-size and batch-reuse effects \cite{arnaboldi2024online,dandi2024benefits,arnaboldi2025repetita}; and proportional-limit descriptions of feature learning \cite{montanari2026phase}.
The generative-leap framework of \cite{damiangenerative} establishes sharp thresholds for efficient hidden-subspace recovery in Gaussian multi-index models. In contrast, we study KRR with a fixed kernel---a setting closely related to the kernel/NTK regime of neural networks \cite{jacot2018neural}---and show that the fitted gradients can contain sufficient information for post-hoc representation recovery, without weight evolution.

\paragraph{Kernel-based feature learning and AGOP/EGOP methods.}

A growing literature shows that feature learning can also be realized through adaptive kernel methods. Recursive feature machines (RFM) \cite{radhakrishnan2024linear} and iteratively reweighted kernel machines (IRKM) \cite{zhu2025iteratively} alternate kernel regression with gradient- or AGOP-based metric updates, while related approaches optimize a linear map inside the kernel \cite{huang2025learning}, use coordinate-wise kernel reweighting \cite{ruan2025compositional}, or adapt local EGOP metrics, where EGOP is the population analogue of AGOP \cite{kokot2026local}. These works adapt the kernel or metric during training; in particular, Huang et al.~\cite{huang2025learning} prove finite-sample excess-risk guarantees for HKRR, a representation-learning variant of KRR that optimizes a linear map inside the kernel, while we study the AGOP of a fitted ordinary fixed-kernel KRR predictor in the polynomial high-dimensional regime. Our closest comparison is \cite{zhu2025iteratively}: both works use gradients of fitted kernel predictors under hypercube covariates, but \cite{zhu2025iteratively} studies sparse and hierarchical targets via iterative metric learning, whereas we study general multi-index models and the eigenspace structure of the resulting AGOP.

\section{Main Results}\label{sec:main}
Throughout this section,  we let the hypercube \(\mathbb H^d=\{-1,1\}^d\) be equipped with the uniform measure \(\tau_d\), and all \(L_2\)-norms on \(\mathbb H^d\) are taken with respect to \(\tau_d\). The main result has two steps. First, we show that the empirical AGOP of the fitted KRR predictor concentrates around a low-degree population target \(M_{\le p}\). Second, we show that if this low-degree target preserves the central subspace, then under a weak coherence condition, the top-\(r\) eigenspace of the empirical AGOP consistently recovers the central subspace \(\mathrm{row}(U)\).

\subsection{Model and low-degree population targets}
We consider the multi-index model
\begin{align}\label{eq:multi_index}
    f^*(x)=h(Ux), \qquad x\in\mathbb R^d,
\end{align}
where \(U\in\mathbb R^{r\times d}\) has orthonormal rows. We refer to \(\mathrm{row}(U)\) as the planted central subspace. 

We use the standard notion of subspace coherence, commonly used in
matrix completion analysis~\cite{candes2012exact}. 
\begin{definition}[Coherence of the central subspace~\cite{candes2012exact}]
\label{def:coherence}
Let \(U\in\mathbb R^{r\times d}\) satisfy \(UU^\top=I_r\). We define the
coherence of its row space by
\begin{align}  
\mu(U)
:=
\frac{d}{r}
\max_{i\in[d]}
\|U_{:,i}\|_2^2 .
\end{align}
\end{definition}
Equivalently, with \(P_{\operatorname{row}(U)}=U^\top U\) the orthogonal projection onto $\operatorname{row}(U)$, this is the usual
subspace coherence
\(
\mu(\operatorname{row}(U))
=
\frac{d}{r}
\max_{i\in[d]}
\|P_{\operatorname{row}(U)}e_i\|_2^2
\)
with respect to the standard basis
\((e_i)_{i=1}^d\). Note that small
coherence means that no single ambient coordinate carries too much mass from
\(\mathrm{row}(U)\).

\begin{assumption}[Finite degree]\label{assum:degree}
The link function \(h\) is a polynomial of total degree at most \(\ell\).
\end{assumption}

Throughout, we work with Boolean-hypercube data \(x\in\mathbb H^d=\{\pm1\}^d\). This setting is  natural for sharp KRR analysis since the required low-degree empirical-kernel-matrix approximations are available for uniform spherical/hypercube designs, whereas the corresponding optimal approximation theory for Gaussian designs remains only partially understood and substantially more delicate~\cite{kaushik2025general}.

On the hypercube, polynomials are naturally reduced modulo the relations \(x_i^2=1\). Thus, before taking low-degree truncations, we first replace each polynomial by its multilinear representative. For a polynomial \(P\), let \(\mathcal H(P)\) denote the unique multilinear polynomial agreeing with \(P\) on \(\mathbb H^d\). Equivalently, on monomials, we have
\[
\mathcal H\!\left[\prod_{i=1}^d x_i^{\alpha_i}\right]
:=
\prod_{i=1}^d x_i^{\alpha_i \bmod 2},
\]
and this relation extends linearly to all polynomials on the hypercube.
We use the standard convention that \(F_q\) denotes the projection of \(F\)
onto the degree-\(q\) Fourier--Walsh characters, and
\(F_{\le p}:=\sum_{q\le p}F_q\). Thus \(\gH(P)_q\) and
\(\gH(P)_{\le p}\) denote the corresponding degree-\(q\) component and
degree-\(\le p\) truncation of \(\gH(P)\). See Appendix~\ref{sec:fourier_exp} for the precise definition of Fourier--Walsh polynomials.

{On the hypercube, only the values of the ambient polynomial
\(f_U^*(x)=h(Ux)\) on \(\mathbb H^d\) are used. To simplify the notation, we 
write
\[
f^*:=\mathcal H(f_U^*)
\]
for its unique multilinear representative.  Thus \(f^*(x)=f_U^*(x)\) for all
\(x\in\mathbb H^d\). Throughout the hypercube
analysis, we let \(f_q^*\) and \(f_{\le p}^*\) denote the Fourier--Walsh components
of this reduced target. The latent link \(h\) itself is not reduced and is
used below for the Hermite comparison.} Next, we define the truncated population AGOP on the hypercube: 
\begin{align}\label{eq:egop_trunc}
M_{\le p}
=\sum_{q=0}^p M_q:=
\sum_{q=0}^p
\E_{x\sim\tau_d}
\big[
\nabla f^*_q(x)\nabla f^*_q(x)^\top
\big].
\end{align}
Since differentiation lowers Walsh degree by one and distinct Walsh degrees are orthogonal, cross terms between different \(q\)'s vanish. Hence we equivalently  have
\[
M_{\le p}
=
\E_{x\sim\tau_d}
\big[
\nabla  f^*_{\le p}(x)
\nabla  f^*_{\le p}(x)^\top
\big].
\]
{We will now show that \(M_{\le p}\) is closely related to the Gaussian latent AGOP of
the degree-\(\le p\) Hermite truncation of \(h\). See Appendix~\ref{proof:M_gaussian_latent} for the Hermite basis and
normalization.} 
Write the Hermite expansion of \(h\) under
\(z\sim\mathcal N(0,I_r)\) as
\[
h(z)
=
\sum_{\alpha\in\mathbb N^r:\,|\alpha|\le \ell}
a_\alpha \He_\alpha(z)
=
\sum_{q=0}^{\ell} h_q(z),
\qquad\textrm{where}\qquad
h_q(z)
:=
\sum_{\alpha\in\mathbb N^r:\,|\alpha|=q}
a_\alpha \He_\alpha(z).
\]
Letting $h_{\le p}(z):=\sum_{q=0}^{p}h_q(z)$, define the degree-\(\le p\) latent gradient covariance
\begin{align}\label{eq:hermite_truncation}
\Sigma_p
:=
\E_{z\sim\mathcal \gN(0,I_r)}
\big[
\nabla h_{\le p}(z)\nabla h_{\le p}(z)^\top
\big]~\in \R^{r\times r}.
\end{align}
The following lemma shows that \(M_{\le p}\) is a coherence-controlled perturbation of
\(U^\top\Sigma_pU\):
\begin{lemma}[Population hypercube-to-Gaussian AGOP comparison]
\label{lem:M_gaussian_latent}
Fix \(p\in\{0,\ldots,\ell\}\). Suppose Assumption~\ref{assum:degree} holds,
and \(r,\ell\) are fixed. Then
\begin{align}\label{eq:M_gaussian_latent}  
\left\|
M_{\le p}-U^\top\Sigma_pU
\right\|_{\op}
=
O_d\left(
\frac{\mu(U)}{d}\|f^*\|_{L_2}^2
\right).
\end{align}
\end{lemma}

\paragraph{Learning setup.}

We observe training data \(\{(x^{(i)},y_i)\}_{i=1}^n\) with
\begin{align}\label{eq:data}
x^{(i)}\overset{\mathrm{i.i.d.}}{\sim}\tau_d,
\qquad
y_i=f^*(x^{(i)})+\varepsilon_i,
\qquad
\varepsilon_i\overset{\mathrm{i.i.d.}}{\sim}\mathcal \gN(0,\sigma_\vepsilon^2),
\end{align}
where the noise is independent of the covariates. Let \(X\in\mathbb R^{n\times d}\) denote the design matrix with rows \(x^{(1)},\dots,x^{(n)}\), and let \(y\in\mathbb R^n\) denote the response vector.

We fit kernel ridge regression with ridge parameter \(\lambda\geq 0\) and inner-product kernel
\begin{align}\label{eq:inner_product_kernel}
K(x,x')
=
g\!\left(\frac{\langle x,x'\rangle}{d}\right),
\end{align}
where $g : \R \rightarrow \R$ is a kernel profile. 
The fitted predictor can be written as
\[
\hat f(x)=K(x,X)\,(K(X,X)+\lambda I_n)^{-1}y,
\]
where
$
K(x,X):=\bigl(K(x,x^{(1)}),\dots,K(x,x^{(n)})\bigr)$ and $
K(X,X):=\bigl(K(x^{(i)},x^{(j)})\bigr)_{i,j\in[n]}
$. Throughout the paper, we take the ridge parameter $\lambda =O_d(1)$. 
We then form the empirical AGOP
\begin{equation}\label{eq:empirical-agop}
\widehat M
:=
\frac1n\sum_{i=1}^n
\nabla \hat f(x^{(i)})\,\nabla \hat f(x^{(i)})^\top.
\end{equation}
Our main results will invoke a combination of the following two regularity assumptions on $g$.
\begin{assumption}[Analyticity near \(0\)]\label{assump:g_0}
There exists \(\varepsilon_0\in(0,1)\) such that \(g\) is analytic on
\((-\varepsilon_0,\varepsilon_0)\) and satisfies \(g^{(k)}(0)\ge 0\) for all \(k\ge 0\).
\end{assumption}

\begin{assumption}[Regularity near \(1\)]\label{assump:g_1}
There exists \(\varepsilon_1\in(0,1)\) such that \(g'\) is Lipschitz on
\((1-\varepsilon_1,1+\varepsilon_1)\).
\end{assumption}

These assumptions cover standard inner-product kernels such as polynomial and exponential kernels, and, on fixed-norm domains, Gaussian kernels.

\subsection{Empirical AGOP approximation and subspace recovery}

The next theorem gives an operator-norm AGOP approximation with explicit
dependence on the subspace coherence \(\mu(U)\).  We assume that the rank \(r\) and largest degree \(\ell\)
are fixed as \(d\to\infty\). The dependence of the constants on \(r\) and
\(\ell\) is recorded in the proof that appears in
Appendix~\ref{proof:main_agop}.

\begin{theorem}[AGOP approximation]\label{thm:main_agop}
Suppose Assumptions~\ref{assum:degree}, \ref{assump:g_0} and \ref{assump:g_1} hold. Suppose further that
\(
\min_{0\le k\le p} g^{(k)}(0)\) and 
\(\lambda+\sum_{k=p+1}^\infty g^{(k)}(0)
\) are strictly positive.
Fix \(p\in\{0,\dots,\ell\}\), and let \(n=d^{p+\delta}\) with \(\delta\in (0,1)\). Then, for every sufficiently small \(\epsilon>0\), the following holds:
\begin{align}\label{eq:main_result_1}
\|\widehat M-M_{\le p}\|_{\mathrm{op}} =O_{d,\mathbb P}\bigl( R_d(U) \cdot \round{
    \|f^*\|_{L_2}^2+\sigma_{\vepsilon}^2}
\bigr),
\end{align}
where we set \begin{align}
    R_d(U)
:=
d^{-\tfrac{\delta}{2}+\epsilon}
+d^{-\tfrac{1-\delta}{2}+\epsilon}
+\mu(U)^{\tfrac{1}{2}}d^{-\tfrac{1+\delta}{4}+\epsilon} 
+\mu(U)d^{-\tfrac{2+\delta}{4}+\epsilon}
+\mu(U)^2d^{-\tfrac{2+\delta}{2}+\epsilon}.
\end{align}
In particular, if there exists a constant \(\gamma>0\) such that
\(
\mu(U)=O_d\!\left(d^{1/2+\delta/4-\gamma}\right),
\)
then, choosing \(\epsilon>0\) sufficiently small depending on
\(\delta\) and \(\gamma\) yields
$
\|\widehat M-M_{\le p}\|_{\op}
=
o_{d,\P}\bigl(
    \|f^*\|_{L_2}^2+\sigma_{\vepsilon}^2
\bigr).$ 
\end{theorem}

We next turn the approximation of AGOP into a subspace recovery statement.  For row-orthonormal \(U,\widehat U\in\mathbb R^{r\times d}\), we define the usual subspace sine matrix
\[
\sin_\Theta(\widehat U,U)
:=
\widehat U (U_\perp)^\top
\in \mathbb R^{r\times(d-r)}.
\]
Its singular values are the sines of the principal angles between
\(\operatorname{row}(\widehat U)\) and \(\operatorname{row}(U)\), and hence are independent of the choice of \(U_\perp\).
We further define $s_p$ and $\rho_p$ which measure the part of \(M_{\le p}\) inside and outside \(\mathrm{row}(U)\):
\begin{align}\label{eq:s_and_rho}
s_p
:=
\lambda_{\min}\!\big(U M_{\le p} U^\top\big),
\qquad
\rho_p
:=
\snorm{M_{\le p}-P_{\mathrm{row}(U)} M_{\le p}P_{\mathrm{row}(U)}},
\end{align}
where $ P_{\mathrm{row}(U)}  = U\tran U$. The next lemma states the subspace recovery result using $s_p$ and $\rho_p$.
The result follows from the Davis--Kahan sin-theta theorem and the proof appears in Appendix~\ref{proof:eigenspace-perturbation}.

\begin{lemma}[Eigenspace perturbation]\label{lem:eigenspace-perturbation}
Let
$\varepsilon_{\mathrm{agop}}
:=
\snorm{\widehat M-M_{\le p}}$ and suppose $s_p
>0$ holds. 
If \(\widehat U\in\mathbb R^{r\times d}\) has orthonormal rows spanning the top-\(r\) eigenspace of \(\widehat M\), then the estimate holds:
\begin{equation}\label{eq:eigenspace-perturbation}
    \snorm{\sin_\Theta(\widehat U,U)}
\le
\min\left\{
1,\,
4\,\frac{\rho_p+\varepsilon_{\mathrm{agop}}}{s_p}
\right\}.
\end{equation}
\end{lemma}
Consequently, Lemma~\ref{lem:eigenspace-perturbation} reduces subspace recovery to
showing
\((\varepsilon_{\mathrm{agop}}+\rho_p)/s_p=o_{d,\mathbb P}(1)\).
The next corollary verifies this criterion under latent nondegeneracy and weak
coherence, by combining Lemma~\ref{lem:M_gaussian_latent},
Theorem~\ref{thm:main_agop}, and Lemma~\ref{lem:eigenspace-perturbation}.
The proof is deferred to Appendix~\ref{proof:cor:main_result}.

\begin{corollary}[Subspace recovery under a weak coherence condition]\label{cor:main_result}
Suppose the assumptions of Theorem~\ref{thm:main_agop} hold. 
Suppose further that there exist constants \(\kappa>0\) and \(\gamma>0\)
such that the following three conditions hold:
\[
\lambda_{\min}(\Sigma_p)\ge \kappa,
\qquad
\mu(U)=O_d\!\left(d^{1/2+\delta/4-\gamma}\right),
\qquad
\frac{\mu(U)}{d}\,\norm{f^*}_{L_2}^2=o_d(\kappa).
\] Let \(\widehat U\in\mathbb R^{r\times d}\) have orthonormal rows spanning the top-\(r\) eigenspace of \(\widehat M\). Then, 
\[
\snorm{\sin_\Theta(\widehat U,U)}
=
o_{d,\mathbb P}\!\left(
\kappa\inv 
\big(\|f^*\|_{L_2}^2+\sigma_\vepsilon^2\big)
\right).
\]
\end{corollary}

The condition \(\lambda_{\min}(\Sigma_p)>0\) in Corollary~\ref{cor:main_result} is the AGOP analogue of the
nondegeneracy assumptions commonly used in feature-learning analyses of
multi-index models~\cite{damian2022neural,zhang2026neural,bietti2023learning}: it ensures that the
degree-\(\le p\) AGOP contains all relevant directions. The coherence condition is a sufficient
technical condition. It is substantially weaker than the bounded or
polylogarithmic incoherence assumptions common in matrix completion
analysis~\cite{candes2012exact}. 
The experiments in Section~\ref{sec:exp} suggest that AGOP-based subspace recovery
succeeds under broader conditions than those guaranteed by
Corollary~\ref{cor:main_result}.

\section{Implications for feature learning}\label{sec:fl}
Recursive Feature Machines (RFM) are iterative kernel methods proposed to
enable feature learning in kernel machines:  they alternate
between fitting a kernel predictor and updating the kernel metric using the
empirical AGOP of that predictor~\cite{radhakrishnan2024linear,rfm_science}. This section shows how one RFM step turns AGOP recovery into feature learning. Although the first KRR fit from the isotropic metric \(M_1=I_d\) mainly captures the low-degree component of the target, its empirical AGOP already recovers the central subspace \(\mathrm{row}(U)\). RFM then uses this AGOP to update the metric, producing an anisotropic rescaling: relevant directions are amplified, while orthogonal directions remain essentially unchanged. 

We consider inner-product kernels of the form
\begin{align}\label{eq:inner_product_kernel_M}
    K_M(x,x')
    \;=\;
    g\!\left(\frac{x^\top M x'}{d}\right) = g\!\left(\frac{\langle M^{1/2}x,\,M^{1/2}x'\rangle}{d}\right),
\end{align}
where \(M\in\mathbb R^{d\times d}\) is positive semidefinite. Therefore,  KRR with kernel \(K_M\) can be interpreted as ordinary inner-product-kernel KRR applied to the transformed inputs \(M^{1/2}x\).

Starting from \(M_1=I_d\), and fixing a safeguard parameter \(\eta>0\) and ridge parameter \(\lambda>0\), RFM alternates between
\begin{itemize}
    \item \textbf{Step 1 (KRR):}
    \[
    \hat f_t(x)
    \;=\;
    K_{M_t}(x,X)\,(K_{M_t}(X,X)+\lambda I_n)^{-1}y.
    \]
    \item \textbf{Step 2 (metric update):}
    \[
    M_{t+1}
    =
    \frac{d}{\operatorname{tr}(\widehat M_t+\eta I_d)}
    \,(\widehat M_t+\eta I_d),
    \qquad
    \widehat M_t
    :=
    \frac1n\sum_{i=1}^n
    \nabla \hat f_t(x^{(i)})\,\nabla \hat f_t(x^{(i)})^\top.
    \]
\end{itemize}
The trace normalization fixes the metric scale across iterations, while the
safeguard \(\eta I_d\), as in~\cite{zhu2025iteratively}, prevents small-AGOP directions from being driven to zero.

At the first iteration, \(\hat f_1\) is exactly the standard KRR predictor. The next theorem, adapted from Theorem~4 of~\cite{mei2022generalization}, records that in the polynomial-sample regime it recovers only the low-degree part of the target. The proof is deferred to Appendix~\ref{proof:test_error_krr}.

\begin{theorem}[Adapted from Theorem~4 of~\cite{mei2022generalization}]\label{thm:test_error_krr}
Under the assumptions of Theorem~\ref{thm:main_agop}, the first-step KRR predictor \(\hat f_1\) satisfies
\begin{align}
    \|\hat f_1- f^*_{\le p}\|_{L_2}^2
    =
    o_{d,\mathbb P}(1)\,
    \big(
        \|f^*\|_{L_2}^2+\sigma_\vepsilon^2
    \big).
\end{align}
\end{theorem}
Combining Theorem~\ref{thm:test_error_krr} with 
Corollary~\ref{cor:main_result} yields a prediction--representation separation. Indeed, if the target function $f^*$ has degree $p^*$ with $p^*>p$, then in the regime $n = d^{p+\delta}$, the empirical AGOP recovers the central
subspace, while the prediction error of the first KRR fit can remain large:
\begin{align}
    \|\hat f_1- f^*\|_{L_2}^2
    = \norm{ f^*_{> p}}_{L_2}^2 + 
    o_{d,\mathbb P}(1)\,
    \big(
        \|f^*\|_{L_2}^2+\sigma_\vepsilon^2
    \big).
\end{align}

In the remainder of this section, suppose that all assumptions of
Corollary~\ref{cor:main_result} hold. Then, more importantly for our purposes, Corollary~\ref{cor:main_result} shows that the AGOP \(\widehat M_1\) already identifies the central subspace. The next proposition shows that the first RFM metric update turns this recovered geometry into an anisotropic rescaling of the inputs. The proof is deferred to Appendix~\ref{proof:project_data}.

\begin{proposition}\label{prop:project_data}
Set the safeguard parameter
\(
\eta
=
d^\zeta\cdot
\max\{d^{-\delta/2},\,d^{-(1-\delta)/2}\}
\)
for a small constant \(\zeta>0\). For an independent sample $x\sim \tau_d$, define
\begin{align}\label{eq:widehat_x}
\widehat x
:=
U^\top \sqrt{\Sigma_p+\eta I_r}\,Ux
+
\sqrt{\eta}\,(U_\perp)^\top U_\perp x.
\end{align}
Let
\(
c_\eta:=\operatorname{tr}(\widehat M_1+\eta I_d).
\)
Then the following estimate holds
\begin{align}\label{eq:normalized_sqrt_approx}
\big\|
M_2^{1/2}x-\sqrt{d/c_\eta}\,\widehat x
\big\|_{L_2}
=
o_{d,\mathbb P}\!\left(
\sqrt{d/c_\eta}\,\|\widehat x\|_{L_2}
\right),
\end{align}
where
\(
c_\eta=(1+o_{d,\mathbb P}(1))\,\eta d.
\)
\end{proposition}
Writing $z_0:=Ux$ and $z_1:=U_\perp x$,
Proposition~\ref{prop:project_data} indicates
\begin{align}
M_2^{1/2}x
&=
(1+o_{d,\mathbb P}(1))
\left(
U^\top \sqrt{\eta^{-1}\Sigma_p+I_r}\,z_0
+
(U_\perp)^\top z_1
\right).
\end{align}
Thus the second iterate applies an anisotropic linear map within the relevant subspace, with singular values at least
\(\sqrt{\eta^{-1}\Sigma_p+I_r}\), while leaving the orthogonal coordinates at unit scale. In particular, since \(\Sigma_p\succeq \kappa I_r\) then every relevant direction is amplified by at least \(\sqrt{1+\kappa/\eta}\). This provides a mechanism by which one RFM update performs feature learning.

\paragraph{Conjectural prediction for the second fit.}
The representation above suggests a heuristic analogy with the spiked-covariate model of~\cite{ghorbani2020neural}. Motivated by this analogy, we conjecture that the relevant sample complexity for the second KRR fit is governed by the effective dimension
\(
d_{\mathrm{eff}}:=d\sqrt{\eta}.
\)
Equivalently, one expects the polynomial threshold to shift from \(p\) to
\(
p_{\mathrm{eff}}
=
p\cdot
\frac{\log d}{\log(d_{\mathrm{eff}})}\geq p.
\)
In particular, this heuristic predicts that the second-step KRR predictor \(\hat f_2\) with independent draws of input data,  should satisfy
\begin{align}
\|\hat f_2-f^*_{\le \lfloor p_{\mathrm{eff}}\rfloor}\|_{L_2}^2
=
o_{d,\mathbb P}(1)\,
\big(
\|f^*\|_{L_2}^2+\sigma_\vepsilon^2
\big).
\end{align}
We leave this as a conjectural extension. The rigorous contribution of this section is the anisotropic rescaling result above, which explains how one RFM update can convert recovered subspace information into a more favorable geometry for the next KRR fit.

\section{Numerical results}\label{sec:exp}
We evaluate two phenomena highlighted by our analysis: recovery of the central subspace from the AGOP, and the empirical improvement produced by iterating RFM.

\paragraph{Target functions.}
Let \(z=Ux\) with \(U\in\R^{r\times d}\).  We consider the two link functions
\[
h(z)=\He_1(z_1)+\frac{1}{\sqrt{24}}\,\He_4(z_1)~~{(L1)},
\qquad
h(z)=\He_1(z_1)\He_1(z_2)+\frac{1}{2}\,\He_2(z_1)\He_2(z_2)~~{(L2)}.
\]
The first is a single-index model (\(r=1\)), and the second is an index-two model (\(r=2\)). In both cases, the low-degree component already contains the full subspace information, while the higher-degree term increases prediction difficulty.

\paragraph{Input data.}
We consider two input distributions:
\[
x\sim \mathrm{Unif}(\{-1,1\}^d),
\qquad
x\sim \mathcal N(0,I_d).
\]
The Boolean distribution matches the design studied in our theoretical analysis. For both input distributions, \(U\) is formed by taking the first \(r\) rows of a Haar-distributed orthogonal matrix.

We also consider a Boolean design, \(x\sim \mathrm{Unif}(\{-1,1\}^d)\), with a sparse matrix \(U\). This setting does not satisfy the incoherence condition in Corollary~\ref{cor:main_result} hence it is for testing how the method behaves when this condition fails. We construct \(U\) so that its rows are sparse and have pairwise disjoint supports. On each support, we draw independent random entries and then normalize the resulting vector to have unit Euclidean norm. The support size is set to $d^{3/10} \approx 4$ when $d=100$.

\paragraph{Experimental protocol.}
We fix the ambient dimension to \(d=100\) and generate the labels by
\[
y=h(Ux)+\varepsilon,
\qquad
\varepsilon\sim\mathcal N(0,0.01).
\]
The test loss, specifically mean squared error, is evaluated on $5,000$ test points. For each sample-size exponent $\alpha$, we set \(n=\lfloor d^\alpha\rfloor\).
For the Gaussian kernel we use bandwidth $d$, for the Laplace kernel $\sqrt{d},$ and
the ridge parameter is \(\lambda=10^{-6}\). The RFM safeguard is \(\eta=10^{-2}\cdot d\). We run RFM with Gaussian and Laplace kernels for five iterations and report averages over 10 independent trials.

We report the experimental results for the Gaussian kernel with hypercube data and dense \(U\) in Figs.~\ref{fig:gauss_hypercube_task1} and~\ref{fig:gauss_hypercube_task2}, corresponding to the target functions \({(L1)}\) and \({(L2)}\), respectively. The remaining plots are provided in Appendix~\ref{app:additional_experiments}.

The results show that, for KRR, corresponding to iteration \(0\), increasing the sample-size exponent leads to improved subspace recovery: the principal angle decreases and eventually becomes small, even though the test loss remains relatively large; see the first and second panels. After additional iterations of RFM, both the test loss and the principal angle improve substantially. These findings support our theoretical results on subspace recovery and are consistent with the feature-learning mechanism suggested by our RFM analysis.
\begin{figure}[H]
    \centering
\includegraphics[width=1.0\linewidth]{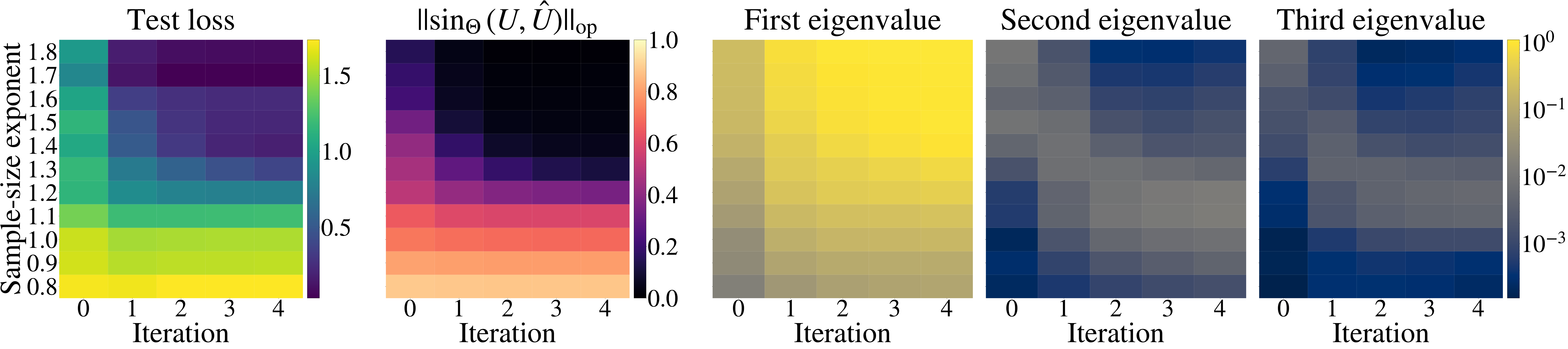}
    \caption{For target function ${(L1)}$, we train RFM with a Gaussian kernel on hypercube data for five iterations. From left to right, the panels display the test loss, the sine of the largest principal angle, and the largest, second-largest and third-largest eigenvalues of the AGOP. }
    \label{fig:gauss_hypercube_task1}
\end{figure}

\begin{figure}[H]
    \centering    \includegraphics[width=1.0\linewidth]{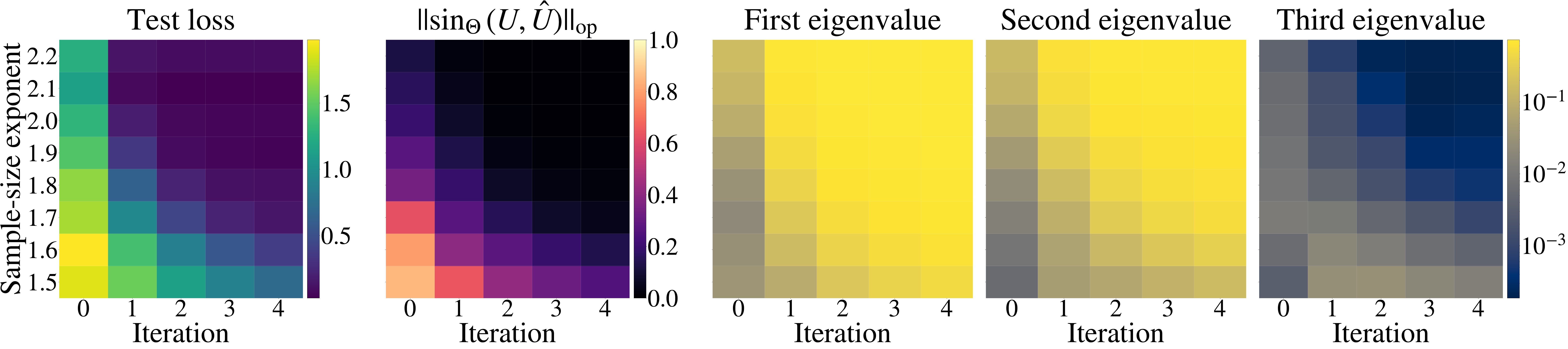}
    \caption{For target function ${(L2)}$, we train RFM with a Gaussian kernel on hypercube data for five iterations. From left to right, the panels display the test loss, the sine of the largest principal angle, and the largest, second-largest and third-largest eigenvalues of the AGOP. }
    \label{fig:gauss_hypercube_task2}
\end{figure}

\section{Conclusion}\label{sec:conclusion}

We studied whether a fitted predictor can reveal low-dimensional structure before it predicts accurately. For polynomial multi-index models on the Boolean hypercube, we showed that the empirical AGOP of a KRR predictor concentrates around the population AGOP of a low-degree truncation of the target. If this
population AGOP is nondegenerate on the planted subspace and the subspace
satisfies a weak coherence condition, then the leading eigenspace of the
empirical AGOP recovers the planted central subspace.

This gives a separation between prediction and representation. In the regime
\(n=d^{p+\delta}\), KRR may learn only the degree-\(\le p\) part of the target,
so its prediction error can remain large when  higher-degree
components are present.  Nevertheless, the AGOP of the fitted predictor can already identify the central
subspace once the degree-\(\le p\) population AGOP is nondegenerate on that subspace.  Thus, representation recovery is governed by the
first degree at which this nondegeneracy holds, rather than by the highest degree needed for
accurate prediction.

Finally, we showed that one RFM update uses the recovered AGOP to rescale the
input anisotropically: directions in the central subspace are amplified, while
orthogonal directions remain essentially unchanged.  This provides a mechanism by which iterative AGOP-based kernel methods can turn early representation recovery into improved subsequent fits in multi-index models.

\printbibliography

\newpage
\appendix
\section*{Appendix}
\startcontents[apx] 
\printcontents[apx]{l}{1}{\section*{Contents}}
\section{Additional numerical results}\label{app:additional_experiments}
In this section, we report the complementary experimental results for Section~\ref{sec:exp}. 
The experiments follow the same protocol as in the main text. In each plot, iteration \(0\) corresponds to KRR, while the later iterations correspond to RFM updates. Across the additional settings, we observe the same qualitative behavior: increasing the sample-size exponent improves the initial subspace estimate, and subsequent RFM iterations further reduce both the test loss and the sine of the largest principal angle. The evolution of the leading AGOP eigenvalues is also consistent with the emergence of the relevant low-dimensional structure.

\textbf{Gaussian kernel.}

\begin{figure}[H]
    \centering
    \includegraphics[width=1.0\linewidth]{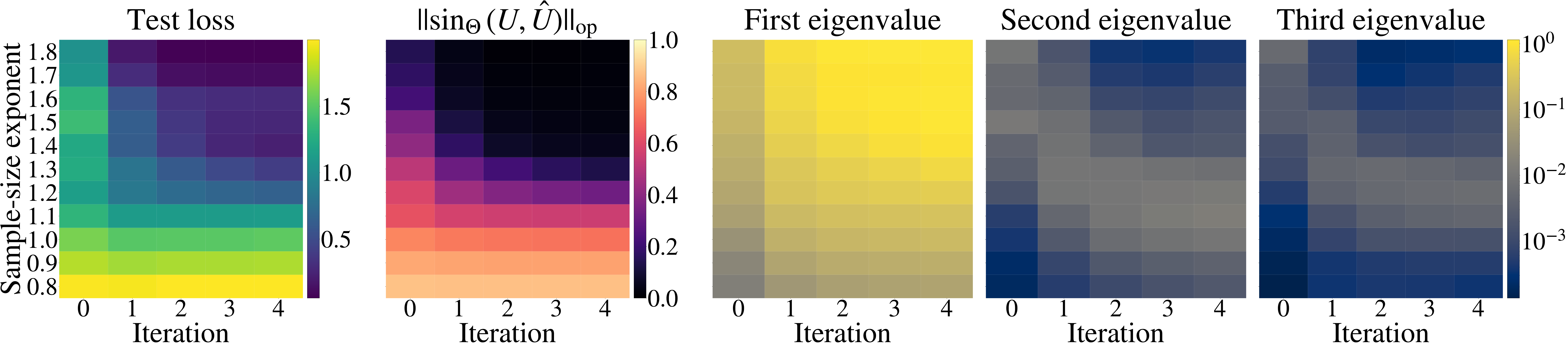}
    \caption{For target function ${(L1)}$, we train RFM with a Gaussian kernel on dense Gaussian data for five iterations. From left to right, the panels display the test loss, the sine of the largest principal angle, and the largest, second-largest and third-largest eigenvalues of the AGOP.}
    \label{fig:app:gaussian_dense_gaussian_task1}
\end{figure}

\begin{figure}[H]
    \centering
    \includegraphics[width=1.0\linewidth]{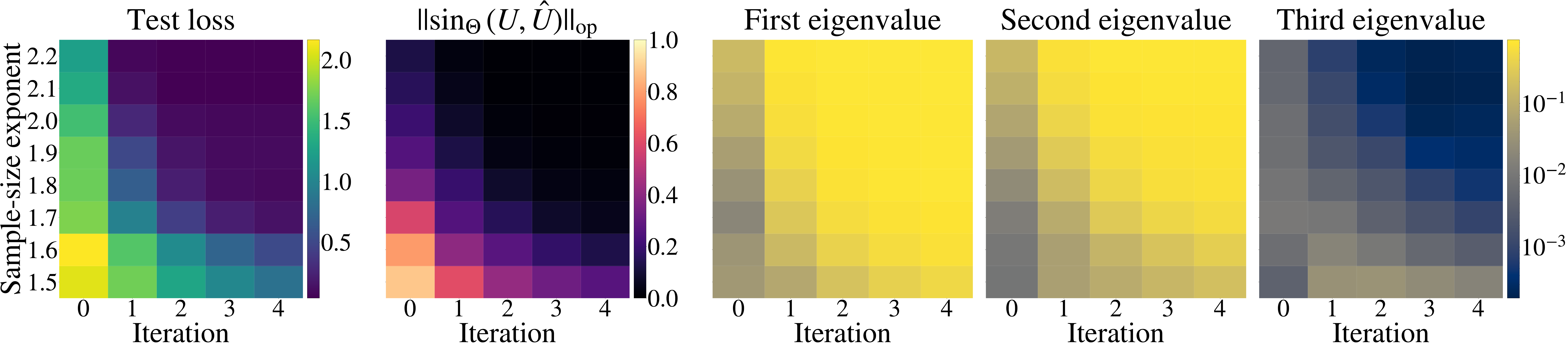}
    \caption{For target function ${(L2)}$, we train RFM with a Gaussian kernel on dense Gaussian data for five iterations. From left to right, the panels display the test loss, the sine of the largest principal angle, and the largest, second-largest and third-largest eigenvalues of the AGOP.}
    \label{fig:app:gaussian_dense_gaussian_task2}
\end{figure}

\begin{figure}[H]
    \centering
    \includegraphics[width=1.0\linewidth]{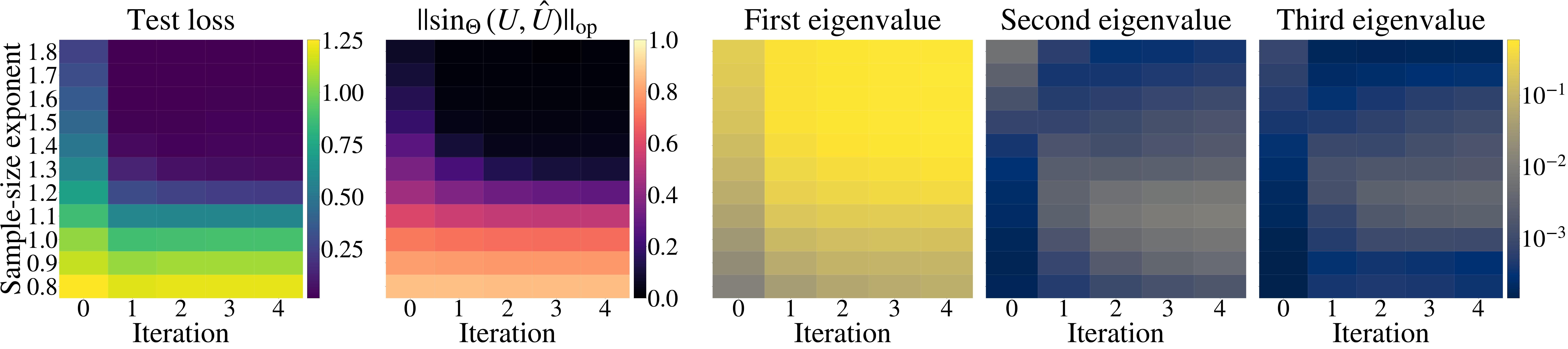}
    \caption{For target function ${(L1)}$, we train RFM with a Gaussian kernel on sparse hypercube data for five iterations. From left to right, the panels display the test loss, the sine of the largest principal angle, and the largest, second-largest and third-largest eigenvalues of the AGOP.}
    \label{fig:app:gaussian_sparseq_hypercube_task1}
\end{figure}

\begin{figure}[H]
    \centering
    \includegraphics[width=1.0\linewidth]{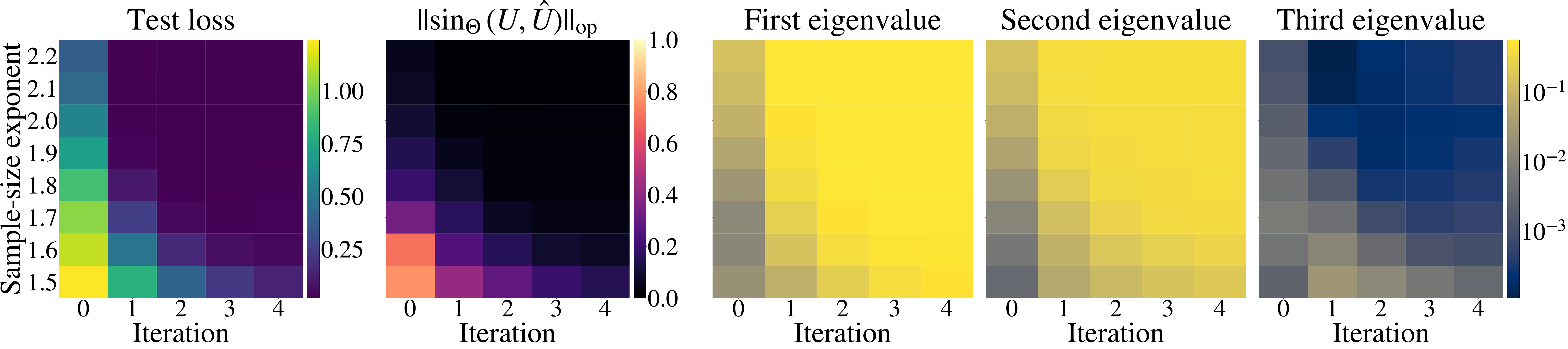}
    \caption{For target function ${(L2)}$, we train RFM with a Gaussian kernel on sparse hypercube data for five iterations. From left to right, the panels display the test loss, the sine of the largest principal angle, and the largest, second-largest and third-largest eigenvalues of the AGOP.}
    \label{fig:app:gaussian_sparseq_hypercube_task2}
\end{figure}

\textbf{Laplacian kernel.}

\begin{figure}[H]
    \centering
    \includegraphics[width=1.0\linewidth]{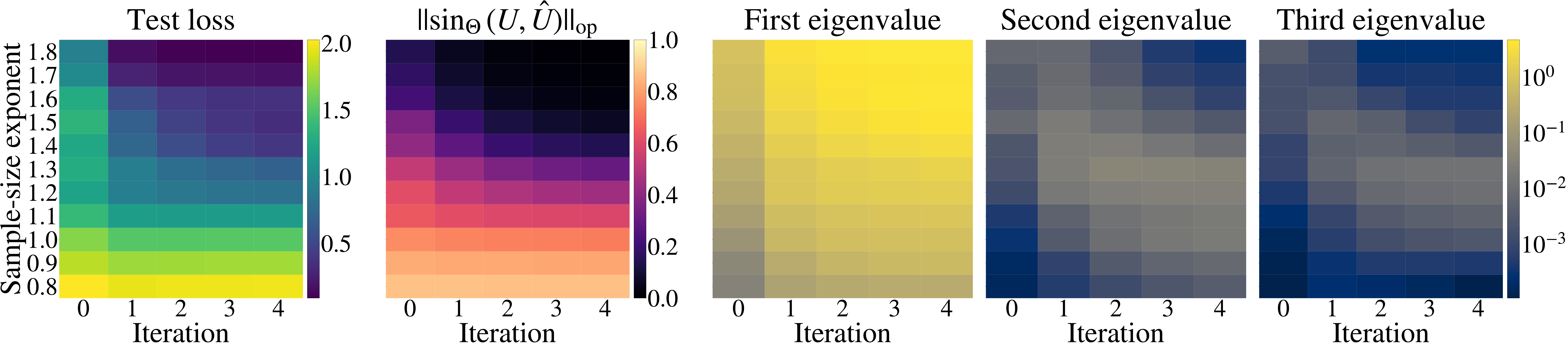}
    \caption{For target function ${(L1)}$, we train RFM with a Laplacian kernel on dense Gaussian data for five iterations. From left to right, the panels display the test loss, the sine of the largest principal angle, and the largest, second-largest and third-largest eigenvalues of the AGOP.}
    \label{fig:app:laplacian_dense_gaussian_task1}
\end{figure}

\begin{figure}[H]
    \centering
    \includegraphics[width=1.0\linewidth]{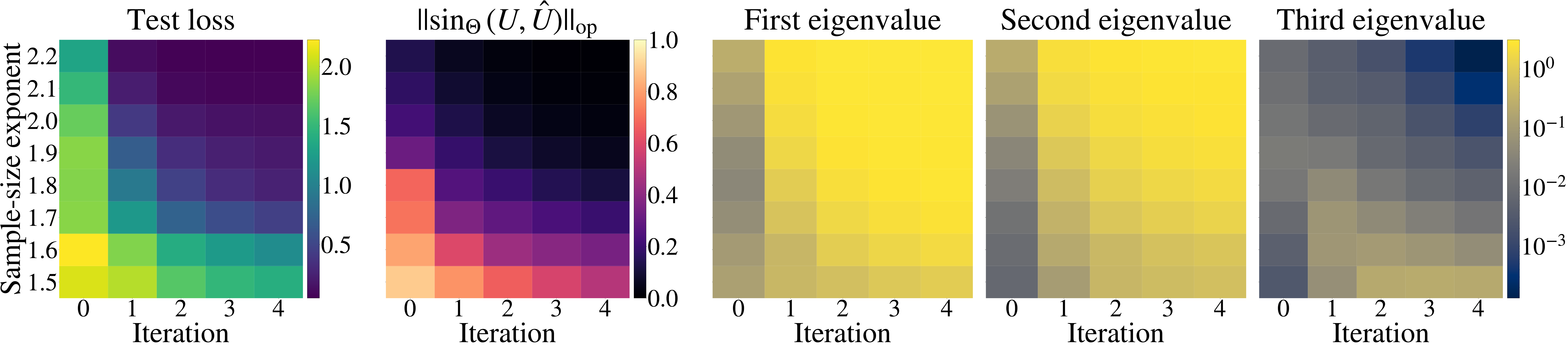}
    \caption{For target function ${(L2)}$, we train RFM with a Laplacian kernel on dense Gaussian data for five iterations. From left to right, the panels display the test loss, the sine of the largest principal angle, and the largest, second-largest and third-largest eigenvalues of the AGOP.}
    \label{fig:app:laplacian_dense_gaussian_task2}
\end{figure}

\begin{figure}[H]
    \centering
    \includegraphics[width=1.0\linewidth]{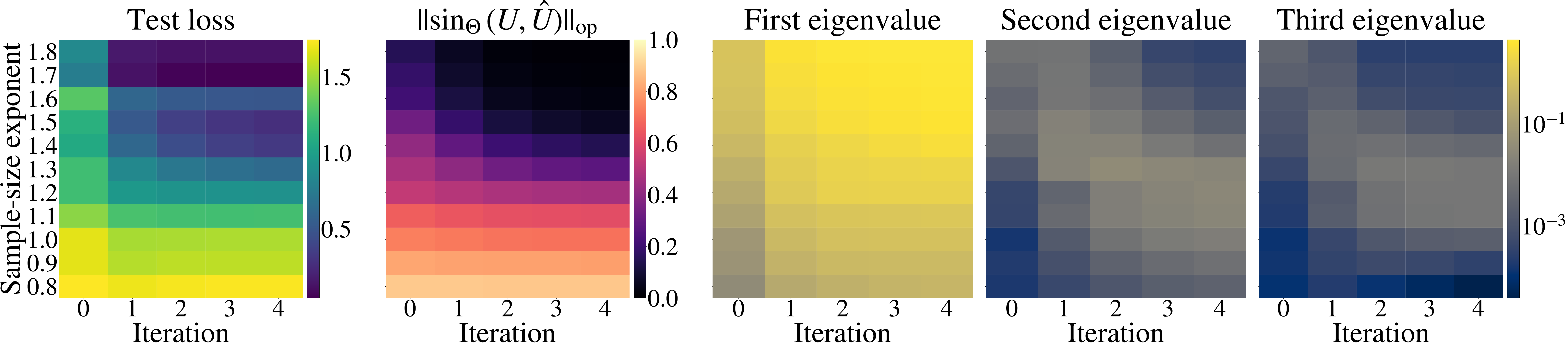}
    \caption{For target function ${(L1)}$, we train RFM with a Laplacian kernel on dense hypercube data for five iterations. From left to right, the panels display the test loss, the sine of the largest principal angle, and the largest, second-largest and third-largest eigenvalues of the AGOP.}
    \label{fig:app:laplacian_dense_hypercube_task1}
\end{figure}

\begin{figure}[H]
    \centering
    \includegraphics[width=1.0\linewidth]{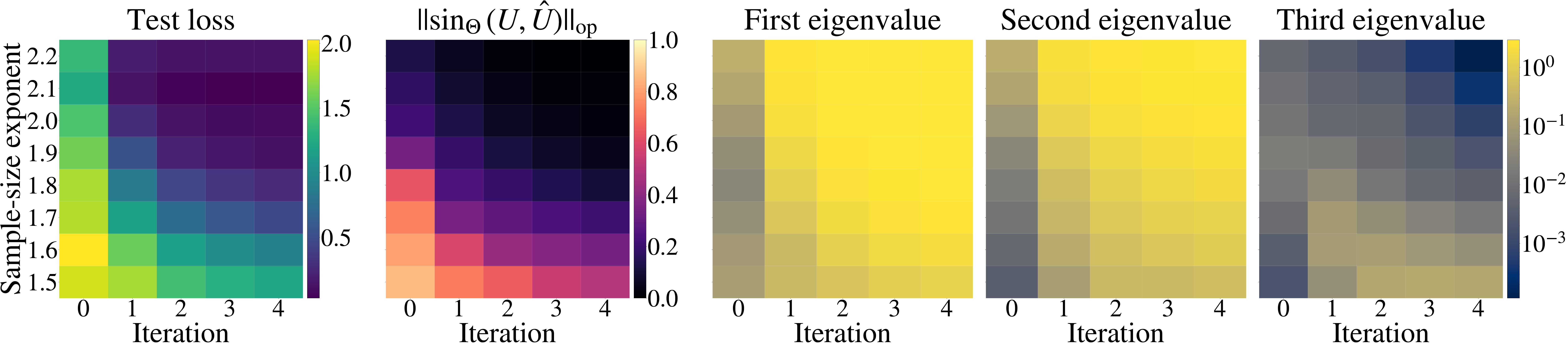}
    \caption{For target function ${(L2)}$, we train RFM with a Laplacian kernel on dense hypercube data for five iterations. From left to right, the panels display the test loss, the sine of the largest principal angle, and the largest, second-largest and third-largest eigenvalues of the AGOP.}
    \label{fig:app:laplacian_dense_hypercube_task2}
\end{figure}

\begin{figure}[H]
    \centering
    \includegraphics[width=1.0\linewidth]{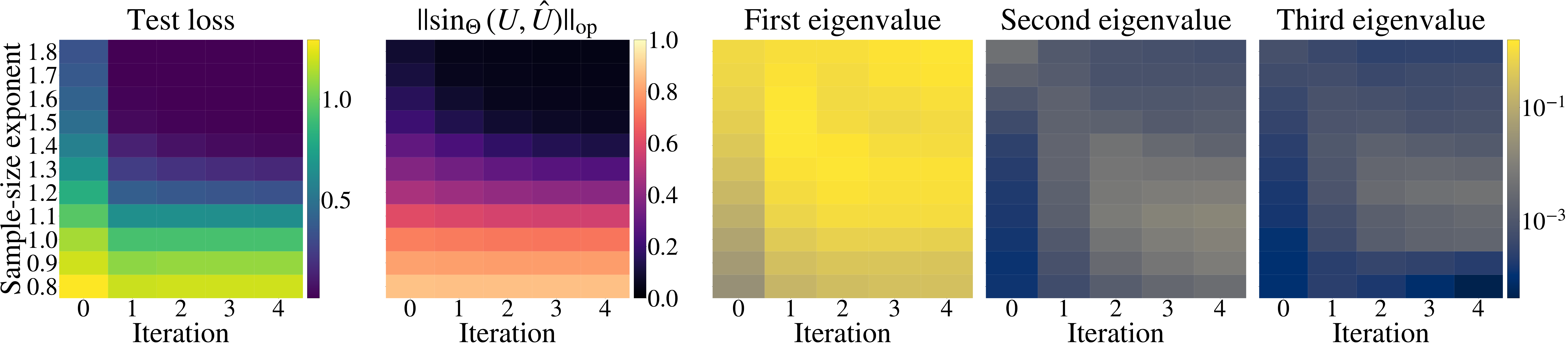}
    \caption{For target function ${(L1)}$, we train RFM with a Laplacian kernel on sparse hypercube data for five iterations. From left to right, the panels display the test loss, the sine of the largest principal angle, and the largest, second-largest and third-largest eigenvalues of the AGOP.}
    \label{fig:app:laplacian_sparseq_hypercube_task1}
\end{figure}

\begin{figure}[H]
    \centering
    \includegraphics[width=1.0\linewidth]{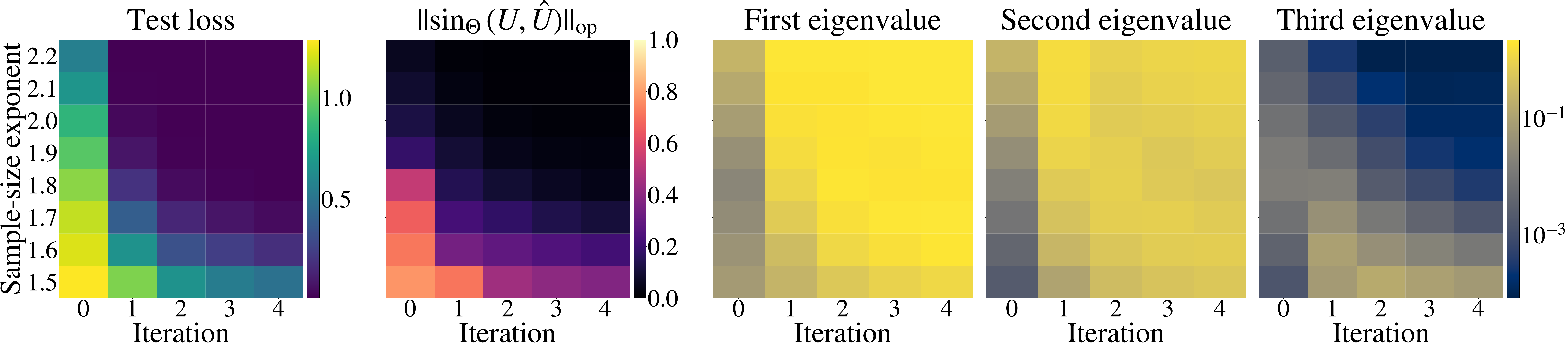}
    \caption{For target function ${(L2)}$, we train RFM with a Laplacian kernel on sparse hypercube data for five iterations. From left to right, the panels display the test loss, the sine of the largest principal angle, and the largest, second-largest and third-largest eigenvalues of the AGOP.}
    \label{fig:app:laplacian_sparseq_hypercube_task2}
\end{figure}
\section{Preliminaries}
\subsection{Preliminaries on orthogonal polynomials}\label{sec:prelin_ortho}
In this section, we record notation and preliminaries on orthogonal polynomials, following the standard monographs on the subject \cite{szeg1939orthogonal,
chihara2011introduction,
andrews2006classical}. Consider a
measure space $(D,\mathcal{A},\mu)$, where $D$ is a set, $\mathcal{A}$ is a $\sigma$-algebra, and $\mu$ is a measure on $(D,\mathcal{A})$. We let $L^2(D,\mu)$ denote the Hilbert space of $\mu$ square-integrable functions on $D$ equipped with the usual inner product and the induced norm
$$\langle f,g \rangle=\int fg\,d\mu\qquad \|f\|=\sqrt{\langle f,f\rangle }.$$
We will suppress the symbol $D$ from $L^2(D,\mu)$ when it is clear from context. Every Hilbert space $L^2(\mu)$ admits an orthonormal basis $\{\phi_j\}_{j\in J}$, meaning  $\phi_j$ are unit norm, pairwise orthogonal, and their linear span is dense in $L^2(\mu)$. A favorable situation occurs when $J$ is countable, in which case $L^2(\mu)$ is called separable. Any function $f\in L^2(\mu)$ in a separable Hilbert space can be expanded in the orthogonal basis:
$$\left\|f-\sum_{i=0}^{m} \langle f,\phi_i \rangle \phi_i\right\|\to 0\quad\textrm{as}\quad m\nearrow |J|.$$
Here, we identify $J$ with the contiguous subset of natural numbers $\mathbb{N}$ starting at zero. We will primarily focus on the following two examples of separable Hilbert spaces along with orthonormal bases.

\subsection{Fourier expansion on the boolean Hypercube}\label{sec:fourier_exp}
Let $\mathbb{H}^d=\{-1,1\}^d$ be the hypercube equipped with the uniform measure $\tau_d$. Then the inner product between any two functions $f,g\colon \mathbb{H}^d\to\R$ is simply 
$$\langle f,g\rangle=\frac{1}{2^d}\sum_{x\in \mathbb{H}^d} f(x)g(x).$$
An orthonormal basis on $L_2(\mathbb{H}^d)$ is furnished by the multi-linear monomials (called Fourier-Walsh)
$$x^{\alpha}:=\prod_{i=1}^d x_i^{\alpha_i}\qquad \textrm{for each}~ \alpha\in \{0,1\}^d.$$
We will denote the degree of the monomial  $x^{\alpha}$ by the symbol $|\alpha|:=\sum_{i=1}^d\alpha_i$. Notice that there is a one-to-one correspondence between binary vectors $\alpha\in\{0,1\}^d$ and subsets $S\subset [d]$ (their support). We will therefore often abuse notation and treat $\alpha$ as both a vector and a set whenever convenient. %

Thus any polynomial $f$ on the hypercube $\mathbb{H}^d$ can be expanded in the monomial basis:
$$f(x)=\sum_{\alpha\in \{0,1\}^d} b_{\alpha} x^{\alpha}\qquad \textrm{with}\qquad b_{\alpha}=\langle f,x^{\alpha}\rangle,$$
where $b_{\alpha}$ is called the Fourier coefficient of $f$ indexed by $\alpha$.
The $p$-truncation of $f$ is then defined to be the truncated series
\begin{equation}\label{eqn:trunc_fourier}
f_{\leq p}(x)=\sum_{\alpha\in \{0,1\}^d:\,|\alpha|\leq p} b_{\alpha} x^{\alpha}.
\end{equation}
Equivalently, $f_{\leq p}$ is the projection of $f$ in $L_2(\mathbb{H}^d)$ onto the span of all Fourier-Walsh monomials $x^{\alpha}$ with $|\alpha|\leq p$. The functions $f_{p}$ and $f_{>p}$ are defined in an obvious way.

The Fourier coefficients $b_{\alpha}$ have a convenient interpretation in terms of discrete derivative of $f$. Namely,  the discrete derivative of any function $g\colon\mathbb{H}^d\to\R$ in direction $x_i$ is defined to be 
\begin{equation}\label{eq:discrete_deri}
    \mathtt{D}_i g(x)=\frac{g(x)-g(x_1,\ldots, x_{i-1}, -x_i, x_{i+1},\ldots, x_d)}{2x_i}.
\end{equation}
The discrete derivative with respect to a set $S=\{i_1,\ldots, i_k\}\subset [d]$ is then defined by iterating:
$$\mathtt{D}_S f(x)=\mathtt{D}_{i_1}\mathtt{D}_{i_2}\ldots\mathtt{D}_{i_k} f(x).$$
If \(G := \gH(g)\colon\R^d\to\R\) denotes the multilinear extension of \(g\),
then this normalization agrees with the ambient derivatives on the hypercube:
\[
    \mathtt{D}_i g(x)=\partial_i G(x),
    \qquad x\in\mathbb H^d,\ i\in[d].
\]
More generally, for every \(S\subset[d]\),
\[
    \mathtt{D}_S g(x)=\partial_S G(x),
    \qquad x\in\mathbb H^d,
\]
where \(\partial_S:=\prod_{i\in S}\partial_i\).

Interestingly, Fourier coefficients correspond precisely to expectations of discrete derivatives:
\begin{equation}\label{eqn:herm_iterated_der}
b_S=\E[\mathtt{D}_S f].
\end{equation}
See for example \cite[Section 2.2]{o2014analysis} for details. Therefore,  Fourier coefficients measure the sensitivity of $f$ to coordinate perturbations.

\bigskip

Consider now $\R^d$ equipped with some probability measure $\mu_d$. We will often have to control the $L_q$-norm $\|f\|_{L_q(\mu_d)}=[\E |f|^q]^{1/q}$ of a polynomial $f$ on $\R^d$.
In general, one would expect that the ratio between the $L_q(\mu_d)$ and $L_2(\mu_d)$ norms depends strongly on the dimension $d$. Interestingly, for a broad class of measures $\mu$ and functions $f$ this is not the case. Indeed, the so-called hypercontractivity inequality ensures that any polynomial $f\colon\R^d\to\R$ of degree at most $\ell$ satisfies the inequality %
\begin{equation}\label{eqn:hypercontrac} 
\norm{f}_{L_q(\mu_d)} \leq (q-1)^{\ell/2}\norm{f}_{L_2(\mu_d)}\qquad \forall q\geq 2,
\end{equation}
where $\mu_d$ can be the standard Gaussian measure~\cite[Chapter 3.2]{ledoux2013probability} on $\R^d$ or the uniform  measure on the hypercube $\mathbb{H}^d$ \cite[Chapter 9]{o2014analysis}. The %
inequality \eqref{eqn:hypercontrac} is widely used in probability and theoretical computer science, and we will use it heavily here as well.

\subsection{Isometric properties of orthogonal polynomials}\label{sec:iso_poly}

In this subsection, we restate several isometric properties of covariance-like matrices induced by orthogonal polynomials. These results are established in \cite{zhu2025iteratively}; we include them here for convenience, with minor notational changes, and omit the proofs.

Fix a probability space $(D,\mathcal{A},\mu_d)$ with $D\subset\R^d$ and a set
$\{\phi_j\}_{j\in \bar{\mathcal{S}}}$
of orthonormal polynomials with respect to $\mu_d$, indexed by some set $\bar{\mathcal{S}}$. For any finite set $\S\subset \bar{\mathcal{S}}$ and a point $x\in D$, we define the concatenated vector
\[
\phi_{\S}(x):=(\phi_j(x))_{j\in \S}\in \R^{|\S|}.
\]
Given a sequence of points $X=(x^{(1)},\ldots, x^{(n)})$, we stack $\phi_{\S}(x^{(i)})$ as rows to form the matrix
\[
\Phi_{\S}(X)=[\phi_j(x^{(i)})]\in \R^{n\times |\S|}.
\]
The rows of $\Phi_{\S}(X)$ are indexed by the data points and the columns by the polynomials in $\S$. To simplify notation, we will often omit $X$ from $\Phi_{\S}(X)$ when it is clear from context. Throughout, we assume that the data points $x^{(1)},\ldots, x^{(n)}$ are sampled independently from $\mu_d$.

We impose the following hypercontractivity assumption throughout this subsection, paralleling \eqref{eqn:hypercontrac} in our two running examples. More precisely, we assume hypercontractivity of the product measure $\mu_d\times \mu_d$ on $\R^d\times \R^d$. Hypercontractivity of the original measure $\mu_d$ then follows immediately.

\begin{assumption}[Hypercontractivity]\label{ass:hypercontr}
{\rm
There exist constants $C_{\ell,q}>0$, indexed by integers $\ell,q\in \mathbb{N}$, such that any polynomial $f$ on $\R^d\times \R^d$ of degree at most $\ell$ satisfies
\[
\|f\|_{L_q(\mu_d\times \mu_d)}\leq C_{\ell,q}\cdot \|f\|_{L_2(\mu_d\times \mu_d)}
\qquad \forall q\geq 2.
\]}
\end{assumption}

In particular, Assumption~\ref{ass:hypercontr} holds in our running example---uniform on the hypercube---with $C_{\ell,q}=(q-1)^{\ell/2}$, since the product uniform measure uniform on the hypercube is again uniform on the hypercube.

We now record the relevant isometry properties of the matrix $\Phi_{\S}$ in the high-dimensional regime $n,d\to\infty$.

\begin{lemma}[Norm control]\label{lem:norm_control}
Suppose Assumption~\ref{ass:hypercontr} holds. Consider the regime $n = d^{p+\delta}$ where $\delta\in(0,1)$ is a constant. Fix a set $\S\subseteq\bar{\S}$ indexing polynomials of degree at most $\ell$. Then the following hold:
\begin{enumerate}
    \item\label{it:smallS} If $|\S| \leq Cd^{p+\delta_0}$ with $\delta_0\in (-\infty, \delta)$, then
    \[
    \|\Phi_{\S}\|_{\rm op}=O_{d,\mathbb{P}}(\sqrt{n}).
    \]
    \item\label{it:bigS} If $|\S| \geq Cd^{p+\delta_0}$ with $\delta_0\in (\delta,\infty)$, then
    \[
    \|\Phi_{\S}\|_{\rm op}=O_{d,\mathbb{P}}(\sqrt{|\S|}).
    \]
\end{enumerate}
\end{lemma}

The asymptotic behavior of $\Phi_{\S}$ depends on the relative scale of $n$ and $|\S|$, namely
\[
\frac{n}{|\S|}\asymp d^{q}
\qquad \text{with} \qquad
\underbrace{q>0}_{\rm Case~I}
\qquad \text{or} \qquad
\underbrace{q<0}_{\rm Case~II}.
\]
The following two theorems show, respectively, that in Case~I the empirical covariance matrix $\frac{1}{n}\Phi_{\S}^\top\Phi_{\S}$ is asymptotically equal to the identity $I_{|\S|}$, while in Case~II the Gram matrix $\frac{1}{|\S|}\Phi_{\S}\Phi_{\S}^\top$ is asymptotically equal to the identity $I_n$.

\begin{theorem}[Asymptotics in Case I]\label{lemma:phi_id_2}
Suppose that Assumption~\ref{ass:hypercontr} holds and consider the regime $n = d^{p+\delta}$ where $\delta\in(0,1)$ is a constant. Fix a set $\S\subseteq\bar{\S}$ of cardinality $|\S| \leq Cd^{p+\delta_0}$, with $\delta_0\in (-\infty,\delta)$, indexing polynomials of degree at most $\ell$. Then for any $\delta'\in (0,\delta-\delta_0)$ there exists a constant $C'$ satisfying
\[
\E\brac{\snorm{\frac{1}{n}\Phi_{\S}^\top \Phi_{\S} - I_{|\S|}}}
\leq C' d^{-\delta'/2}\sqrt{\log(n)}.
\]
Consequently, for any $\epsilon>0$, with probability at least $1 - \tfrac{C'}{\sqrt{\log(n)}}$, we have
\[
\snorm{\Phi_{\S}(X)^\top \Phi_{\S}(X)/ n - I_{|\S|}}
\leq d^{-\delta'/2 + \epsilon}.
\]
\end{theorem}

Next, we turn to the complementary setting in which $n$ is small compared to $|\S|$. In this case, the Gram matrix
\[
\Phi_{\S}\Phi_{\S}^\top
=
[\langle \phi_{\S}(x^{(i)}),\phi_{\S}(x^{(j)})\rangle]_{i,j=1}^n
\]
is close to a multiple of the identity.

\begin{theorem}[Asymptotics in Case II]\label{lemma:gram_matrix}
Suppose that Assumption~\ref{ass:hypercontr} holds and consider the regime $n = d^{p+\delta}$ where $\delta\in(0,1)$ is a constant. Fix a set $\S\subseteq\bar{\S}$ of cardinality $|\S| \geq Cd^{p+\delta_0}$, with $\delta_0\in (\delta,\infty)$, indexing polynomials of degree at most $\ell$. Then for any $\epsilon>0$ there exists a constant $C'$ satisfying
\begin{equation}\label{eqn:bound_needed}
\E\brac{\snorm{ \frac{1}{|\S|}\Phi_{\S}\Phi_{\S}^\top - I_n}}
\leq C'\left(d^{-\frac{p+\delta}{2}+2\epsilon}+\sqrt{\log(n)}\cdot d^{-\frac{\delta_0-\delta}{2}+\epsilon}\right).
\end{equation}
Consequently, in the case $p\geq 1$, $\delta_0=1$, and $\epsilon\in (0,\delta)$, with probability at least $1 - \tfrac{C'}{\sqrt{\log(n)}}$, we have
\[
\snorm{ \Phi_{\S}\Phi_{\S}^\top / |\S|- I_n} \leq d^{-\tfrac{1-\delta}{2}+\epsilon}.
\]
\end{theorem}
\subsection{Approximating kernels by orthogonal polynomials}\label{sec:kernel_approx}

In this subsection, we restate several results that approximate inner-product kernels by quadratic forms in orthogonal polynomials. These results are established in \cite{zhu2025iteratively}; we include them here for convenience, with minor notational changes, and omit the proofs. Throughout, we fix an inner-product kernel \eqref{eq:inner_product_kernel} satisfying the following regularity assumption. Note that Assumption~\ref{assump:g_0} implies that both kernels $K(z,x)$ and $K'(z,x)$ satisfy Assumption~\ref{assump:g}. This will be important later, when we apply the results below with $K$ replaced by $K'$.

\begin{assumption}[Regularity of the kernel]\label{assump:g}
There exists a constant $\varepsilon\in (0,1)$ such that the function $g$ in \eqref{eq:inner_product_kernel} is Lipschitz continuous on $(-1-\varepsilon,1+\varepsilon)$ and is analytic on $(-\varepsilon,\varepsilon)$.
\end{assumption}

Our goal is to approximate $K(x,y)$ by a quadratic form of the form
\[
K(x,y)\approx \phi(x)^\top D\phi(y),
\]
where $\phi\colon\R^d\to\R^p$ has orthonormal polynomials as its coordinate functions. Ideally, $D$ should be relatively simple, for example diagonal. We begin by forming the Taylor approximation of $g$ at the origin and the corresponding approximate kernel:
\[
g_{m}(t):=\sum_{k=0}^{m}\frac{g^{(k)}(0)}{k!}t^k,
\qquad
K_{m}(x,y) := g_{m}\!\left(\frac{\inner{x,y}}{d}\right).
\]

We now specialize to the uniform measure $\tau_d$ over the hypercube $\mathbb{H}^d=\{-1,1\}^d$, as discussed in Section~\ref{sec:fourier_exp}. We consider the Fourier orthogonal basis $\phi_\lambda(x):=x^\lambda$ for each $\lambda\in \{0,1\}^d$. As usual, for any set $\S\subset\{0,1\}^d$, we define the concatenated vector
\[
\phi_\S(x)=\{\phi_{\lambda}(x)\}_{\lambda\in \S}\in \R^{|\S|}.
\]
We write $\Phi(X)$ for the full matrix with entries $[\phi_\lambda(x^{(i)})]_{i,\lambda}$, and $\Phi_\S(X)$ for the corresponding column-submatrix indexed by $\S$. It will be useful to isolate the degree-$k$ basis elements:
\[
\S_k=\{\lambda\in \{0,1\}^d: |\lambda|=k\}.
\]
We set $\Phi_{\S_k}(X)$ and $\Phi_{\leq k}(X)$ to be, respectively, the submatrices of $\Phi$ indexed by degree $k$ and degree at most $k$ basis elements. We will suppress the symbol $X$ from $K(X,X)$ and $\Phi(X)$ throughout this subsection in order to simplify the notation. We stress that all probabilistic statements are with respect to the random vectors $\{x^{(i)}\}_{i=1}^n$ sampled independently from $\tau_d$.

The next theorem, restated from \cite{zhu2025iteratively}, shows that the deviation $K_{m}(X,X)-K(X,X)$ is asymptotically equivalent to a multiple of the identity under favorable conditions. In particular, in the regime $n=d^{p+\delta}$ with $\delta\in (0,1/2)$, the right-hand side of \eqref{eqn:taylor_approx_cube} tends to zero for any approximation order $m\geq 2p$. The approximation of kernels has been previously studied in the linear regime (i.e., $n\asymp d$) by \cite{elkaroui2010spectrum} and in the quadratic regime (i.e., $n\asymp d^2$) by \cite{pandit2025universality}. Note that in \cite{zhu2025iteratively}, it is proved in the more general form for the parameterized kernel
\[
K_w(x,y) = g\!\left(\frac{\inner{\sqrt{w}\odot x,\sqrt{w}\odot y}}{d}\right),
\]
which encompasses the special case $w = \mathbf{1}_d$ recorded below.

\begin{theorem}[Taylor approximation of kernels on the hypercube]\label{thm:poly_approx_cube}
Consider independent, mean-zero, isotropic random vectors $x^{(1)},\ldots,x^{(n)}$ in $\R^d$ and suppose that the coordinates of each vector $x^{(i)}$ are independent and uniformly distributed in $\{-1,1\}$. Suppose that we are in the regime $n = d^{p+\delta}$, where $\delta\in(0,1)$ is a constant. Fix an arbitrary constant $C>2$ and approximation order $m> 2p$. Then the estimate
\begin{equation}\label{eqn:taylor_approx_cube}
\snorm{K(X,X) - K_{m}(X,X) - (g(1) - g_{m}(1)) I_n} \le  c_{g,m}\sqrt{\frac{(C\log(n))^{m+1}}{d^{m-2p+1-2\delta}}}
\end{equation}
holds with probability at least $1-\frac{4}{n^{C-2}}$, where $c_{g,m}<\infty$ is a constant that depends only on $m$ and the regularity constants in Assumption~\ref{assump:g}.
\end{theorem}

We next record a deterministic lemma that expresses the matrix $\left(\frac{XX^\top}{d}\right)^{\odot k}$ as a diagonal quadratic form acting on the Fourier basis. We further decompose the diagonal matrix into the dominant part, corresponding to the degree-$k$ basis elements, and the lower-order error part. 

\begin{lemma}[Conversion from a polynomial kernel to a Fourier basis]\label{lemma:each_mono_cube}
For any $k>0$, the following holds:
\begin{align}\label{eq:x_power}
    \frac{1}{k!}\left(\frac{XX^\top}{d}\right)^{\odot k}
    = d^{-k}\cdot \Phi_{\S_k}\Phi_{\S_k}^\top
    + \sum_{\substack{j:\, 0\leq j<k, \\ k-j~\mathrm{is~even}}}\Phi_{\S_j} \widetilde{D}^{(k)}_{\S_j}\Phi_{\S_j}^\top,
\end{align}
where each $\widetilde{D}^{(k)}_{\S_j}$ is a diagonal matrix with all entries on the order of $\Theta_d(d^{-(j+k)/2})$.
\end{lemma}
At this point we use the permutation symmetry of the kernel
$(x,y)\mapsto \langle x,y\rangle^k$. Indeed, for every permutation $\pi$ of $[d]$, one has
\[
\round{\frac{\langle x,y\rangle}{d}}^k
=
\round{\frac{\langle \pi x,\pi y\rangle}{d}}^k.
\]
Since the Walsh characters $\{x_Sy_S\}_{S\subset[d]}$ form a basis, the coefficient of $x_Sy_S$
can therefore depend only on $|S|$. It follows that each diagonal block
$\widetilde D_{\S_j}^{(k)}$ is in fact a scalar multiple of the identity, say
\[
\widetilde D_{\S_j}^{(k)}=\widetilde d_j^{(k)} I_{|\S_j|},
\qquad
\widetilde d_j^{(k)}=O_d(d^{-(j+k)/2}).
\]
Then equation~\eqref{eq:x_power} can be refined to
\begin{align}\label{eq:x_power_refined}
       \frac{1}{k!}\left(\frac{XX^\top}{d}\right)^{\odot k}
    = d^{-k}\cdot \Phi_{\S_k}\Phi_{\S_k}^\top
    + \sum_{\substack{j:\, 0\leq j<k, \\ k-j~\mathrm{is~even}}} \widetilde d_j^{(k)}\cdot \Phi_{\S_j} \Phi_{\S_j}^\top.
\end{align}

The following lemma, also established in \cite{zhu2025iteratively}, is the main approximation statement in this setting. It shows the asymptotic equivalence of $K(X,X)$ to the sum of a multiple of the identity and a quadratic form in Fourier basis elements of degree at most $p$.

\begin{lemma}[Fourier basis approximation of a kernel]\label{lemma:kernel_to_mono}
Consider the regime $n = d^{p+\delta}$ where $\delta \in (0,1)$ is a constant. Then for any $\epsilon\in (0,\delta)$, the following holds:
\begin{align}
    &\snorm{K - \Phi_{\leq p} D \Phi_{\leq p}^\top - \bigl( g(1) - g_{p}(1)\bigr)I_{n}} = O_{d,\P}\left( \frac{\log(n)}{d^{\tfrac{1-\delta}{2}-\epsilon}}\right),\label{eq:kernel_to_mono}
\end{align}
where $D$ is a diagonal matrix satisfying $\|D_{\S_k} -g^{(k)}(0)d^{-k}I_{|\S_k|}\|_{\rm op} = O_d(d^{-k-1})$ for $k=0,\ldots, p$.
\end{lemma}

Note that the estimate \eqref{eq:kernel_to_mono} allows us to approximate $K$ by the simple expression $\Phi_{\leq p} D \Phi_{\leq p}^\top + \bigl( g(1) - g_{p}(1)\bigr)I_{n}$ up to an error of order $d^{-(1-\delta)/2}$. We provide below a more refined approximation of $K$ with the improved error rate. This sharper estimate will be useful in some of our later arguments. The proof appears in Appendix~\ref{proof:kernel_to_mono_general}.
\begin{lemma}[General refined Fourier-basis approximation of a kernel]
\label{lemma:kernel_to_mono_general}
Consider the regime $n=d^{p+\delta}$ where $\delta\in(0,1)$ is fixed, and let $m\ge p$ be a fixed integer.
Then there exist scalars $\theta_{p+1},\ldots,\theta_{2m+1}$, a scalar $\rho_{p,m}$, and a diagonal matrix $D$ such that
\begin{align}
    \snorm{
        K-\Phi_{\le p}D\Phi_{\le p}\tran
        -\sum_{k=p+1}^{m}\theta_k\,\offd\!\round{\Phi_{\S_k}\Phi_{\S_k}\tran}
        -\rho_{p,m} I_n
    }
    =O_{d,\P}\round{\frac{\log(n)}{d^{\tfrac{m+1-p-\delta}{2}-\epsilon}}},
    \label{eq:kernel_to_mono_general}
\end{align}
for any sufficiently small $\epsilon\in(0,\delta)$. Moreover,
\begin{align}
    \|D_{\S_k}-g^{(k)}(0)d^{-k}I_{|\S_k|}\|_{\rm op}
    &=O_d(d^{-k-1}),
    \qquad k=0,\ldots,p,\label{eq:D_low_deg_general}\\
    \theta_k&=\frac{g^{(k)}(0)}{d^k}+O_d(d^{-k-1}),
    \qquad k=p+1,\ldots,2m+1,\label{eq:theta_asymp}\\
    \rho_{p,m}&=g(1)-g_p(1)+O_d(d^{-1}).\label{eq:rho_pm_asymp}
\end{align}
In fact one may take
\begin{equation}\label{eq:rho_pm_def}
    \rho_{p,m}:=g(1)-g_{2m+1}(1)+\sum_{k=p+1}^{2m+1}\theta_k |\S_k|.
\end{equation}
\end{lemma}

\section{Proof of Lemma~\ref{lem:M_gaussian_latent}}
\label{proof:M_gaussian_latent}
We first fix the Walsh-degree notation used throughout the proof. Namely, for a
multilinear polynomial
\[
Q(x)=\sum_{S\subseteq[d]}\widehat Q(S)\prod_{i\in S}x_i
\]
and for an integer \(m\ge0\), we write
\[
[Q]_{\deg m}
:=
\sum_{S\subseteq[d]:\,|S|=m}
\widehat Q(S)\prod_{i\in S}x_i
\]
for its homogeneous degree-\(m\) Walsh component.

We next record the relevant subspace and the ambient basis in which coordinate
estimates are expressed. Let \(U\in\R^{r\times d}\) have orthonormal rows
\(
(u^{[1]})^\top,\dots,(u^{[r]})^\top.
\)
Thus the corresponding subspace is
\[
\mathcal U
:=
\operatorname{row}(U)
=
\operatorname{span}\{u^{[1]},\dots,u^{[r]}\}.
\]
We fix an arbitrary orthonormal completion
\(
u^{[r+1]},\dots,u^{[d]}
\)
of \(\R^d\). These additional vectors are used only to represent matrices in
an ambient orthonormal basis.

The coherence parameter controls how much of the relevant subspace lies on
each ambient coordinate. We define
\[
q_i:=\sum_{j=1}^r \bigl(u_i^{[j]}\bigr)^2,
\qquad
\rho:=\sum_{i=1}^d q_i^2.
\]
Note that the coherence of $\gU$ can be equivalently written as
\[
\mu(U)
=
\frac{d}{r}
\max_{i\in[d]}
\sum_{j=1}^r \bigl(u_i^{[j]}\bigr)^2.
\]
Then we have
\[
\max_{i\in[d]}q_i
=
\frac{r\mu(U)}{d}.
\]
Also, since \(q_i=\|P_{\mathcal U}e_i\|_2^2\), we have
\(
0\le q_i\le1.
\)
Since the rows of \(U\) are orthonormal, we also have
\(
\sum_{i=1}^d q_i=r.
\)
Consequently, the collision parameter satisfies
\begin{equation}
\label{eq:rho-bound-coherence}
\rho
=
\sum_{i=1}^d q_i^2
\le
\Bigl(\max_i q_i\Bigr)\sum_{i=1}^d q_i
=
\frac{r^2\mu(U)}{d}.
\end{equation}
We shall use the shorthand
\[
\Delta_U
:=
\frac{r^2\mu(U)}{d}.
\]
Thus all repeated-coordinate collision errors below are controlled by
\(
\rho\le\Delta_U.
\)
Notice that orthonormality alone gives \(\rho\le r\), not necessarily
\(\rho\le1\). The bound \(\rho\le1\) is used only in the small-coherence regime
where \(\Delta_U\le1\).

It is useful to make the finite-dimensional dependence on \(r\) explicit. We
define
\[
N_{r,\ell}
:=
\#\{\lambda\in\N^r:\ |\lambda|\le\ell\}
=
\binom{r+\ell}{\ell}.
\]
We also define
\[
\Theta_U
:=
N_{r,\ell}^2\Delta_U
=
N_{r,\ell}^2\frac{r^2\mu(U)}{d},
\]
and we define
\[
\overline\Theta_U
:=
r\Theta_U
=
rN_{r,\ell}^2\frac{r^2\mu(U)}{d}.
\]
The factor \(N_{r,\ell}^2\) is a crude envelope for finite sums over
multi-indices of degree at most \(\ell\). All constants denoted by \(C_\ell\)
below depend only on \(\ell\), and are independent of \(r,d,U\), and \(h\).

We now pass from the ambient variables to the relevant coordinates. We define
the linear forms
\[
z_j(x):=\langle u^{[j]},x\rangle,
\qquad j\in[r].
\]
Thus the target ridge polynomial can be written as
\[
f_U^*(x)=h(Ux)=h(z_1(x),\dots,z_r(x)).
\]
For simplicity, we write \(f^*\) instead of \(f_U^*\) when the dependence on \(U\) is clear.

For \(\lambda\in\N^r\), we write
\(
z^\lambda:=\prod_{j=1}^r z_j^{\lambda_j}.
\)
The associated all-distinct Walsh layer is defined by
\[
H_\lambda
:=
\big[\mathcal H(z^\lambda)\big]_{\deg |\lambda|}.
\]

\paragraph{Hermite notation.}

We compare monomial and Hermite expansions of the latent polynomial \(h\). For
\(n\in\N\), let \(\He_n\) denote the probabilists' Hermite polynomial, namely
\begin{align}\label{eq:hermite}
 \He_n(t)
:=
(-1)^n e^{t^2/2}\frac{d^n}{dt^n}e^{-t^2/2}.   
\end{align}
Equivalently, this polynomial is given by
\[
\He_n(t)
=
n!\sum_{m=0}^{\lfloor n/2\rfloor}
\frac{(-1)^m}{2^m\,m!\,(n-2m)!}\,t^{\,n-2m}.
\]
For a multi-index \(\lambda\in\N^r\), we set
\[
\He_\lambda(z)
:=
\prod_{j=1}^r \He_{\lambda_j}(z_j),
\qquad
\lambda!
:=
\prod_{j=1}^r \lambda_j!.
\]

We next introduce the index sets and coefficients appearing in the
monomial-to-Hermite inversion. For \(\alpha\in\N^r\), define
\[
\mathcal A(\alpha)
:=
\Bigl\{
\lambda\in\N^r:\ \alpha-\lambda\in(2\N)^r
\Bigr\}.
\]
For \(m\ge0\), write
\[
\mathcal A_m(\alpha)
:=
\{\lambda\in\mathcal A(\alpha):\ |\lambda|=m\}.
\]
We also define
\[
\Lambda_m(\alpha)
:=
\{\lambda\in\N^r:\ |\lambda|=m,\ 0\le \lambda_j\le \alpha_j
\text{ for all }j\in[r]\}.
\]
To avoid conflict with the coherence notation, we denote the pair-count
multi-index by
\[
\nu(\alpha,\lambda)
:=
\frac{\alpha-\lambda}{2}.
\]
Thus, when \(\lambda\in\mathcal A(\alpha)\), the vector
\(
\nu=\nu(\alpha,\lambda)
\)
belongs to \(\N^r\). In this case, define
\[
\mathsf A_{\alpha,\lambda}
:=
\prod_{j=1}^r
\frac{\alpha_j!}{\lambda_j!\,2^{\nu_j}},
\]
and define
\[
\mathsf B_{\alpha,\lambda}
:=
\prod_{j=1}^r
\frac{\alpha_j!}{\lambda_j!\,2^{\nu_j}\nu_j!}.
\]
Equivalently, we have
\[
\mathsf B_{\alpha,\lambda}
=
\frac{\mathsf A_{\alpha,\lambda}}{\nu!},
\qquad
\nu!
:=
\prod_{j=1}^r \nu_j!.
\]

Applying the univariate Hermite inversion formula coordinatewise gives
\begin{equation}
\label{eq:monomial-hermite-inversion}
z^\alpha
=
\sum_{\lambda\in\mathcal A(\alpha)}
\mathsf B_{\alpha,\lambda}\,\He_\lambda(z).
\end{equation}
Consequently, every polynomial \(h\) of degree at most \(\ell\) has the unique
Hermite expansion
\begin{equation}
\label{eq:h-hermite-expansion}
h(z)
=
\sum_{\lambda\in\N^r:\,|\lambda|\le \ell}
a_\lambda\,\He_\lambda(z).
\end{equation}
We write its homogeneous Hermite components as
\[
h_q(z)
:=
\sum_{\lambda\in\N^r:\,|\lambda|=q}
a_\lambda\,\He_\lambda(z),
\qquad
0\le q\le \ell.
\]
Similarly, we write its low-degree truncations as
\[
h_{\le L}(z)
:=
\sum_{q=0}^{L}h_q(z),
\qquad
0\le L\le \ell.
\]
The corresponding latent Gaussian gradient covariance is
\[
\Sigma_L
:=
\E_{z\sim\gamma_r}
\big[
\nabla h_{\le L}(z)\nabla h_{\le L}(z)^\top
\big].
\]

It remains to relate the Hermite coefficients to the ordinary monomial
coefficients. Write
\[
h(z)=\sum_{|\alpha|\le \ell} b_\alpha z^\alpha,
\qquad
\Lambda_b
:=
\{\alpha\in\N^r:\ |\alpha|\le \ell,\ b_\alpha\neq0\}.
\]
Comparing this expansion with \eqref{eq:monomial-hermite-inversion} gives
\begin{equation}
\label{eq:hermite-coeff-from-monomials}
a_\lambda
=
\sum_{\substack{\alpha\in\Lambda_b\\ \lambda\in\mathcal A(\alpha)}}
b_\alpha\,\mathsf B_{\alpha,\lambda}.
\end{equation}

We will also use the following Gaussian \(L_2\) identities. Since \(U\) has
orthonormal rows, if \(G\sim N(0,I_d)\), then \(UG\sim N(0,I_r)\). Therefore,
\begin{equation}
\label{eq:fstar-h-gaussian-L2}
\|f^*\|_{L_2(\gamma_d)}
=
\|h\|_{L_2(\gamma_r)}.
\end{equation}
By Hermite orthogonality under \(\gamma_r\), we also have
\begin{equation}
\label{eq:h-gaussian-L2-hermite}
\|h\|_{L_2(\gamma_r)}^2
=
\sum_{\lambda\in\N^r:\,|\lambda|\le \ell}
a_\lambda^2\,\lambda!.
\end{equation}

Finally, we introduce the coefficient notation used to count pair-supports.
For \(\nu\in\N^r\) and \(T\subseteq[d]\), define
\[
C_\nu(T)
:=
[w^\nu]\prod_{i\in T}
\Bigl(
1+\sum_{j=1}^r \bigl(u_i^{[j]}\bigr)^2 w_j
\Bigr).
\]
For the full coordinate set, we write
\(
C_\nu:=C_\nu([d]).
\)
Here
\(
w=(w_1,\dots,w_r)
\)
is a vector of formal variables.

\paragraph{The surrogate and the target components.}

We now define the all-distinct surrogate. For \(q\in\{0,\dots,\ell\}\), set
\[
F_q(x)
:=
\sum_{\lambda\in\N^r:\,|\lambda|=q}
a_\lambda\,H_\lambda(x).
\]
For \(0\le L\le \ell\), define
\[
F_{\le L}(x)
:=
\sum_{q=0}^{L}F_q(x),
\qquad
F(x):=F_{\le \ell}(x).
\]
Thus the full surrogate is
\[
F(x)
=
\sum_{\lambda\in\N^r:\,|\lambda|\le \ell}
a_\lambda\,H_\lambda(x).
\]

For the multilinearized target, write
\[
P:=\mathcal H(f^*).
\]
We denote its Walsh-degree components by
\[
P_q:=[P]_{\deg q},
\qquad
P_{\le L}:=\sum_{q=0}^{L}P_q,
\qquad
0\le L\le \ell.
\]
The associated hypercube population matrix is
\[
M_{\le L}
:=
\E_{x\sim\tau_d}
\big[
\nabla P_{\le L}(x)\nabla P_{\le L}(x)^\top
\big].
\]
Unless a measure is explicitly displayed, all \(L_2\)-norms in the hypercube
part of the proof are taken with respect to \(x\sim\tau_d\).

\paragraph{Proof outline.}

The proof proceeds in six steps. Step 1 approximates the multilinearized target
\(P\) by the all-distinct surrogate \(F\). Step 2 records the structural
properties of the layers \(H_\lambda\). Step 3 identifies the latent Gaussian
covariance at each Hermite degree. Step 4 computes the hypercube gradient
covariance of the surrogate \(F_q\). Step 5 transfers this comparison from
\(F_q\) to the true Walsh component \(P_q\). Step 6 sums over degrees
\(q\le L\) and obtains the covariance transfer, which completes the proof for Lemma~\ref{lem:M_gaussian_latent}.

\paragraph{Step 1: reduction to all-distinct layers.}

We begin with a degree-by-degree reduction for monomials. The proof appears in Appendix~\ref{proof:degree-m-reduction}.

\begin{lemma}[Degree-\(m\) reduction]
\label{lem:degree-m-reduction}
Let \(\alpha\in\N^r\) satisfy \(|\alpha|\le\ell\), and let
\(m\le|\alpha|\) with \(|\alpha|-m\in2\N\). Then
\[
\Biggl\|
[\mathcal H(z^\alpha)]_{\deg m}
-
\sum_{\lambda\in\mathcal A_m(\alpha)}
\mathsf A_{\alpha,\lambda}\,
C_{\nu(\alpha,\lambda)}\,
H_\lambda
\Biggr\|_{L_2}
\le
C_\ell\Delta_U.
\]
\end{lemma}

Summing over the finitely many admissible Walsh degrees gives the corresponding
global reduction.

\begin{corollary}[Monomial reduction to all-distinct layers]
\label{cor:monomial-to-H}
For every \(\alpha\in\N^r\) with \(|\alpha|\le\ell\), we have
\[
\left\|
\mathcal H(z^\alpha)
-
\sum_{\lambda\in\mathcal A(\alpha)}
\mathsf A_{\alpha,\lambda}\,
C_{\nu(\alpha,\lambda)}\,
H_\lambda
\right\|_{L_2}
\le
C_\ell\Delta_U.
\]
\end{corollary}

\begin{proof}[Proof of Corollary~\ref{cor:monomial-to-H}]
We decompose \(\mathcal H(z^\alpha)\) into homogeneous Walsh layers:
\[
\mathcal H(z^\alpha)
=
\sum_{\substack{0\le m\le |\alpha|\\ m\equiv |\alpha|\!\!\!\pmod 2}}
[\mathcal H(z^\alpha)]_{\deg m}.
\]
Applying Lemma~\ref{lem:degree-m-reduction} to each admissible \(m\) gives an
\(O_\ell(\Delta_U)\) error in each admissible Walsh degree. Since
\(|\alpha|\le\ell\), there are only \(C_\ell\) admissible degrees. Therefore
the total error is bounded by
\(
C_\ell\Delta_U.
\)
This proves the claim.
\end{proof}

Next, the following lemma compares the combinatorial coefficients \(C_\nu\) with the Gaussian
coefficients \(\nu!^{-1}\). The proof appears in Appendix~\ref{proof:Cmu-hermite}.

\begin{lemma}[Pair-support coefficients]
\label{lem:Cnu-hermite}
For every \(\nu\in\N^r\) with \(|\nu|\le\ell\), we have
\[
\left|C_\nu-\frac1{\nu!}\right|
\le
C_\ell\Delta_U.
\]
\end{lemma}

The next theorem combines Corollary~\ref{cor:monomial-to-H}, Lemma~\ref{lem:Cnu-hermite}, and
the Hermite expansion of \(h\) then we identify the correct surrogate coefficients. The proof appears in Appendix~\ref{proof:hypercube_linear_clean}.

\begin{theorem}[Hypercube multilinearization in Hermite form]
\label{thm:hypercube_linear_clean}
The following bound holds:
\[
\|P-F\|_{L_2(\tau_d)}
\le
C_\ell
\|f^*\|_{L_2(\gamma_d)}
\Theta_U.
\]
\end{theorem}

Since
\[
P-F=\sum_{q=0}^{\ell}(P_q-F_q),
\]
and since the summands lie in distinct Walsh degrees, orthogonality gives
\begin{equation}
\label{eq:Pq-Fq-orth}
\sum_{q=0}^{\ell}\|P_q-F_q\|_{L_2}^2
=
\|P-F\|_{L_2}^2.
\end{equation}
Consequently, Theorem~\ref{thm:hypercube_linear_clean} implies
\begin{equation}
\label{eq:Pq-Fq-bound}
\sum_{q=0}^{\ell}\|P_q-F_q\|_{L_2}^2
\le
C_\ell
\|f^*\|_{L_2(\gamma_d)}^2
\Theta_U^2.
\end{equation}
In particular, for every \(0\le q\le\ell\), we have
\begin{equation}
\label{eq:Pq-Fq-each}
\|P_q-F_q\|_{L_2}
\le
C_\ell
\|f^*\|_{L_2(\gamma_d)}
\Theta_U.
\end{equation}
The same orthogonality also yields, for every cutoff \(0\le L\le\ell\),
\[
\|P_{\le L}-F_{\le L}\|_{L_2}^2
=
\sum_{q=0}^{L}\|P_q-F_q\|_{L_2}^2
\le
C_\ell
\|f^*\|_{L_2(\gamma_d)}^2
\Theta_U^2.
\]

\paragraph{Step 2: structural properties of the all-distinct layers.}

We next record the structural facts about \(H_\lambda\) used in the covariance
comparison. The proof for Proposition~\ref{prop:H-orth} and \ref{prop:H-contraction} appears in Appendix~\ref{proof:H-orth} and \ref{proof:H-contraction}, respectively.

\begin{proposition}[Approximate orthogonality]
\label{prop:H-orth}
Let \(\alpha,\beta\in\N^r\) satisfy \(|\alpha|=|\beta|=m\le\ell\). Then
\[
\Bigl|
\langle H_\alpha,H_\beta\rangle
-
\mathbf 1_{\{\alpha=\beta\}}\,\alpha!
\Bigr|
\le
C_\ell\Delta_U.
\]
\end{proposition}

\begin{proposition}[Ambient-direction contraction]
\label{prop:H-contraction}
Let \(\lambda\in\N^r\) with \(1\le|\lambda|=q\le\ell\). Then, for every
\(s\in[r]\), we have
\[
(u^{[s]})^\top \nabla H_\lambda
=
\lambda_s H_{\lambda-e_s}+R_{s,\lambda},
\qquad
\|R_{s,\lambda}\|_{L_2}
\le
C_\ell\Delta_U.
\]
Here \(\lambda_sH_{\lambda-e_s}\) is interpreted as \(0\) when
\(\lambda_s=0\). Moreover, for every unit vector
\(v\in\mathcal U^\perp\), we have
\[
v^\top\nabla H_\lambda
=
R_{v,\lambda},
\qquad
\|R_{v,\lambda}\|_{L_2}
\le
C_\ell\Delta_U.
\]
In particular, the second estimate applies to every vector
\(u^{[s]}\), \(s\in[d]\setminus[r]\), in the chosen orthonormal completion.
\end{proposition}

The following lemma gives a normalization comparison between the Gaussian and hypercube
\(L_2\)-norms of the ridge polynomial. The proof appears in Appendix~\ref{proof:gaussian-hypercube-L2-comparison}.

\begin{lemma}[Gaussian and hypercube \(L_2\)-norm comparison]
\label{lem:gaussian-hypercube-L2-comparison}
If \(\Theta_U\le1\), then
\[
\left|
\|f^*\|_{L_2(\tau_d)}^2
-
\|f^*\|_{L_2(\gamma_d)}^2
\right|
\le
C_\ell
\Theta_U
\|f^*\|_{L_2(\gamma_d)}^2.
\]
In particular, if \(\Theta_U\) is sufficiently small, then
\[
\|f^*\|_{L_2(\gamma_d)}^2
\asymp_\ell
\|f^*\|_{L_2(\tau_d)}^2.
\]
\end{lemma}

\paragraph{Step 3: the latent Gaussian covariance at degree \(q\).}

For \(q\in\{1,\dots,\ell\}\) and \(s,t\in[r]\), define
\begin{equation}
\label{eq:Gq-def}
(G_q)_{s,t}
:=
\E_{g\sim N(0,I_r)}
\big[
\partial_s h_q(g)\,\partial_t h_q(g)
\big].
\end{equation}
We extend \(G_q\) to a \(d\times d\) matrix in the ambient basis
\(\{u^{[1]},\dots,u^{[d]}\}\) by setting
\[
(\widetilde G_q)_{s,t}
:=
\begin{cases}
(G_q)_{s,t}, & s,t\in[r],\\
0, & \text{otherwise}.
\end{cases}
\]
We also set
\[
A_q
:=
\sum_{\lambda\in\N^r:\,|\lambda|=q}
a_\lambda^2.
\]
Since \(\lambda!\ge1\), Hermite orthogonality gives
\begin{equation}
\label{eq:Aq-bound}
A_q
\le
\|h_q\|_{L_2(\gamma_r)}^2
\le
\|h\|_{L_2(\gamma_r)}^2
=
\|f^*\|_{L_2(\gamma_d)}^2.
\end{equation}

By the chain
rule we have
\[
(u^{[s]})^\top \nabla_x h_q(Ux)
=
\partial_s h_q(Ux).
\]
Since \(Ux\sim N(0,I_r)\) whenever \(x\sim N(0,I_d)\), we can derive
\begin{align}\label{eq:Gq}
(G_q)_{s,t}
=
\E_{x\sim N(0,I_d)}
\big[
(u^{[s]})^\top \nabla_x h_q(Ux)\,
(\nabla_x h_q(Ux))^\top u^{[t]}
\big].
\end{align}

For \(0\le L\le\ell\), define
\[
G_{\le L}
:=
\E_{g\sim N(0,I_r)}
\big[
\nabla h_{\le L}(g)\nabla h_{\le L}(g)^\top
\big],
\]
which is exactly the latent Gaussian covariance $\Sigma_L$.

Since \(\partial_s h_q\) has Hermite degree \(q-1\), Gaussian orthogonality
gives
\begin{equation}
\label{eq:G-sum-by-degree}
(G_{\le L})_{s,t}
=
\sum_{q=1}^{L}(G_q)_{s,t},
\qquad
(\widetilde G_{\le L})_{s,t}
=
\sum_{q=1}^{L}(\widetilde G_q)_{s,t}.
\end{equation}
Here \(\widetilde G_{\le L}\) denotes the extension of \(G_{\le L}\) by zeros
to the last \(d-r\) ambient coordinates.

\paragraph{Step 4: hypercube gradient covariance of the degree-\(q\) surrogate.}

We now compare the degree-\(q\) hypercube gradient covariance of \(F_q\) with
the latent Gaussian covariance \(G_q\).
The proof appears in Appendix~\ref{proof:ambient-basis-covariance}.
\begin{proposition}[Ambient-basis covariance of the degree-\(q\) surrogate]
\label{prop:ambient-basis-covariance}
Fix \(q\in\{1,\dots,\ell\}\), and assume \(\Theta_U\le1\). For \(s,t\in[d]\),
define
\[
(\Xi_q)_{s,t}
:=
\E_{x\sim\tau_d}
\big[
(u^{[s]})^\top \nabla F_q(x)\,
(\nabla F_q(x))^\top u^{[t]}
\big].
\]
Then
\[
\big|
(\Xi_q)_{s,t}
-
(\widetilde G_q)_{s,t}
\big|
\le
C_\ell A_q
\left(
\Theta_U\,
\mathbf 1_{\{s\le r\ \text{or}\ t\le r\}}
+
\Theta_U^2\,
\mathbf 1_{\{s>r,\ t>r\}}
\right).
\]
The same bounds hold uniformly if either inactive basis vector
\(u^{[s]}\), \(s>r\), or \(u^{[t]}\), \(t>r\), is replaced by an arbitrary unit
vector in \(\mathcal U^\perp\).
\end{proposition}

\paragraph{Step 5: degreewise covariance of the multilinearized target.}

The next proposition transfers the degree-\(q\) comparison from \(F_q\) to \(P_q\). The proof appears in Appendix~\ref{proof:degreewise-final-target-cov}.

\begin{proposition}[Degreewise hypercube gradient covariance of the target]
\label{prop:degreewise-final-target-cov}
For \(q\in\{1,\dots,\ell\}\), define
\[
(\Gamma_q)_{s,t}
:=
\E_{x\sim\tau_d}
\big[
(u^{[s]})^\top \nabla P_q(x)\,
(\nabla P_q(x))^\top u^{[t]}
\big],
\qquad
s,t\in[d].
\]
If \(\Theta_U\) is sufficiently small, then
\[
\big|
(\Gamma_q)_{s,t}
-
(\widetilde G_q)_{s,t}
\big|
\le
C_\ell
\|f^*\|_{L_2(\tau_d)}^2
\left(
\Theta_U\,
\mathbf 1_{\{s\le r\ \text{or}\ t\le r\}}
+
\Theta_U^2\,
\mathbf 1_{\{s>r,\ t>r\}}
\right).
\]
The same bounds hold uniformly if inactive basis vectors are replaced by
arbitrary unit vectors in \(\mathcal U^\perp\).
\end{proposition}

\paragraph{Step 6: covariance of the low-degree multilinearized target.}

We now sum over degrees. This gives the desired comparison between the
hypercube gradient covariance of \(P_{\le L}\) and the lifted Gaussian
covariance of \(h_{\le L}\). The proof appears in Appendix~\ref{proof:final-target-cov}.

\begin{theorem}[Hypercube gradient covariance of the multilinearized target]
\label{thm:final-target-cov}
For \(0\le L\le\ell\), define
\[
(\Gamma_{\le L})_{s,t}
:=
\E_{x\sim\tau_d}
\big[
(u^{[s]})^\top \nabla P_{\le L}(x)\,
(\nabla P_{\le L}(x))^\top u^{[t]}
\big],
\qquad
s,t\in[d].
\]
If \(\Theta_U\) is sufficiently small, then
\[
\big|
(\Gamma_{\le L})_{s,t}
-
(\widetilde G_{\le L})_{s,t}
\big|
\le
C_\ell
\|f^*\|_{L_2(\tau_d)}^2
\left(
\Theta_U\,
\mathbf 1_{\{s\le r\ \text{or}\ t\le r\}}
+
\Theta_U^2\,
\mathbf 1_{\{s>r,\ t>r\}}
\right).
\]
The same mixed and inactive-inactive bounds hold uniformly if inactive basis
vectors are replaced by arbitrary unit vectors in \(\mathcal U^\perp\).

Equivalently, since \(G_{\le L}=\Sigma_L\), we have
\[
\big|
(\Gamma_{\le L})_{s,t}
-
(\Sigma_L)_{s,t}
\big|
\le
C_\ell
\|f^*\|_{L_2(\tau_d)}^2
\Theta_U,
\qquad
s,t\in[r].
\]
If exactly one of \(s,t\) lies in \([r]\), then
\[
|(\Gamma_{\le L})_{s,t}|
\le
C_\ell
\|f^*\|_{L_2(\tau_d)}^2
\Theta_U.
\]
Finally, for \(s,t\in[d]\setminus[r]\), we have
\[
|(\Gamma_{\le L})_{s,t}|
\le
C_\ell
\|f^*\|_{L_2(\tau_d)}^2
\Theta_U^2.
\]
Moreover, the corresponding block-operator comparison holds:
\begin{equation}
\label{eq:block-op-transfer}
\snorm{
M_{\le L}
-
U^\top \Sigma_L U
}
\le
C_\ell
\|f^*\|_{L_2(\tau_d)}^2
\overline\Theta_U .
\end{equation}
\end{theorem}

This proves Lemma~\ref{lem:M_gaussian_latent}.

\section{A General AGOP Approximation Result and Proof of Theorem~\ref{thm:main_agop}}\label{proof:main_agop}
We first prove an AGOP approximation theorem that does not depend on the subspace coherence. The result is stated in coordinates aligned
with the central subspace. We then apply this theorem with
\(S=\{1,\ldots,r\}\) in the rotated basis and use coherence to obtain
the operator-norm bound in Theorem~\ref{thm:main_agop}.

We introduce some notation before the theorem statement. 
Let \(U_\perp\in\mathbb R^{(d-r)\times d}\) have orthonormal rows spanning 
\(\mathrm{row}(U)^\perp\), and define the orthogonal completion
\[
Q_U
:=
\begin{bmatrix}
U\\
U_\perp
\end{bmatrix}
\in\mathbb R^{d\times d}.
\]
Define
\begin{align}
    \Gamma_q = Q_U M_q Q_U\tran \in \R^{d\times d}, \qquad \Gamma_{\leq p} = \sum_{0\leq q\leq p}\Gamma_q.
\end{align}

\begin{theorem}[AGOP approximation]\label{thm:main_agop_nodeloc}
Suppose Assumptions~\ref{assum:degree}, \ref{assump:g_0}, and \ref{assump:g_1} hold. Suppose further that
\[
\min_{0\le k\le p} g^{(k)}(0)>0,
\qquad
\lambda+\sum_{k=p+1}^\infty g^{(k)}(0)>0.
\]
Fix \(p\in\{0,\dots,\ell\}\), and let \(n=d^{p+\delta}\) with \(\delta\in(0,1)\). For every fixed deterministic subset \(S\subset[d]\), the following holds for any sufficiently small
\(\epsilon>0\):
\begin{equation}
\label{eq:agop-op-reduction}
\begin{aligned}
\|\widehat M-M_{\le p}\|_{\op}
&\le
\max\{\mathcal R(S),\mathcal R(S^c)\} + \|(\Gamma_{\leq p})_{S,S^c}\|_{\op}
\\
&\quad+
\sqrt{
\Big(\|(\Gamma_{\leq p})_{S,S}\|_{\op}+\mathcal R(S)\Big)
\Big(\|(\Gamma_{\leq p})_{S^c,S^c}\|_{\op}+\mathcal R(S^c)\Big).
}
\end{aligned}
\end{equation}
Here, with
$A_s:=\sum_{q=0}^p(\Gamma_q)_{s,s}$,
$B_s:=\sum_{q=0}^p\sqrt{(\Gamma_q)_{s,s}}$
and \(A_J=(A_s)_{s\in J} \in \R^{|J|}\), \(B_J=(B_s)_{s\in J}\in \R^{|J|}\), we define
\begin{equation}
\label{eq:R-I-def}
\begin{aligned}
\mathcal R(J)
&:=
O_{d,\P}\!\left(
d^{-\delta+\epsilon}+d^{-1+\delta+\epsilon}
\right)
\left(\|f^*\|_{L_2}^2+\sigma_\vepsilon^2\right) +O_{d,\P}\!\left(d^{-\delta/2+\epsilon}\right)
\left(
\|A_J\|_1+\sqrt{|J|}\|A_J\|_2
\right)
\\
&\quad+
O_{d,\P}\!\left(d^{-(1+\delta)/2+\epsilon/2}\right)
\left(\|f^*\|_{L_2}^2+\sigma_\vepsilon^2\right)^{1/2}
\left(
\|B_J\|_1+\sqrt{|J|}\|B_J\|_2
\right).
\end{aligned}
\end{equation}
\end{theorem}

\begin{proof}

We establish the following key result, and defer the proof to Appendix~\ref{proof:main_krr}.
\begin{theorem}\label{thm:main_krr}
  Suppose Assumptions~\ref{assum:degree}, \ref{assump:g_0}, and \ref{assump:g_1} hold.
Fix \(p\in\{0,\dots,\ell\}\), and let \(n=d^{p+\delta}\) with \(\delta\in(0,1)\).  Then for any $\epsilon>0$, and uniformly for any $s,t\in[d]$,
  the KRR predictor $\hat{f}$ satisfies
\begin{align}\label{eq:main_krr}
     &~~\left| (u^{[s]})\tran \E_n\brac{\nabla \hat{f}(\nabla \hat{f})\tran}(u^{[t]})-  \sum_{0\leq q\leq p}{(\Gamma_{q})_{s,t}} -C_0^2 (d^{2\delta-2})(\Gamma_{p+1})_{s,t}  \right|\notag \\
&=  O_{d,\P}\round{d^{-\frac{\delta}{2}+\epsilon}}
\Bigg[
\sum_{0\leq q\leq p}\round{(\Gamma_{q})_{s,s} + (\Gamma_{q})_{t,t}}
+
d^{2\delta-2}
\round{
(\Gamma_{p+1})_{s,s}
+
(\Gamma_{p+1})_{t,t}
}
\Bigg]
\notag\\
&\qquad+
O_{d,\P}\round{d^{-\frac{1+\delta}{2}+\frac{\epsilon}{2}}}
\round{\sum_{0\leq q\leq p}\round{\sqrt{(\Gamma_{q})_{s,s}} + \sqrt{(\Gamma_{q})_{t,t}}}}
\round{\norm{f^*}_{L_2}^2+\sigma_\vepsilon^2}^{1/2}
\notag\\
&\qquad+
O_{d,\P}\round{d^{\frac{-4+3\delta}{2}+\frac{\epsilon}{2}}}
\round{\sqrt{(\Gamma_{p+1})_{s,s}} + \sqrt{(\Gamma_{p+1})_{t,t}}}
\round{\norm{f^*}_{L_2}^2+\sigma_\vepsilon^2}^{1/2}
\notag\\
&\qquad+
O_{d,\P}\round{
d^{-1-\delta+\epsilon}
+
d^{-2+\delta+\epsilon}
}
\round{\norm{f^*}_{L_2}^2+\sigma_\vepsilon^2}
\end{align}
where $C_0 =(\rho+\lambda)^{-1} g^{(p+1)}(0)$.
\end{theorem}

To simplify the notation, let
\[
R:=\|f^*\|_{L_2}^2+\sigma_\vepsilon^2,
\qquad
\theta_d:=d^{2\delta-2},
\qquad
\kappa_d:=C_0^2\theta_d .
\]
Since \(Q_U\) is orthogonal, the operator norm is unchanged by this
rotation, and hence
\[
\|\widehat M-M_{\le p}\|_{\op}
=
\|\widehat\Gamma-\Gamma_{\le p}\|_{\op}.
\]

We center \(\widehat\Gamma\) at the matrix appearing in
Theorem~\ref{thm:main_krr} by defining
\[
\widetilde E
:=
\widehat\Gamma-\Gamma_{\le p}-\kappa_d\Gamma_{p+1},
\qquad
E:=\widehat\Gamma-\Gamma_{\le p}.
\]

With this notation, the desired error is
\[
E=\widetilde E+\kappa_d\Gamma_{p+1}.
\]

We first verify that the \((p+1)\)-degree component has bounded trace. Since
\(\Gamma_{p+1}\) is the rotated degree-\((p+1)\) population AGOP, the
Fourier--Walsh gradient-energy identity gives
\[
\tr(\Gamma_{p+1})
=
\E_{x\sim\tau_d}
\left[
\left\|
\nabla \mathcal H(f_U^*)_{p+1}(x)
\right\|_2^2
\right]
\lesssim
\left\|
\mathcal H(f_U^*)_{p+1}
\right\|_{L_2}^2
\le
\left\|
\mathcal H(f_U^*)
\right\|_{L_2}^2
\lesssim
\|f^*\|_{L_2}^2
\le R .
\]

For the rest of the proof, introduce the temporary notation
\[
g_s:=(\Gamma_{p+1})_{s,s},
\qquad
h_s:=\sqrt{(\Gamma_{p+1})_{s,s}} .
\]

By Theorem~\ref{thm:main_krr}, uniformly over \(s,t\in[d]\), we have
\[
\begin{aligned}
|\widetilde E_{s,t}|
&\le
O_{d,\P}\!\left(d^{-\delta/2+\epsilon}\right)
\left[
A_s+A_t+\theta_d(g_s+g_t)
\right]
\\
&\quad+
O_{d,\P}\!\left(d^{-(1+\delta)/2+\epsilon/2}\right)
R^{1/2}
\left[
B_s+B_t
\right]
\\
&\quad+
O_{d,\P}\!\left(d^{(-4+3\delta)/2+\epsilon/2}\right)
R^{1/2}
\left[
h_s+h_t
\right]
\\
&\quad+
O_{d,\P}\!\left(
d^{-1-\delta+\epsilon}
+
d^{-2+\delta+\epsilon}
\right)R .
\end{aligned}
\]

We now convert this entrywise control into block operator-norm control. For any
nonnegative vector \(v\in\mathbb R^d\) and any \(J\subset[d]\), the matrix
with entries \(v_s+v_t\), \(s,t\in J\), satisfies
\[
\left\|(v_s+v_t)_{s,t\in J}\right\|_{\op}
\le
\left\|(v_s+v_t)_{s,t\in J}\right\|_F
\lesssim
\|v_J\|_1+\sqrt{|J|}\|v_J\|_2 .
\]

Also, the constant entrywise term contributes at most its size times \(|J|\),
and since \(|J|\le d\), it contributes
\[
O_{d,\P}\!\left(
d^{-\delta+\epsilon}
+
d^{-1+\delta+\epsilon}
\right)R .
\]

Applying these two deterministic bounds to the preceding entrywise estimate
yields, for every deterministic \(J\subset[d]\),
\[
\begin{aligned}
\|\widetilde E_{J,J}\|_{\op}
&\le
O_{d,\P}\!\left(d^{-\delta/2+\epsilon}\right)
\left(
\|A_J\|_1+\sqrt{|J|}\|A_J\|_2
\right)
\\
&\quad+
O_{d,\P}\!\left(d^{-(1+\delta)/2+\epsilon/2}\right)
R^{1/2}
\left(
\|B_J\|_1+\sqrt{|J|}\|B_J\|_2
\right)
\\
&\quad+
\mathcal T_{p+1}(J)
+
O_{d,\P}\!\left(
d^{-\delta+\epsilon}
+
d^{-1+\delta+\epsilon}
\right)R ,
\end{aligned}
\]
where \(\mathcal T_{p+1}(J)\) denotes the contribution of the two terms
involving \(g\) and \(h\).

We next show that \(\mathcal T_{p+1}(J)\) is absorbed by the last line. Since
\(\Gamma_{p+1}\succeq0\) and \(\tr(\Gamma_{p+1})\lesssim R\), we have
\[
\|g_J\|_1+\sqrt{|J|}\|g_J\|_2
\lesssim
\sqrt d\,R,
\qquad
\|h_J\|_1+\sqrt{|J|}\|h_J\|_2
\lesssim
\sqrt d\,R^{1/2}.
\]

Therefore the \((p+1)\)-degree terms satisfy
\[
\mathcal T_{p+1}(J)
\le
O_{d,\P}\!\left(d^{-3(1-\delta)/2+\epsilon}\right)R
+
O_{d,\P}\!\left(d^{-3(1-\delta)/2+\epsilon/2}\right)R .
\]

Since \(\delta\in(0,1)\), the previous display is bounded by
\[
\mathcal T_{p+1}(J)
\le
O_{d,\P}\!\left(d^{-1+\delta+\epsilon}\right)R .
\]

The deterministic bias \(\kappa_d\Gamma_{p+1}\) is absorbed similarly. Indeed,
using \(\|\Gamma_{p+1}\|_{\op}\le \tr(\Gamma_{p+1})\), we get
\[
\kappa_d\|(\Gamma_{p+1})_{J,J}\|_{\op}
\le
C_0^2 d^{2\delta-2}\tr(\Gamma_{p+1})
\lesssim
d^{2\delta-2}R
\le
d^{-1+\delta}R .
\]

Combining the previous bounds gives the diagonal-block estimate
\[
\|E_{J,J}\|_{\op}
=
\|(\widehat\Gamma-\Gamma_{\le p})_{J,J}\|_{\op}
\le
\mathcal R(J).
\]

It remains to control the off-diagonal block.
The standard PSD block inequality gives
\[
\|\widehat\Gamma_{S,N}\|_{\op}
\le
\sqrt{
\|\widehat\Gamma_{S,S}\|_{\op}
\|\widehat\Gamma_{N,N}\|_{\op}
}.
\]

Using the diagonal-block estimate with \(J=S\) and \(J=N\), we obtain
\[
\|\widehat\Gamma_{J,J}\|_{\op}
\le
\|(\Gamma_{\le p})_{J,J}\|_{\op}
+
\mathcal R(J),
\qquad
J\in\{S,N\}.
\]

Consequently, the off-diagonal block of \(E\) satisfies
\[
\begin{aligned}
\|E_{S,N}\|_{\op}
&\le
\|\widehat\Gamma_{S,N}\|_{\op}
+
\|(\Gamma_{\le p})_{S,N}\|_{\op}
\\
&\le
\sqrt{
\Big(\|(\Gamma_{\le p})_{S,S}\|_{\op}+\mathcal R(S)\Big)
\Big(\|(\Gamma_{\le p})_{N,N}\|_{\op}+\mathcal R(N)\Big)
}
+
\|(\Gamma_{\le p})_{S,N}\|_{\op}.
\end{aligned}
\]

Finally, for any symmetric block matrix one has
\[
\left\|
\begin{pmatrix}
A & B\\
B^\top & D
\end{pmatrix}
\right\|_{\op}
\le
\max\{\|A\|_{\op},\|D\|_{\op}\}
+
\|B\|_{\op}.
\]

Applying this inequality to \(E=\widehat\Gamma-\Gamma_{\le p}\), with the
block decomposition \(S\cup N=[d]\), gives
\[
\begin{aligned}
\|\widehat M-M_{\le p}\|_{\op}
&=
\|E\|_{\op}
\\
&\le
\max\{\mathcal R(S),\mathcal R(N)\}
+
\|(\Gamma_{\le p})_{S,N}\|_{\op}
\\
&\quad+
\sqrt{
\Big(\|(\Gamma_{\le p})_{S,S}\|_{\op}+\mathcal R(S)\Big)
\Big(\|(\Gamma_{\le p})_{N,N}\|_{\op}+\mathcal R(N)\Big)
}.
\end{aligned}
\]
This proves the theorem.
\end{proof}
\subsection{Proof of Theorem~\ref{thm:main_agop}}

We work in an orthonormal coordinate system aligned with the relevant subspace.
In this basis,
\[
S=[r],
\qquad
N=[d]\setminus [r].
\]
Throughout the proof, \(r\) and \(\ell\) are fixed. We write
\[
e_f:=\|f^*\|_{L_2}^2+\sigma_\vepsilon^2,
\qquad
\mu:=\mu(U).
\]

Define the coherence scale
\[
\Theta_U
:=
N_{r,\ell}^2\frac{r^2\mu}{d},
\qquad
N_{r,\ell}:=\binom{r+\ell}{\ell}.
\]
Since \(r\) and \(\ell\) are fixed, this scale satisfies
\(
\Theta_U\lesssim \frac{\mu}{d}.
\)
Thus, in the regime where \(\Theta_U\) is sufficiently small,
Theorem~\ref{thm:final-target-cov} gives the block estimates
\[
\begin{aligned}
\|(\Gamma_{\le p})_{S,S}\|_{\op}
&\lesssim e_f,\\
\|(\Gamma_{\le p})_{S,N}\|_{\op}
&\lesssim e_f\frac{\mu}{d},\\
\|(\Gamma_{\le p})_{N,N}\|_{\op}
&\lesssim e_f\frac{\mu^2}{d^2}.
\end{aligned}
\]
The same theorem also gives the corresponding diagonal estimates:
\[
(\Gamma_{\le p})_{s,s}
\lesssim
\begin{cases}
e_f, & s\in S,\\[2mm]
e_f\mu^2/d^2, & s\in N.
\end{cases}
\]

We next estimate the diagonal profiles that enter
\(\mathcal R(S)\) and \(\mathcal R(N)\). Since $A_s$ takes the form
\[
A_s=\sum_{q=0}^p(\Gamma_q)_{s,s}
=(\Gamma_{\le p})_{s,s},
\]
the preceding diagonal bounds imply
\[
A_s
\lesssim
\begin{cases}
e_f, & s\in S,\\[2mm]
e_f\mu^2/d^2, & s\in N.
\end{cases}
\]
The estimate for \(B_s\) follows from Cauchy's inequality. Indeed, because
each \(\Gamma_q\succeq 0\),
\[
B_s
=
\sum_{q=0}^p \sqrt{(\Gamma_q)_{s,s}}
\le
\sqrt{p+1}
\left(
\sum_{q=0}^p(\Gamma_q)_{s,s}
\right)^{1/2}.
\]
Since \(p\) is fixed in the present regime, this gives
\[
B_s
\lesssim
\begin{cases}
e_f^{1/2}, & s\in S,\\[2mm]
e_f^{1/2}\mu/d, & s\in N.
\end{cases}
\]

We now sum these coordinatewise bounds. Since \(|S|=r=O(1)\) and
\(|N|\le d\), we obtain
\[
\begin{aligned}
\|A_S\|_1+\sqrt{|S|}\|A_S\|_2
&\lesssim e_f,\\
\|B_S\|_1+\sqrt{|S|}\|B_S\|_2
&\lesssim e_f^{1/2},\\
\|A_N\|_1+\sqrt{|N|}\|A_N\|_2
&\lesssim e_f\frac{\mu^2}{d},\\
\|B_N\|_1+\sqrt{|N|}\|B_N\|_2
&\lesssim e_f^{1/2}\mu .
\end{aligned}
\]

Substituting these profile estimates into the definition of \(\mathcal R(S)\)
gives the active-block bound
\[
\mathcal R(S)
\lesssim
e_f\left(
d^{-\delta/2+\epsilon}
+
d^{-1+\delta+\epsilon}
\right).
\]
Here the other active-block contributions are dominated by the displayed terms
because \(\delta\in(0,1)\). Similarly, the inactive-block contribution satisfies
\[
\begin{aligned}
\mathcal R(N)
\lesssim
e_f\big(&
d^{-\delta+\epsilon}
+
d^{-1+\delta+\epsilon} \\
&+
\mu^2 d^{-1-\delta/2+\epsilon}
+
\mu d^{-(1+\delta)/2+\epsilon/2}
\big).
\end{aligned}
\]

We now apply Theorem~\ref{thm:main_agop_nodeloc}. With
\(
\varepsilon_{\mathrm{agop}}
:=
\|\widehat M-M_{\le p}\|_{\op},
\)
the theorem yields
\[
\begin{aligned}
\varepsilon_{\mathrm{agop}}
&\le
\max\{\mathcal R(S),\mathcal R(N)\}
+
C e_f\frac{\mu}{d} \\
&\quad+
\Bigl[
\bigl(Ce_f+\mathcal R(S)\bigr)
\bigl(Ce_f\mu^2/d^2+\mathcal R(N)\bigr)
\Bigr]^{1/2},
\end{aligned}
\]
where \(C=C_{r,\ell}\) is a constant depending only on \(r\) and \(\ell\).

Since \(\epsilon>0\) is chosen sufficiently small, we have
\[
\mathcal R(S)=o_{d,\P}(e_f).
\]
Consequently, the active factor inside the square root is \(O_{d,\P,r,\ell}(e_f)\).
For the inactive factor, the preceding bound on \(\mathcal R(N)\) gives
\[
\begin{aligned}
Ce_f\frac{\mu^2}{d^2}+\mathcal R(N)
\lesssim
e_f\big(&
\mu^2 d^{-2}
+
d^{-\delta+\epsilon}
+
d^{-1+\delta+\epsilon} \\
&+
\mu^2 d^{-1-\delta/2+\epsilon}
+
\mu d^{-(1+\delta)/2+\epsilon/2}
\big).
\end{aligned}
\]
Taking square roots, and renaming the arbitrarily small polynomial slack
\(\epsilon\), the square-root term is therefore bounded by
\[
\begin{aligned}
e_f\big(&
d^{-\delta/2+\epsilon}
+
d^{-(1-\delta)/2+\epsilon}
+
\mu d^{-1} \\
&+
\mu d^{-1/2-\delta/4+\epsilon}
+
\mu^{1/2}d^{-(1+\delta)/4+\epsilon}
\big).
\end{aligned}
\]

Combining the preceding estimates gives
\[
\begin{aligned}
\varepsilon_{\mathrm{agop}}
\lesssim
e_f\big(&
d^{-\delta/2+\epsilon}
+
d^{-(1-\delta)/2+\epsilon}
+
\mu d^{-1} \\
&+
\mu^{1/2}d^{-(1+\delta)/4+\epsilon}
+
\mu d^{-1/2-\delta/4+\epsilon} \\
&+
\mu^2 d^{-1-\delta/2+\epsilon}
+
\mu d^{-(1+\delta)/2+\epsilon/2}
\big).
\end{aligned}
\]
Since \(\mu\ge 1\) and \(\delta\in(0,1)\), the terms
\(
\mu d^{-1}\)
 and
\(\mu d^{-(1+\delta)/2+\epsilon/2}
\)
are both absorbed by
\(
\mu d^{-1/2-\delta/4+\epsilon}.
\)
Hence we conclude
\[
\begin{aligned}
\varepsilon_{\mathrm{agop}}
\lesssim
e_f\big(&
d^{-\delta/2+\epsilon}
+
d^{-(1-\delta)/2+\epsilon} \\
&+
\mu^{1/2}d^{-(1+\delta)/4+\epsilon}
+
\mu d^{-1/2-\delta/4+\epsilon} \\
&+
\mu^2 d^{-1-\delta/2+\epsilon}
\big).
\end{aligned}
\]
which proves \eqref{eq:main_result_1}.

Finally, suppose that for some constant \(\gamma>0\), the following holds
\[
\mu(U)\le d^{1/2+\delta/4-\gamma}.
\]
Then we reach
\[
\begin{aligned}
\varepsilon_{\mathrm{agop}}
\lesssim{}&
d^{-\delta/2+\epsilon}
+
d^{-(1-\delta)/2+\epsilon} \\
&+
d^{-\delta/8-\gamma/2+\epsilon}
+
d^{-\gamma+\epsilon}
+
d^{-2\gamma+\epsilon}.
\end{aligned}
\]
Choosing \(\epsilon>0\) sufficiently small, the right-hand side tends to zero, which completes the proof.

\section{Proof of Lemma~\ref{lem:eigenspace-perturbation}}
\label{proof:eigenspace-perturbation}

Let \(\Pi_U:=U^\top U\) and \(A:=\Pi_U M_{\le p}\Pi_U\). We first identify
the relevant eigenspace of \(A\). Since \(UU^\top=I_r\), the matrix \(A\) has
the representation
\[
A
=
U^\top (U M_{\le p} U^\top) U .
\]
Hence the nonzero eigenvalues of \(A\) are exactly the eigenvalues of
\(U M_{\le p} U^\top\). In particular, the eigenvalue identities
\[
\lambda_r(A)=s_p,
\qquad
\lambda_{r+1}(A)=0
\]
hold. Moreover, the range of \(A\) is contained in \(\mathrm{row}(U)\), so the
top-\(r\) eigenspace of \(A\) is precisely \(\mathrm{row}(U)\).

Define the perturbation \(E:=\widehat M-A\). The triangle inequality, together
with the definitions of \(\varepsilon_{\mathrm{agop}}\) and \(\rho_p\), gives
the perturbation bound
\begin{equation}
\label{eq:eigenspace-perturbation-size}
\snorm{E}
\le
\snorm{\widehat M-M_{\le p}}
+
\snorm{M_{\le p}-\Pi_U M_{\le p}\Pi_U}
=
\varepsilon_{\mathrm{agop}}+\rho_p .
\end{equation}
For brevity, set \(\Delta:=\varepsilon_{\mathrm{agop}}+\rho_p\).

We consider two cases. If \(\Delta\ge s_p/4\), then
\(4\Delta/s_p\ge1\). Since \(\snorm{\sin_\Theta(\widehat U,U)}\le1\)
always, the desired estimate is immediate in this case.

It remains to consider the case \(\Delta<s_p/4\). Weyl's inequality gives the
lower bound
\[
\lambda_r(\widehat M)
\ge
\lambda_r(A)-\snorm{E}
\ge
s_p-\Delta
\ge
\frac34 s_p .
\]
The same inequality gives the upper bound
\[
\lambda_{r+1}(\widehat M)
\le
\lambda_{r+1}(A)+\snorm{E}
\le
\Delta
\le
\frac14 s_p .
\]
These two bounds show that the top-\(r\) eigenvalues of \(\widehat M\) are
separated from the remaining spectrum by a gap of at least \(s_p/2\).

We may now apply the Davis--Kahan \(\sin\Theta\) theorem to the pair
\(A\) and \(\widehat M=A+E\). The \(r\)-dimensional invariant subspace of
\(A\) corresponding to its nonzero eigenvalues is \(\mathrm{row}(U)\), while
the top-\(r\) eigenspace of \(\widehat M\) is \(\mathrm{row}(\widehat U)\) by
assumption. Davis--Kahan gives the estimate
\[
\snorm{\sin_\Theta(\widehat U,U)}
\le
\frac{\snorm{E}}{s_p/2}
\le
2\,\frac{\Delta}{s_p}
\le
4\,\frac{\Delta}{s_p}.
\]
Combining the two cases yields
\[
\snorm{\sin_\Theta(\widehat U,U)}
\le
\min\left\{
1,\,
4\,\frac{\rho_p+\varepsilon_{\mathrm{agop}}}{s_p}
\right\}.
\]
Finally, under the condition
\((\rho_p+\varepsilon_{\mathrm{agop}})/s_p=o_d(1)\), the right-hand side of \eqref{eq:eigenspace-perturbation} tends
to zero. Therefore the subspace error satisfies
\[
\snorm{\sin_\Theta(\widehat U,U)}=o_d(1),
\]
which proves that \(\widehat U\) consistently recovers \(\mathrm{row}(U)\).

\section{Proof of Corollary~\ref{cor:main_result}}\label{proof:cor:main_result}

Set
\[
e_M:=\snorm{M_{\le p}-U^\top\Sigma_pU},
\qquad
\varepsilon_{\mathrm{agop}}:=\snorm{\widehat M-M_{\le p}}.
\]
Lemma~\ref{lem:M_gaussian_latent} gives the deterministic comparison bound
\[
e_M
=
O_d\!\left(
\frac{\mu(U)}{d}\norm{f^*}_{L_2(\tau_d)}^2
\right).
\]
The gap-stability assumption therefore implies \(e_M=o_d(\kappa)\). Moreover,
the weak coherence condition gives \(\mu(U)/d=o_d(1)\), so the same bound also
implies \(e_M=o_d(e_f)\).

Let \(P_U:=P_{\mathrm{row}(U)}=U^\top U\). Since \(U^\top\Sigma_pU\) is
supported on \(\mathrm{row}(U)\), the identity
\[
P_U(U^\top\Sigma_pU)P_U=U^\top\Sigma_pU
\]
holds. Hence the definition of \(\rho_p\) gives the inequality
\[
\rho_p
=
\snorm{M_{\le p}-P_UM_{\le p}P_U}
\le
\snorm{M_{\le p}-U^\top\Sigma_pU}
+
\snorm{P_U(M_{\le p}-U^\top\Sigma_pU)P_U}
\le
2e_M.
\]
Thus \(\rho_p=o_d(e_f)\). Similarly, the compression error satisfies
\[
\snorm{UM_{\le p}U^\top-\Sigma_p}
=
\snorm{U(M_{\le p}-U^\top\Sigma_pU)U^\top}
\le
e_M,
\]
where the equality uses \(UU^\top=I_r\). Weyl's inequality then yields
\[
s_p
=
\lambda_{\min}(UM_{\le p}U^\top)
\ge
\lambda_{\min}(\Sigma_p)-e_M
\ge
\kappa-e_M.
\]
Since \(e_M=o_d(\kappa)\), this lower bound implies \(s_p\ge \kappa/2\) for all
sufficiently large \(d\).

It remains to control the empirical AGOP error. Under
\(\mu(U)=O_d(d^{1/2+\delta/4-\gamma})\), choose \(\epsilon>0\) sufficiently small
relative to \(\delta\) and \(\gamma\). Then the rate \(R_d(U)\) in
Theorem~\ref{thm:main_agop} satisfies \(R_d(U)=o_d(1)\), and therefore
Theorem~\ref{thm:main_agop} gives the stochastic bound
\[
\varepsilon_{\mathrm{agop}}
=
o_{d,\mathbb P}(e_f).
\]
Combining this bound with \(\rho_p=o_d(e_f)\) gives
\[
\rho_p+\varepsilon_{\mathrm{agop}}
=
o_{d,\mathbb P}(e_f).
\]
Finally, Lemma~\ref{lem:eigenspace-perturbation} and the lower bound
\(s_p\ge\kappa/2\) imply
\[
\snorm{\sin_\Theta(\widehat U,U)}
\le
4\,\frac{\rho_p+\varepsilon_{\mathrm{agop}}}{s_p}
\le
8\kappa^{-1}(\rho_p+\varepsilon_{\mathrm{agop}})
=
o_{d,\mathbb P}\!\left(\kappa^{-1}e_f\right).
\]
This proves the claim.

\section{Proof of Theorem~\ref{thm:test_error_krr}}
\label{proof:test_error_krr}

In this section, we follow the notation of~\cite{mei2022generalization}. In
particular, the space denoted by \(L_p\) elsewhere in the paper is written as
\(L^p\) in this proof.

Set \(g_d:=\mathcal H(f_U^*)\in L^2(\tau_d)\). For each
\(k\in\{0,\ldots,d\}\), let
\[
V_{d,k}:=\mathrm{span}\{\chi_S:\ |S|=k\},
\qquad
\chi_S(x):=\prod_{i\in S}x_i .
\]
The Walsh characters form an orthonormal basis of \(L^2(\tau_d)\), giving the
orthogonal decomposition
\[
L^2(\tau_d)=\bigoplus_{k=0}^d V_{d,k},
\qquad
\dim(V_{d,k})=\binom{d}{k}.
\]
Because the first-step kernel is an inner-product kernel on the hypercube, each
space \(V_{d,k}\) is an eigenspace of the associated kernel operator. Let
\(\xi_{d,k}^2\) denote the corresponding eigenvalue, and let \(P_k\) be the
orthogonal projector onto \(V_{d,k}\). With this notation,
\[
g_d=\sum_{k=0}^d P_k g_d,
\qquad
g_{d,\le p}:=\sum_{k=0}^p P_k g_d,
\qquad
g_{d,>p}:=\sum_{k\ge p+1}P_k g_d .
\]
By definition of the low-degree truncation of the multilinearized target,
\(g_{d,\le p}=\mathcal H(f_U^*)_{\le p}\).

We next identify the leading eigenspace appearing in Theorem~4 of
\cite{mei2022generalization}. Let
\[
m:=\sum_{k=0}^p \binom{d}{k}.
\]
Under the assumptions of Theorem~\ref{thm:main_agop}, the hypercube
verification in Appendix~D.2 of~\cite{mei2022generalization} applies. In
particular, the degree ordering of the kernel eigenspaces is compatible with
the Walsh decomposition above: the leading \(m\) eigenspaces are exactly
\[
\bigoplus_{k=0}^p V_{d,k}.
\]
Equivalently, the projector \(P_{\le m}\) in the notation of
\cite{mei2022generalization} coincides with the Walsh projector
\(P_{\le p}:=\sum_{k=0}^p P_k\). Hence \(P_{>m}g_d=g_{d,>p}\).

The same hypercube spectral estimates give the fixed-degree scaling
\(\xi_{d,k}^2\asymp d^{-k}\). Since \(p\) and \(\ell\) are fixed, there exist
constants \(c,C>0\), independent of \(d\), such that
\begin{equation}
\label{eq:krr_eigenvalue_scales}
\min_{0\le k\le p}\xi_{d,k}^2
\ge
c\,d^{-p},
\qquad
\max_{p+1\le k\le \ell}\xi_{d,k}^2
\le
C\,d^{-p-1},
\end{equation}
where the second bound is void when \(\ell\le p\).

Let \(\lambda_1\) denote the regularization parameter used in the first KRR
step, so that \(\hat f_1=\hat f_{\lambda_1}\). By the assumptions of
Theorem~\ref{thm:main_agop}, the sample size and regularization satisfy
\[
n=d^{p+\delta},
\qquad
0<\delta<1,
\qquad
\lambda_1\in[0,\lambda^\star],
\qquad
\lambda^\star:=\operatorname{Tr}(H_{d,>m}).
\]
The hypercube kernel assumptions imply
\(\operatorname{Tr}(H_{d,>m})=\Theta(1)\). Therefore the effective
regularization parameter from Theorem~4 of~\cite{mei2022generalization},
\[
\gamma_{\mathrm{eff}}
:=
\lambda_1+\operatorname{Tr}(H_{d,>m}),
\]
satisfies
\begin{equation}
\label{eq:krr_effective_regularization_scale}
\gamma_{\mathrm{eff}}=\Theta(1),
\qquad
\frac{\gamma_{\mathrm{eff}}}{n}
=
\Theta(d^{-p-\delta}).
\end{equation}

Theorem~4 of~\cite{mei2022generalization} gives the effective population
estimator in diagonal form:
\[
\hat f^{\mathrm{eff}}_{\gamma_{\mathrm{eff}}}
=
\sum_{k=0}^d s_{d,k}\,P_k g_d,
\qquad
s_{d,k}
:=
\frac{\xi_{d,k}^2}
{\xi_{d,k}^2+\gamma_{\mathrm{eff}}/n}.
\]
For the low-degree modes \(k\le p\), the lower bound in
\eqref{eq:krr_eigenvalue_scales} and the scale estimate
\eqref{eq:krr_effective_regularization_scale} give the shrinkage bound
\[
1-s_{d,k}
=
\frac{\gamma_{\mathrm{eff}}/n}
{\xi_{d,k}^2+\gamma_{\mathrm{eff}}/n}
\le
\frac{\gamma_{\mathrm{eff}}/n}{\xi_{d,k}^2}
\lesssim
d^{-\delta}.
\]
For the high-degree modes that can appear in \(g_d\), namely
\(p+1\le k\le \ell\), the upper bound in
\eqref{eq:krr_eigenvalue_scales} gives
\[
s_{d,k}
=
\frac{\xi_{d,k}^2}
{\xi_{d,k}^2+\gamma_{\mathrm{eff}}/n}
\le
\frac{\xi_{d,k}^2}{\gamma_{\mathrm{eff}}/n}
\lesssim
d^{\delta-1}.
\]

We now bound the deterministic bias of the effective estimator. Since \(g_d\)
has degree at most \(\ell\), the projection \(P_k g_d\) vanishes for every
\(k>\ell\). Orthogonality of the Walsh degree decomposition and the preceding
shrinkage bounds give
\begin{align}
\label{eq:krr_effective_bias}
\big\|
\hat f^{\mathrm{eff}}_{\gamma_{\mathrm{eff}}}
-
g_{d,\le p}
\big\|_{L^2(\tau_d)}^2
&=
\sum_{k=0}^p
(1-s_{d,k})^2
\|P_k g_d\|_{L^2(\tau_d)}^2
+
\sum_{k=p+1}^{\ell}
s_{d,k}^2
\|P_k g_d\|_{L^2(\tau_d)}^2
\notag\\
&\lesssim
d^{-2\delta}
\sum_{k=0}^p
\|P_k g_d\|_{L^2(\tau_d)}^2
+
d^{2\delta-2}
\sum_{k=p+1}^{\ell}
\|P_k g_d\|_{L^2(\tau_d)}^2
\notag\\
&\le
\bigl(d^{-2\delta}+d^{2\delta-2}\bigr)
\|g_d\|_{L^2(\tau_d)}^2 .
\end{align}
Since \(0<\delta<1\), both \(d^{-2\delta}\) and \(d^{2\delta-2}\) converge to
zero. Thus \eqref{eq:krr_effective_bias} implies
\begin{equation}
\label{eq:krr_effective_bias_small}
\big\|
\hat f^{\mathrm{eff}}_{\gamma_{\mathrm{eff}}}
-
g_{d,\le p}
\big\|_{L^2(\tau_d)}^2
=
o_d(1)\,
\|g_d\|_{L^2(\tau_d)}^2 .
\end{equation}

It remains to compare the empirical KRR estimator with the effective population
estimator. Theorem~4 of~\cite{mei2022generalization}, together with the
identification \(P_{>m}g_d=g_{d,>p}\), gives the stochastic comparison
\begin{equation}
\label{eq:krr_empirical_effective_comparison}
\big\|
\hat f_1
-
\hat f^{\mathrm{eff}}_{\gamma_{\mathrm{eff}}}
\big\|_{L^2(\tau_d)}^2
=
o_{d,\mathbb P}(1)
\left(
\|g_d\|_{L^2(\tau_d)}^2
+
\|g_{d,>p}\|_{L^{2+\eta}(\tau_d)}^2
+
\sigma_\vepsilon^2
\right).
\end{equation}

We control the \(L^{2+\eta}\)-term using the bounded degree of \(g_d\). Since
\(\deg(g_d)\le\ell\), the function \(g_{d,>p}\) also has degree at most
\(\ell\). Set \(q_0:=\lceil 2+\eta\rceil\). Monotonicity of \(L^q\)-norms on a
probability space and the Boolean hypercontractivity estimate from Lemma~18 of
\cite{mei2022generalization} give
\[
\|g_{d,>p}\|_{L^{2+\eta}(\tau_d)}^2
\le
\|g_{d,>p}\|_{L^{q_0}(\tau_d)}^2
\le
(q_0-1)^\ell
\|g_{d,>p}\|_{L^2(\tau_d)}^2 .
\]
Since \(g_{d,>p}\) is an orthogonal projection of \(g_d\), the preceding
estimate implies
\[
\|g_{d,>p}\|_{L^{2+\eta}(\tau_d)}^2
\lesssim
\|g_d\|_{L^2(\tau_d)}^2 .
\]
Substituting this bound into
\eqref{eq:krr_empirical_effective_comparison} yields
\begin{equation}
\label{eq:krr_empirical_effective_simplified}
\big\|
\hat f_1
-
\hat f^{\mathrm{eff}}_{\gamma_{\mathrm{eff}}}
\big\|_{L^2(\tau_d)}^2
=
o_{d,\mathbb P}(1)
\left(
\|g_d\|_{L^2(\tau_d)}^2
+
\sigma_\vepsilon^2
\right).
\end{equation}

Finally, we combine the empirical-to-effective error with the deterministic
effective bias. The elementary inequality
\(\|a+b\|_2^2\le 2\|a\|_2^2+2\|b\|_2^2\), applied in \(L^2(\tau_d)\) with
\[
a=
\hat f_1-\hat f^{\mathrm{eff}}_{\gamma_{\mathrm{eff}}},
\qquad
b=
\hat f^{\mathrm{eff}}_{\gamma_{\mathrm{eff}}}-g_{d,\le p},
\]
gives
\begin{align}
\label{eq:krr_final_combination}
\|\hat f_1-g_{d,\le p}\|_{L^2(\tau_d)}^2
&\le
2\big\|
\hat f_1
-
\hat f^{\mathrm{eff}}_{\gamma_{\mathrm{eff}}}
\big\|_{L^2(\tau_d)}^2
+
2\big\|
\hat f^{\mathrm{eff}}_{\gamma_{\mathrm{eff}}}
-
g_{d,\le p}
\big\|_{L^2(\tau_d)}^2
\notag\\
&=
o_{d,\mathbb P}(1)
\left(
\|g_d\|_{L^2(\tau_d)}^2
+
\sigma_\vepsilon^2
\right)
+
o_d(1)\,
\|g_d\|_{L^2(\tau_d)}^2
\notag\\
&=
o_{d,\mathbb P}(1)
\left(
\|g_d\|_{L^2(\tau_d)}^2
+
\sigma_\vepsilon^2
\right).
\end{align}
Since \(g_{d,\le p}=\mathcal H(f_U^*)_{\le p}\), the estimate
\eqref{eq:krr_final_combination} is the desired claim.

\section{Proof of Proposition~\ref{prop:project_data}}\label{proof:project_data}

Let
\[
B:=U^\top \Sigma_p U,
\qquad
\Delta_d:=\snorm{\widehat M_1-B},
\qquad
m_d:=\max\{d^{-\delta/2},\,d^{-1+\delta}\}.
\]
By construction,
\[
M_2=\frac{d}{c_\eta}\,(\widehat M_1+\eta I_d),
\qquad
c_\eta=\trace(\widehat M_1+\eta I_d).
\]

We first identify the deterministic approximation term. Let
\[
Q:=
\begin{bmatrix}
U\\
U_\perp
\end{bmatrix}
\in O(d).
\]
Since \(U_\perp\) spans \(\mathrm{row}(U_\perp)\), we have
\[
I_d = U^\top U + (U_\perp)^\top U_\perp.
\]
Using this decomposition, we may write
\[
B+\eta I_d
=
U^\top (\Sigma_p+\eta I_r)U
+
\eta (U_\perp)^\top U_\perp
=
Q^\top
\begin{bmatrix}
\Sigma_p+\eta I_r & 0\\
0 & \eta I_{d-r}
\end{bmatrix}
Q.
\]
Taking principal square roots yields
\[
\sqrt{B+\eta I_d}
=
Q^\top
\begin{bmatrix}
\sqrt{\Sigma_p+\eta I_r} & 0\\
0 & \sqrt{\eta}\,I_{d-r}
\end{bmatrix}
Q.
\]
Therefore, by the definition of \(\widehat x\), we have
\[
\widehat x = \sqrt{B+\eta I_d}\,x.
\]
This identity allows us to rewrite the approximation error as
\begin{align*}
\bigl\|M_2^{1/2}x-\sqrt{\tfrac{d}{c_\eta}}\,\widehat x\bigr\|_{L_2}
&=
\sqrt{\frac{d}{c_\eta}}\,
\bigl\|
\bigl(\sqrt{\widehat M_1+\eta I_d}-\sqrt{B+\eta I_d}\bigr)x
\bigr\|_{L_2} \\
&\le
\sqrt{\frac{d}{c_\eta}}\,
\snorm{\sqrt{\widehat M_1+\eta I_d}-\sqrt{B+\eta I_d}}
\,\|x\|_{L_2}.
\end{align*}

We next control the square-root perturbation term. Since both
\(\widehat M_1+\eta I_d\) and \(B+\eta I_d\) are bounded below by \(\eta I_d\), the square-root map is operator-Lipschitz on \([\eta,\infty)\), which gives
\[
\snorm{\sqrt{\widehat M_1+\eta I_d}-\sqrt{B+\eta I_d}}
\le
\frac{1}{2\sqrt{\eta}}\,
\snorm{\widehat M_1-B}
=
\frac{\Delta_d}{2\sqrt{\eta}}.
\]
Substituting this bound into the previous display, we obtain
\begin{align}\label{eq:project_data_short_intermediate}
\bigl\|M_2^{1/2}x-\sqrt{\tfrac{d}{c_\eta}}\,\widehat x\bigr\|_{L_2}
\le
\frac{\sqrt d}{2\sqrt{c_\eta\eta}}\,
\Delta_d\,\|x\|_{L_2}.
\end{align}

It remains to identify the scale of \(c_\eta\). Expanding the trace, we have
\[
c_\eta
=
\eta d+\trace(B)+\trace(\widehat M_1-B).
\]
Since \(UU^\top=I_r\), this gives
\[
\trace(B)=\trace(U^\top\Sigma_p U)=\trace(\Sigma_p),
\]
and therefore we have
\[
|c_\eta-\eta d|
\le
\trace(\Sigma_p)+d\,\Delta_d.
\]
Now Corollary~\ref{cor:main_result}, applied with \(\epsilon=\zeta/2\), yields
\[
\Delta_d = O_{d,\P}\!\bigl(d^{\zeta/2} m_d\bigr).
\]
Because \(\eta=d^\zeta m_d\), we have
\[
d\,\Delta_d
=
O_{d,\P}\!\bigl(d^{-\zeta/2}\eta d\bigr)
=
o_{d,\P}(\eta d).
\]
Moreover, the following holds
\[
\eta d \ge d^\zeta \cdot d^{-1+\delta}\cdot d = d^{\zeta+\delta}\to\infty,
\]
while \(\trace(\Sigma_p)=O_d(1)\). Hence \(\trace(\Sigma_p)=o(\eta d)\), and we conclude that
\[
c_\eta=(1+o_{d,\P}(1))\eta d.
\]

Finally, inserting this asymptotic into \eqref{eq:project_data_short_intermediate} yields
\[
\frac{\sqrt d}{\sqrt{c_\eta\eta}}
=
(1+o_{d,\P}(1))\frac{1}{\eta},
\]
and therefore we obtain
\[
\bigl\|M_2^{1/2}x-\sqrt{\tfrac{d}{c_\eta}}\,\widehat x\bigr\|_{L_2}
\le
(1+o_{d,\P}(1))\frac{\Delta_d}{2\eta}\,\|x\|_{L_2}.
\]
Using again \(\Delta_d = O_{d,\P}(d^{\zeta/2}m_d)\) and \(\eta=d^\zeta m_d\), we obtain
\[
\frac{\Delta_d}{\eta}=O_{d,\P}(d^{-\zeta/2}),
\]
which gives
\[
\bigl\|M_2^{1/2}x-\sqrt{\tfrac{d}{c_\eta}}\,\widehat x\bigr\|_{L_2}
=
O_{d,\P}(d^{-\zeta/2})\,\|x\|_{L_2}.
\]
This is exactly \eqref{eq:normalized_sqrt_approx}.

\section{Proof for preliminaries}
\subsection{Proof of Lemma~\ref{lemma:kernel_to_mono_general}}\label{proof:kernel_to_mono_general}
We begin with the polynomial truncation. Applying Theorem~\ref{thm:poly_approx_cube} with truncation order $2m+1$, we obtain
\begin{align}\label{eq:mono_taylor_approx_general1}
    \snorm{K-K_{2m+1}-\round{g(1)-g_{2m+1}(1)}I_n}
    =O_{d,\P}\round{\frac{\log^{m+1}(n)}{d^{m+1-p-\delta}}}.
\end{align}
We next expand $K_{2m+1}$ in the Fourier basis. By definition of the truncated kernel, we have
\[
K_{2m+1}
=
\sum_{\ell=0}^{2m+1} g^{(\ell)}(0)\cdot
\frac{1}{\ell!}\round{\frac{XX\tran}{d}}^{\odot \ell}.
\]
Applying the refined identity \eqref{eq:x_power_refined} term by term, this gives
\begin{align*}
    K_{2m+1}
    =
    \sum_{\ell=0}^{2m+1} g^{(\ell)}(0)
    \round{
        d^{-\ell}\Phi_{\S_\ell}\Phi_{\S_\ell}\tran
        +
        \sum_{\substack{j:\,0\le j<\ell,\\ \ell-j~\mathrm{is~even}}}
        \widetilde d_j^{(\ell)}\,\Phi_{\S_j}\Phi_{\S_j}\tran
    }.
\end{align*}
Collecting together the terms with the same Fourier degree, we may therefore write
\begin{align}\label{eq:K_expression_split}
    K_{2m+1}
    =
    \Phi_{\le p}D\Phi_{\le p}\tran
    +
    \sum_{k=p+1}^{2m+1}\theta_k\,\Phi_{\S_k}\Phi_{\S_k}\tran,
\end{align}
where, for each $j=0,\ldots,p$,
\[
D_{\S_j}
:=
\round{
\frac{g^{(j)}(0)}{d^j}
+
\sum_{\substack{\ell:\,j<\ell\le 2m+1,\\ \ell-j~\mathrm{is~even}}}
g^{(\ell)}(0)\,\widetilde d_j^{(\ell)}
}\,I_{|\S_j|},
\]
and, for each $k=p+1,\ldots,2m+1$,
\[
\theta_k
:=
\frac{g^{(k)}(0)}{d^k}
+
\sum_{\substack{\ell:\,k<\ell\le 2m+1,\\ \ell-k~\mathrm{is~even}}}
g^{(\ell)}(0)\,\widetilde d_k^{(\ell)}.
\]

These formulas immediately yield the required coefficient bounds. Indeed, since
$\widetilde d_j^{(\ell)}=O_d(d^{-(j+\ell)/2})$ and $\ell-j$ is a positive even integer in the correction terms, we necessarily have $\ell\ge j+2$. It follows that
\[
\widetilde d_j^{(\ell)}=O_d(d^{-j-1}).
\]
Because the number of indices $\ell$ is bounded in terms of $m$, the whole correction sum is still of order $O_d(d^{-j-1})$. Hence, for $j=0,\ldots,p$, we obtain
\[
\|D_{\S_j}-g^{(j)}(0)d^{-j}I_{|\S_j|}\|_{\rm op}
=
O_d(d^{-j-1}),
\]
while for $k=p+1,\ldots,2m+1$, we have
\[
\theta_k
=
\frac{g^{(k)}(0)}{d^k}
+
O_d(d^{-k-1}).
\]
This proves the coefficient estimates in the statement.

\subsection{Proof of Lemma~\ref{lem:degree-m-reduction}}
\label{proof:degree-m-reduction}

Let \(X\sim\tau_d\). For each \(S\subseteq[d]\) with \(|S|=m\), set
\(a_S:=\E[z^\alpha(X)X^S]\). The degree-\(m\) Walsh component of
\(\mathcal H(z^\alpha)\) is
\[
[\mathcal H(z^\alpha)]_{\deg m}(x)
=
\sum_{|S|=m} a_S x^S .
\]
For \(\lambda\in\Lambda_m(\alpha)\), write
\[
H_\lambda(x)
=
\sum_{|S|=m} h_{\lambda,S}x^S,
\qquad
h_{\lambda,S}:=\E[z^\lambda(X)X^S].
\]

Fix \(S\subseteq[d]\) with \(|S|=m\). Expanding
\(z^\alpha=\prod_{j=1}^r z_j^{\alpha_j}\) amounts to assigning each slot
corresponding to direction \(j\) to a coordinate \(i\in[d]\). Such an
assignment contributes to \(a_S\) precisely when, for each coordinate \(i\), the
total multiplicity assigned to \(i\) has parity \(\mathbf 1_{\{i\in S\}}\).

We first count the principal assignments. Fix
\(\lambda\in\mathcal A_m(\alpha)\), and put
\(\nu=\nu(\alpha,\lambda)=(\alpha-\lambda)/2\). A contributing assignment is
called principal of type \(\lambda\) if it has the following structure: each
coordinate \(i\in S\) is hit exactly once; among these singleton hits, exactly
\(\lambda_j\) come from direction \(j\); the remaining \(2\nu_j\) slots of
direction \(j\) are grouped into \(\nu_j\) same-direction pairs; and the
pair-support coordinates are distinct and lie in \(S^c\).

Let \(G_{\lambda,S}\) denote the total contribution of the principal
assignments of type \(\lambda\). A direct count gives the identity
\begin{equation}
\label{eq:G-lambda-S}
G_{\lambda,S}
=
\mathsf A_{\alpha,\lambda}\,
h_{\lambda,S}\,
C_{\nu(\alpha,\lambda)}(S^c).
\end{equation}
The factor \(h_{\lambda,S}\) accounts for the singleton placements on \(S\),
while \(C_{\nu(\alpha,\lambda)}(S^c)\) encodes the pair-support coordinates in
\(S^c\). For direction \(j\), the slot combinatorics contribute the factor
\[
\binom{\alpha_j}{\lambda_j}
\frac{(2\nu_j)!}{2^{\nu_j}\nu_j!}\nu_j!
=
\frac{\alpha_j!}{\lambda_j!\,2^{\nu_j}} .
\]
This is the \(j\)-th factor in
\[
\mathsf A_{\alpha,\lambda}
=
\prod_{j=1}^r
\frac{\alpha_j!}{\lambda_j!\,2^{\nu_j}} .
\]

Summing over all admissible \(\lambda\), the coefficient \(a_S\) decomposes as
\[
a_S
=
\sum_{\lambda\in\mathcal A_m(\alpha)} G_{\lambda,S}
+
B_S,
\]
where \(B_S\) denotes the total contribution of the non-principal assignments.

We next replace \(C_\nu(S^c)\) by \(C_\nu\). Since all coefficients in the
generating function defining \(C_\nu(T)\) are nonnegative, the difference
\(C_\nu-C_\nu(S^c)\) is obtained by forcing at least one pair-support
coordinate to lie in \(S\). Since \(|\nu|\le\ell\), this difference satisfies
\[
|C_\nu-C_\nu(S^c)|
\le
C_\ell\sum_{i\in S} q_i .
\]
Because \(|S|=m\le\ell\), the preceding bound is at most
\(C_\ell m\max_i q_i\le C_\ell\Delta_U\). Using
\eqref{eq:G-lambda-S}, the finiteness of \(\mathcal A_m(\alpha)\), and the
bound \(\|H_\lambda\|_{L_2}\le \|z^\lambda\|_{L_2}\le C_\ell\), the omission
error satisfies
\begin{equation}
\label{eq:principal-omission-l2}
\sum_{|S|=m}
\left|
\sum_{\lambda\in\mathcal A_m(\alpha)}
\mathsf A_{\alpha,\lambda}
\bigl(C_{\nu(\alpha,\lambda)}(S^c)-C_{\nu(\alpha,\lambda)}\bigr)
h_{\lambda,S}
\right|^2
\le
C_\ell\Delta_U^2 .
\end{equation}

It remains to bound the non-principal assignments. We first record an
elementary coefficient estimate. For a fixed ordered label list
\(\sigma=(\sigma_1,\dots,\sigma_m)\in[r]^m\), define
\[
K_{\sigma,S}
:=
\sum_{\psi:\,[m]\to S\ \mathrm{bijection}}
\prod_{a=1}^m
\bigl|u_{\psi(a)}^{[\sigma_a]}\bigr| .
\]
The following bound will be used repeatedly:
\begin{equation}
\label{eq:K-sigma-S-bound}
\sum_{|S|=m} K_{\sigma,S}^2
\le
C_\ell .
\end{equation}
Indeed, \(K_{\sigma,S}\) is the absolute-value analogue of a degree-\(m\)
Walsh coefficient of a product of \(m\) linear forms. Since the degree-\(m\)
projection is an \(L_2(\tau_d)\)-contraction, Hölder's inequality and the
Khintchine inequality give
\[
\left(\sum_{|S|=m} K_{\sigma,S}^2\right)^{1/2}
\le
\left\|
\prod_{a=1}^m
\left(\sum_{i=1}^d |u_i^{[\sigma_a]}|X_i\right)
\right\|_{L_2}
\le
\prod_{a=1}^m
\left\|
\sum_{i=1}^d |u_i^{[\sigma_a]}|X_i
\right\|_{L_{2m}}
\le
C_\ell .
\]
This proves \eqref{eq:K-sigma-S-bound}.

We now decompose the non-principal assignments. We overcount them by making
two bounded-complexity choices. First, for each coordinate \(i\in S\), choose
one distinguished hit among the odd number of hits at \(i\). Second, choose
one distinguished witness of non-principality. Since \(|\alpha|\le\ell\), the
number of such choices is bounded by \(C_\ell\).

The distinguished singleton hits on \(S\) have an ordered label list
\(\sigma=(\sigma_1,\dots,\sigma_m)\in[r]^m\), and their total absolute
contribution is bounded by \(K_{\sigma,S}\). After these distinguished
singleton hits are fixed, all remaining multiplicities are even. We impose the
following priority rule for the remaining defects.
\begin{enumerate}
    \item First, consider assignments for which at least one remaining block touches
\(S\). Choose such a block as the distinguished witness. Let
\(j_1,\dots,j_t\), with \(t\ge2\), be its fixed labels. For a support
coordinate \(i\in S\), this block contributes at most
\[
\prod_{a=1}^t |u_i^{[j_a]}|
\le
q_i^{t/2}
\le
q_i,
\]
where the last inequality uses \(q_i\le1\). Summing over the possible support
coordinate in \(S\) gives the bound
\[
\sum_{i\in S} q_i
\le
m\max_i q_i
\le
C_\ell\Delta_U .
\]
\item 
Second, consider the remaining assignments for which no remaining block
touches \(S\), but some coordinate in \(S^c\) carries remaining multiplicity at
least \(4\). Choose such a block as the distinguished witness. If its fixed
labels are \(j_1,\dots,j_t\), with \(t\ge4\), then its total absolute
contribution is at most
\[
\sum_{i=1}^d
\prod_{a=1}^t |u_i^{[j_a]}|
\le
\sum_{i=1}^d q_i^{t/2}
\le
\sum_{i=1}^d q_i^2
=
\rho
\le
\Delta_U .
\]
\item 

Third, consider the remaining assignments for which no remaining block touches
\(S\), and no coordinate in \(S^c\) carries remaining multiplicity at least
\(4\). At this stage, every remaining block is a \(2\)-block in \(S^c\), and
the supports of these \(2\)-blocks are distinct. If the assignment is still
non-principal, then at least one such \(2\)-block is mixed-direction. Choose a
mixed block with labels \(j\neq k\) as the distinguished witness. After all
other remaining blocks have been fixed, the allowed support coordinates for
this mixed block are \(i\notin T\), where \(T\) contains \(S\) and the supports
of the other remaining blocks. In particular, \(|T|\le C_\ell\). The
orthogonality of the rows of \(U\) gives the identity
\[
\sum_{i\notin T} u_i^{[j]}u_i^{[k]}
=
-\sum_{i\in T} u_i^{[j]}u_i^{[k]} .
\]
This identity bounds the absolute contribution of the mixed block by
\[
\left|
\sum_{i\notin T} u_i^{[j]}u_i^{[k]}
\right|
\le
\sum_{i\in T} q_i
\le
C_\ell\max_i q_i
\le
C_\ell\Delta_U .
\]
\end{enumerate}

It remains to control the non-distinguished even blocks. For a fixed block
with labels \(j_1,\dots,j_t\), where \(t\ge2\), Hölder's inequality gives
\[
\sum_{i=1}^d
\prod_{a=1}^t |u_i^{[j_a]}|
\le
\prod_{a=1}^t \|u^{[j_a]}\|_t
\le
\prod_{a=1}^t \|u^{[j_a]}\|_2
=
1 .
\]
Restrictions on the allowed support coordinates can only decrease this
absolute sum. Since the number of remaining blocks is bounded by \(\ell\), all
non-distinguished even blocks together contribute at most \(C_\ell\).

The three prioritized defect estimates give the pointwise bound
\[
|B_S|
\le
C_\ell\Delta_U
\sum_{\sigma\in\mathcal S_{\alpha,m}} K_{\sigma,S},
\]
where \(\mathcal S_{\alpha,m}\) is a finite set of ordered label lists with
cardinality bounded by \(C_\ell\). Squaring, summing over \(S\), and using
\eqref{eq:K-sigma-S-bound} gives
\begin{equation}
\label{eq:BS-bound-clean}
\sum_{|S|=m} |B_S|^2
\le
C_\ell\Delta_U^2 .
\end{equation}

Finally, define the coefficient error by
\[
E_S
:=
a_S
-
\sum_{\lambda\in\mathcal A_m(\alpha)}
\mathsf A_{\alpha,\lambda}
C_{\nu(\alpha,\lambda)}
h_{\lambda,S}.
\]
The estimates \eqref{eq:principal-omission-l2} and
\eqref{eq:BS-bound-clean} imply
\[
\sum_{|S|=m} |E_S|^2
\le
C_\ell\Delta_U^2 .
\]
By Parseval's identity, the left-hand side is exactly the squared
\(L_2\)-error of the degree-\(m\) approximation. Taking square roots proves the
lemma.

\subsection{Proof of Lemma~\ref{lem:Cnu-hermite}}
\label{proof:Cmu-hermite}

Define \(L_i(w):=\sum_{j=1}^r (u_i^{[j]})^2w_j\). The generating function
factorizes as
\[
\prod_{i=1}^d (1+L_i(w))
=
\exp\left(\sum_{i=1}^d L_i(w)\right)\exp(R(w)),
\qquad
R(w):=
\sum_{i=1}^d \bigl(\log(1+L_i(w))-L_i(w)\bigr).
\]
Since \(\sum_i (u_i^{[j]})^2=1\) for every \(j\), the linear term satisfies
\(\sum_{i=1}^d L_i(w)=\sum_{j=1}^r w_j\). Thus the coefficient of \(w^\nu\) in
the first exponential factor is
\[
[w^\nu]\exp\left(\sum_{i=1}^d L_i(w)\right)
=
[w^\nu]\exp\left(\sum_{j=1}^r w_j\right)
=
\frac{1}{\nu!}.
\]

We first consider the case \(\Delta_U\ge1\). Since
\(1+L_i(w)\le\exp(L_i(w))\) coefficientwise, the coefficient \(C_\nu\)
satisfies
\[
0\le C_\nu
\le
[w^\nu]\exp\left(\sum_{i=1}^d L_i(w)\right)
=
\frac{1}{\nu!}.
\]
In this case, the desired error bound follows from
\[
\left|C_\nu-\frac{1}{\nu!}\right|
\le
\frac{1}{\nu!}
\le
C_\ell\Delta_U .
\]

It remains to consider the case \(\Delta_U<1\). In this case,
\(\rho\le\Delta_U<1\). The series \(R(w)\) has no constant or linear terms.
Therefore, every coefficient of \(R(w)\) up to total degree \(|\nu|\le\ell\)
is a finite linear combination, with coefficients depending only on \(\ell\),
of coefficients of \(L_i(w)^m\), where \(2\le m\le|\nu|\). For such \(m\), the
coefficient bound
\[
|[w^\omega]L_i(w)^m|
\le
C_\ell q_i^m
\le
C_\ell q_i^2
\]
holds because \(0\le q_i\le1\). After summing over \(i\), every relevant
coefficient of \(R(w)\) is bounded by
\[
C_\ell\sum_{i=1}^d q_i^2
=
C_\ell\rho .
\]
Since \(\rho<1\) and \(|\nu|\le\ell\), every coefficient of
\(\exp(R(w))-1\) up to total degree \(|\nu|\) is also \(O_\ell(\rho)\).
Multiplication by \(\exp(\sum_j w_j)\), whose coefficients up to degree
\(\ell\) are bounded by \(C_\ell\), gives
\[
\left|C_\nu-\frac{1}{\nu!}\right|
\le
C_\ell\rho
\le
C_\ell\Delta_U .
\]
This proves the lemma.

\subsection{Proof of Theorem~\ref{thm:hypercube_linear_clean}}
\label{proof:hypercube_linear_clean}
Write the monomial expansion of \(h\) as
\[
h(z)=\sum_{\alpha\in\Lambda_b}b_\alpha z^\alpha.
\]
We first record the coefficient comparison
\begin{equation}
\label{eq:explicit-coeff-norm}
\sum_{|\alpha|\le\ell}|b_\alpha|
+
\sum_{|\lambda|\le\ell}|a_\lambda|
\le
C_\ell N_{r,\ell}^{1/2}\|h\|_{L_2(\gamma_r)}.
\end{equation}
Indeed, Hermite orthogonality gives
\[
\sum_{|\lambda|\le\ell}|a_\lambda|
\le
N_{r,\ell}^{1/2}
\left(\sum_{|\lambda|\le\ell}a_\lambda^2\lambda!\right)^{1/2}.
\]
Using \eqref{eq:h-gaussian-L2-hermite}, the preceding display becomes
\[
\sum_{|\lambda|\le\ell}|a_\lambda|
\le
N_{r,\ell}^{1/2}\|h\|_{L_2(\gamma_r)}.
\]
The same bound for the monomial coefficients follows by expanding each
\(\He_\lambda\) into monomials. Since \(|\lambda|\le\ell\), the
\(\ell_1\)-norm of the monomial coefficient vector of \(\He_\lambda\) is
bounded by \(C_\ell\). This proves \eqref{eq:explicit-coeff-norm}.

Applying Corollary~\ref{cor:monomial-to-H} to each monomial in \(h\), and then
using \eqref{eq:explicit-coeff-norm}, gives
\[
\mathcal H(f^*)
=
\sum_{|\lambda|\le\ell} c_\lambda H_\lambda
+
E,
\]
where
\[
\|E\|_{L_2(\tau_d)}
\le
C_\ell N_{r,\ell}^{1/2}
\|h\|_{L_2(\gamma_r)}
\Delta_U.
\]
Here the coefficients \(c_\lambda\) are given by
\[
c_\lambda
=
\sum_{\substack{\alpha\in\Lambda_b\\ \lambda\in\mathcal A(\alpha)}}
b_\alpha
\mathsf A_{\alpha,\lambda}
C_{\nu(\alpha,\lambda)}.
\]

On the other hand, the Hermite coefficient formula
\eqref{eq:hermite-coeff-from-monomials} gives
\[
a_\lambda
=
\sum_{\substack{\alpha\in\Lambda_b\\ \lambda\in\mathcal A(\alpha)}}
b_\alpha
\mathsf A_{\alpha,\lambda}
\frac1{\nu(\alpha,\lambda)!}.
\]
Subtracting the two coefficient formulas and applying
Lemma~\ref{lem:Cnu-hermite} yields
\[
|c_\lambda-a_\lambda|
\le
C_\ell N_{r,\ell}^{1/2}
\|h\|_{L_2(\gamma_r)}
\Delta_U.
\]
Moreover, the layer norm satisfies
\[
\|H_\lambda\|_{L_2(\tau_d)}
\le
\|z^\lambda\|_{L_2(\tau_d)}
\le
C_\ell,
\]
where the last inequality follows from Hölder's inequality and the Khintchine
inequality. Since there are at most \(N_{r,\ell}\) indices \(\lambda\), we get
\[
\left\|
\sum_{|\lambda|\le\ell}(c_\lambda-a_\lambda)H_\lambda
\right\|_{L_2(\tau_d)}
\le
C_\ell N_{r,\ell}^{3/2}
\|h\|_{L_2(\gamma_r)}
\Delta_U.
\]
Since \(N_{r,\ell}^{3/2}\le N_{r,\ell}^2\), the preceding display is bounded
by
\(
C_\ell
\|h\|_{L_2(\gamma_r)}
\Theta_U.
\)
Combining this estimate with the bound on \(E\), and using
\eqref{eq:fstar-h-gaussian-L2}, proves
\[
\|P-F\|_{L_2(\tau_d)}
\le
C_\ell
\|f^*\|_{L_2(\gamma_d)}
\Theta_U.
\]
This completes the proof.

\subsection{Proof of Proposition~\ref{prop:H-orth}}
\label{proof:H-orth}

The cases \(m=0\) and \(m=1\) are immediate, so assume \(m\ge2\). Choose
ordered label lists \(\sigma_1,\dots,\sigma_m\in[r]\) and
\(\tau_1,\dots,\tau_m\in[r]\) such that
\[
\#\{p:\sigma_p=j\}=\alpha_j,
\qquad
\#\{p:\tau_p=j\}=\beta_j.
\]
With this choice, the two monomials can be written as
\[
z^\alpha=\prod_{p=1}^m z_{\sigma_p},
\qquad
z^\beta=\prod_{p=1}^m z_{\tau_p}.
\]
Expanding each linear form shows that the top-degree layers keep exactly the
injective terms. Thus
\[
H_\alpha(x)
=
\sum_{\substack{i_1,\dots,i_m\in[d]\\\text{all distinct}}}
\Bigl(\prod_{p=1}^m u_{i_p}^{[\sigma_p]}\Bigr)
x_{i_1}\cdots x_{i_m},
\]
and similarly
\[
H_\beta(x)
=
\sum_{\substack{v_1,\dots,v_m\in[d]\\\text{all distinct}}}
\Bigl(\prod_{p=1}^m u_{v_p}^{[\tau_p]}\Bigr)
x_{v_1}\cdots x_{v_m}.
\]

Let \(X\sim\tau_d\). The expectation is nonzero precisely when the two
injective tuples have the same underlying coordinate set. Hence we have
\[
\langle H_\alpha,H_\beta\rangle
=
\sum_{\substack{i_1,\dots,i_m\\\text{all distinct}}}
\sum_{\pi\in S_m}
\prod_{p=1}^m
u_{i_p}^{[\sigma_p]}
u_{i_p}^{[\tau_{\pi^{-1}(p)}]}.
\]

We compare this restricted sum with the unrestricted sum
\[
M
:=
\sum_{i_1,\dots,i_m\in[d]}
\sum_{\pi\in S_m}
\prod_{p=1}^m
u_{i_p}^{[\sigma_p]}
u_{i_p}^{[\tau_{\pi^{-1}(p)}]}.
\]
For fixed \(\pi\), the sum over \((i_1,\dots,i_m)\) factorizes as
\[
\sum_{i_1,\dots,i_m}
\prod_{p=1}^m
u_{i_p}^{[\sigma_p]}
u_{i_p}^{[\tau_{\pi^{-1}(p)}]}
=
\prod_{p=1}^m
\langle u^{[\sigma_p]},u^{[\tau_{\pi^{-1}(p)}]}\rangle.
\]
By orthonormality, this product equals \(1\) exactly when
\(\sigma_p=\tau_{\pi^{-1}(p)}\) for every \(p\), and otherwise it equals
\(0\). Such permutations exist if and only if \(\alpha=\beta\). In that case,
their number is
\[
\prod_{j=1}^r \alpha_j!
=
\alpha!.
\]
Therefore, we get
\[
M
=
\mathbf 1_{\{\alpha=\beta\}}\alpha!.
\]

It remains to control the contribution of tuples with collisions. Since
\(\langle H_\alpha,H_\beta\rangle\) is obtained from \(M\) by restricting to
tuples with all entries distinct, we have
\[
M-\langle H_\alpha,H_\beta\rangle
=
\sum_{\pi\in S_m}
\sum_{\substack{i_1,\dots,i_m\\\text{not all distinct}}}
\prod_{p=1}^m
u_{i_p}^{[\sigma_p]}
u_{i_p}^{[\tau_{\pi^{-1}(p)}]}.
\]
Using the elementary bound
\[
\mathbf 1_{\{\text{not all distinct}\}}
\le
\sum_{1\le p<q\le m}\mathbf 1_{\{i_p=i_q\}},
\]
and then taking absolute values, we get
\[
|M-\langle H_\alpha,H_\beta\rangle|
\le
\sum_{\pi\in S_m}
\sum_{1\le p<q\le m}
\sum_{\substack{i_1,\dots,i_m\\ i_p=i_q}}
\prod_{c=1}^m
\big|
u_{i_c}^{[\sigma_c]}
u_{i_c}^{[\tau_{\pi^{-1}(c)}]}
\big|.
\]

Fix \(\pi\in S_m\) and \(p<q\), and write \(i=i_p=i_q\). For any
\(j,k\in[r]\), we have
\[
|u_i^{[j]}u_i^{[k]}|
\le
\frac12\bigl((u_i^{[j]})^2+(u_i^{[k]})^2\bigr)
\le
q_i.
\]
Thus the two constrained positions contribute at most
\(
q_i^2.
\)

For each remaining position, Cauchy--Schwarz gives
\(
\sum_{v=1}^d
|u_v^{[\sigma_c]}u_v^{[\tau_{\pi^{-1}(c)}]}|
\le
1.
\)
Therefore each constrained collision sum is bounded by
\(
\sum_{i=1}^d q_i^2
=
\rho.
\)
Summing over \(\pi\in S_m\) and over all pairs \(p<q\), we obtain
\[
\big|
\langle H_\alpha,H_\beta\rangle
-
\mathbf 1_{\{\alpha=\beta\}}\alpha!
\big|
\le
m!\binom{m}{2}\rho.
\]
Since \(m\le\ell\) and \(\rho\le\Delta_U\), the desired bound follows.

\subsection{Proof of Proposition~\ref{prop:H-contraction}}
\label{proof:H-contraction}

We first prove the active-direction estimate. Fix \(s\in[r]\), and choose an
ordered label list \(\sigma_1,\dots,\sigma_q\in[r]\) such that
\(\#\{p:\sigma_p=j\}=\lambda_j\) for each \(j\in[r]\). The all-distinct layer
\(H_\lambda\) admits the injective expansion
\[
H_\lambda(x)
=
\sum_{\substack{t_1,\dots,t_q\in[d]\\\text{all distinct}}}
\left(\prod_{p=1}^q u_{t_p}^{[\sigma_p]}\right)
x_{t_1}\cdots x_{t_q}.
\]
For a tuple \((t_v)_{v\neq p}\), write \(S=\{t_v:v\neq p\}\). Differentiating
the preceding expansion term by term gives the directional derivative formula
\[
(u^{[s]})^\top\nabla H_\lambda(x)
=
\sum_{p=1}^q
\sum_{\substack{(t_v)_{v\neq p}\\\text{all distinct}}}
\left(\prod_{v\neq p}u_{t_v}^{[\sigma_v]}\right)
x^S
\sum_{t\notin S}
u_t^{[s]}u_t^{[\sigma_p]}.
\]
The orthonormality of the rows of \(U\) gives, for each \(j\in[r]\), the
identity
\[
\sum_{t\notin S}u_t^{[s]}u_t^{[j]}
=
\mathbf 1_{\{s=j\}}
-
\sum_{t\in S}u_t^{[s]}u_t^{[j]}.
\]
Substituting this identity into the derivative formula separates the main term
from the error term. The contribution of \(\mathbf 1_{\{s=\sigma_p\}}\) is
exactly \(\lambda_sH_{\lambda-e_s}\), with the convention that this term is
zero when \(\lambda_s=0\). The remaining terms define \(R_{s,\lambda}\), giving
the decomposition
\[
(u^{[s]})^\top\nabla H_\lambda
=
\lambda_sH_{\lambda-e_s}+R_{s,\lambda}.
\]

It remains to estimate the remainder. For every \((q-1)\)-element set \(S\),
the correction factor satisfies
\[
\left|
\sum_{t\in S}u_t^{[s]}u_t^{[\sigma_p]}
\right|
\le
|S|\max_i q_i
\le
C_\ell\Delta_U.
\]
After grouping the \(p\)-th summand by the underlying set \(S\), its Walsh
coefficients are those of \(H_{\lambda-e_{\sigma_p}}\) multiplied by this
correction factor. Since
\(\|H_{\lambda-e_{\sigma_p}}\|_{L_2}
\le \|z^{\lambda-e_{\sigma_p}}\|_{L_2}\le C_\ell\), Parseval's identity gives
\[
\|R_{s,\lambda}^{(p)}\|_{L_2}
\le
C_\ell\Delta_U.
\]
Summing over \(p\in[q]\), with \(q\le\ell\), yields
\[
\|R_{s,\lambda}\|_{L_2}
\le
C_\ell\Delta_U.
\]

We now prove the inactive-direction estimate. Let \(v\in\mathcal U^\perp\) be
a unit vector. The same injective expansion gives
\[
v^\top\nabla H_\lambda(x)
=
\sum_{p=1}^q
\sum_{\substack{(t_a)_{a\neq p}\\\text{all distinct}}}
\left(\prod_{a\neq p}u_{t_a}^{[\sigma_a]}\right)
x^S
\sum_{t\notin S}
v_tu_t^{[\sigma_p]}.
\]
Since \(v\perp u^{[\sigma_p]}\), the inner sum satisfies
\[
\sum_{t\notin S}v_tu_t^{[\sigma_p]}
=
-\sum_{t\in S}v_tu_t^{[\sigma_p]}.
\]
Thus \(v^\top\nabla H_\lambda\) consists only of remainder terms; write
\(v^\top\nabla H_\lambda=R_{v,\lambda}\).

Let \(c_S^{(p)}\) denote the coefficient of \(x^S\) in
\(H_{\lambda-e_{\sigma_p}}\). By Cauchy--Schwarz, the coefficient multiplying
\(x^S\) in the \(p\)-th remainder satisfies
\[
\left|
\sum_{i\in S}v_i u_i^{[\sigma_p]}
\right|^2
\le
(q-1)\sum_{i\in S}v_i^2(u_i^{[\sigma_p]})^2.
\]
Consequently, Parseval's identity gives the preliminary bound
\begin{equation}
\label{eq:inactive-remainder-preclaim}
\|R_{v,\lambda}^{(p)}\|_{L_2}^2
\le
C_\ell
\sum_{i=1}^d
v_i^2(u_i^{[\sigma_p]})^2
\sum_{S\ni i}|c_S^{(p)}|^2.
\end{equation}

\begin{claim}\label{cl:c_s_p}
For every \(i\in[d]\), the coefficient localization bound holds:
\[
\sum_{S\ni i}|c_S^{(p)}|^2
\le
C_\ell q_i.
\]
\end{claim}

\begin{proof}[Proof of Claim~\ref{cl:c_s_p}]
Set \(\eta:=\lambda-e_{\sigma_p}\). Since \(c_S^{(p)}\) is the coefficient of
\(x^S\) in \(H_\eta\), Parseval's identity gives
\[
\sum_{S\ni i}|c_S^{(p)}|^2
=
\|\partial_i H_\eta\|_{L_2}^2.
\]
Choose an ordered label list \(\rho_1,\dots,\rho_{q-1}\) for \(\eta\). The
injective expansion of \(H_\eta\) is
\[
H_\eta(x)
=
\sum_{\substack{t_1,\dots,t_{q-1}\in[d]\\\text{all distinct}}}
\left(\prod_{a=1}^{q-1}u_{t_a}^{[\rho_a]}\right)
x_{t_1}\cdots x_{t_{q-1}}.
\]
Differentiating this expansion with respect to \(x_i\) exposes one label at a
time. In the term exposing the label \(\rho_a\), the coefficient
\(u_i^{[\rho_a]}\) appears, and the remaining factor is an all-distinct layer
of degree \(q-2\) with coordinate \(i\) excluded from its support. The same
Khintchine and Hölder bounds used above show that this remaining layer has
\(L_2(\tau_d)\)-norm at most \(C_\ell\). Therefore the derivative norm satisfies
\[
\|\partial_iH_\eta\|_{L_2}
\le
C_\ell
\sum_{a=1}^{q-1}|u_i^{[\rho_a]}|.
\]
Since \(q\le\ell\), Cauchy--Schwarz gives the bound
\[
\|\partial_iH_\eta\|_{L_2}^2
\le
C_\ell
\sum_{a=1}^{q-1}(u_i^{[\rho_a]})^2.
\]
The exposed labels are among the relevant directions, so
\[
\sum_{a=1}^{q-1}(u_i^{[\rho_a]})^2
\le
C_\ell\sum_{j=1}^r (u_i^{[j]})^2
=
C_\ell q_i.
\]
Combining the preceding two estimates proves the claim.
\end{proof}

Substituting Claim~\ref{cl:c_s_p} into
\eqref{eq:inactive-remainder-preclaim} gives
\[
\|R_{v,\lambda}^{(p)}\|_{L_2}^2
\le
C_\ell
\sum_{i=1}^d
v_i^2(u_i^{[\sigma_p]})^2q_i.
\]
Since \((u_i^{[\sigma_p]})^2\le q_i\), the last sum satisfies
\[
\sum_{i=1}^d
v_i^2(u_i^{[\sigma_p]})^2q_i
\le
\sum_{i=1}^d v_i^2 q_i^2
\le
(\max_i q_i)^2.
\]
Using \(\max_i q_i\le\Delta_U\), we obtain
\[
\|R_{v,\lambda}^{(p)}\|_{L_2}^2
\le
C_\ell\Delta_U^2.
\]
Equivalently, the \(p\)-th remainder satisfies
\[
\|R_{v,\lambda}^{(p)}\|_{L_2}
\le
C_\ell\Delta_U.
\]
Summing over \(p\in[q]\), with \(q\le\ell\), gives
\[
\|R_{v,\lambda}\|_{L_2}
\le
C_\ell\Delta_U.
\]
This completes the proof.

\subsection{Proof of Lemma~\ref{lem:gaussian-hypercube-L2-comparison}}
\label{proof:gaussian-hypercube-L2-comparison}

Since \(P=\mathcal H(f^*)\) agrees with \(f^*\) on \(\{\pm1\}^d\), we have
\(\|P\|_{L_2(\tau_d)}=\|f^*\|_{L_2(\tau_d)}\). By
Theorem~\ref{thm:hypercube_linear_clean}, there is a decomposition
\begin{equation}
\label{eq:L2-comparison-surrogate-decomposition}
P=F+E,
\qquad
\|E\|_{L_2(\tau_d)}
\le
C_\ell \|f^*\|_{L_2(\gamma_d)}\Theta_U .
\end{equation}
Using \eqref{eq:fstar-h-gaussian-L2}, the error bound in
\eqref{eq:L2-comparison-surrogate-decomposition} may equivalently be written as
\[
\|E\|_{L_2(\tau_d)}
\le
C_\ell \|h\|_{L_2(\gamma_r)}\Theta_U .
\]

We first compare the surrogate \(F\) with the latent polynomial \(h\). Different
Walsh degrees are exactly orthogonal, and Proposition~\ref{prop:H-orth}
controls the inner products within each fixed degree. Therefore the squared
norms satisfy
\[
\left|
\|F\|_{L_2(\tau_d)}^2
-
\sum_{|\lambda|\le\ell}a_\lambda^2\lambda!
\right|
\le
C_\ell\Delta_U
\left(\sum_{|\lambda|\le\ell}|a_\lambda|\right)^2 .
\]
The Hermite identity
\(\sum_{|\lambda|\le\ell}a_\lambda^2\lambda!
=\|h\|_{L_2(\gamma_r)}^2\), the coefficient bound
\eqref{eq:explicit-coeff-norm}, and the inequality
\(N_{r,\ell}\Delta_U\le \Theta_U\) give the key comparison
\begin{equation}
\label{eq:L2-comparison-F-h}
\left|
\|F\|_{L_2(\tau_d)}^2
-
\|h\|_{L_2(\gamma_r)}^2
\right|
\le
C_\ell\Theta_U\|h\|_{L_2(\gamma_r)}^2 .
\end{equation}
Since \(\Theta_U\le1\), the comparison \eqref{eq:L2-comparison-F-h} also
implies
\[
\|F\|_{L_2(\tau_d)}
\le
C_\ell\|h\|_{L_2(\gamma_r)} .
\]

We now transfer the comparison from \(F\) to \(P\). The decomposition
\eqref{eq:L2-comparison-surrogate-decomposition} gives the deterministic bound
\begin{align}
\label{eq:L2-comparison-P-F}
\left|
\|P\|_{L_2(\tau_d)}^2
-
\|F\|_{L_2(\tau_d)}^2
\right|
&\le
2\|F\|_{L_2(\tau_d)}\|E\|_{L_2(\tau_d)}
+
\|E\|_{L_2(\tau_d)}^2 \notag\\
&\le
C_\ell\Theta_U\|h\|_{L_2(\gamma_r)}^2 .
\end{align}
In the last step, we used the preceding bound on \(\|F\|_{L_2(\tau_d)}\), the
error estimate for \(E\), and the assumption \(\Theta_U\le1\).

Combining \eqref{eq:L2-comparison-F-h} and
\eqref{eq:L2-comparison-P-F} gives
\[
\left|
\|P\|_{L_2(\tau_d)}^2
-
\|h\|_{L_2(\gamma_r)}^2
\right|
\le
C_\ell\Theta_U\|h\|_{L_2(\gamma_r)}^2 .
\]
Finally, using
\(\|P\|_{L_2(\tau_d)}=\|f^*\|_{L_2(\tau_d)}\) and
\eqref{eq:fstar-h-gaussian-L2}, we obtain the desired estimate
\begin{equation}
\label{eq:gaussian-hypercube-L2-comparison-final}
\left|
\|f^*\|_{L_2(\tau_d)}^2
-
\|f^*\|_{L_2(\gamma_d)}^2
\right|
\le
C_\ell
\Theta_U
\|f^*\|_{L_2(\gamma_d)}^2 .
\end{equation}

If \(\Theta_U\) is sufficiently small, then
\eqref{eq:gaussian-hypercube-L2-comparison-final} implies the two-sided bound
\[
(1-C_\ell\Theta_U)\|f^*\|_{L_2(\gamma_d)}^2
\le
\|f^*\|_{L_2(\tau_d)}^2
\le
(1+C_\ell\Theta_U)\|f^*\|_{L_2(\gamma_d)}^2 .
\]
This proves the stated norm equivalence.

\subsection{Proof of Proposition~\ref{prop:ambient-basis-covariance}}
\label{proof:ambient-basis-covariance}

For \(s\in[d]\), define
\[
D_s:=(u^{[s]})^\top\nabla F_q.
\]
By Proposition~\ref{prop:H-contraction}, we have
\[
D_s
=
\begin{cases}
\displaystyle
\sum_{|\lambda|=q}a_\lambda\lambda_sH_{\lambda-e_s}
+\mathcal R_s,
& s\in[r],
\\[2ex]
\mathcal R_s,
& s\in[d]\setminus[r],
\end{cases}
\]
where the remainder satisfies
\[
\|\mathcal R_s\|_{L_2}
\le
C_\ell\Delta_U
\sum_{|\lambda|=q}|a_\lambda|.
\]
If
\[
N_{r,q}:=\#\{\lambda\in\N^r:\ |\lambda|=q\},
\]
then Cauchy--Schwarz gives
\[
\sum_{|\lambda|=q}|a_\lambda|
\le
N_{r,q}^{1/2}A_q^{1/2}.
\]
Therefore, we have
\[
\|\mathcal R_s\|_{L_2}
\le
C_\ell N_{r,q}^{1/2}\Delta_U A_q^{1/2}.
\]

Assume first that \(s\in[r]\). Reindexing with \(\gamma=\lambda-e_s\) gives
\[
D_s
=
\sum_{|\gamma|=q-1}
(\gamma_s+1)a_{\gamma+e_s}H_\gamma
+
\mathcal R_s.
\]
Denote the leading term by
\[
L_s
:=
\sum_{|\gamma|=q-1}
(\gamma_s+1)a_{\gamma+e_s}H_\gamma.
\]
Using Proposition~\ref{prop:H-orth} within degree \(q-1\), we obtain
\[
\|L_s\|_{L_2}^2
\le
C_\ell A_q
+
C_\ell\Delta_U
\left(
\sum_{|\gamma|=q-1}|(\gamma_s+1)a_{\gamma+e_s}|
\right)^2.
\]
By Cauchy--Schwarz, this implies
\[
\|L_s\|_{L_2}^2
\le
C_\ell A_q
+
C_\ell N_{r,q}\Delta_U A_q.
\]
Since \(N_{r,q}\le N_{r,\ell}\) and \(\Theta_U=N_{r,\ell}^2\Delta_U\le1\), we
get
\(
\|L_s\|_{L_2}
\le
C_\ell A_q^{1/2}.
\)
Together with the remainder bound, this yields
\[
\|D_s\|_{L_2}
\le
C_\ell A_q^{1/2},
\qquad s\in[r].
\]
If \(s\in[d]\setminus[r]\), then \(D_s=\mathcal R_s\), and hence
\[
\|D_s\|_{L_2}
\le
C_\ell A_q^{1/2}\Theta_U.
\]

We first consider the active-active block. If \(s,t\in[r]\), then expanding the
covariance gives
\[
\E[D_sD_t]
=
\langle L_s,L_t\rangle
+
O_\ell(\Theta_U A_q).
\]
Using Proposition~\ref{prop:H-orth}, the leading inner product satisfies
\[
\langle L_s,L_t\rangle
=
(G_q)_{s,t}
+
O_\ell(N_{r,q}\Delta_U A_q).
\]
Since \(N_{r,q}\Delta_U\le\Theta_U\), we obtain
\[
\big|
(\Xi_q)_{s,t}-(\widetilde G_q)_{s,t}
\big|
\le
C_\ell\Theta_U A_q,
\qquad s,t\in[r].
\]

Next suppose exactly one of \(s,t\) is inactive. Then we immediately have
\[
(\widetilde G_q)_{s,t}=0.
\]
One derivative has \(L_2\)-norm at most \(C_\ell A_q^{1/2}\), while the other
has \(L_2\)-norm at most \(C_\ell A_q^{1/2}\Theta_U\). Therefore, we have
\[
\big|
(\Xi_q)_{s,t}-(\widetilde G_q)_{s,t}
\big|
\le
C_\ell\Theta_U A_q.
\]

Finally, if \(s,t>r\), then both derivatives are inactive remainders. Thus we have
\[
\big|
(\Xi_q)_{s,t}-(\widetilde G_q)_{s,t}
\big|
\le
C_\ell N_{r,q}\Delta_U^2A_q.
\]
Since \(N_{r,q}\le N_{r,\ell}\), this is bounded by
\(
C_\ell\Theta_U^2A_q.
\)

Combining the three cases proves the entrywise estimate. The same proof works
uniformly if any inactive basis vector is replaced by an arbitrary unit vector
in \(\mathcal U^\perp\), because Proposition~\ref{prop:H-contraction} is
uniform over such unit vectors.

\subsection{Proof of Proposition~\ref{prop:degreewise-final-target-cov}}
\label{proof:degreewise-final-target-cov}
Set
\(
E_q:=P_q-F_q.
\)
By \eqref{eq:Pq-Fq-each}, we have
\[
\|E_q\|_{L_2}
\le
C_\ell
\|f^*\|_{L_2(\gamma_d)}
\Theta_U.
\]

Since \(E_q\) is homogeneous multilinear of degree \(q\), for any
\(v\in\R^d\), its directional derivative has the expansion
\[
v^\top\nabla E_q
=
\sum_{|T|=q-1}
\left(
\sum_{j\notin T}
v_j\widehat{E_q}(T\cup\{j\})
\right)x^T.
\]
By Cauchy--Schwarz, this implies
\[
\|v^\top\nabla E_q\|_{L_2}^2
\le
q\|v\|_2^2\|E_q\|_{L_2}^2.
\]
Thus, for every unit vector \(v\), we have
\[
\|v^\top\nabla E_q\|_{L_2}
\le
C_\ell
\|f^*\|_{L_2(\gamma_d)}
\Theta_U.
\]

For convenience, write
\[
D_{s,q}^P:=(u^{[s]})^\top\nabla P_q,
\qquad
D_{s,q}^F:=(u^{[s]})^\top\nabla F_q,
\qquad
D_{s,q}^E:=(u^{[s]})^\top\nabla E_q.
\]
Since \(P_q=F_q+E_q\), these derivatives satisfy
\[
D_{s,q}^P
=
D_{s,q}^F+D_{s,q}^E.
\]
Expanding the covariance difference gives
\[
(\Gamma_q)_{s,t}-(\Xi_q)_{s,t}
=
\E[D_{s,q}^ED_{t,q}^F]
+
\E[D_{s,q}^FD_{t,q}^E]
+
\E[D_{s,q}^ED_{t,q}^E].
\]
Taking absolute values and applying Cauchy--Schwarz yields
\[
\big|(\Gamma_q)_{s,t}-(\Xi_q)_{s,t}\big|
\le
\|D_{s,q}^E\|_{L_2}\|D_{t,q}^F\|_{L_2}
+
\|D_{s,q}^F\|_{L_2}\|D_{t,q}^E\|_{L_2}
+
\|D_{s,q}^E\|_{L_2}\|D_{t,q}^E\|_{L_2}.
\]

The derivative bound for \(E_q\) gives
\[
\|D_{s,q}^E\|_{L_2}
\le
C_\ell
\|f^*\|_{L_2(\gamma_d)}
\Theta_U
\]
for all \(s\in[d]\). The proof of
Proposition~\ref{prop:ambient-basis-covariance} gives
\[
\|D_{s,q}^F\|_{L_2}
\le
C_\ell A_q^{1/2},
\qquad
s\in[r],
\]
and also gives
\[
\|D_{s,q}^F\|_{L_2}
\le
C_\ell A_q^{1/2}\Theta_U,
\qquad
s\in[d]\setminus[r].
\]
Using \eqref{eq:Aq-bound}, we have
\[
A_q^{1/2}
\le
\|f^*\|_{L_2(\gamma_d)}.
\]
Therefore, we have
\[
\big|
(\Gamma_q)_{s,t}-(\Xi_q)_{s,t}
\big|
\le
C_\ell
\|f^*\|_{L_2(\gamma_d)}^2
\left(
\Theta_U\mathbf 1_{\{s\le r\ \text{or}\ t\le r\}}
+
\Theta_U^2\mathbf 1_{\{s>r,\ t>r\}}
\right).
\]
Combining this estimate with Proposition~\ref{prop:ambient-basis-covariance}
gives
\[
\big|
(\Gamma_q)_{s,t}-(\widetilde G_q)_{s,t}
\big|
\le
C_\ell
\|f^*\|_{L_2(\gamma_d)}^2
\left(
\Theta_U\mathbf 1_{\{s\le r\ \text{or}\ t\le r\}}
+
\Theta_U^2\mathbf 1_{\{s>r,\ t>r\}}
\right).
\]
If \(\Theta_U\) is sufficiently small, then
Lemma~\ref{lem:gaussian-hypercube-L2-comparison} allows
\(\|f^*\|_{L_2(\gamma_d)}^2\) to be replaced by
\(\|f^*\|_{L_2(\tau_d)}^2\). The uniform version with arbitrary inactive
unit vectors follows from the same unit-vector derivative estimate and the
uniform inactive-direction estimate in Proposition~\ref{prop:H-contraction}.

\subsection{Proof of Theorem~\ref{thm:final-target-cov}}\label{proof:final-target-cov}
Since
\[
P_{\le L}
=
\sum_{q=0}^{L}P_q,
\]
the derivative \((u^{[s]})^\top\nabla P_q\) is homogeneous multilinear of
degree \(q-1\). Therefore the contributions from distinct degrees are
orthogonal in \(L_2(\tau_d)\). Hence, for \(q\neq q'\), we have
\[
\Big\langle
(u^{[s]})^\top\nabla P_q,\,
(u^{[t]})^\top\nabla P_{q'}
\Big\rangle_{L_2(\tau_d)}
=
0.
\]
It follows that
\begin{equation}
\label{eq:Gamma-sum-by-degree}
(\Gamma_{\le L})_{s,t}
=
\sum_{q=1}^{L}(\Gamma_q)_{s,t}.
\end{equation}
On the Gaussian side, \eqref{eq:G-sum-by-degree} gives
\begin{equation}
\label{eq:Gtilde-sum-by-degree}
(\widetilde G_{\le L})_{s,t}
=
\sum_{q=1}^{L}(\widetilde G_q)_{s,t}.
\end{equation}
Subtracting \eqref{eq:Gtilde-sum-by-degree} from
\eqref{eq:Gamma-sum-by-degree}, applying
Proposition~\ref{prop:degreewise-final-target-cov} for each
\(q=1,\dots,L\), and absorbing the finite number of degrees into \(C_\ell\)
proves the entrywise bounds. The same summation gives the stated uniform
bounds involving arbitrary inactive unit vectors.

It remains to pass from the entrywise comparison to the block-operator
comparison. In the orthonormal basis
\[
\{u^{[1]},\dots,u^{[d]}\},
\]
the matrix \(M_{\le L}\) is represented by \(\Gamma_{\le L}\), while
\(U^\top\Sigma_LU\) is represented by \(\widetilde G_{\le L}\). Hence the
desired operator norm equals
\[
\|\Gamma_{\le L}-\widetilde G_{\le L}\|_{\op}.
\]

The active-active block has size \(r\times r\). Its entrywise bound therefore
contributes at most
\(
C_\ell r\,
\|f^*\|_{L_2(\tau_d)}^2
\Theta_U
\)
to the operator norm. For the active-inactive block, the uniform
inactive-direction estimate implies that, for every unit vector
\(v\in\mathcal U^\perp\) and every \(s\in[r]\), the corresponding mixed
bilinear form is bounded by
\[
C_\ell
\|f^*\|_{L_2(\tau_d)}^2
\Theta_U.
\]
Therefore the active-inactive block has operator norm at most
\[
C_\ell\sqrt r\,
\|f^*\|_{L_2(\tau_d)}^2
\Theta_U.
\]
This is absorbed by
\(
C_\ell r\,
\|f^*\|_{L_2(\tau_d)}^2
\Theta_U.
\)
For the inactive-inactive block, the uniform inactive-inactive estimate gives
operator norm at most
\[
C_\ell
\|f^*\|_{L_2(\tau_d)}^2
\Theta_U^2.
\]
Since \(\Theta_U\le1\) in the regime of interest and \(r\ge1\), this is also
absorbed by
\(
C_\ell
\|f^*\|_{L_2(\tau_d)}^2
r\Theta_U.
\)
Combining the three block estimates gives
\[
\snorm{
M_{\le L}
-
U^\top \Sigma_L U
}
\le
C_\ell
\|f^*\|_{L_2(\tau_d)}^2
r\Theta_U.
\]
Since \(\overline\Theta_U=r\Theta_U\), this proves
\eqref{eq:block-op-transfer}.

\section{Proof of Theorem~\ref{thm:main_krr}}\label{proof:main_krr}
Extend $u^{[1]},\ldots,u^{[r]}$ to an orthonormal basis $u^{[1]},\ldots,u^{[d]}$ of $\R^d$, with $u^{[r+1]}\ldots,u^{[d]} \subseteq U^\perp$.
For any $u^{[s]}, u^{[t]}\in \{u^{[i]}\}_{i=1}^d$, we focus on the estimate
\begin{align}\label{eqn:deriv_est}
    (u^{[s]})\tran \E_n\brac{\nabla \hat{f}(\nabla \hat{f})\tran}u^{[t]} &=  \frac{1}{nd^2} \sum_{k=1}^n \sum_{i,j = 1}^d u_i^{[s]}\partial_i \hat{f}(x^{(k)})\partial_j \hat{f}(x^{(k)}) u_j^{[t]} \notag\\
    &=\frac{1}{nd^2}y\tran K_\lambda\inv \Diag(Xu^{[s]}) (K')^2   \Diag(Xu^{[t]}) K_\lambda\inv y.
\end{align}
By the assumption $\gH(f^*)\in L_2(\tau_d)$, we can expand $\gH(f^*)$ into Fourier-Walsh basis:
\begin{equation}\label{eqn:hermite_expans_target}
    \gH(f_U^*)(x^{(i)}) = \phi(x^{(i)}) \tran \mathsf{c}=\sum_{S\subseteq [d], |S|\leq \ell}\mathsf{c}_{S}\phi_{S}(x^{(i)}),
\end{equation}
for some coefficients \begin{align}
    \mathsf{c}_S:= \E_{x\sim \tau_d}\brac{\gH(f^*_U) x^S},~~~S\subseteq [d],
\end{align}
which satisfies $\|\mathsf{c}\|_2<\infty$. 
Since $\E[y^2_i]=\E|\gH(f_U^*)|^2 +\sigma_\vepsilon^2$ we deduce the estimate 
$$\E[\|y\|_2^2] = n (\E|\gH(f_U^*)|^2+\sigma_\vepsilon^2) = n(\norm{f_U^*}^2_{L_2(\tau_d)} + \sigma_\vepsilon^2).$$
Markov's inequality  subsequently shows $$\norm{y}^2 = O_{d,\P}(n\log(n)\cdot (\norm{f_U^*}_{L_2(\tau_d)}^2+\sigma_\vepsilon^2)).$$
We summarize this observation in the following proposition.

\begin{proposition}\label{prop:bound_y}
With probability at least $1-1/\log(n)$ we have $\norm{y}^2 \leq n\log(n)(\norm{f_U^*}_{L_2(\tau_d)
}^2+\sigma_\vepsilon^2)$.
\end{proposition}
We omit the subscript $U$ in $f^*_U$ for simplicity.

Applying Lemma~\ref{lemma:kernel_to_mono} with the two kernels $K$ and $K'$ shows that for any $\epsilon>0$ there exist matrices $\Delta,\Delta_1$ satisfying $\snorm{\Delta},\snorm{\Delta_1} = O_{d,\P}(d^{(\delta-1)/2+\epsilon})$ and 
\begin{align}
    K &= \Phi_{\leq p}D\Phi_{\leq p}\tran + \rho I_n + \Delta_1\label{eqn:K_decomp_cube},\\
    K' &= \Phi_{\leq p}D'\Phi_{\leq p}\tran + \rho' I_n + \Delta,\label{eqn:K'_decomp_cube} 
\end{align}
where we define $\rho := g(1) - g_{p}(1)$ and $\rho':= g'(1) - g_{p+1}(1)$, and the diagonal matrices $D$ and $D'$ satisfy $\|D_{\S_k} -g^{(k)}(0)d^{-k}I_{|\S_k|}\|_{\rm op} = O_d(d^{-k-1})$ and $\|D'_{\S_k} -g^{(k+1)}(0) d^{-k}I_{|\S_k|}\|_{\rm op} = O_d(d^{-k-1})$ for $k=0,\ldots, p$.

With the expression \eqref{eqn:K'_decomp_cube} in place of $K'$, equation \eqref{eqn:deriv_est} becomes
\begin{align*}
     \frac{1}{nd^2}y\tran K_\lambda\inv \Diag(Xu^{[s]}) \round{\Phi_{\leq p}D'\Phi_{\leq p}\tran + \rho' I_n + \Delta}^2  \Diag(Xu^{[t]}) K_\lambda\inv y.
\end{align*}
Letting $\beta = K_\lambda\inv y$ and expanding the square gives
\begin{align*}
   &~~~(u^{[s]})\tran \E_n\brac{\nabla \hat{f}(\nabla \hat{f})\tran} u^{[t]} \\
   &= \frac{1}{nd^2}\underbrace{\beta\tran \Diag(X u^{[s]})  \round{\Phi_{\leq p}D'\Phi_{\leq p}\tran }^2 \Diag(X u^{[t]}) \beta}_{T_1(s,t)} +  \frac{1}{nd^2} \underbrace{\beta\tran \Diag(Xu^{[s]})   (\rho' I_n + \Delta)^2\Diag(X u^{[t]}) \beta}_{T_2(s,t)} \\
    &~~~~+\frac{1}{nd^2}\underbrace{\beta\tran \Diag(X u^{[s]})\round{\round{\Phi_{\leq p}D'\Phi_{\leq p}\tran }  (\rho' I_n + \Delta)  +  (\rho' I_n + \Delta) \round{\Phi_{\leq p}D'\Phi_{\leq p}\tran } }\Diag(X u^{[t]}) \beta}_{T_{3}(s,t)}
\end{align*}
The following claim is the central ingredient of the proof. Once $T_1(s,t)$, $T_2(s,t)$ and $T_3(s,t)$ are estimated, the desired bound for $(u^{[s]})\tran \E_n\brac{\nabla \hat{f}(\nabla \hat{f})\tran} u^{[t]}$ follows immediately. Accordingly, the remainder of this section is devoted to the proof of this claim. In particular, the bound for $T_3(s,t)$  follows immediately from the bound for $T_1(s,t)$ and $T_2(s,t)$ using Cauchy-Schwarz.

Recall that
\begin{align}\label{eq:def_Gamma_rep} 
(\Gamma_q)_{s,t}
:=
\E\!\big[
(u^{[s]})^\top \nabla \gH(f^*)_q\,
(\nabla \gH(f^*)_q)^\top u^{[t]}
\big],
\qquad s,t\in[d].
\end{align}
To simplify the notation, for \(0\le q\le p+1\), define
\[
    A_q(s,t):=
    \sqrt{(\Gamma_q)_{s,s}}+\sqrt{(\Gamma_q)_{t,t}},
\]
and define
\[
    A_{\le p}(s,t):=
    \sum_{0\le q\le p}A_q(s,t),
    \qquad
    B_{\le p}(s,t):=
    \sum_{0\le q\le p}
    \bigl((\Gamma_q)_{s,s}+(\Gamma_q)_{t,t}\bigr).
\]

\begin{claim}\label{cl:main_agop}
   The following estimates hold for any $\epsilon>0$ uniformly over all $s,t \in [d]$:
   \begin{enumerate}[label=(\alph*)]
          \item $ \frac{1}{nd^2} \abs{T_2(s,t)} = {O}_{d,\P}\left(d^{-2+\epsilon} \right)\cdot \round{\norm{f^*}_{L_2}^2+\sigma_\vepsilon^2}.$\label{cl:main_agop_a}
       \item $ \abs{\frac{1}{nd^2} T_1(s,t)-{\sum_{0\leq q\leq p}(\Gamma_{q})_{s,t} + C_0^2 d^{2\delta-2}(\Gamma_{p+1})_{s,t}  }}\\
=O_{d,\P}\round{d^{-\frac{\delta}{2}+\epsilon}}
\Bigg[
B_{\leq p}(s,t)
+
d^{2\delta-2}
\round{
(\Gamma_{p+1})_{s,s}
+
(\Gamma_{p+1})_{t,t}
}
\Bigg]
\notag\\
+
O_{d,\P}\round{d^{-\frac{1+\delta}{2}+\frac{\epsilon}{2}}}
A_{\leq p}(s,t)
\round{\norm{f^*}_{L_2}^2+\sigma_\vepsilon^2}^{1/2}
\notag\\
\qquad+
O_{d,\P}\round{d^{\frac{-4+3\delta}{2}+\frac{\epsilon}{2}}}
A_{p+1}(s,t)
\round{\norm{f^*}_{L_2}^2+\sigma_\vepsilon^2}^{1/2}
\notag\\
\qquad+
O_{d,\P}\round{
d^{-1-\delta+\epsilon}
+
d^{-2+\delta+\epsilon}
}
\round{\norm{f^*}_{L_2}^2+\sigma_\vepsilon^2},$\label{cl:main_agop_b}

       \item $ \frac{1}{nd^2} \abs{T_3(s,t)}= O_{d,\P}\round{d^{-{\delta}+\epsilon}}
\Bigg[
B_{\leq p}(s,t)
+
d^{2\delta-2}
\round{
(\Gamma_{p+1})_{s,s}
+
(\Gamma_{p+1})_{t,t}
}
\Bigg]
\notag\\
+
O_{d,\P}\round{d^{-\frac{1+3\delta}{2}+\frac{\epsilon}{2}}}
A_{\leq p}(s,t)
\round{\norm{f^*}_{L_2}^2+\sigma_\vepsilon^2}^{1/2}
\notag\\
\qquad+
O_{d,\P}\round{d^{\frac{-4+\delta}{2}+\frac{\epsilon}{2}}}
A_{p+1}(s,t)
\round{\norm{f^*}_{L_2}^2+\sigma_\vepsilon^2}^{1/2}
\notag\\
\qquad+
O_{d,\P}\round{
d^{-1-2\delta+\epsilon}
+
d^{-2+\delta+\epsilon}
}
\round{\norm{f^*}_{L_2}^2+\sigma_\vepsilon^2},$\label{cl:main_agop_c}
   \end{enumerate}
\end{claim}

\subsection{Proof of Item~\ref{cl:main_agop_a} of Claim~\ref{cl:main_agop}}
Sub-multiplicativity of the operator norm directly implies
\begin{align*}
     T_2(s,t)  &\leq \snorm{\rho' I + \Delta}^2   \max_i |(u^{[s]})\tran x^{(i)}|  \max_i |(u^{[t]})\tran x^{(i)}|\norm{K_{\lambda}^{-1}y}_2^2.
\end{align*}
Note that the expression \eqref{eqn:K_decomp_cube} directly implies 
\begin{align}\label{eq:k_lambda_inv}
    \snorm{K_\lambda\inv} \leq 1/(\rho+\lambda -\snorm{\Delta_1}) = O_{d,\P}(1).
\end{align}
where we use the fact that $D$ is positive semi-deifnite since $D$ is diagonally dominant. Therefore, taking into account Proposition~\ref{prop:bound_y}, we deduce
\begin{equation}\label{eqn:we_gonna_need}
\norm{K_\lambda\inv y}_2\leq \snorm{K_\lambda\inv}\|y\|_2\leq O_{d,\mathbb{P}}\left(\sqrt{n\log(n)\round{\norm{f^*}_{L_2}^2+\sigma_\vepsilon^2}}\right)
\end{equation}
By Hoeffding’s inequality we have for any $t>0$, with probability at least $1 - 2e^{-t^2/2}$ over the random data $x^{(i)}$, the following holds for any unit vector $u$:
\begin{equation}\label{eqn:nnorm_bound_gauss}
 |u\tran x^{(i)}| \leq t.
\end{equation}
Letting $t = \log(d)$ and taking the union bound over $\{x^{(i)}\}_{i=1}^n$ and $\{u^{[j]}\}_{j=1}^d$, we have
\begin{align*}
    \max_{i\in[n], j\in[d]} \abs{(u^{[j]})\tran x^{(i)}} = O_{d,\P}(\log(d)).
\end{align*}
With $ \snorm{\rho' I + \Delta}^2 \leq 2|\rho'|^2 + 2\snorm{\Delta}^2 =O_{d,\P}(1)$, consequently, we conclude 
\begin{align}
 \frac{1}{nd^2}|T_{2}(s,t)|^2 ={O}_{d,\P}\left(\frac{ \log^2(d)}{d^{2}}\cdot \round{\norm{f^*}_{L_2}^2+\sigma_\vepsilon^2}\right) = {O}_{d,\P}\left(d^{-2+\epsilon}\cdot \round{\norm{f^*}_{L_2}^2+\sigma_\vepsilon^2}\right)   
\end{align}
for any $\epsilon>0$
as desired.

\subsection{Proof of Item~\ref{cl:main_agop_b} of Claim~\ref{cl:main_agop}}
Expanding the square and invoking Theorem~\ref{lemma:phi_id_2}, the term $\frac{1}{nd^2} T_1(s,t)$ can be reduced to
\begin{align}
    \frac{1}{nd^2} T_1(s,t) = \frac{1}{d^2}\beta\tran \Diag(X u^{[s]})  \Phi_{\leq p}D' (I+\widetilde{\Delta}) D' \Phi_{\leq p}\tran  \Diag(Xu^{[t]}) \beta,
\end{align}
where $\snorm{\widetilde{\Delta}} = O_{d,\P}(d^{-\delta/2+\epsilon})$. 
To simplify the notation, letting $$h^{[j]} := \tfrac{1}{d}D' \Phi_{\leq p}\tran  \Diag(Xu^{[j]}) \beta$$ we write
\begin{align}\label{eq:t_1_decomp}
    \frac{1}{nd^2} T_1(s,t) = (h^{[s]})\tran h^{[t]} + (h^{[s]})\tran \widetilde{\Delta} h^{[t]}.
\end{align}
Observe that $\round{\sum_{i=1}^d u_i  X_i} \Phi_{\leq p}$ lies in the subspace  $\mathrm{span}\{\Phi_1,\Phi_2,\ldots,\Phi_{p+1}\}$. Collecting coefficients with respect to this basis, we obtain
\begin{align}\label{eq:agop_decomp}
    \frac{1}{d}\beta\tran \round{\sum_{i=1}^d u_i  X_i} \Phi_{\leq p} D' = \beta\tran \Phi_{\leq p+1} E, 
\end{align}
where the sparse matrix $E\in \R^{L(p+1) \times  L(p)}$ corresponds to coefficients.
A direct computation yields
\begin{align*}
    E_{S,J} =\begin{cases}
        \tfrac{1}{d} D'_{J,J}\cdot u_i,~~&\mathrm{if}~S=J\sqcup \{i\}~\mathrm{or}~J \setminus \{i\}; \\
       0,~~&\mathrm{otherwise.}
    \end{cases}
\end{align*}
Here $E_{J\setminus \{i\},J}$ is induced by coordinate collapse. We observe that each column $E_{:,J}$ is a sparse vector with at most $d$ non-zero entries.

Next, we analyze $(h^{[s]})\tran h^{[t]}$ and $(h^{[s]})\tran \widetilde{\Delta} h^{[t]}$ in equation~\eqref{eq:agop_decomp} to bound $T_1$.  Combining these two bounds completes the estimate for $T_1$. Specifically, we will prove the following holds with constant $C_0 =(\rho+\lambda)^{-1} g^{(p+1)}(0)$:
\begin{align}\label{eq:final_t_prod}
&~~~\abs{(h^{[s]})\tran h^{[t]} -  \round{\sum_{0\leq q\leq p}(\Gamma_{q})_{s,t} + C_0^2 d^{2\delta-2}(\Gamma_{p+1})_{s,t}  }}\notag\\
&=O_{d,\P}\round{d^{-\delta+\frac{\epsilon}{2}}}
B_{\le p}(s,t)\notag\\
&\qquad+
O_{d,\P}\round{d^{-\frac{1+\delta}{2}+\frac{\epsilon}{2}}}\round{
A_{\le p}(s,t) +\round{d^{\frac{-3+4\delta}{2}}}
A_{p+1}(s,t)}
\round{\norm{f^*}_{L_2}^2+\sigma_\vepsilon^2}^{1/2}\notag\\
&\qquad+
O_{d,\P}\round{d^{-1-\delta+\epsilon} + d^{-2+\delta+\epsilon}}
\round{\norm{f^*}_{L_2}^2+\sigma_\vepsilon^2}
\end{align}
and
\begin{align}\label{eq:final_t_prod_delta}
\abs{(h^{[s]})\tran \widetilde{\Delta}h^{[t]}}
&=
O_{d,\P}\round{d^{-\frac{\delta}{2}+\epsilon}}
B_{\leq p}(s,t)
\notag\\
&\qquad+
O_{d,\P}\round{d^{\frac{3\delta}{2}-2+\epsilon}}
\round{
(\Gamma_{p+1})_{s,s}
+
(\Gamma_{p+1})_{t,t}
}
\notag\\
&\qquad+
O_{d,\P}\round{
d^{-1-\frac{3\delta}{2}+\epsilon}
+
d^{-2+\frac{\delta}{2}+\epsilon}
}
\round{\norm{f^*}_{L_2}^2+\sigma_\vepsilon^2}.
\end{align}

\paragraph{Term $(h^{[s]})\tran h^{[t]}$. }
We decompose indices of vector~\eqref{eq:agop_decomp}  into $j=0,\ldots,p-1$ and $j= p$, which correspondingly decomposes $(h^{[s]})\tran h^{[t]}$ into two terms:
\begin{align}
  (h^{[s]})\tran h^{[t]}
  = \inner{h^{[s]}_{\leq p-1}, h^{[t]}_{\leq p-1}} +\inner{h^{[s]}_{\S_p}, h^{[t]}_{\S_p}}.\label{eq:t_11}
\end{align} 

Following the same convention, we let $c_{\leq q}$  denote the vector $[c_S]_{\cup_{S:|S|\leq q}}$.
To simplify the notation, corresponding to each direction $j\in[d]$, we denote 
\begin{align}
    r^{[j]}:= (E^{[j]})\tran D_{\leq p+1}\inv c_{\leq p+1}.
\end{align}
Next, we present two key lemmas that show $h$ can be approximated by $r$. The proof appears in Appendix~\ref{proof:lemma:agop_leq_p} and \ref{proof:lemma:agop_p_plus_1} respectively.
\begin{lemma}\label{lemma:agop_leq_p} 
    The following estimate holds for any $\epsilon>0$ and uniformly for any $j\in[d]$
    \begin{align}\label{eq:agop_leq_p}
        \norm{h^{[j]}_{\leq p-1} - r^{[j]}_{\leq p-1}}_2^2=
        O_{d,\P}(d^{-2\delta+\epsilon}) \norm{r^{[j]}_{\leq p-1}}_2^2 +   O_{d,\P}(d^{-1-\delta+\epsilon})(\norm{f^*}_{L_2}^2+\sigma_\vepsilon^2). 
    \end{align}
\end{lemma}

\begin{lemma}\label{lemma:agop_p_plus_1}
    The following estimate holds for any $\epsilon>0$ and uniformly for any $j\in[d]$ with constant $C_0 =(\rho+\lambda)^{-1} g^{(p+1)}(0)$:
    \begin{align}\label{eq:agop_p_plus_1}
    \norm{h^{[j]}_{\S_p} -  C_0\cdot d^{\delta-1} r^{[j]}_{\S_p} }_2^2
    &=O_{d,\P}(d^{-2+\epsilon}) \norm{r^{[j]}_{\S_p}}_2^2 +    O_{d,\P}(d^{-2+\delta+\epsilon})(\norm{f^*}_{L_2}^2+\sigma_\vepsilon^2).
    \end{align}
\end{lemma}

The following proposition connects $\inner{r^{[s]}_{\S_q},r^{[t]}_{\S_q}}$ to $\Gamma_{s,t}$. The proof appears in Appendix~\ref{proof:r_prod}.
\begin{proposition}\label{prop:r_prod}
The following estimate holds for any $q= 0,1,\ldots,p$

\begin{align}\label{eq:r_prod_q}
    \abs{\inner{r^{[s]}_{\S_q},r^{[t]}_{\S_q}} - \round{\Gamma_{q+1}}_{s,t}} = {O}_d(d^{-2})A_{q+1}(s,t)\norm{f^*}_{L_2} + O_{d}(d^{-4}) \norm{f^*}_{L_2}^2,
\end{align}
\end{proposition}
Then, we have the following estimate. The proof appears in Appendix~\ref{proof:h_prod}.

\begin{proposition}\label{prop:h_prod}
for any \(\epsilon>0\), the following estimates hold:
\begin{align}\label{eq:h_prod_1}
&\abs{
\inner{h^{[s]}_{\leq p-1}, h^{[t]}_{\leq p-1}}
-
\sum_{0\leq q\leq p}\round{\Gamma_{q}}_{s,t}
}\notag\\
&=
O_{d,\P}\round{d^{-\delta+\frac{\epsilon}{2}}}
B_{\le p}(s,t)\notag\\
&\qquad+
O_{d,\P}\round{d^{-\frac{1+\delta}{2}+\frac{\epsilon}{2}}}
A_{\le p}(s,t)
\round{\norm{f^*}_{L_2}^2+\sigma_\vepsilon^2}^{1/2}\notag\\
&\qquad+
O_{d,\P}\round{d^{-1-\delta+\epsilon}}
\round{\norm{f^*}_{L_2}^2+\sigma_\vepsilon^2},
\end{align}
and
\begin{align}\label{eq:h_prod_2}
&\abs{
\inner{h^{[s]}_{\S_p}, h^{[t]}_{\S_p}}
-
C_0^2 d^{2\delta-2}\round{\Gamma_{p+1}}_{s,t}
}\notag\\
&=
O_{d,\P}\round{d^{\frac{-4+3\delta}{2}+\frac{\epsilon}{2}}}
A_{p+1}(s,t)
\round{\norm{f^*}_{L_2}^2+\sigma_\vepsilon^2}^{1/2}\notag\\
&\qquad+
O_{d,\P}\round{d^{-2+\delta+\epsilon}}
\round{\norm{f^*}_{L_2}^2+\sigma_\vepsilon^2}.
\end{align}
\end{proposition}

Plugging the above two estimates into equation~\eqref{eq:t_11} finishes the bound for \eqref{eq:final_t_prod}.

\paragraph{Term $(h^{[s]})\tran \widetilde{\Delta}h^{[t]}$.}

First, applying \eqref{eq:t_11} with \(s=t\), together with
Proposition~\ref{prop:h_prod}, gives
\begin{align}
&\abs{
\norm{h^{[s]}}_2^2
-
\frac12 B_{\leq p}(s,s)
-
C_0^2 d^{2\delta-2}(\Gamma_{p+1})_{s,s}
}\notag\\
&\leq
O_{d,\P}\round{d^{-\delta+\frac{\epsilon}{2}}}
B_{\leq p}(s,s)\notag\\
&\qquad+
O_{d,\P}\round{d^{-\frac{1+\delta}{2}+\frac{\epsilon}{2}}}
A_{\leq p}(s,s)
\round{\norm{f^*}_{L_2}^2+\sigma_\vepsilon^2}^{1/2}\notag\\
&\qquad+
O_{d,\P}\round{d^{\frac{-4+3\delta}{2}+\frac{\epsilon}{2}}}
A_{p+1}(s,s)
\round{\norm{f^*}_{L_2}^2+\sigma_\vepsilon^2}^{1/2}\notag\\
&\qquad+
O_{d,\P}\round{d^{-1-\delta+\epsilon}+d^{-2+\delta+\epsilon}}
\round{\norm{f^*}_{L_2}^2+\sigma_\vepsilon^2}.
\label{eq:h_s_diag_pre_delta}
\end{align}
Here the factor \(1/2\) appears because
\(B_{\leq p}(s,s)=2\sum_{0\le q\le p}(\Gamma_q)_{s,s}\).

Next, since
\[
    A_{\leq p}(s,s)^2
    =
    4\round{\sum_{0\leq q\leq p}\sqrt{(\Gamma_q)_{s,s}}}^2
    \leq
    2(p+1)B_{\leq p}(s,s),
\]
the weighted AM-GM inequality implies, after applying the estimate with a
smaller \(\epsilon\) and then renaming it as \(\epsilon\), that
\begin{align}
&
O_{d,\P}\round{d^{-\frac{1+\delta}{2}+\frac{\epsilon}{2}}}
A_{\leq p}(s,s)
\round{\norm{f^*}_{L_2}^2+\sigma_\vepsilon^2}^{1/2}
\notag\\
&=
o_{d,\P}(1)B_{\leq p}(s,s)
+
O_{d,\P}\round{d^{-1-\delta+\epsilon}}
\round{\norm{f^*}_{L_2}^2+\sigma_\vepsilon^2}.
\label{eq:h_s_low_absorb_delta}
\end{align}
Similarly, since
\(
    A_{p+1}(s,s)
    =
    2\sqrt{(\Gamma_{p+1})_{s,s}},
\)
another application of weighted AM-GM gives
\begin{align}
&
O_{d,\P}\round{d^{\frac{-4+3\delta}{2}+\frac{\epsilon}{2}}}
A_{p+1}(s,s)
\round{\norm{f^*}_{L_2}^2+\sigma_\vepsilon^2}^{1/2}
\notag\\
&=
o_{d,\P}(1)d^{2\delta-2}(\Gamma_{p+1})_{s,s}
+
O_{d,\P}\round{d^{-2+\delta+\epsilon}}
\round{\norm{f^*}_{L_2}^2+\sigma_\vepsilon^2}.
\label{eq:h_s_top_absorb_delta}
\end{align}
Therefore, substituting \eqref{eq:h_s_low_absorb_delta} and
\eqref{eq:h_s_top_absorb_delta} into \eqref{eq:h_s_diag_pre_delta}, we obtain
\begin{align}
\norm{h^{[s]}}_2^2
&=
\round{\frac12+o_{d,\P}(1)}B_{\leq p}(s,s)
+
\round{C_0^2+o_{d,\P}(1)}
d^{2\delta-2}(\Gamma_{p+1})_{s,s}
\notag\\
&\qquad+
O_{d,\P}\round{d^{-1-\delta+\epsilon}+d^{-2+\delta+\epsilon}}
\round{\norm{f^*}_{L_2}^2+\sigma_\vepsilon^2}.
\label{eq:h_s_diag_delta}
\end{align}
Likewise, replacing \(s\) by \(t\), we have
\begin{align}
\norm{h^{[t]}}_2^2
&=
\round{\frac12+o_{d,\P}(1)}B_{\leq p}(t,t)
+
\round{C_0^2+o_{d,\P}(1)}
d^{2\delta-2}(\Gamma_{p+1})_{t,t}
\notag\\
&\qquad+
O_{d,\P}\round{d^{-1-\delta+\epsilon}+d^{-2+\delta+\epsilon}}
\round{\norm{f^*}_{L_2}^2+\sigma_\vepsilon^2}.
\label{eq:h_t_diag_delta}
\end{align}
Consequently, adding \eqref{eq:h_s_diag_delta} and \eqref{eq:h_t_diag_delta}, and using
\(
    \frac12 B_{\leq p}(s,s)+\frac12 B_{\leq p}(t,t)
    =
    B_{\leq p}(s,t),
\)
we get
\begin{align}
\norm{h^{[s]}}_2^2+\norm{h^{[t]}}_2^2
&=
\round{1+o_{d,\P}(1)}B_{\leq p}(s,t)
\notag\\
&\qquad+
\round{C_0^2+o_{d,\P}(1)}
d^{2\delta-2}
\round{
(\Gamma_{p+1})_{s,s}
+
(\Gamma_{p+1})_{t,t}
}
\notag\\
&\qquad+
O_{d,\P}\round{d^{-1-\delta+\epsilon}+d^{-2+\delta+\epsilon}}
\round{\norm{f^*}_{L_2}^2+\sigma_\vepsilon^2}.
\label{eq:h_s_t_diag_delta}
\end{align}
Finally, by Cauchy's inequality and AM-GM, the following holds
\[
    \abs{(h^{[s]})\tran \widetilde{\Delta}h^{[t]}}
    \leq
    \snorm{\widetilde{\Delta}}\norm{h^{[s]}}_2\norm{h^{[t]}}_2
    \leq
    \frac12\snorm{\widetilde{\Delta}}
    \round{
    \norm{h^{[s]}}_2^2+\norm{h^{[t]}}_2^2
    }.
\]
Thus, using
\(
    \snorm{\widetilde{\Delta}}
    =
    O_{d,\P}\round{d^{-\frac{\delta}{2}+\epsilon}},
\)
and substituting \eqref{eq:h_s_t_diag_delta}, we obtain
\begin{align}
\abs{(h^{[s]})\tran \widetilde{\Delta}h^{[t]}}
&=
O_{d,\P}\round{d^{-\frac{\delta}{2}+\epsilon}}
B_{\leq p}(s,t)
\notag\\
&\qquad+
O_{d,\P}\round{d^{\frac{3\delta}{2}-2+\epsilon}}
\round{
(\Gamma_{p+1})_{s,s}
+
(\Gamma_{p+1})_{t,t}
}
\notag\\
&\qquad+
O_{d,\P}\round{
d^{-1-\frac{3\delta}{2}+\epsilon}
+
d^{-2+\frac{\delta}{2}+\epsilon}
}
\round{\norm{f^*}_{L_2}^2+\sigma_\vepsilon^2}.
\end{align}
which gives \eqref{eq:final_t_prod_delta}.

\section{Proof of lemmas and propositions for the main results}

\subsection{Proof of Lemma~\ref{lemma:agop_leq_p}}\label{proof:lemma:agop_leq_p}

For notational simplicity, we suppress the superscript \([t]\) in
\(u^{[t]}\). A direct coordinate calculation gives the following identity: for
every \(J\in \bigcup_{j=0}^p \S_j\),
\begin{align}
    (D_{\leq p+1}^{-1}E_{:,J})_{S}
    =
    \begin{cases}
       u_i,
       &\text{if } S=J\sqcup\{i\},\\[1mm]
       \displaystyle
       \frac{1}{d^2}\bigl(g^{|S|}(0)\bigr)^{-1}g^{|S|+2}(0)\,u_i,
       &\text{if } S=J\setminus\{i\},\\[1mm]
       0,
       &\text{otherwise.}
    \end{cases}
    \label{eq:D_inv_E}
\end{align}
Consequently, this identity implies the uniform bound
\begin{align}
    \norm{D_{\leq p+1}^{-1}E_{:,J}}_2
    =
    O_d(1).
    \label{eq:bound_e_j}
\end{align}

We next restrict to \(J\in\bigcup_{j=0}^{p-1}\S_j\). In this case,
\(E_{\S_{p+1},J}=0\). Therefore, the corresponding coordinate of \(h\) satisfies
\begin{align*}
    h_J
    &=
    \beta^\top \Phi_{\leq p+1}E_{:,J}  \\
    &=
    \beta^\top
    \left(
    \Phi_{\leq p}E_{:L(p),J}
    +
    \Phi_{\S_{p+1}}E_{\S_{p+1},J}
    \right)  \\
    &=
    \beta^\top \Phi_{\leq p}E_{:L(p),J}.
\end{align*}
Since \(D_{\leq p+1}\) is diagonal, we define
\begin{align}
    \hat e_J
    :=
    (D_{\leq p+1}^{-1}E_{:,J})_{:L(p)}
    =
    D_{\leq p}^{-1}E_{:L(p),J}
    \in \R^{L(p)}.
    \label{eq:e_hat_def}
\end{align}
Thus, when the context is clear, we write \(D\) in place of \(D_{\leq p}\).

By Lemma~\ref{lemma:kernel_to_mono_general}, applied with \(m=2p+1\), the
regularized kernel admits the refined decomposition
\begin{align}
    K_\lambda
    =
    \Phi_{\leq p}D\Phi_{\leq p}^\top
    +
    (\rho_{p,2p+1}+\lambda)H,
    \label{eq:K_refined_decomp}
\end{align}
where \(H\) is given by
\begin{align}
    H
    =
    I_n
    +
    (\rho_{p,2p+1}+\lambda)^{-1}
    \left[
    \sum_{k=p+1}^{2p+1}
    \theta_k\,\offd\!\left(\Phi_{\S_k}\Phi_{\S_k}^\top\right)
    +
    \Delta
    \right].
    \label{eq:H_prelim}
\end{align}
Moreover, the perturbation \(\Delta\) satisfies
\begin{align}
    \snorm{\Delta}
    =
    O_{d,\P}\!\left(
    \frac{\log n}{d^{\frac{2p+2-p-\delta}{2}-\epsilon}}
    \right)
    =
    O_{d,\P}(n^{-1/2}).
    \label{eq:Delta_bound_agop}
\end{align}
Here the coefficients satisfy
\(\theta_k=g^{(k)}(0)d^{-k}+O_d(d^{-k-1})\) for
\(k=p+1,\ldots,2p+1\), and
\(\rho_{p,2p+1}=g(1)-g_p(1)+O_d(d^{-1})\).

For notational simplicity, write
\begin{align*}
    \rho:=\rho_{p,2p+1}.
\end{align*}
Then \(H\) can be written as
\begin{align}
    H
    =
    I_n+\Delta_1,
    \label{eq:H}
\end{align}
where the perturbation is
\begin{align}
    \Delta_1
    :=
    (\rho+\lambda)^{-1}
    \left[
    \sum_{k=p+1}^{2p+1}
    \theta_k\,\offd\!\left(\Phi_{\S_k}\Phi_{\S_k}^\top\right)
    +
    \Delta
    \right].
    \label{eq:Delta_1_def}
\end{align}
By Lemma~\ref{lemma:kernel_to_mono}, this perturbation satisfies
\begin{align}
    \snorm{\Delta_1}
    =
    O_{d,\P}\!\left(d^{(\delta-1)/2+\epsilon}\right).
    \label{eq:Delta_1_bound}
\end{align}

Set
\begin{align}
    \Lambda:=(nD)^{-1}.
    \label{eq:Lambda_def}
\end{align}
Since \(E_{:L(p),J}=D\hat e_J\), the Sherman--Morrison--Woodbury formula gives
\begin{align}
   h_J
   &=
   \beta^\top \Phi_{\leq p}D\hat e_J \notag\\
   &=
   (f^*(X)+\vepsilon)^\top
   K_{\lambda}^{-1}\Phi_{\leq p}D\hat e_J \notag\\
   &=
   \frac{1}{n}
   (f^*(X)+\vepsilon)^\top
   H^{-1}\Phi_{\leq p}
   \left(
   \underbrace{
   (\rho+\lambda)\Lambda
   +
   \frac{\Phi_{\leq p}^\top H^{-1}\Phi_{\leq p}}{n}
   }_{=:G}
   \right)^{-1}
   \hat e_J \notag\\
   &=
   \frac{1}{n}
   \vepsilon^\top H^{-1}\Phi_{\leq p}G^{-1}\hat e_J
   +
   \frac{1}{n}
   f^*(X)^\top H^{-1}\Phi_{\leq p}G^{-1}\hat e_J .
   \label{eq:smw_<_p}
\end{align}

We will use the following norm bounds. Their proof is given in
Appendix~\ref{proof:cl:h_g_norm}.

\begin{claim}\label{cl:h_g_norm}
The following bounds hold:
\begin{align}
    \snorm{H^{-1}}
    =
    O_{d,\P}(1),
    \qquad
    \snorm{G^{-1}}
    =
    O_{d,\P}(1).
    \label{eq:h_g_bound}
\end{align}
\end{claim}

We first bound the noise term in \eqref{eq:smw_<_p}. Define
\begin{align}
   h_J^{(3)}
   :=
   \frac{1}{n}\vepsilon^\top H^{-1}\Phi_{\leq p}G^{-1}\hat e_J .
   \label{eq:h3_def}
\end{align}
Conditioning on \(X\) and using the independence of \(\vepsilon\), we obtain
\begin{align*}
 \E_{\vepsilon}\!\left[\left(h_J^{(3)}\right)^2\right]
 &=
 \frac{\sigma_\vepsilon^2}{n^2}
 \norm{H^{-1}\Phi_{\leq p}G^{-1}\hat e_J}_2^2 \\
 &\le
 \frac{\sigma_\vepsilon^2}{n^2}
 \snorm{H^{-1}}^2
 \snorm{\Phi_{\leq p}}^2
 \snorm{G^{-1}}^2
 \norm{\hat e_J}_2^2 \\
 &=
 O_{d,\P}(n^{-1})\,\sigma_\vepsilon^2.
\end{align*}
In the last line, we used \eqref{eq:h_g_bound},
Lemma~\ref{lem:norm_control}, and \eqref{eq:bound_e_j}. More explicitly,
these bounds give
\begin{align*}
    \snorm{H^{-1}}=O_{d,\P}(1),
    \qquad
    \snorm{G^{-1}}=O_{d,\P}(1),
    \qquad
    \snorm{\Phi_{\leq p}}^2=O_{d,\P}(n),
    \qquad
    \norm{\hat e_J}_2=O_d(1).
\end{align*}
After summing over all \(J\in\S_j\), \(j=0,\ldots,p-1\), and applying
Markov's inequality, we get, for every \(\epsilon>0\),
\begin{align}
    \sum_{j=0}^{p-1}
    \sum_{J\in\S_j}
    \left(h_J^{(3)}\right)^2
    =
    O_{d,\P}\!\left(d^{-1-\delta+\epsilon}\right)
    \sigma_\vepsilon^2.
    \label{eq:less_p_eps}
\end{align}

We next decompose the signal part according to degree. Namely, write
\begin{align}
    f^*(X)
    =
    \Phi_{\leq p}\mathsf c_{\leq p}
    +
    \Phi_{>p}\mathsf c_{>p}.
    \label{eq:fstar_decomp}
\end{align}
Substituting this decomposition into the second term in
\eqref{eq:smw_<_p} gives
\begin{align*}
    \frac{1}{n}
    f^*(X)^\top H^{-1}\Phi_{\leq p}G^{-1}\hat e_J
    &=
    \underbrace{
    \frac{1}{n}
    \mathsf c_{\leq p}^\top
    \Phi_{\leq p}^\top H^{-1}\Phi_{\leq p}
    G^{-1}\hat e_J
    }_{=:h_J^{(1)}} \\
    &\quad+
    \underbrace{
    \frac{1}{n}
    \mathsf c_{>p}^\top
    \Phi_{>p}^\top H^{-1}\Phi_{\leq p}
    G^{-1}\hat e_J
    }_{=:h_J^{(2)}} .
\end{align*}

The remaining estimates are supplied by the following two propositions. Their
proofs are given in Appendices~\ref{proof:key_feature_product} and
\ref{proof:key_feature_product_2}, respectively.

\begin{proposition}\label{cl:key_feature_product}
For every \(\epsilon>0\), the following estimate holds uniformly over
\(\{u^{[j]}\}_{j=1}^d\):
\begin{align}
   \sum_{J:|J|\leq p-1}
   \left(
   h_J^{(1)}
   -
   \mathsf c_{\leq p}^\top \hat e_J
   \right)^2
   &=
   O_{d,\P}\!\left(d^{-2\delta+\epsilon}\right)
   \norm{r_{\leq p-1}}_2^2
   +
   O_{d,\P}\!\left(d^{-p-\delta+\epsilon}\right)
   \norm{f^*}_{L_2}^2 .
   \label{eq:key_bounds_1}
\end{align}
\end{proposition}

\begin{proposition}\label{cl:key_feature_product_2}
For every \(\epsilon>0\), the following estimate holds uniformly over
\(\{u^{[j]}\}_{j=1}^d\):
\begin{align}
   \sum_{J:|J|\leq p-1}
   \left(h_J^{(2)}\right)^2
   &=
   O_{d,\P}\!\left(d^{-1-\delta+\epsilon}\right)
   \norm{f^*}_{L_2}^2 .
   \label{eq:key_bounds_2}
\end{align}
\end{proposition}

We now combine the three pieces. Since
\(h_J=h_J^{(1)}+h_J^{(2)}+h_J^{(3)}\) for \(|J|\le p-1\), and since
\(r_J=\mathsf c_{\leq p}^\top\hat e_J\), the elementary inequality
\((a+b+c)^2\le 3a^2+3b^2+3c^2\) gives
\begin{align*}
    \norm{h_{\leq p-1}-r_{\leq p-1}}_2^2
    &=
    \sum_{j=0}^{p-1}
    \sum_{J\in\S_j}
    \left(
    h_J-\mathsf c_{\leq p}^\top\hat e_J
    \right)^2 \\
    &\le
    3
    \sum_{J:|J|\le p-1}
    \left(
    h_J^{(1)}
    -
    \mathsf c_{\leq p}^\top\hat e_J
    \right)^2
    +
    3
    \sum_{J:|J|\le p-1}
    \left(h_J^{(2)}\right)^2 \\
    &\quad+
    3
    \sum_{J:|J|\le p-1}
    \left(h_J^{(3)}\right)^2 .
\end{align*}
Using \eqref{eq:key_bounds_1}, \eqref{eq:key_bounds_2}, and
\eqref{eq:less_p_eps}, we therefore obtain
\begin{align*}
    \norm{h_{\leq p-1}-r_{\leq p-1}}_2^2
    &=
    O_{d,\P}\!\left(d^{-2\delta+\epsilon}\right)
    \norm{r_{\leq p-1}}_2^2
    +
    O_{d,\P}\!\left(d^{-p-\delta+\epsilon}\right)
    \norm{f^*}_{L_2}^2 \\
    &\quad+
    O_{d,\P}\!\left(d^{-1-\delta+\epsilon}\right)
    \left(
    \norm{f^*}_{L_2}^2
    +
    \sigma_\vepsilon^2
    \right).
\end{align*}
Since \(p\ge 1\), the term \(d^{-p-\delta+\epsilon}\norm{f^*}_{L_2}^2\) is
absorbed by the \(d^{-1-\delta+\epsilon}\norm{f^*}_{L_2}^2\) term. Hence
\begin{align*}
    \norm{h_{\leq p-1}-r_{\leq p-1}}_2^2
    &=
    O_{d,\P}\!\left(d^{-2\delta+\epsilon}\right)
    \norm{r_{\leq p-1}}_2^2
    +
    O_{d,\P}\!\left(d^{-1-\delta+\epsilon}\right)
    \left(
    \norm{f^*}_{L_2}^2
    +
    \sigma_\vepsilon^2
    \right).
\end{align*}
This proves the desired estimate.

\subsection{Proof of Lemma~\ref{lemma:agop_p_plus_1}}\label{proof:lemma:agop_p_plus_1}

Fix \(J\in\S_p\). We begin by substituting the decomposition of \(y\) into
the definition of \(h_J\). Namely, we use
\[
y
=
\Phi_{\S_{p+1}}\mathsf c_{\S_{p+1}}
+
\Phi_{\Q}\mathsf c_{\Q}
+
\vepsilon .
\]
Since
\[
h_J
=
\beta^\top \Phi_{\S_{p+1}}E_{\S_{p+1},J},
\]
the substitution gives
\begin{align}
\label{eq:p_plus_1_decomp}
h_J
&=
\mathsf c_{\S_{p+1}}^\top
\Phi_{\S_{p+1}}^\top K_\lambda^{-1}
\Phi_{\S_{p+1}}E_{\S_{p+1},J} \notag\\
&\quad
+
\mathsf c_{\Q}^\top
\Phi_{\Q}^\top K_\lambda^{-1}
\Phi_{\S_{p+1}}E_{\S_{p+1},J}
+
\vepsilon^\top K_\lambda^{-1}
\Phi_{\S_{p+1}}E_{\S_{p+1},J}.
\end{align}
To simplify notation, set
\[
e_J:=E_{\S_{p+1},J}.
\]

We first estimate the noise contribution. Taking expectation over
\(\vepsilon\), conditionally on the design, gives
\[
\E_{\vepsilon}\left[
\left(
\vepsilon^\top K_\lambda^{-1}\Phi_{\S_{p+1}}e_J
\right)^2
\right]
=
\sigma_\vepsilon^2
\left\|
K_\lambda^{-1}\Phi_{\S_{p+1}}e_J
\right\|_2^2.
\]
Using the estimates
\(\|K_\lambda^{-1}\|_{\op}=O_{d,\P}(1)\) from
\eqref{eq:k_lambda_inv} and
\(\|\Phi_{\S_{p+1}}e_J\|_2^2
=
O_{d,\P}(d^{p+\delta+\epsilon})\|e_J\|_2^2\) from
\eqref{prop:bound_y}, we obtain
\[
\E_{\vepsilon}\left[
\left(
\vepsilon^\top K_\lambda^{-1}\Phi_{\S_{p+1}}e_J
\right)^2
\right]
=
O_{d,\P}\left(d^{-p-2+\delta+\epsilon}\right)
\sigma_\vepsilon^2.
\]
After summing over \(J\in\S_p\) and applying Markov's inequality, the preceding
bound yields
\begin{align}
\label{eq:eq_p_eps}
\sum_{J\in\S_p}
\left(
\vepsilon^\top K_\lambda^{-1}\Phi_{\S_{p+1}}e_J
\right)^2
=
O_{d,\P}\left(d^{-2+\delta+\epsilon}\right)
\sigma_\varepsilon^2 .
\end{align}

We next estimate the first deterministic term in
\eqref{eq:p_plus_1_decomp}. The proof of the following estimate is given in
Appendix~\ref{proof:prop:term_p_plus_one}.

\begin{proposition}
\label{prop:term_p_plus_one}
For every \(\epsilon>0\), the following estimate holds:
\begin{align}
\label{eq:p_1_term1}
&\sum_{J\in\S_p}
\left(
\frac{1}{n}
\mathsf c_{\S_{p+1}}^\top
\Phi_{\S_{p+1}}^\top K_\lambda^{-1}
\Phi_{\S_{p+1}}e_J
-
\frac{1}{\rho+\lambda}
\mathsf c_{\S_{p+1}}^\top e_J
\right)^2
\notag\\
&\qquad
=
O_{d,\P}\left(d^{-2\delta+\epsilon}\right)
\sum_{J\in\S_p}
\left(
\mathsf c_{\S_{p+1}}^\top e_J
\right)^2
+
O_d\left(d^{-2p-2-\delta+\epsilon}\right)
\left\|\mathsf c_{\S_{p+1}}\right\|_2^2 .
\end{align}
\end{proposition}

We then estimate the second deterministic term in
\eqref{eq:p_plus_1_decomp}. The proof of the next estimate is given in
Appendix~\ref{proof:prop:term_p_plus_one_q}.

\begin{proposition}
\label{prop:term_p_plus_one_q}
For every \(\epsilon>0\), the following estimate holds:
\begin{align}
\label{eq:p_1_term2}
\sum_{J\in\S_p}
\left(
\frac{1}{n}
\mathsf c_{\Q}^\top
\Phi_{\Q}^\top K_\lambda^{-1}
\Phi_{\S_{p+1}}e_J
\right)^2
=
O_d\left(d^{-2p-2-\delta+\epsilon}\right)
\left\|\mathsf c_{\Q}\right\|_2^2 .
\end{align}
\end{proposition}

We now combine the three estimates. First, applying
\((a+b+c)^2\le 3(a^2+b^2+c^2)\) to the decomposition
\eqref{eq:p_plus_1_decomp} gives
\begin{align}
\label{eq:p_1_term3_first}
&\sum_{J\in\S_p}
\left(
h_J
-
\frac{n}{\rho+\lambda}
\mathsf c_{\S_{p+1}}^\top e_J
\right)^2
\notag\\
&\qquad
\le
3\sum_{J\in\S_p}
\left(
\mathsf c_{\Q}^\top
\Phi_{\Q}^\top K_\lambda^{-1}
\Phi_{\S_{p+1}}e_J
\right)^2
+
3\sum_{J\in\S_p}
\left(
\mathsf c_{\S_{p+1}}^\top
\Phi_{\S_{p+1}}^\top K_\lambda^{-1}
\Phi_{\S_{p+1}}e_J
-
\frac{n}{\rho+\lambda}
\mathsf c_{\S_{p+1}}^\top e_J
\right)^2
\notag\\
&\qquad\quad
+
3\sum_{J\in\S_p}
\left(
\vepsilon^\top K_\lambda^{-1}
\Phi_{\S_{p+1}}e_J
\right)^2 .
\end{align}
Multiplying the estimates in Propositions~\ref{prop:term_p_plus_one} and
\ref{prop:term_p_plus_one_q} by \(n^2\), and then using
\eqref{eq:eq_p_eps}, gives
\begin{align}
\label{eq:p_1_term3}
&\sum_{J\in\S_p}
\left(
h_J
-
\frac{n}{\rho+\lambda}
\mathsf c_{\S_{p+1}}^\top e_J
\right)^2
\notag\\
&\qquad
=
O_{d,\P}\left(d^{-2\delta+\epsilon}\right)
\sum_{J\in\S_p}
\left(
n\,\mathsf c_{\S_{p+1}}^\top e_J
\right)^2
+
O_d\left(d^{-2+\delta+\epsilon}\right)
\left(
\|f^*\|_{L_2}^2+\sigma_\vepsilon^2
\right).
\end{align}

It remains to rewrite the leading deterministic term in the desired
normalization. For every \(J\in\S_p\), the diagonal expansion gives
\[
D_{J,J}
=
g^{(p+1)}(0)d^{-p-1}
+
O_d(d^{-p-2}).
\]
Therefore, multiplying by \(n\) and using the scaling of \(n\), we have
\[
nD_{J,J}
=
g^{(p+1)}(0)d^{-1+\delta}
+
O_d(d^{-2+\delta}).
\]
Consequently, the following identity holds:
\[
n\,\mathsf c_{\S_{p+1}}^\top e_J
=
nD_{J,J}
\left(
D_{J,J}^{-1}
\mathsf c_{\S_{p+1}}^\top e_J
\right).
\]
Substituting the preceding estimate for \(nD_{J,J}\), we obtain
\[
n\,\mathsf c_{\S_{p+1}}^\top e_J
=
\left(
g^{(p+1)}(0)d^{-1+\delta}
+
O_d(d^{-2+\delta})
\right)
\left(
\mathsf c_{\S_{p+1}}^\top
D_{\S_{p+1}}^{-1}e_J
\right).
\]

We now substitute this expression into \eqref{eq:p_1_term3}. The result is
\begin{align}
\label{eq:p_1_term4}
&\sum_{J\in\S_p}
\left(
h_J
-
\frac{g^{(p+1)}(0)}{\rho+\lambda}
d^{-1+\delta}
\mathsf c_{\S_{p+1}}^\top
D_{\S_{p+1}}^{-1}e_J
\right)^2
\notag\\
&\qquad
=
O_{d,\P}\left(d^{-2+\epsilon}\right)
\sum_{J\in\S_p}
\left(
\mathsf c_{\S_{p+1}}^\top
D_{\S_{p+1}}^{-1}e_J
\right)^2
+
O_d\left(d^{-2+\delta+\epsilon}\right)
\left(
\|f^*\|_{L_2}^2+\sigma_\vepsilon^2
\right).
\end{align}

Finally, define
\(
C_0:=(\rho+\lambda)^{-1}g^{(p+1)}(0).
\)
Also recall the notation
\(
r_J:=
\mathsf c_{\S_{p+1}}^\top D_{\S_{p+1}}^{-1}e_J.
\)
With these definitions, \eqref{eq:p_1_term4} becomes
\[
\sum_{J\in\S_p}
\left(
h_J
-
C_0 d^{-1+\delta}r_J
\right)^2
=
O_{d,\P}\left(d^{-2+\epsilon}\right)
\sum_{J\in\S_p}r_J^2
+
O_d\left(d^{-2+\delta+\epsilon}\right)
\left(
\|f^*\|_{L_2}^2+\sigma_\vepsilon^2
\right).
\]
This is the desired estimate.

\subsection{Proof of Proposition~\ref{prop:r_prod}}\label{proof:r_prod}

We set
\[
P:=\mathcal H(f_U^*),
\qquad
P_q:=[P]_{\deg q}.
\]
W also write \(P_q=\mathcal H(f_U^*)_q\)

By definition, we have
\begin{align*}
    \sum_{J:|J|=q}\sum_{i\in[d]\setminus J} u_i \mathsf{c}_{J \sqcup \{i\}} x^{J} = u\tran \nabla P_{q+1}.
\end{align*}
With this expression, we can deduce
\begin{align}\label{eq:u_p_term1}
    \E\brac{(u^{[s]})\tran \nabla P_{q+1} (\nabla P_{q+1})\tran u^{[t]}} =   \sum_{J:|J|=q} \round{\sum_{i\in[d]\setminus J} u_i^{[s]} \mathsf{c}_{J \sqcup \{i\}}} \round{\sum_{i\in[d]\setminus J} u^{[t]}_i \mathsf{c}_{J \sqcup \{i\}}}.
\end{align}
For any $J \in \S_q$, we can write $r_J$ as
\begin{align*}
    r_{J} &= E_{:,J}\tran D_{\leq p+1} \mathsf{c}_{\leq p+1}\\
    &=   \sum_{i\in [d]\setminus J} u_i \mathsf{c}_{J\sqcup\{i\}} +  \frac{1}{d}\sum_{i\in J} D'_{J,J}D_{J\setminus\{i\},J\setminus\{i\}}\inv u_i \mathsf{c}_{J\setminus\{i\}},
\end{align*}
which implies
\begin{align}\label{eq:u_p_term2}
    \abs{r_J - \sum_{i\in [d]\setminus J} u_i \mathsf{c}_{J\sqcup\{i\}} } = \abs{\frac{1}{d}\sum_{i\in J} D'_{J,J}D_{J\setminus\{i\},J\setminus\{i\}}\inv u_i \mathsf{c}_{J\setminus\{i\}}} = {O}_d(d^{-2}) |\sum_{i\in J} u_i \mathsf{c}_{J \setminus \{i\}}|.
\end{align}
Using basic inequality for any $a,b,c,d\in \R^n$
\begin{align}\label{eq:basic_ineq}
     |\inner{a,b} - \inner{c,d}| \leq \norm{d}_2\norm{a-c}_2 + \norm{c}_2\norm{b-d}_2 + \norm{a-c}_2 \norm{b-d}_2,
\end{align}   
Eq.~\eqref{eq:u_p_term1} and \eqref{eq:u_p_term2} together gives
\begin{align}
\abs{\inner{r^{[s]}_{\S_q},r^{[t]}_{\S_q}} - \round{\Gamma_{q+1}}_{s,t}}&=\abs{\sum_{J:|J|=q} r_J^{[s]}r_J^{[t]} -  \E\brac{(u^{[s]})\tran \nabla P_{q+1} (\nabla P_{q+1})\tran u^{[t]}}} \notag\\
    &\leq {O}_d(d^{-2})\sqrt{\sum_{J:|J|=q} \round{\sum_{i\in[d]\setminus J} u_i^{[s]} \mathsf{c}_{J \sqcup \{i\}}}^2} \sqrt{\sum_{J:|J|=q}\round{\sum_{i\in J} u_i^{[s]} \mathsf{c}_{J \setminus \{i\}}}^2} \notag\\
    &~~~+{O}_d(d^{-2})\sqrt{\sum_{J:|J|=q} \round{\sum_{i\in[d]\setminus J} u_i^{[t]} \mathsf{c}_{J \sqcup \{i\}}}^2} \sqrt{\sum_{J:|J|=q}\round{\sum_{i\in J} u_i^{[t]} \mathsf{c}_{J \setminus \{i\}}}^2}\notag\\
    &~~~+{O}_d(d^{-4})\sqrt{\sum_{J:|J|=q}\round{\sum_{i\in J} u_i^{[s]} \mathsf{c}_{J \setminus \{i\}}}^2}\sqrt{\sum_{J:|J|=q}\round{\sum_{i\in J} u_i^{[t]} \mathsf{c}_{J \setminus \{i\}}}^2}.\label{eq:r_to_gamma}
\end{align}
By Cauchy-Schwarz inequality, we have
\begin{align}
    \sqrt{\sum_{J:|J|=q}\round{\sum_{i\in J} u_i \mathsf{c}_{J \setminus \{i\}}}^2}\leq   \sqrt{\sum_{J:|J|=q} \sum_{i\in J} \abs{\mathsf{c}_{J \setminus \{i\}}}^2}\leq \norm{f^*}_{L_2}.
\end{align}
Plugging this back into equation~\eqref{eq:r_to_gamma} yields
\begin{align*}
    \abs{\inner{r^{[s]}_{\S_q},r^{[t]}_{\S_q}} - \round{\Gamma_{q+1}}_{s,t}}&=\abs{\sum_{J:|J|=q} r_J^{[s]}r_J^{[t]} -  \E\brac{(u^{[s]})\tran \nabla P_{q+1} (\nabla P_{q+1})\tran u^{[t]}}}\\ &\leq {O}_d(d^{-2})\round{\sqrt{(\Gamma_{q+1})_{s,s}} + \sqrt{(\Gamma_{q+1})_{t,t}}} \norm{f^*}_{L_2} + O_{d}(d^{-4}) \norm{f^*}_{L_2}^2,
\end{align*}
which completes the proof.

\subsection{Proof of Proposition~\ref{prop:h_prod}}\label{proof:h_prod}
We first record two residual-size estimates which follow from the diagonal case of Proposition~\ref{prop:r_prod}. Indeed, by applying Proposition~\ref{prop:r_prod} with \(s=t\), and then summing over the orthogonal components, we obtain
\begin{align}
&\norm{r_{\le p-1}^{[s]}}_2^2
+
\norm{r_{\le p-1}^{[t]}}_2^2 \notag\\
&=
B_{\le p}(s,t)
+
O_d(d^{-2})A_{\le p}(s,t)\norm{f^*}_{L_2}
+
O_d(d^{-4})\norm{f^*}_{L_2}^2 .
\label{eq:res_size_sq_corr}
\end{align}
Moreover, taking square roots componentwise in the same diagonal estimate gives
\begin{align}
\norm{r_{\le p-1}^{[s]}}_2
+
\norm{r_{\le p-1}^{[t]}}_2
=
O_d\round{
A_{\le p}(s,t)+d^{-2}\norm{f^*}_{L_2}
}.
\label{eq:res_size_corr}
\end{align}
Similarly, for the top component, the same argument gives
\begin{align}
\norm{r_{\S_p}^{[s]}}_2
+
\norm{r_{\S_p}^{[t]}}_2
=
O_d\round{
A_{p+1}(s,t)+d^{-2}\norm{f^*}_{L_2}
}.
\label{eq:res_size_top_corr}
\end{align}

We now prove \eqref{eq:h_prod_1}. First, for \(u\in\{s,t\}\), the approximation estimate \eqref{eq:agop_leq_p} gives
\[
\norm{h_{\leq p-1}^{[u]}-r_{\leq p-1}^{[u]}}_2
=
O_{d,\P}\round{d^{-\delta+\frac{\epsilon}{2}}}
\norm{r_{\leq p-1}^{[u]}}_2
+
O_{d,\P}\round{d^{-\frac{1+\delta}{2}+\frac{\epsilon}{2}}}
\round{\norm{f^*}_{L_2}^2+\sigma_\vepsilon^2}^{1/2}.
\]
Therefore, applying \eqref{eq:basic_ineq} with
\[
a=h_{\leq p-1}^{[s]},
\qquad
b=h_{\leq p-1}^{[t]},
\qquad
c=r_{\leq p-1}^{[s]},
\qquad
d=r_{\leq p-1}^{[t]},
\]
we obtain
\[
\begin{aligned}
&\abs{
\inner{h_{\leq p-1}^{[s]},h_{\leq p-1}^{[t]}}
-
\inner{r_{\leq p-1}^{[s]},r_{\leq p-1}^{[t]}}
} \\
&\le
\norm{r_{\leq p-1}^{[t]}}_2
\norm{h_{\leq p-1}^{[s]}-r_{\leq p-1}^{[s]}}_2
+
\norm{r_{\leq p-1}^{[s]}}_2
\norm{h_{\leq p-1}^{[t]}-r_{\leq p-1}^{[t]}}_2 \\
&\qquad+
\norm{h_{\leq p-1}^{[s]}-r_{\leq p-1}^{[s]}}_2
\norm{h_{\leq p-1}^{[t]}-r_{\leq p-1}^{[t]}}_2 .
\end{aligned}
\]
Consequently, substituting the preceding approximation estimate and using AM-GM yields
\[
\begin{aligned}
&\abs{
\inner{h_{\leq p-1}^{[s]},h_{\leq p-1}^{[t]}}
-
\inner{r_{\leq p-1}^{[s]},r_{\leq p-1}^{[t]}}
} \\
&\le
O_{d,\P}\round{d^{-\delta+\frac{\epsilon}{2}}}
\round{
\norm{r_{\leq p-1}^{[s]}}_2^2
+
\norm{r_{\leq p-1}^{[t]}}_2^2
} \\
&\qquad+
O_{d,\P}\round{d^{-\frac{1+\delta}{2}+\frac{\epsilon}{2}}}
\round{
\norm{r_{\leq p-1}^{[s]}}_2
+
\norm{r_{\leq p-1}^{[t]}}_2
}
\round{\norm{f^*}_{L_2}^2+\sigma_\vepsilon^2}^{1/2} \\
&\qquad+
O_{d,\P}\round{d^{-1-\delta+\epsilon}}
\round{\norm{f^*}_{L_2}^2+\sigma_\vepsilon^2}.
\end{aligned}
\]
Next, using \eqref{eq:res_size_sq_corr} and \eqref{eq:res_size_corr}, and using
\[
\norm{f^*}_{L_2}
\le
\round{\norm{f^*}_{L_2}^2+\sigma_\vepsilon^2}^{1/2},
\]
we get
\[
\begin{aligned}
&\abs{
\inner{h_{\leq p-1}^{[s]},h_{\leq p-1}^{[t]}}
-
\inner{r_{\leq p-1}^{[s]},r_{\leq p-1}^{[t]}}
} \\
&=
O_{d,\P}\round{d^{-\delta+\frac{\epsilon}{2}}}
B_{\le p}(s,t)\\
&\qquad+
O_{d,\P}\round{d^{-\frac{1+\delta}{2}+\frac{\epsilon}{2}}}
A_{\le p}(s,t)
\round{\norm{f^*}_{L_2}^2+\sigma_\vepsilon^2}^{1/2}\\
&\qquad+
O_{d,\P}\round{d^{-1-\delta+\epsilon}}
\round{\norm{f^*}_{L_2}^2+\sigma_\vepsilon^2}.
\end{aligned}
\]
Here the first term is kept as \(B_{\le p}(s,t)\), rather than being replaced by
\(A_{\le p}(s,t)(\norm{f^*}_{L_2}^2+\sigma_\vepsilon^2)\), because
\(A_{\le p}(s,t)\) already has the scale of \(\norm{f^*}_{L_2}\).

On the other hand, by the orthogonality of the components and Proposition~\ref{prop:r_prod}, we also have
\[
\begin{aligned}
&\abs{
\inner{r_{\leq p-1}^{[s]},r_{\leq p-1}^{[t]}}
-
\sum_{0\leq q\leq p}\round{\Gamma_q}_{s,t}
} \\
&\le
\sum_{0\le q\le p}
\abs{
\inner{r_{\S_{q-1}}^{[s]},r_{\S_{q-1}}^{[t]}}
-
\round{\Gamma_q}_{s,t}
} \\
&=
O_d(d^{-2})A_{\le p}(s,t)\norm{f^*}_{L_2}
+
O_d(d^{-4})\norm{f^*}_{L_2}^2 .
\end{aligned}
\]
Since \(0<\delta<1\), the last display may be weakened as
\[
\begin{aligned}
&\abs{
\inner{r_{\leq p-1}^{[s]},r_{\leq p-1}^{[t]}}
-
\sum_{0\leq q\leq p}\round{\Gamma_q}_{s,t}
} \\
&=
O_d\round{d^{-\frac{1+\delta}{2}+\frac{\epsilon}{2}}}
A_{\le p}(s,t)
\round{\norm{f^*}_{L_2}^2+\sigma_\vepsilon^2}^{1/2}\\
&\qquad+
O_d\round{d^{-1-\delta+\epsilon}}
\round{\norm{f^*}_{L_2}^2+\sigma_\vepsilon^2}.
\end{aligned}
\]
Finally, combining the previous two estimates by the triangle inequality gives \eqref{eq:h_prod_1}.

We next prove \eqref{eq:h_prod_2}. First, for \(u\in\{s,t\}\), the \(\S_p\)-component approximation gives
\[
\norm{
h_{\S_p}^{[u]}
-
C_0d^{\delta-1}r_{\S_p}^{[u]}
}_2
=
O_{d,\P}\round{d^{\frac{\delta-2}{2}+\frac{\epsilon}{2}}}
\round{\norm{f^*}_{L_2}^2+\sigma_\vepsilon^2}^{1/2}.
\]
Therefore, applying \eqref{eq:basic_ineq} with
\[
a=h_{\S_p}^{[s]},
\qquad
b=h_{\S_p}^{[t]},
\qquad
c=C_0d^{\delta-1}r_{\S_p}^{[s]},
\qquad
d=C_0d^{\delta-1}r_{\S_p}^{[t]},
\]
we obtain
\[
\begin{aligned}
&\abs{
\inner{h_{\S_p}^{[s]},h_{\S_p}^{[t]}}
-
C_0^2d^{2\delta-2}
\inner{r_{\S_p}^{[s]},r_{\S_p}^{[t]}}
} \\
&\le
C_0d^{\delta-1}
\norm{r_{\S_p}^{[t]}}_2
\norm{
h_{\S_p}^{[s]}-C_0d^{\delta-1}r_{\S_p}^{[s]}
}_2 \\
&\qquad+
C_0d^{\delta-1}
\norm{r_{\S_p}^{[s]}}_2
\norm{
h_{\S_p}^{[t]}-C_0d^{\delta-1}r_{\S_p}^{[t]}
}_2 \\
&\qquad+
\norm{
h_{\S_p}^{[s]}-C_0d^{\delta-1}r_{\S_p}^{[s]}
}_2
\norm{
h_{\S_p}^{[t]}-C_0d^{\delta-1}r_{\S_p}^{[t]}
}_2 .
\end{aligned}
\]
Consequently, substituting the preceding \(\S_p\)-component approximation gives
\[
\begin{aligned}
&\abs{
\inner{h_{\S_p}^{[s]},h_{\S_p}^{[t]}}
-
C_0^2d^{2\delta-2}
\inner{r_{\S_p}^{[s]},r_{\S_p}^{[t]}}
} \\
&\le
O_{d,\P}\round{
d^{\delta-1}
d^{\frac{\delta-2}{2}+\frac{\epsilon}{2}}
}
\round{
\norm{r_{\S_p}^{[s]}}_2
+
\norm{r_{\S_p}^{[t]}}_2
}
\round{\norm{f^*}_{L_2}^2+\sigma_\vepsilon^2}^{1/2}\\
&\qquad+
O_{d,\P}\round{d^{-2+\delta+\epsilon}}
\round{\norm{f^*}_{L_2}^2+\sigma_\vepsilon^2}.
\end{aligned}
\]
Next, using \eqref{eq:res_size_top_corr}, we get
\[
\begin{aligned}
&\abs{
\inner{h_{\S_p}^{[s]},h_{\S_p}^{[t]}}
-
C_0^2d^{2\delta-2}
\inner{r_{\S_p}^{[s]},r_{\S_p}^{[t]}}
} \\
&=
O_{d,\P}\round{d^{\frac{-4+3\delta}{2}+\frac{\epsilon}{2}}}
A_{p+1}(s,t)
\round{\norm{f^*}_{L_2}^2+\sigma_\vepsilon^2}^{1/2}\\
&\qquad+
O_{d,\P}\round{d^{-2+\delta+\epsilon}}
\round{\norm{f^*}_{L_2}^2+\sigma_\vepsilon^2}.
\end{aligned}
\]
It remains to replace
\(\inner{r_{\S_p}^{[s]},r_{\S_p}^{[t]}}\)
by \(\round{\Gamma_{p+1}}_{s,t}\). By Proposition~\ref{prop:r_prod} with \(q=p\), we have
\[
\begin{aligned}
&C_0^2d^{2\delta-2}
\abs{
\inner{r_{\S_p}^{[s]},r_{\S_p}^{[t]}}
-
\round{\Gamma_{p+1}}_{s,t}
} \\
&\le
C_0^2d^{2\delta-2}
\round{
O_d(d^{-2})A_{p+1}(s,t)\norm{f^*}_{L_2}
+
O_d(d^{-4})\norm{f^*}_{L_2}^2
}.
\end{aligned}
\]
Since \(0<\delta<1\), the preceding display implies
\[
\begin{aligned}
&C_0^2d^{2\delta-2}
\abs{
\inner{r_{\S_p}^{[s]},r_{\S_p}^{[t]}}
-
\round{\Gamma_{p+1}}_{s,t}
} \\
&=
O_d\round{d^{\frac{-4+3\delta}{2}}}
A_{p+1}(s,t)
\round{\norm{f^*}_{L_2}^2+\sigma_\vepsilon^2}^{1/2}\\
&\qquad+
O_d\round{d^{-2+\delta}}
\round{\norm{f^*}_{L_2}^2+\sigma_\vepsilon^2}.
\end{aligned}
\]
Finally, combining the last two estimates by the triangle inequality proves \eqref{eq:h_prod_2}, and the proof is complete.

\subsection{Proof of Claim~\ref{cl:h_g_norm}}\label{proof:cl:h_g_norm}
\begin{proof}The following estimate for any $\Delta$ with $\snorm{\Delta} = o_{d,\P}(1)$:
\begin{align}\label{eq:i_plus_delta_inv}
    \snorm{(I+\Delta)\inv - I} = O_{d,\P}(\snorm{\Delta}),
\end{align}
which can be deduced using Neumann series.

    Based on the definition of $D$, we immediately have $\snorm{(nD)\inv} = O_{d}(d^{-\delta})$. Notice that $ \snorm{\frac{\Phi_{\leq p}\tran  \Phi_{\leq p}}{n} - I_n} = O_{d,\P}(d^{-\delta/2+\epsilon})$ (Theorem~\ref{lemma:phi_id_2}) and 
\begin{align}\label{eq:e_inv_minus_i}
    \snorm{H\inv - I_n} = O_d(\snorm{\Delta_1}) = O_{d,\P}(d^{(\delta-1)/2+\epsilon}),
\end{align}
which follows from Eq.~\eqref{eq:i_plus_delta_inv}.
Thus we deduce
\begin{align}\label{eq:psi_e_psi_minus_i}
 \snorm{\frac{\Phi_{\leq p}\tran H\inv \Phi_{\leq p}}{n} -I} &= \snorm{\frac{\Phi_{\leq p}\tran  \Phi_{\leq p}}{n} +\frac{\Phi_{\leq p}\tran  (H\inv - I) \Phi_{\leq p}}{n} - I}\notag\\
 &\leq \snorm{\frac{\Phi_{\leq p}\tran  \Phi_{\leq p}}{n} - I} + \snorm{\frac{\Phi_{\leq p}}{\sqrt{n}}}^2\snorm{H\inv - I}\notag\\
 &= O_{d,\P}(d^{-\delta/2+\epsilon} + d^{(\delta-1)/2+\epsilon}).
\end{align}
With this bound, we can deduce
\begin{align*}
    \snorm{G-I} =  O_d(d^{-\delta}) +O_{d,\P}(d^{-\delta/2+\epsilon} + d^{(\delta-1)/2+\epsilon})= O_{d,\P}(d^{-\delta/2+\epsilon} + d^{(\delta-1)/2+\epsilon}).
\end{align*}
Invoking Eq.~\eqref{eq:i_plus_delta_inv} again we obtain
\begin{align}\label{eq:g_inv_minus_i}
     \snorm{G\inv-I} =  O_{d,\P}(d^{-\delta/2+\epsilon} + d^{(\delta-1)/2+\epsilon}).
\end{align}
\end{proof}

\subsection{Proof of Proposition~\ref{cl:key_feature_product}}\label{proof:key_feature_product}

First, we express $H\inv$ and $G\inv$ in \eqref{eq:smw_<_p}  using feature matrices. The proof appears in Appendix~\ref{proof:prop:h_g_inv}.
\begin{proposition}\label{prop:h_g_inv}
The following equalities hold:
\begin{align}\label{eq:h_inv}
  H\inv  = 
I_n + \underbrace{\sum_{q=1}^{c_H} \round{ \sum_{j=p+1}^{2p+1 } - \frac{c_{j}}{d^j}\offd\round{\Phi_{\S_{j}} \Phi_{\S_{j}}\tran}}^{q}}_{:=\Delta_H} + \Delta_1,
\end{align}
and
\begin{align}\label{eq:g_inv}
G^{-1} &=I_{|L(p)|}+ \underbrace{\sum_{r=1}^{c_G} \Biggl[ -(\rho+\lambda)\Lambda +  \tfrac{1}{n}\offd\round{\Phi_{\leq p}\tran \Phi_{\leq p}}+ \frac{\Phi_{\leq p}\tran}{\sqrt{n}} \Delta_H\frac{\Phi_{\leq p}}{\sqrt{n}} \Biggr]^{r}}_{:=\Delta_G}+ \Delta_2,
\end{align}
where $c_j, c_G, c_H>0$ are constants and matrices $\Delta_1,\Delta_2$ satisfy $$\snorm{\Delta_1}, \snorm{\Delta_2} = O_{d,\P}(n^{-\tfrac{1}{2}}).$$
\end{proposition}

To simplify the notation, we write
\begin{align}
    A_H:=\frac1n\Phi_{\le p}^\top\Delta_H\Phi_{\le p},
\qquad
A_G:=\frac1n\Phi_{\le p}^\top\Phi_{\le p}\Delta_G,
\qquad
A_{HG}:=\frac1n\Phi_{\le p}^\top\Delta_H\Phi_{\le p}\Delta_G
\end{align}
Plugging the expressions of $H\inv =I + \Delta_H+\Delta_1$ and $G\inv =I + \Delta_G+\Delta_2$ into $\Phi_{\leq p} \tran H\inv \Phi_{\leq p} G\inv$ and expanding the powers, we see that it can be written as
\begin{align}\label{eq:finit_summation}
\tfrac{1}{n}\Phi_{\leq p}\tran H\inv \Phi_{\leq p}G\inv \notag  
  &=\frac{1}{{n}} \Phi_{\leq p}\tran \Phi_{\leq p} +A_H + A_G + A_{HG}\notag\\
  &~~~+\underbrace{\tfrac{1}{n}(\Phi_{\leq p})\tran H\inv(\Phi_{\leq p}) \Delta_2+ \tfrac{1}{n}(\Phi_{\leq p})\tran \Delta_1(\Phi_{\leq p}) G\inv}_{:=\Delta_3}.
\end{align}
We can immediately deduce that $\Delta_3$ satisfies 
\begin{align}\label{eq:bound_Delta_3}
    \snorm{\Delta_3} = O_{d,\P}(1)(\snorm{\Delta_2}+\snorm{\Delta_4}) =  O_{d,\P}(n^{-\tfrac{1}{2}}),
\end{align}
where we use the estimate of $\snorm{H\inv} = O_{d,\P}(1)$ \eqref{eq:e_inv_minus_i}, $\snorm{G\inv} = O_{d,\P}(1)$ \eqref{eq:g_inv_minus_i} and $\snorm{\Phi_{\leq p}/\sqrt{n}} = O_{d,\P}(1)$ (Lemma~\ref{lem:norm_control})

The following proposition estimates $A_H$, $A_G$ and $A_{HG}$. The proof appears in Appendix~\ref{proof:prop:control_error_delta}.

\begin{proposition}\label{prop:control_error_delta}
Let 
\[
P_G(\Lambda):=\sum_{r=1}^{c_G} \bigl(-(\rho+\lambda)\Lambda\bigr)^r.
\]
For all 
\begin{align*}
   N \in \{A_H,\ A_G-P_G(\Lambda),\ A_{HG}\},
\end{align*}
the following equalities hold:
\begin{align}\label{eq:control_error_delta}
    \E\brac{(\mathsf{c}_{\leq p}\tran N \hat{e}_J )^2} &= O_d(d^{-2\delta})(\mathsf{c}_{\leq p}\tran \hat{e}_J)^2 +O_d(d^{-4-2\delta}) \cdot  \round{\sum_{i\in J} \mathsf{c}_{J\setminus \{i\}}^2} +  O_{d,\P}(n\inv) \norm{\mathsf{c}_{\leq p}}_2^2.
\end{align}
\end{proposition}

Then, we plug the decomposition \eqref{eq:finit_summation} into the expression of $h_J^{(1)}$ and using Cauchy-Schwartz inequality we obtain
\begin{align}\label{eq:finit_summation_2}
      &~~~{\round{h_J^{(1)}-\inner{\mathsf{c}_{\leq p}, \round{I+P_G(\Lambda)} \hat{e}_J}  }^2}\notag \\
      &\leq 5 {\round{ \mathsf{c}_{\leq p}\tran \round{\tfrac{1}{n}\Phi_{\leq p}\tran \Phi_{\leq p} + P_G(\Lambda)}\hat{e}_J-\inner{\mathsf{c}_{\leq p}, \round{I+P_G(\Lambda)} \hat{e}_J}}^2} + 5{\round{ \mathsf{c}_{\leq p} \tran A_H \hat{e}_J}^2} \notag\\
      &~~~+5{\round{ \mathsf{c}_{\leq p} \tran A_G \hat{e}_J}^2} + 5{\round{ \mathsf{c}_{\leq p} \tran \round{A_{HG} - P_G(\Lambda)} (\hat{e}_J)}^2} + 5\round{ \mathsf{c}_{\leq p} \tran\Delta_3 (\hat{e}_J)}^2.
\end{align}
Using equation~\eqref{eq:bound_Delta_3}, the last  term $\round{ \mathsf{c}_{\leq p} \tran\Delta_3 (\hat{e}_J)}^2$ can be bounded by 
\begin{align}\label{eq:bound_Delta_4}
  \snorm{\Delta_3 }^2\cdot  \norm{\mathsf{c}_{\leq p}}_2^2 \cdot \norm{\hat{e}_J}^2_2 = O_{d,\P}(n\inv)\norm{\mathsf{c}_{\leq p}}_2^2.
\end{align}
Taking the expectation of the first term and expanding the square gives
\begin{align}
    &~~~\E\brac{\round{ \mathsf{c}_{\leq p}\tran \round{\tfrac{1}{n}\Phi_{\leq p}\tran \Phi_{\leq p} + P_G(\Lambda)}\hat{e}_J-\inner{\mathsf{c}_{\leq p}, \round{I+P_G(\Lambda)} \hat{e}_J}}^2}\notag\\
    &=\E\brac{\round{\mathsf{c}_{\leq p}\tran \round{\tfrac{1}{ n} \Phi_{\leq p}\tran \Phi_{\leq p}+P_G(\Lambda)} \hat{e}_J}^2 }- \round{\mathsf{c}_{\leq p}\tran (I+P_G(\Lambda)) \hat{e}_J}^2\notag\\
    &\overset{(a)}{=} \round{\mathsf{c}_{\leq p}\tran (I+P_G(\Lambda)) \hat{e}_J}^2 + O_d(\tfrac{1}{n}) \norm{\mathsf{c}_{\leq p}}_2^2 \norm{(I+P_G(\Lambda))\hat{e}_J}_2^2- \round{\mathsf{c}_{\leq p}\tran (I+P_G(\Lambda)) \hat{e}_J}^2\notag\\
    &= O_d(\tfrac{1}{n}) \norm{\mathsf{c}_{\leq p}}_2^2,\label{eq:control_error_delta_2}
\end{align}
where $(a)$ follows from Lemma~\ref{lemma:feature_product_2}.

For the second, third and fourth terms in equation~\eqref{eq:finit_summation_2}, we take the expectation and the bound~\eqref{eq:control_error_delta} immediately applies.

Applying Proposition~\ref{prop:uniform_bound} shows the bound~\eqref{eq:control_error_delta} and \eqref{eq:control_error_delta_2} hold uniformly for all $u^{[1]},\ldots u^{[d]}$, with an extra multiplicative factor $d^\epsilon$. Summing over $J$ with $|J|\leq p-1$ and applying Markov's inequality, we have with probability $1-1/\log(d)$, the following holds
\begin{align*}
&~~~~\sum_{J:|J|\leq p-1}\round{\frac{1}{{n}} \Phi_{\leq p}\tran \Phi_{\leq p} +A_H + A_G + A_{HG}-\inner{\mathsf{c}_{\leq p}, (I+P_G(\Lambda))\hat{e}_J}}^2\\
&=  O_{d,\P}(d^{-1-\delta+\epsilon})\log(d)\norm{\mathsf{c}_{\leq p}}_2^2.
\end{align*}
Combining with equation~\eqref{eq:bound_Delta_4}, we conclude
\begin{align*}
\sum_{J:|J|\leq p-1}\round{ h_J^{(1)}-\inner{\mathsf{c}_{\leq p}, (I+P_G(\Lambda))\hat{e}_J}}^2 =  O_d(d^{-1-\delta+\epsilon})\norm{\mathsf{c}_{\leq p}}_2^2 =  O_{d,\P}(d^{-1-\delta+\epsilon})\norm{f^*}_{L_2}^2,
\end{align*}
where we absorb the logarithmic actor into $d^\epsilon$.

Lastly, using the Cauchy-Schwarz inequality, we obtain
\begin{align}\label{eq:less_p_last}
\sum_{J:|J|\leq p-1}\round{ h_J^{(1)}-\inner{\mathsf{c}_{\leq p}, \hat{e}_J}}^2 = 2 \sum_{J:|J|\leq p-1}\round{\mathsf{c}_{\leq p}\tran P_G(\Lambda) \hat{e}_J}^2  + O_{d,\P}(d^{-1-\delta+\epsilon})\norm{f^*}_{L_2}^2.
\end{align}
The following claim shows that $\round{\mathsf{c}_{\leq p}\tran P_G(\Lambda) \hat{e}_J}^2$ is of a smaller order then combining it with the equation~\eqref{eq:less_p_last} completes the proof. The proof of the claim appears in Appendix~\ref{proof:small_cpe}.
\begin{claim}\label{cl:small_cpe}
    The following estimate holds:
    \begin{align}
        \round{\mathsf{c}_{\leq p}\tran P_G(\Lambda) \hat{e}_J}^2 = O_d(d^{-2\delta})  \round{\mathsf{c}_{\leq p} \hat{e}_J}^2 + {O}_d(d^{-8-2\delta})\round{\sum_{i\in J} \mathsf{c}_{J\setminus \{i\}}^2}.
    \end{align}
\end{claim}

\subsection{Proof of Proposition~\ref{cl:key_feature_product_2}}\label{proof:key_feature_product_2}
The proof follows analogously to the proof of Proposition~\ref{cl:key_feature_product}. For 
\begin{align*}
    h_J^{(2)} = \frac{1}{n} \mathsf{c}_{>p} \tran \Phi_{> p}\tran H\inv \Phi_{\leq p} G\inv \hat{e}_J,
\end{align*}
plugging the expression of $H\inv$~\eqref{eq:h_inv} and $G\inv$~\eqref{eq:g_inv} gives
\begin{align*}
    \eta_J = \frac{1}{n} \mathsf{c}_{>p} \tran \Phi_{> p}\tran (I+\Delta_H+\Delta_1) \Phi_{\leq p} (I+\Delta_G+\Delta_2) \hat{e}_J
\end{align*}
As in the proof of Proposition~\ref{cl:key_feature_product} , the terms involving the remainder matrices \(\Delta_1,\Delta_2\) are
easier and are absorbed into the same bound, so it suffices to treat the principal part
\[
X_J
:=
\frac1n\,\mathsf{c}_{>p}\tran \Phi_{>p}^\top (I+\Delta_H)\Phi_{\le p}(I+\Delta_G)\hat{e}_J.
\]
We decompose both endpoint spaces by exact degree. On the low-degree side, we use the
same notation as in the proof of Proposition~\ref{prop:control_error_delta}:
\[
W_q:=n^{-1/2}\Phi_{\mathcal S_q},
\qquad
W=[W_0,\dots,W_p],
\qquad
\Lambda=\sum_{q=0}^p \beta_q P_q,
\quad
\beta_q:=\frac{d^q}{n}.
\]
On the high-degree side, write
\[
\Phi_{>p}=[\Phi_{\S_{p+1}},\dots,\Phi_{\S_{\ell}}],
\qquad
V_j:=|\S_j|^{-1/2}\Phi_{\S_j}
\quad (p+1\le j\le \ell),
\]
and decompose
\[
\mathsf{c}_{>p}=(\mathsf{c}_{\S_{p+1}},\dots,\mathsf{c}_{\S_{\ell}}).
\]
Exactly as in the proof of Proposition~\ref{prop:control_error_delta}, every summand in \(\Delta_H\) is a finite product of matrices
\(\offd(V_jV_j^\top)\) with deterministic coefficient \(O_d(1)\), and every word in
\(\Delta_G\) is obtained by expanding powers of
\[
B_0:=-(\rho+\lambda)\Lambda,
\qquad
B_1:=\offd(W^\top W),
\qquad
B_2:=W^\top\Delta_H W.
\]
Hence, after expanding \(I+\Delta_H\) and \(I+\Delta_G\), every summand contributing to
\(X_J\) is of the form
\[
\frac1n\,\mathsf{c}_{\S_j}^\top \Phi_{\S_j}^\top H_0\,\Phi_{\le p}\,\Xi\,\hat{e}_J,
\]
where \(p+1\le j\le \ell\), \(H_0\) is either \(I\) or one of the words appearing in
\(\Delta_H\), and \(\Xi\) is either \(I\) or one of the words appearing in \(\Delta_G\).

We now expand the low-degree factor \(\Phi_{\le p}\Xi\) exactly as in the proof for the
low-degree endpoint words, by inserting the exact-degree projectors \(P_q\) around every
copy of \(\Lambda\) and using the decompositions of \(B_1\) and \(B_2\).
This yields a finite representation
\begin{align}\label{eq:X_J}
 X_J
=
\sum_{\eta\in\Omega}
\alpha_\eta\,
\bigl((a_{\S_{j_\eta}})^\top M_\eta (P_{q_\eta}\hat{e}_J)\bigr),   
\end{align}
where \(|\Omega|=O_{p,c_H,c_G}(1)\), each \(M_\eta\) is a normalized chain whose families
belong to $\S_0,\ldots,\S_\ell$, 
and whose endpoint families are $
\S_{j_\eta}
$ and $
\mathcal \S_{q_\eta}$.
In particular, the endpoint families are always different, since \(j_\eta>p\) whereas
\(q_\eta\le p\).

Moreover, each coefficient \(\alpha_\eta\) is deterministic and satisfies
\begin{align}\label{eq:alpha_eta}
  |\alpha_\eta| = O_d\round{\sqrt{\frac{|\S_{j_\eta}|}{n}}}.  
\end{align}
Indeed, the outer factor \(n^{-1}\Phi_{\S_j}^\top \Phi_{\mathcal \S_q}\) contributes
\[
\frac1n\Phi_{\S_j}^\top\Phi_{\mathcal \S_q}
=
\sqrt{\frac{|\S_j|}{n}}\;V_j^\top W_q,
\]
while all remaining coefficients coming from \(\Delta_H\) are \(O_d(1)\), and every
additional factor coming from \(\Lambda\) is one of the \(\beta_q\le 1\).

Fix one \(\eta\in\Omega\). Since the endpoint families of \(M_\eta\) are different,
Lemma~\ref{lemma:feature_product_2} gives
\begin{align}\label{eq:each_M_eta}
\E\!\left[
\bigl(\mathsf{c}_{\S_{j_\eta}}^\top M_\eta (P_{q_\eta}\hat{e}_J)\bigr)^2
\right]
\le
O_d\!\left(
(nw^{(1)}w^{(m+1)})^2 n^{-1}
\|a_{\S_{j_\eta}}\|_2^2\|P_{q_\eta}\hat{e}_J\|_2^2
\right).
\end{align}
Note that the weights satisfy
\[
w^{(1)}=|\S_{j_\eta}|^{-1/2},
\qquad
w^{(m+1)}=n^{-1/2},
\]
therefore we have
\[
(nw^{(1)}w^{(m+1)})^2
=
\frac{n}{|\S_{j_\eta}|}.
\]
Substituting this into \ref{eq:each_M_eta} yields
\[
\E\!\left[
\bigl( \mathsf{c}_{\S_{j_\eta}}^\top M_\eta (P_{q_\eta}\hat{e}_J)\bigr)^2
\right]
\le
O_d\!\left(
\frac1{|\S_{j_\eta}|}
\|\mathsf{c}_{\S_{j_\eta}}\|_2^2\|P_{q_\eta}\hat{e}_J\|_2^2
\right).
\]
Combining this with \eqref{eq:alpha_eta}, we obtain
\begin{align}\label{eq:each_M_eta_2}
\E\!\left[
\bigl(\alpha_\eta (a_{\S_{j_\eta}})^\top M_\eta (P_{q_\eta}\hat{e}_J)\bigr)^2
\right]
\le
O_d\!\left(
n^{-1}\|a_{\S_{j_\eta}}\|_2^2\|P_{q_\eta}\hat{e}_J\|_2^2
\right).
\end{align}
Since \(|\Omega|=O_{p,c_H,c_G}(1)\), Cauchy--Schwarz in \((23)\) gives
\[
\E[X_J^2]
\le
O_{p,c_H,c_G}(1)
\sum_{\eta\in\Omega}
\E\!\left[
\bigl(\alpha_\eta (\mathsf{c}_{\S_{j_\eta}})^\top M_\eta (P_{q_\eta}\hat{e}_J)\bigr)^2
\right].
\]
Using \eqref{eq:each_M_eta_2} and the orthogonality of the exact-degree blocks,
\[
\sum_{j=p+1}^{\ell}\|\mathsf{c}_{\S_j}\|_2^2=\|\mathsf{c}_{>p}\|_2^2,
\qquad
\sum_{q=0}^p \|P_q\hat{e}_J\|_2^2=\|\hat{e}_J\|_2^2,
\]
we conclude that
\[
\E[X_J^2]
\le
O_d\!\left(n^{-1}\|\mathsf{c}_{>p}\|_2^2\|\hat{e}_J\|_2^2\right) = O_d\!\left(n^{-1}\|\mathsf{c}_{>p}\|_2^2\right).
\]
Finally, the same bound holds for the terms containing \(\Delta_1\) or \(\Delta_2\):
these are simpler, since the extra factor \(\Delta_1\) or \(\Delta_2\) is controlled by
its spectral norm, exactly as in the estimates~\eqref{eq:bound_Delta_3}.  Then using hypercontractivity and Markov's inequality we conclude the proof.

\subsection{Proof of Proposition~\ref{prop:term_p_plus_one}}\label{proof:prop:term_p_plus_one}

Note that Sherman-Morrison-Woodbury formula for $K_\lambda\inv$ gives
\begin{align}\label{eq:decomp_k_inv}
    K_\lambda\inv = (\rho+\lambda)\inv H\inv - \tfrac{1}{n}(\rho+\lambda)\inv H\inv \Phi_{\leq p} G\inv \Phi_{\leq p}\tran H\inv.
\end{align}
Then correspondingly, we have the decomposition
\begin{align*}
    &~~~\tfrac{1}{n}\mathsf{c}_{\S_{p+1}}\tran \Phi_{\S_{p+1}}\tran K_{\lambda}\inv \Phi_{\S_{p+1}} e_J\\
    &= \tfrac{1}{(\rho+\lambda)} \underbrace{\round{ \tfrac{1}{n}\mathsf{c}_{\S_{p+1}}\tran \Phi_{\S_{p+1}}\tran H\inv \Phi_{\S_{p+1}} e_J - \tfrac{1}{n^2}\mathsf{c}_{\S_{p+1}}\tran \Phi_{\S_{p+1}}\tran  H\inv \Phi_{\leq p} G\inv \Phi_{\leq p}\tran H\inv\Phi_{\S_{p+1}} e_J}}_{:=\gD_J}.
\end{align*}

We now expand $\gD$ as a whole. Since both summands in $\gD$ involve the decomposition of $H\inv$, it is more convenient to classify terms only after expanding the entire expression.

For brevity, write
\[
\Phi := \Phi_{\S_{p+1}},
\qquad
\mathsf{c} := \mathsf{c}_{\S_{p+1}}.
\]
Using
\[
H\inv = I_n + \Delta_H + \Delta_1,
\qquad
G\inv = I + \Delta_G + \Delta_2,
\]
we obtain
\begin{align*}
\gD_J
&=
\frac{1}{n}\mathsf{c}\tran \Phi\tran (I_n+\Delta_H+\Delta_1)\Phi e_J \\
&\quad
-
\frac{1}{n^2}
\mathsf{c}\tran \Phi\tran
(I_n+\Delta_H+\Delta_1)\Phi_{\leq p}(I+\Delta_G+\Delta_2)\Phi_{\leq p}\tran
(I_n+\Delta_H+\Delta_1)\Phi e_J.
\end{align*}
We decompose the full expansion of $\gD_J$ into three disjoint classes:
\[
\gD_J = \gD_J^{\mathrm{lead}} + \gD_J^{\mathrm{corr}} + \gD_J^{\mathrm{rem}}.
\]
The leading part collects the terms with no $\Delta$-factor:
\[
\gD_J^{\mathrm{lead}}
:=
\frac{1}{n}\mathsf{c}\tran \Phi\tran \Phi e_J
-
\frac{1}{n}\mathsf{c}\tran \Phi\tran \Pi_p \Phi e_J,
\qquad
\Pi_p := \tfrac{1}{n}\Phi_{\leq p}\Phi_{\leq p}\tran.
\]
The structured correction part collects the terms that contain at least one factor $\Delta_H$ or $\Delta_G$, but no factor $\Delta_1$ or $\Delta_2$. Define
\[
\Pi_{p,G} := \tfrac{1}{n}\Phi_{\leq p}\Delta_G\Phi_{\leq p}\tran
\]
and
\[
\mathcal M_p
:=
\Big\{
\Delta_H\Pi_p,\ \Pi_p\Delta_H,\ \Delta_H\Pi_p\Delta_H,\ 
\Delta_H\Pi_{p,G},\ \Pi_{p,G}\Delta_H,\ \Delta_H\Pi_{p,G}\Delta_H,\ \Pi_{p,G}
\Big\}.
\]
Then we have  the expression
\[
\gD_J^{\mathrm{corr}}
:=
\frac{1}{n}\mathsf{c}\tran \Phi\tran \Delta_H \Phi e_j
-
\frac{1}{n}\mathsf{c}\tran \Phi\tran \Big(\sum_{M\in\mathcal M_p} M\Big)\Phi e_J.
\]

Finally, we define
\[
\gD_J^{\mathrm{rem}}
:=
\gD_J - \gD_J^{\mathrm{lead}} - \gD_J^{\mathrm{corr}}.
\]
By construction, every term in $\gD_J^{\mathrm{rem}}$ contains at least one factor $\Delta_1$ or $\Delta_2$.

Thus, the three classes are:
\begin{enumerate}
    \item $\gD_J^{\mathrm{rem}}$: terms containing $\Delta_1$ or $\Delta_2$;
    \item $\gD_J^{\mathrm{corr}}$: terms containing $\Delta_H$ or $\Delta_G$, but not $\Delta_1$ or $\Delta_2$;
    \item $\gD_J^{\mathrm{lead}}$: terms containing no $\Delta$-factor.
\end{enumerate}

We treat these three classes separately.

\paragraph{Analysis of $\gD_J^{\mathrm{rem}}$.}
An analogous analysis to equation~\eqref{eq:bound_Delta_3} immediately gives 
\begin{align}
    &~~~|\gD_J^{\mathrm{rem}}|^2 \notag\\
    &\leq \norm{\tfrac{1}{\sqrt{n}}\Phi \mathbf{c}}_2^2\norm{\tfrac{1}{\sqrt{n}}\Phi e_J}_2^2\snorm{\Delta_1}^2 \notag \\
    &~~~~+ \norm{\tfrac{1}{\sqrt{n}}\Phi \mathbf{c}}_2^2(2\snorm{H\inv}^2 \snorm{G\inv}^2  + \snorm{H\inv}^4)\snorm{\tfrac{1}{\sqrt{n}}\Phi_{\leq p}}^2\norm{\tfrac{1}{\sqrt{n}}\Phi e_J}_2^2\notag  \\
    &= O_{d,\P}(d^{-p-\delta} \log(d))\norm{\mathsf{c}}_2^2\norm{\tfrac{1}{\sqrt{n}}\Phi e_J}_2^2.
\end{align}
where we use the estimate of $\snorm{H\inv} = O_{d,\P}(1)$ (Eq.~\eqref{eq:e_inv_minus_i}), $\snorm{G\inv} = O_{d,\P}(1)$ (Eq.~\eqref{eq:g_inv_minus_i}), $\snorm{\tfrac{1}{\sqrt{n}}\Phi_{\leq p}} = O_{d,\P}(1)$ (Lemma~\ref{lem:norm_control}) and $\norm{\tfrac{1}{\sqrt{n}}\Phi \mathsf{c}}_2 = O_{d,\P}(\log(d))\norm{\mathsf{c}}_2$ (Proposition~\ref{prop:bound_y}).

Then we take the expectation on both sides and obtain
\begin{align}\label{eq:D_rem}
      \E\brac{|\gD_J^{\mathrm{rem}}|^2}     &= O_{d,\P}(d^{-p-\delta} \log^2(d))\norm{\mathsf{c}}_2^2 \norm{e_J}^2.
\end{align}

\paragraph{Analysis of $\gD_J^{\mathrm{lead}}$.}
Applying Lemma~\ref{lemma:feature_product_2} immediately gives
\begin{align*}
    \E\brac{\round{\frac{1}{n}\mathsf{c}\tran \Phi\tran \Phi e)J}^2} = (\mathsf{c}\tran e_J)^2 + O_d(n\inv) \norm{\mathsf{c}}_2^2 \norm{e_J}_2^2,
\end{align*}
and
\begin{align*}
    \E\brac{\round{\frac{1}{n}\mathsf{c}\tran \Phi\tran \Pi_p \Phi e_J}^2} =O_d(d^{-2\delta}) (\mathsf{c}\tran e_J)^2 + O_d(n\inv) \norm{\mathsf{c}}_2^2 \norm{e_J}_2^2.
\end{align*}
Therefore, we can deduce
\begin{align}\label{eq:D_lead}
    \E\brac{(\gD_J^{\mathrm{lead}} -\mathsf{c}\tran e_J)^2} &\leq 2 \E\brac{\round{\frac{1}{n}a\tran \Phi\tran \Phi e_J - (\mathsf{c}\tran e_J)}^2} + 2 \E\brac{\round{\frac{1}{n}\mathsf{c}\tran \Phi\tran \Pi_p \Phi e_J}^2}\notag\\
    &= O_d(d^{-2\delta}) (\mathsf{c}\tran e_J)^2 + O_d(n\inv) \norm{\mathsf{c}}_2^2 \norm{e_J}_2^2.
\end{align}

\paragraph{Analysis of $\gD_J^{\mathrm{corr}}$.}
We use the following proposition to bound each summand in $\gD_J^{\mathrm{corr}}$. The proof appears in Appendix~\ref{proof:control_m_p}.
\begin{proposition}\label{prop:control_m_p}
Define
\[
\Pi_p := \tfrac{1}{n}\Phi_{\leq p} \Phi_{\leq p}\tran,
\qquad
\Pi_{p,G} := \tfrac{1}{n}\Phi_{\leq p}\,\Delta_G\,\Phi_{\leq p}\tran,
\]
and the finite collection
\[
\mathcal M_p
:=\Big\{
\Delta_H\Pi_p,\ \Pi_p\Delta_H,\ \Delta_H\Pi_p\Delta_H,\ 
\Delta_H\Pi_{p,G},\ \Pi_{p,G}\Delta_H,\ \Delta_H\Pi_{p,G}\Delta_H,\ \Pi_{p,G} \Big\}.
\]
Then, for every $M\in\mathcal M_p$, the following holds 
\begin{align}
   \E\brac{\left(\frac{1}{n^2}\,\mathsf{c} \tran \Phi^\top M \Phi\,e_J\right)^2}
=O_d(d^{-2\delta}) (\mathsf{c}\tran e_J)^2 + O_d(n\inv) \norm{\mathsf{c}}_2^2 \norm{e_J}_2^2. 
\end{align}
\end{proposition}
Using Cauchy Schwartz inequality we obtain 
\begin{align}\label{eq:D_corr}
     (\gD_J^{\mathrm{corr}})^2 \leq |\gM_p| \sum_{M\in \gM_p} \E\brac{\left(\frac{1}{n^2}\,\mathsf{c} \tran \Phi^\top M \Phi\,e_J\right)^2} =O_d(d^{-2\delta}) (\mathsf{c}\tran e_J)^2 + O_d(n\inv) \norm{\mathsf{c}}_2^2 \norm{e_J}_2^2.
\end{align}
Lastly, combining equation~\eqref{eq:D_lead}, \eqref{eq:D_rem} and \eqref{eq:D_corr} and using Cauchy-Schwartz inequality again gives
\begin{align}
   \E\brac{(\gD - \mathsf{c}\tran e_J)^2} &\leq 3\E\brac{(\gD_J^{\mathrm{lead}} -\mathsf{c}\tran e_J)^2} + 3  \E\brac{(\gD_J^{\mathrm{rem}})^2} +  3\E\brac{(\gD_J^{\mathrm{corr}})^2} \notag\\
   &= O_d(d^{-2\delta }) (\mathsf{c}\tran e_J)^2 + O_{d,\P}(d^{-p-\delta})\log(d) \norm{\mathsf{c}}_2^2 \norm{e_J}_2^2.
\end{align}
Using hypercontractivity, we have the following holds uniformly for all $\{u^{[j]}\}_{j=1}^d$
\begin{align}
   \E\brac{(\gD - \mathsf{c}\tran e_J)^2} = O_d(d^{-2\delta+\epsilon }) (\mathsf{c}\tran e_J)^2 + O_{d,\P}(d^{-p-\delta+\epsilon})\norm{\mathsf{c}}_2^2 \norm{e_J}_2^2.
\end{align}
With $\norm{e_J}_2 = O_d(d^{-p-1})$, taking the summation over $J\in\S_p$ and using Markov's inequality completes the proof.

\subsection{Proof of Proposition~\ref{prop:term_p_plus_one_q}}\label{proof:prop:term_p_plus_one_q}

Using the decomposition of \(K_\lambda^{-1}\) in \eqref{eq:decomp_k_inv}, we obtain
\begin{align}\label{eq:term-p-plus-one-q-split}
&~~~~\frac1n \mathsf{c}_{\Q}^\top \Phi_{\Q}^\top K_\lambda^{-1}\Phi_{\S_{p+1}} e_J \notag\\
&=
\frac{1}{(\rho+\lambda)n}\,
\mathsf{c}_{\Q}^\top \Phi_{\Q}^\top H^{-1}\Phi_{\S_{p+1}} e_J \notag\\
&~~~~-
\frac{1}{(\rho+\lambda)n^2}\,
\mathsf{c}_{\Q}^\top \Phi_{\Q}^\top H^{-1}\Phi_{\le p}G^{-1}\Phi_{\le p}^\top H^{-1}\Phi_{\S_{p+1}} e_J.
\end{align}
It therefore suffices to bound the two terms on the right-hand side.

We expand \(H^{-1}\) and \(G^{-1}\) using \eqref{eq:h_inv} and \eqref{eq:g_inv}, and decompose
\(\Phi_{\le p}\) into exact-degree blocks exactly as in the proof of
Proposition~\ref{prop:term_p_plus_one}. In particular, every occurrence of an
off-diagonal factor is expanded as
\[
\offd(B)=B-\Diag(B).
\]
Thus each of the two terms in \eqref{eq:term-p-plus-one-q-split} becomes a finite sum
of expressions of the form
\[
c_\eta\, \mathsf{c}_{\Q}^\top M_\eta e_J,
\]
where \(c_\eta\) is deterministic, the number of summands is \(O_{p,c_H,c_G}(1)\), and
\(M_\eta\) is a matrix chain of the form covered by
Lemma~\ref{lemma:feature_product_2}. The left endpoint family of \(M_\eta\) is always
\(\Q\), while the right endpoint family is always \(\S_{p+1}\).

Now the key simplification is that
\[
\Q\cap \S_{p+1}=\varnothing.
\]
Hence the endpoint families in every such chain are distinct, so the leading term in
Lemma~\ref{lemma:feature_product_2} vanishes automatically. Therefore, for every \(\eta\),
\[
\E\!\left[\bigl(\mathsf{c}_{\Q}^\top M_\eta e_J\bigr)^2\right]
=
O_d\!\left((n w_\eta^{(1)}w_\eta^{(m_\eta+1)})^2\,n^{-1}\,
\|\mathsf{c}_{\Q}\|_2^2\|e_J\|_2^2\right).
\]
The deterministic coefficients \(c_\eta\) are estimated exactly as in the proof of
Proposition~\ref{prop:term_p_plus_one}; in particular,
\[
|c_\eta|^2 (n w_\eta^{(1)}w_\eta^{(m_\eta+1)})^2 n^{-1}
=
O_d(n\inv).
\]
Since there are only \(O_{p,c_H,c_G}(1)\) summands, we obtain
\begin{equation}\label{eq:term-p-plus-one-q-H-bound}
\E\brac{\left(
\frac1n \mathsf{c}_{\Q}^\top \Phi_{\Q}^\top H^{-1}\Phi_{\S_{p+1}} e_J
\right)^2}
=
O_{d,\P}\!\left(d^{-p-\delta+\epsilon}\,
\|\mathsf{c}_{\Q}\|_2^2\|e_J\|_2^2\right),
\end{equation}
and similarly
\begin{equation}\label{eq:term-p-plus-one-q-HGH-bound}
\E\brac{\left(
\frac1{n^2} \mathsf{c}_{\Q}^\top \Phi_{\Q}^\top H^{-1}\Phi_{\le p}G^{-1}\Phi_{\le p}^\top H^{-1}\Phi_{\S_{p+1}} e_J
\right)^2}
=
O_{d,\P}\!\left(d^{-p-\delta+\epsilon}\,
\|\mathsf{c}_{\Q}\|_2^2\|e_J\|_2^2\right).
\end{equation}
Substituting \eqref{eq:term-p-plus-one-q-H-bound} and
\eqref{eq:term-p-plus-one-q-HGH-bound} into \eqref{eq:term-p-plus-one-q-split}
and using Cauchy-Schwartz inequality gives
\begin{align}
   \E\brac{\round{\frac1n \mathsf{c}_{\Q}^\top \Phi_{\Q}^\top K_\lambda^{-1}\Phi_{\S_{p+1}} e_J}^2} =O_{d,\P}\!\left(d^{-p-\delta+\epsilon}\,
\|\mathsf{c}_{\Q}\|_2^2\|e_J\|_2^2\right).
\end{align}
With $\norm{e_J}_2 = O_d(d^{-p-1})$, then taking the summation over $J\in\S_p$ and applying Markov's inequality completes the proof.

\subsection{Proof of Proposition~\ref{prop:h_g_inv}}\label{proof:prop:h_g_inv}

Note that for any $\Delta = o_{d,\P}(1)$, we can express $(I+\Delta)\inv$ using the Neumann series:
\begin{align*}
    (I + \Delta)\inv = I + \sum_{k=1}^\infty (-\Delta)^k.
\end{align*}

We compute the inverse of $H$ \eqref{eq:H} using Neumann series yielding
\begin{align}\label{eq:raw_h_inv}
    H\inv &= I_n + \sum_{q=1}^\infty \round{\sum_{j=p+1}^{2p+1}-\underbrace{(\rho+\lambda)\inv d^j \theta_j }_{:=c_{j}}d^{-j}{ \offd(\Phi_{\S_{j}}\Phi_{\S_{j}}\tran)} - (\rho+\lambda)\inv\Delta}^{q}.
\end{align}
Note that the triangle inequality on Eq.~\eqref{eq:kernel_to_mono} and \eqref{eq:kernel_to_mono_general} shows for $j\geq p+1$ we have 
\begin{align}\label{eq:off_p+1}
    d^{-j}{ \offd(\Phi_{\S_{j}}\Phi_{\S_{j}}\tran)} = O_{d,\P}(d^{(p+\delta-j)/2+\epsilon}) = o_{d,\P}(1),
\end{align}
where $\offd(\cdot)$ denotes the off-diagonal component of a matrix.
Next, we show that we can truncate higher-order terms in $H\inv$ at constant order to achieve a desired error on the order of $O_{d,\P}(n^{-\tfrac{1}{2}})$.

We set a constant $c_H := \ceil{\frac{p-1+2\delta+2\epsilon}{1-\delta-2\epsilon}}$.  
 Using the triangle inequality and the formula for the summation of a geometric series, we can derive
\begin{align*}
    \snorm{\sum_{q=c_H+1}^\infty\round{\sum_{j=p+1}^{2p+1}-c_{j}d^{-j}{ \offd(\Phi_{\S_{j}}\Phi_{\S_{j}}\tran)} + (\rho+\lambda)\inv\Delta}^{q}} &\leq \sum_{q=c_H+1}^\infty (O_{d,\mathbb{P}}(d^{(\delta-1)/2+\epsilon}))^q\\
&= O_{d,\P}(d^{\round{(\delta-1)/2+\epsilon}(c_H+1)})\\
&=O_{d,\mathbb{P}}(n^{-\tfrac{1}{2}}).
\end{align*}
We expand the binomial for each $q$ and bound the terms that contain $\Delta$, which yields
\begin{align}\label{eq:h_inv_1}
    H\inv &= I_n+\sum_{q=1}^{c_H}\round{\sum_{j=p+1}^{2p+1}-c_{j}d^{-j}{ \offd(\Phi_{\S_{j}}\Phi_{\S_{j}}\tran)} + (\rho+\lambda)\inv\Delta}^{q}\notag\\
    &=I_n+\sum_{q=1}^{c_H}\round{\sum_{j=p+1}^{2p+1}-c_{j}d^{-j}{ \offd(\Phi_{\S_{j}}\Phi_{\S_{j}}\tran)}}^{q}+\Delta_1,
\end{align}
for some $\Delta_1$ satisfying
$\snorm{\Delta_1} = O_{d,\P}(n^{-\tfrac{1}{2}})$. We denote $$\Delta_H:=\sum_{q=1}^{c_H}\round{\sum_{j=p+1}^{2p+1}-c_{j}d^{-j}{ \offd(\Phi_{\S_{j}}\Phi_{\S_{j}}\tran)}}^{q},$$
which satisfies $ \snorm{\Delta_H} = O_{d,\P}(d^{-\tfrac{1-\delta}{2}+\epsilon})$ with a quick calculation using Eq.~\eqref{eq:off_p+1}.

We now plug the  expression \eqref{eq:h_inv_1} for $H\inv$ into $G$ in Eq.~\eqref{eq:smw_<_p}. We further substitute $\frac{\Phi_{\leq p}\tran \Phi_{\leq p}}{n}$ for $ \offd\round{\frac{\Phi_{\leq p}\tran \Phi_{\leq p}}{n}} + I_{L(p)}$ and we obtain
\begin{align*}
    G = (\rho+\lambda)\Lambda  + I + \offd\round{\frac{\Phi_{\leq p}\tran \Phi_{\leq p}}{n}}  + \frac{\Phi_{\leq p}\tran}{\sqrt{n}}\Delta_H\frac{\Phi_{\leq p}}{\sqrt{n}} + \Delta_2,
\end{align*}
with $\snorm{\Delta_2} \leq \snorm{\Phi_{\leq p} / \sqrt{n}}^2 \snorm{\Delta_1} = O_{d,\P}(n^{-\tfrac{1}{2}})$.

Expressing $G\inv$ using Neumann series again gives us
\begin{align*}
    G\inv = I +\sum_{r=1}^\infty \round{ -(\rho+\lambda)\Lambda -\offd\round{\frac{\Phi_{\leq p}\tran \Phi_{\leq p}}{n}}  - \frac{\Phi_{\leq p}\tran}{\sqrt{n}}\Delta_H\frac{\Phi_{\leq p}}{\sqrt{n} } + \Delta_2}^{r},
\end{align*}
Similarly, we can truncate the infinite summation while maintaining the same error order. Note that the spectral norm of each item in the parentheses is of the order $o_{d,\P}(1)$, which can be deduced using equation~\eqref{eq:off_p+1} and the following estimates:
\begin{align*}
  &\snorm{\Lambda} = O_d(d^{-\delta}), \qquad \snorm{\offd\round{\frac{\Phi_{\leq p}\tran \Phi_{\leq p}}{n}}} = O_{d,\P}(d^{-\tfrac{\delta}{2} + \epsilon}),\\
  &\snorm{\frac{\Phi_{\leq p}\tran}{\sqrt{n}}} = O_{d,\P}(1), \qquad \snorm{\Delta_2} = O_{d,\P}(n^{-1/2}),
\end{align*}
where particularly the second estimate can be  deduced from Lemma~\ref{lemma:gram_matrix} since $\diag\round{\frac{\Phi_{\leq p}\tran \Phi_{\leq p}}{n}} = I_n$.

 Therefore, an analogous argument for $H\inv$ shows that we can find a constant $c_G = \max(c_H, \ceil{\frac{p+\epsilon}{\delta-2\epsilon}})$ and a matrix $\Delta_3$ that satisfies $\snorm{\Delta_3} = O_{d,\P}(n^{-\tfrac{1}{2}})$ such that the following holds
\begin{align}\label{eq:g_inv_1}
        G\inv = I +\sum_{r=1}^{c_G} \round{ -(\rho+\lambda)\Lambda\inv -\offd\round{\frac{\Phi_{\leq p}\tran \Phi_{\leq p}}{n}}  - \frac{\Phi_{\leq p}\tran}{\sqrt{n}}\Delta_H \frac{\Phi_{\leq p}}{\sqrt{n} }}^{r} +\Delta_3.
\end{align}
Then we conclude the proof for Eq.~\eqref{eq:g_inv}.

\subsection{Proof of Proposition~\ref{prop:control_error_delta}}\label{proof:prop:control_error_delta}
To simplify the notation, write
\[
W:=n^{-1/2}\Phi_{\le p},
\qquad
V_j:=|\S_j|^{-1/2}\Phi_{\S_j}\quad (p+1\le j\le 2p+1).
\]
Then we have
\[
\frac1n\Phi_{\le p}^\top\Phi_{\le p}=W^\top W
=I+\offd(W^\top W),
\]
and for \(p+1\le j\le 2p+1\), 
\[
-\frac{c_j}{d^j}\offd(\Phi_{\S_j}\Phi_{\S_j}^\top)
=
c'_j\,\offd(V_jV_j^\top),
\qquad
c'_j:=-c_j\,\frac{|\S_j|}{d^j}=O_d(1).
\]
Hence every summand in \(\Delta_H\) is a finite product of matrices
\(\offd(V_jV_j^\top)\) with deterministic coefficient of the order \(O_d(1)\).

We also write
\[
R_G:=A_G-P_G(\Lambda),
\]
and denote
\[
B_0:=-(\rho+\lambda)\Lambda,
\qquad
B_1:=\offd(W^\top W),
\qquad
B_2:=W^\top\Delta_H W.
\]
Then we have
\[
\Delta_G=\sum_{r=1}^{c_G}(B_0+B_1+B_2)^r,
\]
and consequently
\[
A_H=B_2,
\qquad
A_G=(I+B_1)\Delta_G,
\qquad
A_{HG}=B_2\Delta_G.
\]
We now decompose the low-degree space according to exact degree.
For \(0\le q\le p\), let
\[
W_q:=n^{-1/2}\Phi_{\S_q},
\]
and let \(P_q:\R^{L(p)}\to\R^{L(p)}\) be the orthogonal projector onto the
coordinates indexed by \(\S_q\). Then
\[
W=[W_0,\dots,W_p],
\qquad
W_q=WP_q,
\qquad
I=\sum_{q=0}^p P_q,
\qquad
\Lambda=\sum_{q=0}^p \beta_q P_q,
\quad
\beta_q:=\frac{d^q}{n}.
\]
We have the decomposition
\begin{align}\label{eq:B1}
 B_1
=\offd(W^\top W)
=
\sum_{q=0}^p \offd(W_q^\top W_q)
+
\sum_{\substack{0\le q,q'\le p\\ q\neq q'}} W_q^\top W_{q'}.   
\end{align}
Next, write
\[
\Delta_H=\sum_{\omega\in\Omega_H} c_\omega H_\omega,
\qquad
H_\omega:=\prod_{\ell=1}^{m_\omega}
\offd\!\bigl(V_{j_{\omega,\ell}}V_{j_{\omega,\ell}}^\top\bigr),
\]
where \(\Omega_H\) is a finite index set and each \(c_\omega=O_d(1)\).
Then we have
\begin{align}\label{eq:B2}
    B_2
=
\sum_{\omega\in\Omega_H}\sum_{q,q'=0}^p
c_\omega\,W_q^\top H_\omega W_{q'}.
\end{align}
Thus every word contributing to \(A_H\), \(R_G\), or \(A_{HG}\) can be expanded
using~\eqref{eq:B1}, \eqref{eq:B2} and
\begin{align}\label{eq:B0}
    B_0=-(\rho+\lambda)\sum_{q=0}^p \beta_q P_q.
\end{align}
Since \(p,c_H,c_G\) are fixed, this produces only \(O_{p,c_H,c_G}(1)\) summands.

We also define the blockwise quantities
\begin{align}\label{eq:Bt}
 \mathcal B_t(b_1,b_2)
:=
\sum_{q=0}^p \beta_q^{2t}
\bigl((P_qb_1)^\top(P_qb_2)\bigr)^2,
\qquad t\ge 0.   
\end{align}

Next we present the key claim. The proof appears in Appendix~\ref{proof:cl_less_p}.
\begin{proposition}\label{cl:less_p}
Let \(T\) be one expanded word contributing to \(A_H\), \(R_G\), or \(A_{HG}\).
Assume that \(T\) contains exactly \(s\) copies of \(B_0\), and that \(T\) contains
at least one off-diagonal factor, namely either a factor \(\offd(W_q^\top W_q)\)
or one of the factors inside some \(H_\omega\). Then
\begin{align}\label{eq:b_t_p}
    \E\!\left[(b_1^\top T b_2)^2\right]
\le
\sum_{t=0}^s O_d(d^{-2\delta})\,\mathcal B_t(b_1,b_2)
+
O_d\!\left(n^{-1}\|b_1\|_2^2\|b_2\|_2^2\right).
\end{align}
\end{proposition}

We now return to \(A_H\), \(R_G\), and \(A_{HG}\).

The pure \(\Lambda\)-words occur only when we take the identity part of \(W^\top W\)
and every factor in \(\Delta_G\) is \(B_0\); collecting these terms gives precisely
\[
P_G(\Lambda)=\sum_{r=1}^{c_G}\bigl(-(\rho+\lambda)\Lambda\bigr)^r.
\]
Hence, by definition, every word contributing to \(R_G\) contains at least one
off-diagonal factor. The same is clearly true for every word in \(A_H=B_2\) and
in \(A_{HG}=B_2\Delta_G\).

Applying the claim word-by-word, and using that each word contains at most \(c_G\)
copies of \(B_0\), we obtain, for each \(X\in\{A_H,R_G,A_{HG}\}\),
\begin{align}\label{eq:b_X_b}  
\E\!\left[(b_1^\top X b_2)^2\right]
\le
\sum_{t=0}^{c_G} O_d(d^{-2\delta})\,\mathcal B_t(b_1,b_2)
+
O_d\!\left(n^{-1}\|b_1\|_2^2\|b_2\|_2^2\right).
\end{align}

The final piece of the proof is the upper bound for $\sum_{t=0}^s \mathcal B_t(b_1,b_2)$. Combining the bound below with Eq.~\eqref{eq:b_X_b} completes the proof.
\begin{claim}\label{cl:b_t}
    The following estimate holds:
    \begin{align}
        \sum_{t=0}^{c_G}\gB_t(\mathsf{c}_{\leq p},\hat{e}_J) = O_d(1) \round{\mathsf{c}_{\leq p} \hat{e}_J}^2 +  O_d(d^{-4}) \cdot  \round{\sum_{i\in J} \mathsf{c}_{J\setminus \{i\}}^2}.
    \end{align}
\end{claim}

\subsection{Proof of Claim~\ref{cl:small_cpe}}\label{proof:small_cpe}
Expanding $P_G(\Lambda)$ gives
    \begin{align}
         \mathsf{c}_{\leq p}\tran P_G(\Lambda) \hat{e}_J &= \mathsf{c}_{\leq p}\tran\round{\sum_{r=1}^{c_G} \bigl(-(\rho+\lambda) \Lambda\bigr)^r}\hat{e}_J \notag\\
         &= \sum_{r=1}^{c_G} (-(\rho+\lambda))^r \mathsf{c}_{\leq p}\tran \Lambda^r \hat{e}_J.
    \end{align}
Talking the square and applying Cauchy–Schwarz inequality yields
\begin{align*}
    \round{    \mathsf{c}_{\leq p}\tran P_G(\Lambda) \hat{e}_J}^2 \leq c_G\sum_{r=1}^{c_G} (\rho+\lambda)^{2r} \round{\mathsf{c}_{\leq p}\tran \Lambda^r \hat{e}_J}^2.
\end{align*}
For each $r = 1,\ldots,c_G$,  we have
\begin{align*}
   \round{\mathsf{c}_{\leq p}\tran \Lambda^r \hat{e}_J}^2 &=  \round{\mathsf{c}_{\leq p}\tran (nD)^{-r} D\inv E_{:L(p), J}}^2\\
   &=   \round{  \sum_{i\in [d] \setminus J}(nD_{J\sqcup \{i\}})^{-r}\mathsf{c}_{J\sqcup \{i\}} u_i+{ (nD_{J\setminus \{i\}})^{-r} d^{-1}D'_{J,J} D_{J\setminus\{i\}}^{-1} \sum_{i\in S}\mathsf{c}_{J\setminus \{i\}} u_i}}^2\\
   &\leq O_d(d^{-2r\round{p+\delta-|J|-1}}) \underbrace{\round{ \sum_{i\in [d] \setminus J}\mathsf{c}_{J\sqcup \{i\}} u_i}^2}_{:=\gI^2} \\
   &~~~~+ O_d(d^{-2r\round{p+\delta-|J|+1}}) \underbrace{\round{ d^{-1}D'_{J,J} D_{J\setminus\{i\}}^{-1}\sum_{i\in J}\mathsf{c}_{J\setminus \{i\}} u_i}^2}_{:=\gJ^2},
\end{align*}
where we use the basic inequality $(a+b)^2 \leq 2(a^2 + b^2)$ in the last inequality.

Now using Young's inequality $ax^2 +by^2 \leq 2a(x+y)^2 + (2a+b)y^2$, we have
\begin{align*}
    &~~~O_d(d^{-2r\round{p+\delta-|J|-1}}) \gI^2 + O_d(d^{-2r\round{p+\delta-|J|+1}}) \gJ^2 \\
    &\leq  O_d(d^{-2r\round{p+\delta-|J|-1}}) (\gI + \gJ)^2 + O_d(d^{-2r\round{p+\delta-|J|+1}}) \gJ^2\\
    &= O_d(d^{-2r\round{p+\delta-|J|-1}})  \round{\mathsf{c}_{\leq p} \hat{e}_J}^2 + O_d(d^{-2r\round{p+\delta-|J|+1}})  O_d(d^{-4}) \cdot  \round{\sum_{i\in J} \mathsf{c}_{J\setminus \{i\}}^2},
\end{align*}
where we apply the bound for $\gJ^2$ from equation~\eqref{eq:j_square} in the last equality.

When $r=1$ and $|J| = p-1$, we have the smallest upper bound which is
\begin{align}
 O_d(d^{-2\delta})  \round{\mathsf{c}_{\leq p} \hat{e}_J}^2 + {O}_d(d^{-8-2\delta})\round{\sum_{i\in J} \mathsf{c}_{J\setminus \{i\}}^2}.
\end{align}
Then we complete the proof.

\subsection{Proof of Proposition~\ref{prop:control_m_p}}\label{proof:control_m_p}

Set
\[
V:=V_{p+1}=|\S_{p+1}|^{-1/2}\Phi_{\S_{p+1}},
\qquad
\gamma:=\frac{|\S_{p+1}|}{n},
\]
and define
\[
Q_M:=\frac1n\,a^\top \Phi_{\S_{p+1}}^\top M \Phi_{\S_{p+1}}\,e,
\qquad M\in\mathcal N_p.
\]
Since
\[
\Phi_{\S_{p+1}}=|\S_{p+1}|^{1/2}V,
\qquad
\Pi_p=\frac1n\Phi_{\le p}\Phi_{\le p}^\top=WW^\top,
\qquad
\Pi_{p,G}=\frac1n\Phi_{\le p}\Delta_G\Phi_{\le p}^\top=W\Delta_GW^\top,
\]
every expanded word contributing to \(Q_M\) has the form
\begin{equation}\label{eq:Mp-word-form}
\gamma\, a^\top V^\top H_0\,W\,\Xi\,W^\top H_1\,V\,e,
\end{equation}
where \(H_0,H_1\) are either \(I\) or one of the words appearing in the expansion
of \(\Delta_H\), and \(\Xi\) is either \(I\) or one of the words appearing in the
expansion of \(\Delta_G\).

As in the low-degree endpoint proof, we decompose the low-degree space by exact degree:
for \(0\le q\le p\),
\[
W_q:=WP_q=n^{-1/2}\Phi_{\mathcal S_q},
\qquad
\Lambda=\sum_{q=0}^p \beta_qP_q,
\qquad
\beta_q:=\frac{d^q}{n}.
\]
Expanding every occurrence of \(B_0,B_1,B_2\), and then resolving every occurrence of
\(W\) and \(W^\top\) into the exact-degree blocks \(W_q,W_q^\top\), we obtain a finite
decomposition
\begin{equation}\label{eq:Mp-decomposition}
Q_M=\sum_{\eta\in\Omega}\gamma\,\alpha_\eta\, a^\top M_\eta e,
\end{equation}
where \(|\Omega|=O_{p,c_H,c_G}(1)\), and for each \(\eta\in\Omega\):

\begin{enumerate}
\item \(\alpha_\eta\) is the product of all deterministic scalar coefficients coming
from the chosen words in \(\Delta_H\) and \(\Delta_G\), together with the factors
\(\beta_q\) contributed by the selected copies of \(B_0\). In particular,
\[
|\alpha_\eta|=O_d(1).
\]

\item \(M_\eta\) is the corresponding normalized matrix chain obtained after removing
those scalar coefficients. Concretely, \(M_\eta\) is a product of matrices chosen from
\[
V^\top,\quad V,\quad W_q,\quad W_q^\top,\quad
\offd(V_jV_j^\top),\quad \offd(W_q^\top W_q),\quad W_q^\top W_{q'},
\]
with endpoint family \(\S_{p+1}\) on both sides.
\end{enumerate}

We claim that every summand in \eqref{eq:Mp-decomposition} satisfies
\begin{equation}\label{eq:Mp-summand-bound}
\E\!\left[
\bigl(\gamma\,\alpha_\eta\, a^\top M_\eta e\bigr)^2
\right]
\le
o_d(1)\,(a^\top e)^2
+
O_d\!\left(n^{-1}\|a\|_2^2\|e\|_2^2\right).
\end{equation}

There are two cases.

\medskip
\noindent\emph{Case 1: \(M_\eta\) contains at least one off-diagonal replacement.}
Let
\[
\mathfrak S^\eta_1,\dots,\mathfrak S^\eta_{m_\eta+1}
\]
be the sequence of families appearing in the chain \(M_\eta\), and let
\(\mathcal V_\eta,\mathcal E_\eta\) be the corresponding sets of sample-space and
feature-space off-diagonal replacements.
Then \(M_\eta\) satisfies the hypotheses of Corollary~\ref{cor:off_diag}.

The crucial point is that \(M_\eta\) always contains at least one internal low-degree
family \(\mathcal S_q\): indeed, every word in \eqref{eq:Mp-word-form} contains the
middle factor \(W\,\Xi\,W^\top\), and after exact-degree expansion this contributes
at least one block \(W_q\) or \(W_q^\top\). Therefore
\[
\kappa_{M_\eta}
:=
\min_{2\le k\le m_\eta}\frac{|\mathfrak S^\eta_k|}{n}
\le
\frac{|\mathcal S_q|}{n}
=
O_d(d^{-2\delta}),
\]
since \(p\) is fixed and \(|\mathcal S_q|\asymp d^q=O_d(d^{-\delta})\) for all \(q\le p\).

Applying Corollary~\ref{cor:off_diag} to \(M_\eta\), we obtain
\begin{equation}\label{eq:Mp-case1-cor}
\E\!\left[(a^\top M_\eta e)^2\right]
\le
O_d\!\bigl(\kappa_{M_\eta}|\mathcal E_\eta|\bigr)\,
(n w^{(1)}w^{(m_\eta+1)})^2 (a^\top e)^2
+
O_d\!\left((n w^{(1)}w^{(m_\eta+1)})^2 n^{-1}\|a\|_2^2\|e\|_2^2\right).
\end{equation}
Here the endpoint family is \(\S_{p+1}\) on both sides, so
\[
w^{(1)}=w^{(m_\eta+1)}=|\S_{p+1}|^{-1/2},
\qquad
(n w^{(1)}w^{(m_\eta+1)})^2
=
\left(\frac{n}{|\S_{p+1}|}\right)^2
=
\gamma^{-2}.
\]
Also \(|\mathcal E_\eta|=O_{p,c_H,c_G}(1)\), since the total number of factors is
uniformly bounded. Hence \eqref{eq:Mp-case1-cor} simplifies to
\[
\E\!\left[(a^\top M_\eta e)^2\right]
\le
O_d(d^{-2\delta})\,\gamma^{-2}(a^\top e)^2
+
O_d\!\left(\gamma^{-2}n^{-1}\|a\|_2^2\|e\|_2^2\right).
\]
Multiplying by \(\gamma^2|\alpha_\eta|^2\), and using \(|\alpha_\eta|=O_d(1)\),
yields \eqref{eq:Mp-summand-bound}.

\medskip
\noindent\emph{Case 2: \(M_\eta\) contains no off-diagonal replacement.}
This can happen only when \(M=\Pi_{p,G}\), \(H_0=H_1=I\), and the chosen word in
\(\Delta_G\) is a pure \(B_0\)-word. Indeed, every occurrence of \(\Delta_H\) contributes
an off-diagonal factor, and every occurrence of \(B_1\) does as well. Thus
\[
\Xi=B_0^r
\qquad\text{for some }1\le r\le c_G,
\]
and \eqref{eq:Mp-decomposition} reduces to
\begin{equation}\label{eq:Mp-pure-B0}
Q_M
=
\gamma\sum_{q=0}^p \alpha_q\, a^\top V^\top W_qW_q^\top V e,
\qquad
\alpha_q=O_d(\beta_q^r).
\end{equation}
For each fixed \(q\), the chain \(V^\top W_qW_q^\top V\) has endpoint family
\(\S_{p+1}\) on both sides and only one low-degree family \(\mathcal D_q\). Therefore
Lemma~\ref{lemma:feature_product_2} gives
\[
\E\!\left[
\bigl(a^\top V^\top W_qW_q^\top V e\bigr)^2
\right]
\le
C_d\,\gamma^{-2}
\Bigl((a^\top e)^2+n^{-1}\|a\|_2^2\|e\|_2^2\Bigr).
\]
Using Cauchy--Schwarz in \eqref{eq:Mp-pure-B0}, we obtain
\[
\E[Q_M^2]
\le
C_p\sum_{q=0}^p |\alpha_q|^2\,\gamma^2
\E\!\left[
\bigl(a^\top V^\top W_qW_q^\top V e\bigr)^2
\right].
\]
Hence we have
\[
\E[Q_M^2]
\le
C_d\left(\sum_{q=0}^p |\alpha_q|^2\right)
\Bigl((a^\top e)^2+n^{-1}\|a\|_2^2\|e\|_2^2\Bigr).
\]
Since \(|\alpha_q|=O_d(\beta_q^r)\), \(r\ge 1\), and
\[
\max_{0\le q\le p}\beta_q=\frac{d^p}{n}=O_d(d^{-\delta})
\]
we have
\[
\sum_{q=0}^p |\alpha_q|^2=O_d(d^{-2\delta}),
\]
so \eqref{eq:Mp-summand-bound} follows in this case as well.

Combining the two cases proves \eqref{eq:Mp-summand-bound}. Finally, since
\(|\Omega|=O_{p,c_H,c_G}(1)\), Cauchy--Schwarz in \eqref{eq:Mp-decomposition} yields,
for every \(M\in\mathcal N_p\),
\[
\E\!\left[
\left(
\frac1n\,a^\top \Phi_{\S_{p+1}}^\top M \Phi_{\S_{p+1}}\,e
\right)^2
\right]
\le
O_d(d^{-2\delta})\,(a^\top e)^2
+
O_d\!\left(n^{-1}\|a\|_2^2\|e\|_2^2\right).
\]
This proves Proposition~\ref{prop:control_m_p}.

\subsection{Proof of Proposition~\ref{cl:less_p}}\label{proof:cl_less_p}
Expanding \(T\) using \eqref{eq:B1}, \eqref{eq:B2}, and the block decomposition of \(B_0\) \eqref{eq:B0},
we may write
\begin{align}\label{eq:b_T_b_inter}
    b_1^\top T b_2
=
\sum_{\eta\in\Omega(T)}
\alpha_\eta\,
\bigl((P_{q_-^\eta}b_1)^\top M_\eta (P_{q_+^\eta}b_2)\bigr),
\end{align}

where \(\Omega(T)\) is a finite set of cardinality \(O_{p,c_H,c_G}(1)\),
each coefficient \(\alpha_\eta\) is deterministic and has the form
\begin{align}\label{eq:alpha_eta_2}
    \alpha_\eta
=
O_d\!\Bigl(\prod_{\ell=1}^s \beta_{u_{\eta,\ell}}\Bigr),
\qquad
u_{\eta,\ell}\in\{0,\dots,p\},
\end{align}
and each \(M_\eta\) is a normalized matrix chain whose families belong to $\mathcal S_0,\ldots,\S_{2p+1}$ 
with the same off-diagonal replacements as in the original word \(T\).
In particular, Lemma~\ref{lemma:feature_product_2} and
Corollary~\ref{cor:off_diag} apply to each \(M_\eta\).

Fix one \(\eta\in\Omega(T)\).

If \(q_-^\eta\neq q_+^\eta\), then the endpoint families differ, and
Lemma~\ref{lemma:feature_product_2} gives
\begin{align}\label{eq:M_eta_noneq}
 \E\!\left[
\bigl((P_{q_-^\eta}b_1)^\top M_\eta (P_{q_+^\eta}b_2)\bigr)^2
\right]
=
O_d\!\left(
n^{-1}\|P_{q_-^\eta}b_1\|_2^2\|P_{q_+^\eta}b_2\|_2^2
\right).   
\end{align}

Now suppose \(q_-^\eta=q_+^\eta=q\).

If every low-degree family occurring in \(M_\eta\) is equal to \(\mathcal S_q\),
then \(M_\eta\) has the same low-degree family at both endpoints and, by assumption,
contains at least one off-diagonal replacement. Hence
Corollary~\ref{cor:off_diag} applies and yields
\[
\E\!\left[
\bigl((P_qb_1)^\top M_\eta (P_qb_2)\bigr)^2
\right]
=
O_d(d^{2(q-p-\delta)})\,\bigl((P_qb_1)^\top(P_qb_2)\bigr)^2
+
O_d\!\left(
n^{-1}\|P_qb_1\|_2^2\|P_qb_2\|_2^2
\right).
\]
If \(M_\eta\) contains another low-degree family \(\mathcal S_{q'}\) with \(q'\neq q\),
then \(M_\eta\) contains at least two distinct families of size \(o(n)\), namely
\(\mathcal S_q\) and \(\mathcal S_{q'}\). Therefore we are in the second case of
Lemma~\ref{lemma:feature_product_2}, and
\[
\E\!\left[
\bigl((P_qb_1)^\top M_\eta (P_qb_2)\bigr)^2
\right]
=
O_d(d^{2(q' \wedge q) - 2p -2\delta})\,\bigl((P_qb_1)^\top(P_qb_2)\bigr)^2
+
O_d\!\left(
n^{-1}\|P_qb_1\|_2^2\|P_qb_2\|_2^2
\right).
\]
Thus, whenever the endpoints agree, we have the uniform bound
\begin{align}\label{eq:M_eta_eq}
 \E\!\left[
\bigl((P_qb_1)^\top M_\eta (P_qb_2)\bigr)^2
\right]
\le
O_d(d^{2(q - p - \delta)})\,\bigl((P_qb_1)^\top(P_qb_2)\bigr)^2
+
O_d\!\left(
n^{-1}\|P_qb_1\|_2^2\|P_qb_2\|_2^2
\right).   
\end{align}
It remains to account for the coefficient \(\alpha_\eta\).
If the endpoints agree and are equal to \(q\), let \(t_\eta(q)\in\{0,\dots,s\}\)
denote the number of copies of \(B_0\) that act on the degree-\(q\) block in the
summand indexed by \(\eta\). Then, by Eq.~\eqref{eq:alpha_eta_2}, we have
\begin{align}\label{eq:alpha_eta_est}
 |\alpha_\eta|^2
\lesssim_d
\beta_q^{2t_\eta(q)}\,\beta_p^{\,2(s-t_\eta(q))}.  
\end{align}
Since \(\beta_p=d^p/n=o_d(1)\), the factor \(\beta_p^{2(s-t_\eta(q))}\) is
\(o_d(1)\) whenever \(t_\eta(q)<s\), while if \(t_\eta(q)=s\) it is simply \(1\).
Combining \eqref{eq:M_eta_eq} with \eqref{eq:alpha_eta_est}, and using \eqref{eq:M_eta_noneq} when the endpoints differ,
we obtain
\[
\E\!\left[
\bigl(\alpha_\eta (P_{q_-^\eta}b_1)^\top M_\eta (P_{q_+^\eta}b_2)\bigr)^2
\right]
\le
\sum_{t=0}^s
O_d(d^{-2\delta})\sum_{q=0}^p
\beta_q^{2t}\bigl((P_qb_1)^\top(P_qb_2)\bigr)^2
+
O_d\!\left(n^{-1}\|b_1\|_2^2\|b_2\|_2^2\right).
\]
This is exactly
\[
\E\!\left[
\bigl(\alpha_\eta (P_{q_-^\eta}b_1)^\top M_\eta (P_{q_+^\eta}b_2)\bigr)^2
\right]
\le
\sum_{t=0}^s O_d(d^{-2\delta})\,\mathcal B_t(b_1,b_2)
+
O_d\!\left(n^{-1}\|b_1\|_2^2\|b_2\|_2^2\right).
\]
Since \(|\Omega(T)|=O_{p,c_H,c_G}(1)\), Cauchy--Schwarz in \eqref{eq:b_T_b_inter} proves \eqref{eq:b_t_p}.

\subsection{Proof of Claim~\ref{cl:b_t}}\label{proof:cl_b_t}
Since $(nD)_{\S_q}^{-1}$ is a diagonal matrix with each entry of the order $O_d(d^{q-p-\delta}) =o_d(1)$ for $q \leq p$, expanding $\gB_t$ gives
    \begin{align}\label{eq:bt_est}
        \sum_{t=0}^{c_G}\gB_t(\mathsf{c}_{\leq p},\hat{e}_J) &= \sum_{t=0}^{c_G}\sum_{q=0}^p \inner{\mathsf{c}_{\S_q},(nD_{\S_q})^{-2t}D_{\S_q}\inv E_{\S_q,J}}^2\notag\\
        &=(1+o_d(1))  \sum_{q=0}^p \inner{\mathsf{c}_{\S_q},D_{\S_q}\inv E_{\S_q,J}}^2.
    \end{align}
Plugging the expression of $D_{\S_q}\inv E_{\S_q,J}$ namely \eqref{eq:D_inv_E} into the above equation gives
\begin{align*}
   \sum_{q=0}^p \inner{\mathsf{c}_{\S_q},D_{\S_q}\inv E_{\S_q,J}}^2 =  \underbrace{\round{ \sum_{i\in [d] \setminus J}\mathsf{c}_{J\sqcup \{i\}} u_i}^2}_{:=\gI^2} +\underbrace{\round{ d^{-1}D'_{J,J} D_{J\setminus\{i\}}^{-1} \sum_{i\in J}\mathsf{c}_{J\setminus \{i\}} u_i}^2}_{:=\gJ^2}.
\end{align*}
By Cauchy-Schwarz, we obtain
\begin{align}\label{eq:j_square}
   \gJ^2&\leq d^{-2}\cdot (D'_{J,J} D_{J\setminus\{i\}}^{-1})^2 \cdot (\sum_{i\in J} u_i^2) \round{\sum_{i\in J} \mathsf{c}_{J\setminus \{i\}}^2}  \\
   &= O_d(d^{-4}) \cdot  \round{\sum_{i\in J} \mathsf{c}_{J\setminus \{i\}}^2}.
\end{align}
thus by Young's inequality we have
\begin{align*}
    \gI^2 + \gJ^2 \leq 2(\gI+\gJ)^2 + 3 \J^2 = 2\round{\mathsf{c}_{\leq p} \hat{e}_J}^2 +  O_d(d^{-4}) \cdot  \round{\sum_{i\in J} \mathsf{c}_{J\setminus \{i\}}^2}.
\end{align*}
Combining it with Eq.~\eqref{eq:bt_est} completes the proof.

\section{Product of random Fourier-Walsh matrices}

In this section, we state and prove a key lemma for our main results. The proof is adapted from~\cite{zhu2025spectral}.

\begin{lemma}\label{lemma:feature_product_2}
Fix families \(\S_1,\ldots,\S_{m+1}\subseteq 2^{[d]}\), and let
\[
M:=A_1^\top A_2A_2^\top \cdots A_mA_m^\top A_{m+1},
\qquad
A_i:=w^{(i)}X_{\S_i},
\]
where
\[
w^{(i)}=n^{-1/2}\wedge |\S_i|^{-1/2}.
\]
Assume the following:
\begin{enumerate}
    \item[\rm (i)] every \(S\in \S_i\) satisfies \(|S|\le p_i\);
    \item[\rm (ii)] if \(\S_i\cap \S_j\neq\varnothing\), then \(\S_i=\S_j\).
\end{enumerate}
Let \(b_1\in\R^{|\S_1|}\), \(b_2\in\R^{|\S_{m+1}|}\), and set
\[
\kappa:=(n\,w^{(1)}w^{(m+1)})^2.
\]

If \(\S_1\neq \S_{m+1}\), then the following estimate for second-moment holds:
\[
\E\!\left[\bigl(b_1^\top M b_2\bigr)^2\right]
=
\kappa\cdot O_d\!\left(n^{-1}\|b_1\|_2^2\|b_2\|_2^2\right).
\]

Assume now that \(\S_1=\S_{m+1}\). If every distinct internal family
\(\S_k\neq \S_1\), \(2\le k\le m\), satisfies \(n=O(|\S_k|)\), then the
second-moment expansion
\begin{align}\label{eq:second_moment_1}
\E\!\left[\bigl(b_1^\top M b_2\bigr)^2\right]
=
\kappa\left[
(b_1^\top b_2)^2
+
O_d\!\left(n^{-1}\|b_1\|_2^2\|b_2\|_2^2\right)
\right]
\end{align}
holds, and the estimate for first-moment expansion also holds:
\begin{align}\label{eq:first_moment_1}
\E\!\left[b_1^\top M b_2\right]
=
\sqrt{\kappa}\left[
b_1^\top b_2
+
O_d\!\left(n^{-\tfrac{1}{2}}\|b_1\|_2\|b_2\|_2\right)
\right].
\end{align}
In the complementary case, which necessarily has \(m\ge 2\), the following estimates hold:
\begin{align}\label{eq:second_moment_2}
\E\!\left[\bigl(b_1^\top M b_2\bigr)^2\right]
=
\kappa\left[
O_d\!\left(\min_{2\le k\le m}\frac{|\S_k|^2}{n^2}\right)(b_1^\top b_2)^2
+
O_d\!\left(n^{-1}\|b_1\|_2^2\|b_2\|_2^2\right)
\right]
\end{align}
and
\begin{align}\label{eq:first_moment_2}
\E\!\left[b_1^\top M b_2\right]
=
\sqrt{\kappa}\left[
O_d\!\left(\min_{2\le k\le m}\frac{|\S_k|}{n}\right)
(b_1^\top b_2)
+
O_d\!\left(n^{-\tfrac{1}{2}}\|b_1\|_2\|b_2\|_2\right)
\right].
\end{align}
\end{lemma}

\begin{proof}[Proof of Lemma~\ref{lemma:feature_product_2}]
We first prove the second-moment estimates. The first-moment estimates follow
from the same collapse procedure with one copy of the chain instead of two; the
necessary modifications are given at the end of the proof.

For readability, introduce the random variable
\[
\mathbf S
:=
\round{
b_1^\top A_1^\top (A_2A_2^\top)\cdots
(A_mA_m^\top)A_{m+1}b_2
}^2 .
\]
Expanding the square gives the identity
\[
\E[\mathbf S]
=
\sum_{i_1,\ldots,i_{2m}\in[n]}
\E\Biggl[
\begin{aligned}[t]
&(A_1b_1)_{i_1}
(A_2A_2^\top)_{i_1,i_2}
\cdots
(A_mA_m^\top)_{i_{m-1},i_m}
(A_{m+1}b_2)_{i_m} \\
&\quad{}\times
(A_1b_1)_{i_{m+1}}
(A_2A_2^\top)_{i_{m+1},i_{m+2}}
\cdots
(A_mA_m^\top)_{i_{2m-1},i_{2m}}
(A_{m+1}b_2)_{i_{2m}}
\end{aligned}
\Biggr].
\]

For each partition \(\pi\in\Pi_{2m}\), define the exact equality class
\[
\mathcal I_\pi
:=
\{(i_1,\ldots,i_{2m})\in[n]^{2m}:
i_s=i_t \text{ if and only if } s\sim_\pi t\}.
\]
The resulting partition decomposition is
\[
\E[\mathbf S]
=
\sum_{\pi\in\Pi_{2m}}\E[\mathbf S_\pi],
\]
where the contribution associated with \(\pi\) is defined by
\[
\E[\mathbf S_\pi]
:=
\sum_{(i_1,\ldots,i_{2m})\in\mathcal I_\pi}
\E\Biggl[
\begin{aligned}[t]
&(A_1b_1)_{i_1}
(A_2A_2^\top)_{i_1,i_2}
\cdots
(A_mA_m^\top)_{i_{m-1},i_m}
(A_{m+1}b_2)_{i_m} \\
&\quad{}\times
(A_1b_1)_{i_{m+1}}
(A_2A_2^\top)_{i_{m+1},i_{m+2}}
\cdots
(A_mA_m^\top)_{i_{2m-1},i_{2m}}
(A_{m+1}b_2)_{i_{2m}}
\end{aligned}
\Biggr].
\]
We now fix \(\pi\in\Pi_{2m}\) and estimate \(\E[\mathbf S_\pi]\).

At each stage of the reduction, the surviving sample positions are relabeled
from left to right as \(i_1,\ldots,i_q\). The current product consists of
factors of the form
\[
(A_1b_1)_{i_r},\qquad
(A_uA_u^\top)_{i_r,i_{r+1}},\qquad
(A_{m+1}b_2)_{i_r},
\]
with family labels inherited from the unreduced product. We repeatedly use the
following two reductions.

First, suppose that the current partition forces two adjacent sample positions
to be equal, say \(i_r=i_{r+1}\), and that the factor between them is
\((A_uA_u^\top)_{i_r,i_{r+1}}\). The diagonal identity gives the scalar value
\[
(A_uA_u^\top)_{i_r,i_r}
=
(w^{(u)})^2|\S_u|
=
1\wedge\frac{|\S_u|}{n}.
\]
Thus the Gram factor is replaced by a scalar, which we denote by
\[
C_F(u):=(w^{(u)})^2|\S_u|
=
1\wedge\frac{|\S_u|}{n}.
\]
One of the two equal sample variables is then deleted. We call this operation a
\emph{feature collapse}.

Second, suppose that \(\{r\}\) is a singleton block of the current partition.
For fixed values of the other current block labels, the label \(i_r\) has
\(n+O_m(1)\) admissible choices. Define the endpoint pairing by
\[
\langle b_1,b_2\rangle_\cap
:=
\sum_{S\in\S_1\cap\S_{m+1}}(b_1)_S(b_2)_S .
\]
Also set
\[
C_S(u,v)
:=
(n+O_m(1))w^{(u)}w^{(v)}\mathbbm 1_{\S_u=\S_v}.
\]
Independence across samples and orthogonality of distinct feature families give
the following local replacements:
\[
\begin{aligned}
(A_uA_u^\top)_{i_{r-1},i_r}
(A_vA_v^\top)_{i_r,i_{r+1}}
&\rightsquigarrow
C_S(u,v)(A_uA_u^\top)_{i_{r-1},i_{r+1}},\\
(A_1b_1)_{i_r}
(A_vA_v^\top)_{i_r,i_{r+1}}
&\rightsquigarrow
C_S(1,v)(A_1b_1)_{i_{r+1}},\\
(A_uA_u^\top)_{i_{r-1},i_r}
(A_{m+1}b_2)_{i_r}
&\rightsquigarrow
C_S(u,m+1)(A_{m+1}b_2)_{i_{r-1}},\\
(A_1b_1)_{i_r}
(A_{m+1}b_2)_{i_r}
&\rightsquigarrow
C_S(1,m+1)\langle b_1,b_2\rangle_\cap .
\end{aligned}
\]
Here \(\rightsquigarrow\) means that the contribution is unchanged after
summing over the singleton sample label and deleting that variable. In the
first line, the family label on the remaining Gram factor is immaterial because
the factor \(\mathbbm 1_{\S_u=\S_v}\) forces \(\S_u=\S_v\). We call these
operations \emph{sample collapses}.

We iterate the two reductions until neither applies. Let \(\pi^\sharp\) be the
final partition, and let \(c_\pi\) be the product of all scalar coefficients
generated during the reduction. With the convention that any endpoint pairing
\(\langle b_1,b_2\rangle_\cap\) produced before the procedure stops is kept as
an external scalar factor, the collapse procedure gives the relation
\begin{equation}
\label{eq:feature_product_reduced_contribution}
\E[\mathbf S_\pi]
=
\begin{cases}
c_\pi\,\langle b_1,b_2\rangle_\cap^2,
& |\pi^\sharp|=0,\\[2mm]
c_\pi\,\E[\mathbf S_{\pi^\sharp}],
& |\pi^\sharp|\ge 1.
\end{cases}
\end{equation}
All feature-collapse and sample-collapse coefficients have absolute value at
most one, so the coefficient \(c_\pi\) satisfies
\[
0\le c_\pi\le 1.
\]

We now record the coefficient estimate for fully collapsed partitions. The proof appears in Appendix~\ref{proof:c_pi_estimate}.

\begin{proposition}[Fully collapsed coefficient]\label{prop:c_pi_estimate}
Assume \(\S_1=\S_{m+1}\). Fix \(\pi\in\Pi_{2m}\) such that
\(|\pi^\sharp|=0\) and \(c_\pi>0\). For a distinct family \(\T\), write
\[
w_\T:=n^{-1/2}\wedge |\T|^{-1/2}.
\]
Let \(N_F(\T)\) be the number of feature collapses performed on \(\T\), and let
\(N_S(\T)\) be the number of nonterminal sample collapses performed on \(\T\),
excluding the two terminal endpoint collapses that produce the two factors
\(\langle b_1,b_2\rangle_\cap\). Then the coefficient admits the factorization
\[
\kappa^{-1}c_\pi
=
\bigl(1+O_m(n^{-1})\bigr)
\prod_{\T}
\round{1\wedge\frac{|\T|}{n}}^{N_F(\T)}
\round{1\wedge\frac{n}{|\T|}}^{N_S(\T)}.
\]
In particular, if no collapse produces a loss, then the coefficient satisfies
\[
c_\pi
=
\kappa\bigl(1+O_m(n^{-1})\bigr).
\]
Moreover, if at least one collapse on a family \(\T\) produces a loss, then
\[
c_\pi
=
\kappa\cdot
O\!\left(
\min\left\{\frac{|\T|}{n},\frac{n}{|\T|}\right\}
\right).
\]
Finally, if \(\T\neq\S_1\) is a distinct internal family, say
\(\T=\S_k\) for some \(2\le k\le m\), and \(|\T|<n\), then every fully
collapsed second-moment contribution satisfies
\[
c_\pi
=
\kappa\cdot O\!\left(\frac{|\T|^2}{n^2}\right).
\]
\end{proposition}

We next bound the residual terms, namely the terms for which the collapse
procedure stops before the entire product disappears. The proof appears in Appendix~\ref{proof:pi_sharp_small}.

\begin{proposition}\label{prop:pi_sharp_small}
Fix \(\pi\in\Pi_{2m}\), and suppose that the recursive collapse procedure
stops at a final partition \(\pi^\sharp\) with \(|\pi^\sharp|\ge 1\). Then the following estimate for  residual holds:
\[
\bigl|\E[\mathbf S_{\pi^\sharp}]\bigr|
=
O_d(n^{-1})\|b_1\|_2^2\|b_2\|_2^2
\]
\end{proposition}

We now assemble the second-moment estimates. The bound
\(\kappa^{-1}=O_d(1)\) follows from
\(|\S_1|,|\S_{m+1}|\le 2^d\). Proposition~\ref{prop:pi_sharp_small}, together
with \(0\le c_\pi\le 1\), implies that every non-fully-collapsed partition
contributes a term satisfying
\begin{equation}
\label{eq:feature_product_nonfull_second_bound}
c_\pi\,\E[\mathbf S_{\pi^\sharp}]
=
\kappa\cdot
O_d\!\left(
n^{-1}\|b_1\|_2^2\|b_2\|_2^2
\right).
\end{equation}

If \(\S_1\neq\S_{m+1}\), then the endpoint pairing vanishes:
\[
\langle b_1,b_2\rangle_\cap=0.
\]
Thus every fully collapsed contribution vanishes. Since the number of
partitions of \([2m]\) depends only on \(m\), the non-fully-collapsed bound in
\eqref{eq:feature_product_nonfull_second_bound} gives
\[
\E\!\left[
\bigl(b_1^\top M b_2\bigr)^2
\right]
=
\kappa\cdot
O_d\!\left(
n^{-1}\|b_1\|_2^2\|b_2\|_2^2
\right).
\]

Assume now that \(\S_1=\S_{m+1}\) holds. The endpoint pairing is then
\[
\langle b_1,b_2\rangle_\cap
=
b_1^\top b_2.
\]
Suppose first that every distinct internal family \(\S_k\neq\S_1\),
\(2\le k\le m\), satisfies \(n=O(|\S_k|)\). The fully collapsed contribution in
which all collapses are loss-free has coefficient
\[
\kappa\bigl(1+O_m(n^{-1})\bigr).
\]
Every other fully collapsed contribution contains at least one lossy collapse.
If the loss occurs on a family \(\T\) with \(|\T|<n\), then the loss
\(|\T|/n\) is \(O_d(n^{-1})\). If the loss occurs on a distinct internal family
\(\T\neq\S_1\) with \(|\T|>n\), then the loss is \(n/|\T|\); the assumption
\(n=O(|\T|)\), together with \(|\T|\le 2^d\), implies that this loss is also
\(O_d(n^{-1})\). Hence all fully collapsed contributions except the loss-free
one are absorbed into the \(O_d(n^{-1})\) error. The fully collapsed terms
therefore satisfy
\[
\sum_{\substack{\pi\in\Pi_{2m}\\ |\pi^\sharp|=0}}
\E[\mathbf S_\pi]
=
\kappa\left[
(b_1^\top b_2)^2
+
O_d\!\left(
n^{-1}\|b_1\|_2^2\|b_2\|_2^2
\right)
\right].
\]
Combining this estimate with
\eqref{eq:feature_product_nonfull_second_bound} gives
\[
\E\!\left[
\bigl(b_1^\top M b_2\bigr)^2
\right]
=
\kappa\left[
(b_1^\top b_2)^2
+
O_d\!\left(
n^{-1}\|b_1\|_2^2\|b_2\|_2^2
\right)
\right],
\]
which proves \eqref{eq:second_moment_1}.

It remains to consider the complementary regime. In this case, there is a
distinct internal family \(\T=\S_k\neq\S_1\), \(2\le k\le m\), with
\(|\T|<n\) along the relevant range. Define
\[
\rho_2
:=
\min_{\substack{2\le k\le m\\ \S_k\neq\S_1,\ |\S_k|<n}}
\frac{|\S_k|^2}{n^2}.
\]
By Proposition~\ref{prop:c_pi_estimate}, every fully collapsed second-moment
contribution must incur two feature-collapse losses from such an internal
family. This gives the bound
\[
\sum_{\substack{\pi\in\Pi_{2m}\\ |\pi^\sharp|=0}}
\E[\mathbf S_\pi]
=
\kappa\cdot
O_d(\rho_2)(b_1^\top b_2)^2.
\]
Since all families are nonempty and have size at most \(2^d\), replacing
\(\rho_2\) by \(\min_{2\le k\le m}|\S_k|^2/n^2\) changes only the
\(O_d(\cdot)\) constant. After adding the non-fully-collapsed contribution from
\eqref{eq:feature_product_nonfull_second_bound}, we obtain
\[
\E\!\left[
\bigl(b_1^\top M b_2\bigr)^2
\right]
=
\kappa\left[
O_d\!\left(
\min_{2\le k\le m}\frac{|\S_k|^2}{n^2}
\right)(b_1^\top b_2)^2
+
O_d\!\left(
n^{-1}\|b_1\|_2^2\|b_2\|_2^2
\right)
\right],
\]
which proves \eqref{eq:second_moment_2}.

We finally prove the first-moment estimates. Set
\[
\mathbf T:=b_1^\top M b_2.
\]
The partition decomposition is now taken over \([m]\), because there is only
one copy of the chain. The same feature and sample collapses apply. A fully
collapsed first-moment term has one terminal endpoint collapse instead of two,
so the terminal coefficient has the form
\[
(n+O_m(1))w^{(1)}w^{(m+1)}
=
\sqrt{\kappa}\bigl(1+O_m(n^{-1})\bigr).
\]
The analogue of Proposition~\ref{prop:pi_sharp_small}, proved by the same
incidence-graph argument with one chain instead of two, gives the stronger
residual estimate
\[
\bigl|\E[\mathbf T_{\pi^\sharp}]\bigr|
=
O_d(n^{-1})\|b_1\|_2\|b_2\|_2
\]
for every non-fully-collapsed first-moment residual. Since
\(\kappa^{-1/2}=O_d(1)\), the weaker estimate needed below also holds:
\begin{equation}
\label{eq:feature_product_nonfull_first_bound}
\bigl|\E[\mathbf T_{\pi^\sharp}]\bigr|
=
\sqrt{\kappa}\,
O_d(n^{-1/2})\|b_1\|_2\|b_2\|_2.
\end{equation}

Assume first that every distinct internal family
\(\S_k\neq\S_1\), \(2\le k\le m\), satisfies \(n=O(|\S_k|)\). The fully
collapsed loss-free contribution has value
\[
\sqrt{\kappa}\,
\bigl(1+O_m(n^{-1})\bigr)b_1^\top b_2.
\]
As in the second-moment argument, all other fully collapsed contributions are
absorbed into the \(O_d(n^{-1})\) error. The Cauchy--Schwarz inequality gives
the auxiliary bound
\[
n^{-1}|b_1^\top b_2|
\le
n^{-1/2}\|b_1\|_2\|b_2\|_2.
\]
Combining the fully collapsed estimate with
\eqref{eq:feature_product_nonfull_first_bound} gives
\[
\E\!\left[
b_1^\top M b_2
\right]
=
\sqrt{\kappa}\left[
b_1^\top b_2
+
O_d\!\left(
n^{-1/2}\|b_1\|_2\|b_2\|_2
\right)
\right],
\]
which proves \eqref{eq:first_moment_1}.

In the complementary regime, define
\[
\rho_1
:=
\min_{\substack{2\le k\le m\\ \S_k\neq\S_1,\ |\S_k|<n}}
\frac{|\S_k|}{n}.
\]
A fully collapsed first-moment contribution must incur one feature-collapse
loss from a distinct internal family with size smaller than \(n\). This gives
the fully collapsed bound
\[
\sum_{\substack{\pi\in\Pi_m\\ |\pi^\sharp|=0}}
\E[\mathbf T_\pi]
=
\sqrt{\kappa}\,
O_d(\rho_1)(b_1^\top b_2).
\]
Again, because all family sizes lie between \(1\) and \(2^d\), replacing
\(\rho_1\) by \(\min_{2\le k\le m}|\S_k|/n\) changes only the
\(O_d(\cdot)\) constant. Adding the non-fully-collapsed first-moment
contribution from \eqref{eq:feature_product_nonfull_first_bound} yields
\[
\E\!\left[
b_1^\top M b_2
\right]
=
\sqrt{\kappa}\left[
O_d\!\left(
\min_{2\le k\le m}\frac{|\S_k|}{n}
\right)(b_1^\top b_2)
+
O_d\!\left(
n^{-1/2}\|b_1\|_2\|b_2\|_2
\right)
\right],
\]
which proves \eqref{eq:first_moment_2}.
\end{proof}

\begin{corollary}\label{cor:off_diag}
Fix families \(\S_1,\ldots,\S_{m+1}\subseteq 2^{[d]}\), and define
\[
M:=A_1^\top A_2A_2^\top\cdots A_mA_m^\top A_{m+1},
\qquad
A_i:=w^{(i)}X_{\S_i},
\]
where the weights are given by
\[
w^{(i)}:=n^{-1/2}\wedge |\S_i|^{-1/2}.
\]
Assume that the families \(\S_1,\ldots,\S_{m+1}\) satisfy the hypotheses of
Lemma~\ref{lemma:feature_product_2}. Assume moreover that the endpoint
families agree, namely,
\[
\S_1=\S_{m+1}.
\]
Let \(b_1,b_2\in\R^{|\S_1|}\) and let \(\mathcal V,\mathcal E\subseteq\{2,\ldots,m\}\) be such that
\[
\mathcal V\cup\mathcal E\neq\emptyset,
\qquad
\mathcal V\cap\bigl(\mathcal E\cup(\mathcal E+1)\bigr)=\emptyset,
\qquad
\mathcal E+1:=\{j+1:j\in\mathcal E\}.
\]
Assume also that the adjacent matrices agree along \(\mathcal E\), that is,
\[
A_j=A_{j+1}\qquad\text{for every }j\in\mathcal E.
\]
Finally, assume the size conditions
\[
n\le |\S_j|\quad\text{for every }j\in\mathcal V,
\qquad
|\S_j|\le n\quad\text{for every }j\in\mathcal E.
\]

Define \(M_{\mathcal V,\mathcal E}\) to be the matrix obtained from \(M\) by
the following replacements. For each \(j\in\mathcal V\), replace
\[
A_jA_j^\top
\quad\text{by}\quad
\offd(A_jA_j^\top).
\]
Also, for each \(j\in\mathcal E\), replace
\[
A_j^\top A_{j+1}
\quad\text{by}\quad
\offd(A_j^\top A_{j+1}).
\]
Set
\[
\kappa:=(n\,w^{(1)}w^{(m+1)})^2.
\]

Let \(\mathfrak B\) denote the set of distinct internal families
\(\T\neq\S_1\) for which the first-regime condition \(n=O(|\T|)\) fails.
Equivalently, \(\mathfrak B\) is the set of distinct internal families that
force the ``otherwise'' case of Lemma~\ref{lemma:feature_product_2}. Define
\[
\rho
:=
\begin{cases}
0, & \mathfrak B=\emptyset,\\[1mm]
\displaystyle \min_{\T\in\mathfrak B}\frac{|\T|}{n},
& \mathfrak B\neq\emptyset.
\end{cases}
\]
Then
\[
\E\!\left[
\bigl(b_1^\top M_{\mathcal V,\mathcal E}b_2\bigr)^2
\right]
=
\kappa\cdot O_d(\rho^2)(b_1^\top b_2)^2
+
\kappa\cdot O_d\!\left(
n^{-1}\|b_1\|_2^2\|b_2\|_2^2
\right).
\]
\end{corollary}

\begin{proof}[Proof of Corollary~\ref{cor:off_diag}]
We first record the diagonal identities used in the expansion. If
\(j\in\mathcal V\), then \(n\le |\S_j|\), and therefore
\(w^{(j)}=|\S_j|^{-1/2}\). Hence
\[
\Diag(A_jA_j^\top)
=
(w^{(j)})^2|\S_j|\,I_n
=
I_n.
\]
Similarly, if \(j\in\mathcal E\), then \(A_j=A_{j+1}\) and
\(|\S_j|\le n\). Thus \(w^{(j)}=w^{(j+1)}=n^{-1/2}\), and hence
\[
\Diag(A_j^\top A_{j+1})
=
n\,w^{(j)}w^{(j+1)}\,I_{|\S_j|}
=
I_{|\S_j|}.
\]
Consequently, for \(j\in\mathcal V\), we have
\[
\offd(A_jA_j^\top)
=
A_jA_j^\top-I_n.
\]
Likewise, for \(j\in\mathcal E\), we have
\[
\offd(A_j^\top A_{j+1})
=
A_j^\top A_{j+1}-I_{|\S_j|}.
\]

We now expand the off-diagonal constraints by inclusion--exclusion. For
\(J\subseteq\mathcal V\) and \(K\subseteq\mathcal E\), define the remaining
index set by
\[
\Gamma_{J,K}
:=
[m+1]\setminus\bigl(J\cup(K+1)\bigr)
=
\{i_1<\cdots<i_r\}.
\]
On this reduced index set, define the corresponding chain by
\[
M_{J,K}
:=
(A_{i_1}^\top A_{i_2})
(A_{i_2}^\top A_{i_3})
\cdots
(A_{i_{r-1}}^\top A_{i_r}).
\]
The condition
\[
\mathcal V\cap\bigl(\mathcal E\cup(\mathcal E+1)\bigr)=\emptyset
\]
ensures that the deletions coming from \(J\) and \(K\) do not interfere with
one another. Therefore, inclusion--exclusion gives
\begin{equation}
\label{eq:cor-offdiag-inclusion-exclusion}
M_{\mathcal V,\mathcal E}
=
\sum_{J\subseteq\mathcal V}
\sum_{K\subseteq\mathcal E}
(-1)^{|J|+|K|}M_{J,K}.
\end{equation}

We next verify that each reduced chain \(M_{J,K}\) is still covered by
Lemma~\ref{lemma:feature_product_2}. Deleting an index from \(\mathcal V\)
removes only a family \(\S_j\) with \(n\le |\S_j|\), and therefore it cannot
remove any family in \(\mathfrak B\). Deleting \(j+1\) with \(j\in\mathcal E\)
only shortens a block of equal adjacent matrices, since \(A_j=A_{j+1}\). Thus
no distinct bad internal family is removed entirely.

The endpoint families are also unchanged. The left endpoint is never deleted.
If the right endpoint is deleted, then \(K\) contains a terminal interval
\(\{\ell,\ell+1,\ldots,m\}\) for some \(\ell\). In that case, the new right
endpoint is \(\ell\), and the equalities along this terminal block imply
\[
A_\ell=A_{\ell+1}=\cdots=A_{m+1}.
\]
Thus the right endpoint family and the right endpoint weight remain those of
\(\S_{m+1}\). Hence every reduced chain has the same endpoint families and the
same value of \(\kappa\).

For notational convenience, write the index set of reduced chains as
\[
\mathcal A
:=
\{(J,K):J\subseteq\mathcal V,\ K\subseteq\mathcal E\}.
\]
For \(a=(J,K)\in\mathcal A\), define
\[
\sigma_a:=(-1)^{|J|+|K|},
\qquad
Z_a:=b_1^\top M_{J,K}b_2.
\]
Finally, denote
\[
\beta:=b_1^\top b_2,
\qquad
N:=\|b_1\|_2\|b_2\|_2.
\]
With this notation, equation~\eqref{eq:cor-offdiag-inclusion-exclusion}
becomes
\begin{equation}
\label{eq:cor-offdiag-bilinear-expansion}
b_1^\top M_{\mathcal V,\mathcal E}b_2
=
\sum_{a\in\mathcal A}\sigma_a Z_a.
\end{equation}

We now distinguish the two regimes of Lemma~\ref{lemma:feature_product_2}.

\medskip
\noindent
\emph{First regime: \(\mathfrak B=\emptyset\).}
In this case, every reduced chain lies in the first endpoint-matched regime of
Lemma~\ref{lemma:feature_product_2}. We use the following mixed second-moment
estimate: for all \(a,a'\in\mathcal A\),
\begin{equation}
\label{eq:cor-offdiag-mixed-second-moment}
\E[Z_aZ_{a'}]
=
\kappa\left[
\beta^2
+
O_d\!\left(n^{-1}N^2\right)
\right].
\end{equation}
This is the polarized form of the second-moment estimate in
Lemma~\ref{lemma:feature_product_2}. Equivalently, it follows by applying the
same collapse argument to the product \(Z_aZ_{a'}\). The fully collapsed
loss-free contribution gives the common leading term \(\kappa\beta^2\), while
all lossy or non-fully-collapsed contributions are absorbed into the
\(O_d(n^{-1}N^2)\) error.

Since \(\mathcal V\cup\mathcal E\neq\emptyset\), the alternating coefficients
satisfy
\begin{equation}
\label{eq:cor-offdiag-alternating-cancellation}
\sum_{a\in\mathcal A}\sigma_a
=
\left(\sum_{J\subseteq\mathcal V}(-1)^{|J|}\right)
\left(\sum_{K\subseteq\mathcal E}(-1)^{|K|}\right)
=
0.
\end{equation}
Using \eqref{eq:cor-offdiag-bilinear-expansion},
\eqref{eq:cor-offdiag-mixed-second-moment}, and
\eqref{eq:cor-offdiag-alternating-cancellation}, we obtain
\[
\begin{aligned}
\E\!\left[
\bigl(b_1^\top M_{\mathcal V,\mathcal E}b_2\bigr)^2
\right]
&=
\sum_{a,a'\in\mathcal A}\sigma_a\sigma_{a'}\E[Z_aZ_{a'}] \\
&=
\kappa\beta^2
\left(\sum_{a\in\mathcal A}\sigma_a\right)^2
+
\kappa\cdot O_d\!\left(n^{-1}N^2\right) \\
&=
\kappa\cdot O_d\!\left(n^{-1}N^2\right).
\end{aligned}
\]
Since \(\rho=0\) in the present regime, this gives the desired bound.

\medskip
\noindent
\emph{Second regime: \(\mathfrak B\neq\emptyset\).}
In this case, every reduced chain remains in the second regime of
Lemma~\ref{lemma:feature_product_2}. Indeed, as observed above, the deletions
do not remove any family in \(\mathfrak B\). Therefore,
Lemma~\ref{lemma:feature_product_2} gives, for every \(a\in\mathcal A\),
\begin{equation}
\label{eq:cor-offdiag-bad-second-moment}
\E[Z_a^2]
=
\kappa\left[
O_d(\rho^2)\beta^2
+
O_d\!\left(n^{-1}N^2\right)
\right].
\end{equation}
Since \(|\mathcal A|\le 2^m\), Cauchy--Schwarz applied to the finite sum in
\eqref{eq:cor-offdiag-bilinear-expansion} gives
\[
\E\!\left[
\bigl(b_1^\top M_{\mathcal V,\mathcal E}b_2\bigr)^2
\right]
\le
|\mathcal A|
\sum_{a\in\mathcal A}\E[Z_a^2].
\]
Combining this estimate with \eqref{eq:cor-offdiag-bad-second-moment}, we
obtain
\[
\E\!\left[
\bigl(b_1^\top M_{\mathcal V,\mathcal E}b_2\bigr)^2
\right]
=
\kappa\cdot O_d(\rho^2)\beta^2
+
\kappa\cdot O_d\!\left(n^{-1}N^2\right).
\]
Finally, substituting back
\(
\beta=b_1^\top b_2\) and 
\(N=\|b_1\|_2\|b_2\|_2,
\)
yields
\[
\E\!\left[
\bigl(b_1^\top M_{\mathcal V,\mathcal E}b_2\bigr)^2
\right]
=
\kappa\cdot O_d(\rho^2)(b_1^\top b_2)^2
+
\kappa\cdot O_d\!\left(
n^{-1}\|b_1\|_2^2\|b_2\|_2^2
\right).
\]
This completes the proof.
\end{proof}

\subsection{Proof of Proposition~\ref{prop:c_pi_estimate}}
\label{proof:c_pi_estimate}

A feature collapse on \(\T\) has contribution
\[
w_\T^2|\T|
=
1\wedge\frac{|\T|}{n}.
\]
Similarly, the contribution of a nonterminal sample collapse on \(\T\) is
\[
(n+O_m(1))w_\T^2
=
\bigl(1+O_m(n^{-1})\bigr)
\left(1\wedge\frac{n}{|\T|}\right).
\]
The two terminal endpoint collapses have combined contribution
\[
\bigl((n+O_m(1))w^{(1)}w^{(m+1)}\bigr)^2
=
\kappa\bigl(1+O_m(n^{-1})\bigr).
\]
Multiplying all collapse contributions gives the factorization
\begin{equation}
\label{eq:c_pi_factorization}
\kappa^{-1}c_\pi
=
\bigl(1+O_m(n^{-1})\bigr)
\prod_{\T}
\round{1\wedge\frac{|\T|}{n}}^{N_F(\T)}
\round{1\wedge\frac{n}{|\T|}}^{N_S(\T)}.
\end{equation}

Separating the feature families according to whether \(|\T|<n\) or
\(|\T|>n\), and noting that families with \(|\T|=n\) contribute a factor equal
to one, gives the equivalent representation
\begin{equation}
\label{eq:c_pi_factorization_split}
\kappa^{-1}c_\pi
=
\bigl(1+O_m(n^{-1})\bigr)
\prod_{|\T|<n}
\round{\frac{|\T|}{n}}^{N_F(\T)}
\prod_{|\T|>n}
\round{\frac{n}{|\T|}}^{N_S(\T)}.
\end{equation}
The first two consequences follow immediately from
\eqref{eq:c_pi_factorization_split}.

It remains to prove the last claim. Let \(\T\neq\S_1\) be an internal feature
family with \(|\T|<n\). In each copy of the chain, sample collapses can merge
adjacent \(\T\)-factors, but they cannot remove the final surviving
\(\T\)-Gram factor, because the endpoints have family \(\S_1\neq\T\).
Consequently, each copy must contain at least one feature collapse on \(\T\).
Since there are two copies of the chain, the bound
\[
N_F(\T)\ge 2
\]
holds. Applying this bound to \eqref{eq:c_pi_factorization_split} gives
\[
c_\pi
=
\kappa\cdot O\!\left(\frac{|\T|^2}{n^2}\right).
\]
This proves the proposition.

\subsection{Proof of Proposition~\ref{prop:pi_sharp_small}}
\label{proof:pi_sharp_small}

Let \(q\) be the number of surviving sample positions after the collapse
procedure stops. Set
\[
r:=|\pi^\sharp|,
\]
so that \(r\) is the number of blocks in the final partition. Since no further
sample collapse is possible, the partition \(\pi^\sharp\) has no singleton
block. This observation gives the counting inequality
\begin{equation}
\label{eq:pi_sharp_q_r_count}
q\ge 2r.
\end{equation}
Moreover, since \(|\pi^\sharp|\ge 1\), the integer \(r\) satisfies \(r\ge 1\).

During the previous reductions, one entire copy of the chain may already have
collapsed and produced a scalar endpoint factor
\(\langle b_1,b_2\rangle_\cap\). Let \(\Theta\) denote the product of any such
external endpoint factors, with the convention that \(\Theta=1\) if no such
factor is present. We keep \(\Theta\) outside the graph notation below.

Let \(G^\sharp\) be the incidence multigraph whose vertices are the blocks of
\(\pi^\sharp\). Each surviving factor \((A_uA_u^\top)_{i_s,i_t}\) gives an
internal edge \(e\) joining the two blocks containing \(s\) and \(t\). Each
surviving endpoint factor \((A_1b_1)_{i_s}\) or
\((A_{m+1}b_2)_{i_s}\) gives a half-edge \(h\) attached to the block containing
\(s\). Since no further feature collapse is possible, no internal edge is a
loop.

Write \(E=E(G^\sharp)\) and \(H=H(G^\sharp)\). For each internal edge
\(e\in E\), let \(\F_e=\S_{u_e}\) be the corresponding feature family. The
edge factor associated with \(e\) admits the representation
\[
(A_{u_e}A_{u_e}^\top)_{ij}
=
\omega_e\sum_{S_e\in\F_e}x_i^{S_e}x_j^{S_e},
\qquad
\omega_e:=(w^{(u_e)})^2\le n^{-1}.
\]
Similarly, for each half-edge \(h\in H\), let \(\F_h\) be either \(\S_1\) or
\(\S_{m+1}\), and let \(c_h\) be the corresponding coefficient vector, either
\(b_1\) or \(b_2\). The endpoint factor attached to \(h\) has the form
\[
g_h(i)
=
\eta_h\sum_{S_h\in\F_h}c_h(S_h)x_i^{S_h},
\qquad
\eta_h\le n^{-1/2}.
\]

For each admissible assignment of distinct sample labels to the \(r\) blocks,
independence across the distinct block labels factors the expectation over the
vertices of \(G^\sharp\). Since the number of such assignments is at most
\(n^r\), taking absolute values gives the reduction
\begin{equation}
\label{eq:pi_sharp_initial_reduction}
\bigl|\E[\mathbf S_{\pi^\sharp}]\bigr|
\le
|\Theta|\,n^r
\Bigl(\prod_{e\in E}\omega_e\Bigr)
\Bigl(\prod_{h\in H}\eta_h\Bigr)
\Gamma,
\end{equation}
where the remaining feature sum \(\Gamma\) is defined by
\[
\Gamma
:=
\sum_{(S_f)_{f\in E\cup H}}
\Bigl(\prod_{h\in H}|c_h(S_h)|\Bigr)
\prod_{T\in\pi^\sharp}
\mu_T\bigl((S_f)_{f\in\partial T}\bigr).
\]
Here, for each block \(T\in\pi^\sharp\), the local Walsh moment is defined as
\[
\mu_T\bigl((S_f)_{f\in\partial T}\bigr)
:=
\E\Bigl[\prod_{f\in\partial T}x^{S_f}\Bigr].
\]
Since each \(\mu_T\) is a Fourier--Walsh moment, it satisfies the pointwise
bound
\[
0\le
\mu_T\bigl((S_f)_{f\in\partial T}\bigr)
\le 1.
\]

We next bound \(\Gamma\). Since all feature families are subcollections of
\(2^{[d]}\), and since the number of surviving edges and half-edges depends
only on \(m\), the internal feature sums contribute only an \(O_d(1)\) factor.
This gives the estimate
\[
\Gamma
\le
O_d(1)
\prod_{h\in H}\sum_{S_h\in\F_h}|c_h(S_h)|.
\]
For each endpoint feature family, Cauchy--Schwarz gives the bound
\[
\sum_{S_h\in\F_h}|c_h(S_h)|
\le
|\F_h|^{1/2}\|c_h\|_2
\le
O_d(\|c_h\|_2).
\]
Combining the preceding two estimates yields
\begin{equation}
\label{eq:pi_sharp_gamma_bound}
\Gamma
\le
O_d\left(\prod_{h\in H}\|c_h\|_2\right).
\end{equation}

It remains to account for the endpoint coefficients. In the original
second-moment expansion, there are exactly two \(b_1\)-endpoint factors and two
\(b_2\)-endpoint factors. Each such factor either remains as a half-edge or has
already been absorbed into the external scalar \(\Theta\). This accounting gives
the endpoint bound
\[
|\Theta|
\prod_{h\in H}\|c_h\|_2
\le
\|b_1\|_2^2\|b_2\|_2^2.
\]
Together with \eqref{eq:pi_sharp_gamma_bound}, this bound implies
\begin{equation}
\label{eq:pi_sharp_endpoint_bound}
|\Theta|\,\Gamma
\le
O_d\bigl(\|b_1\|_2^2\|b_2\|_2^2\bigr).
\end{equation}

We now count the powers of \(n\). Each surviving sample position appears in
exactly two surviving factors. The corresponding degree count is the identity
\[
2|E|+|H|=2q.
\]
Since \(\omega_e\le n^{-1}\) for every internal edge and
\(\eta_h\le n^{-1/2}\) for every half-edge, the product of weights satisfies
\begin{equation}
\label{eq:pi_sharp_weight_bound}
\Bigl(\prod_{e\in E}\omega_e\Bigr)
\Bigl(\prod_{h\in H}\eta_h\Bigr)
\le
n^{-|E|-|H|/2}
=
n^{-q}.
\end{equation}
Substituting \eqref{eq:pi_sharp_endpoint_bound} and
\eqref{eq:pi_sharp_weight_bound} into
\eqref{eq:pi_sharp_initial_reduction} gives
\[
\bigl|\E[\mathbf S_{\pi^\sharp}]\bigr|
\le
n^{r-q}\,
O_d\bigl(\|b_1\|_2^2\|b_2\|_2^2\bigr).
\]
The counting inequality \eqref{eq:pi_sharp_q_r_count}, together with
\(r\ge 1\), implies
\[
r-q\le -r\le -1.
\]
The preceding estimate therefore gives
\[
\bigl|\E[\mathbf S_{\pi^\sharp}]\bigr|
=
O_d(n^{-1})\|b_1\|_2^2\|b_2\|_2^2.
\]
This proves the proposition.

\section{Technical lemmas and propositions}

\begin{proposition}\label{prop:uniform_bound}
Let \(X\) be a random element of \(\H^{n\times d}\), and let
\(f_1,\dots,f_d:\H^{n\times d}\to\R\).
Assume hypercontractivity: for every \(p\ge 2\) and every \(i\in[d]\),
\[
\|f_i(X)\|_{L_p}\le C_{p,\ell}\,\|f_i(X)\|_{L_2}.
\]
Then, for every \(\epsilon>0\), there exists a constant
\(C_{\epsilon,\ell}\), depending only on \(\epsilon\) and \(\ell\), such that
\[
\Big\|\max_{i\in[d]} f_i(X)\Big\|_{L_2}^2
\le
C_{\epsilon,\ell}\, d^{\epsilon}\,
\max_{i\in[d]}\|f_i(X)\|_{L_2}^{2}.
\]
More explicitly, one may take \(C_{\epsilon,\ell}=C_{2q,\ell}^{\,2}\) for
any choice of \(q>1\) satisfying \(1/q\le \epsilon\).
\end{proposition}

\begin{proof}[Proof of Proposition~\ref{prop:uniform_bound}]
Fix \(\epsilon>0\), and choose \(q>1\) such that \(1/q\le \epsilon\);
for instance, one may take \(q=\max\{2,1/\epsilon\}\). Define
\[
M(X):=\max_{i\in[d]} f_i(X).
\]
For any real numbers \(a_1,\dots,a_d\), the inequality
\((\max_i a_i)^2\le \max_i a_i^2\) holds. Applying this pointwise gives the
initial reduction
\begin{equation}
\label{eq:uniform_reduce_squares}
\|M\|_{L_2}^2
=
\E\big[M(X)^2\big]
\le
\E\Big[\max_{i\in[d]} f_i(X)^2\Big].
\end{equation}

Set
\[
Y(X):=\max_{i\in[d]} f_i(X)^2 .
\]
The random variable \(Y\) is nonnegative. By Lyapunov's inequality, or
equivalently the monotonicity of \(L_p\)-norms on a probability space, the
\(L_1\)-norm of \(Y\) is bounded by its \(L_q\)-norm. Thus the right-hand side
of \eqref{eq:uniform_reduce_squares} satisfies
\[
\E\Big[\max_{i\in[d]} f_i(X)^2\Big]
=
\|Y\|_{L_1}
\le
\|Y\|_{L_q}
=
\big(\E[Y(X)^q]\big)^{1/q}.
\]
Using the definition of \(Y\), this bound can be written as
\[
\big(\E[Y(X)^q]\big)^{1/q}
=
\Big(\E\Big[\max_{i\in[d]} |f_i(X)|^{2q}\Big]\Big)^{1/q}.
\]

The elementary inequality \(\max_i b_i\le \sum_{i=1}^d b_i\), valid for
nonnegative numbers \(b_i\), gives the moment bound
\begin{align}
\label{eq:uniform_moment_bound}
\Big(\E\Big[\max_{i\in[d]} |f_i(X)|^{2q}\Big]\Big)^{1/q}
&\le
\Big(\E\Big[\sum_{i=1}^d |f_i(X)|^{2q}\Big]\Big)^{1/q} \notag\\
&=
\Big(\sum_{i=1}^d \E\big[|f_i(X)|^{2q}\big]\Big)^{1/q} \notag\\
&\le
d^{1/q}\,
\max_{i\in[d]}
\Big(\E\big[|f_i(X)|^{2q}\big]\Big)^{1/q}.
\end{align}

For each \(i\in[d]\), the assumed hypercontractive estimate, applied with
exponent \(2q\), gives
\[
\|f_i(X)\|_{L_{2q}}
\le
C_{2q,\ell}\,\|f_i(X)\|_{L_2}.
\]
Equivalently, after squaring both sides, this estimate becomes
\begin{equation}
\label{eq:uniform_hypercontractive_square}
\Big(\E\big[|f_i(X)|^{2q}\big]\Big)^{1/q}
=
\|f_i(X)\|_{L_{2q}}^{2}
\le
C_{2q,\ell}^{\,2}\,\|f_i(X)\|_{L_2}^{2}.
\end{equation}

Combining \eqref{eq:uniform_reduce_squares},
\eqref{eq:uniform_moment_bound}, and
\eqref{eq:uniform_hypercontractive_square} gives the estimate
\[
\Big\|\max_{i\in[d]} f_i(X)\Big\|_{L_2}^2
\le
d^{1/q}\, C_{2q,\ell}^{\,2}\,
\max_{i\in[d]} \|f_i(X)\|_{L_2}^{2}.
\]
Since \(1/q\le \epsilon\) and \(d\ge 1\), the factor \(d^{1/q}\) is bounded
above by \(d^\epsilon\). Therefore, with
\(C_{\epsilon,\ell}:=C_{2q,\ell}^{\,2}\), the desired estimate follows.
\end{proof}

\begin{proposition}[Weighted sums of Fourier--Walsh polynomials]
\label{prop:fourier_basis_product_card}
Let \(1\le m\le q\), let \(p=(p_1,\dots,p_q)\in\mathbb N^q\), and fix
\(S^*\subset[d]\) with \(|S^*|=p^*\).
For each \(i\in[q]\), let \(\mathcal S_i\) be a collection of subsets of
\([d]\) such that \(|S|\le p_i\) for every \(S\in\mathcal S_i\).
Then, for every choice of vectors
\(a^{[i]}\in\mathbb R^{|\mathcal S_i|}\), \(i=1,\dots,m\), one has
\begin{equation}
\label{eq:weighted_walsh_goal}
\left|
\sum_{S_1\in\mathcal S_1,\dots,S_q\in\mathcal S_q}
a^{[1]}_{S_1}\cdots a^{[m]}_{S_m}\,
\mathbb E\!\left[x^{S_1}\cdots x^{S_q}x^{S^*}\right]
\right|
\le
C_{q,p}\,
d^{\sum_{j=m+1}^q p_j/2}
\prod_{i=1}^m \|a^{[i]}\|_2 .
\end{equation}
\end{proposition}

\begin{proof}[Proof of Proposition~\ref{prop:fourier_basis_product_card}]
Let \(\mu\) denote the underlying probability measure, and write
\(\|\cdot\|_{L_r}\) for the corresponding \(L_r(\mu)\)-norm.
For each \(i\in[m]\), define the weighted Walsh polynomial \(f_i\) by
\[
f_i(x):=\sum_{S\in\mathcal S_i} a^{[i]}_S\,x^S .
\]
For each \(j=m+1,\dots,q\), define the unweighted Walsh polynomial \(g_j\) by
\[
g_j(x):=\sum_{S\in\mathcal S_j} x^S .
\]
Expanding these definitions shows that the expression inside the absolute
value in \eqref{eq:weighted_walsh_goal} is the quantity
\begin{equation}
\label{eq:def_L}
L:=
\mathbb E_\mu\!\left[
x^{S^*}\,
\prod_{i=1}^m f_i(x)\,
\prod_{j=m+1}^q g_j(x)
\right].
\end{equation}

The identity \(|x^{S^*}|=1\) holds pointwise, so multiplication by
\(x^{S^*}\) preserves every \(L_r(\mu)\)-norm. An application of
Cauchy--Schwarz gives the estimate
\[
|L|
\le
\left\|
\prod_{i=1}^m f_i
\prod_{j=m+1}^q g_j
\right\|_{L_2}.
\]
To estimate this \(L_2\)-norm, set
\[
h_r :=
\begin{cases}
f_r, & 1\le r\le m,\\
g_r, & m+1\le r\le q .
\end{cases}
\]
Hölder's inequality, applied with exponent \(q\) to the \(q\) factors
\(|h_r|^2\), yields the bound
\[
\left\|
\prod_{r=1}^q h_r
\right\|_{L_2}
=
\left(
\mathbb E_\mu \prod_{r=1}^q |h_r|^2
\right)^{1/2}
\le
\prod_{r=1}^q \|h_r\|_{L_{2q}} .
\]
Combining the preceding two estimates gives the key reduction
\begin{equation}
\label{eq:basic_product_bound}
|L|
\le
\prod_{i=1}^m \|f_i\|_{L_{2q}}
\prod_{j=m+1}^q \|g_j\|_{L_{2q}} .
\end{equation}

It remains to bound the factors on the right-hand side of
\eqref{eq:basic_product_bound}. Since \(f_i\) is a Walsh polynomial of degree
at most \(p_i\), the Bonami hypercontractive inequality gives the estimate
\[
\|f_i\|_{L_{2q}}
\le
(2q-1)^{p_i/2}\|f_i\|_{L_2}.
\]
The orthonormality of the Fourier--Walsh basis gives the identity
\[
\|f_i\|_{L_2}^2
=
\sum_{S\in\mathcal S_i} |a^{[i]}_S|^2
=
\|a^{[i]}\|_2^2 .
\]
Combining these two estimates yields
\begin{equation}
\label{eq:fi_l2q_bound}
\|f_i\|_{L_{2q}}
\le
(2q-1)^{p_i/2}\|a^{[i]}\|_2 .
\end{equation}

The same argument applies to the unweighted factors. Since \(g_j\) is a
Walsh polynomial of degree at most \(p_j\), the Bonami hypercontractive
inequality gives the estimate
\[
\|g_j\|_{L_{2q}}
\le
(2q-1)^{p_j/2}\|g_j\|_{L_2}.
\]
The orthonormality of the Fourier--Walsh basis gives the identity
\[
\|g_j\|_{L_2}^2
=
\sum_{S\in\mathcal S_j} 1
=
|\mathcal S_j| .
\]
Since every set in \(\mathcal S_j\) has cardinality at most \(p_j\), the
cardinality of \(\mathcal S_j\) satisfies the bound
\[
|\mathcal S_j|
\le
\sum_{r=0}^{p_j}\binom{d}{r}
\le
C_{p_j} d^{p_j}.
\]
Combining the preceding three estimates gives
\begin{equation}
\label{eq:gj_l2q_bound}
\|g_j\|_{L_{2q}}
\le
C_{q,p} d^{p_j/2}.
\end{equation}

Substituting \eqref{eq:fi_l2q_bound} and \eqref{eq:gj_l2q_bound} into
\eqref{eq:basic_product_bound} yields the final estimate
\[
|L|
\le
C_{q,p}
\left(\prod_{i=1}^m \|a^{[i]}\|_2\right)
d^{\sum_{j=m+1}^q p_j/2}.
\]
By the representation of \(L\) in \eqref{eq:def_L}, this estimate is exactly
the desired bound \eqref{eq:weighted_walsh_goal}.
\end{proof}

\end{document}